%% file: VOLDNN_arXiv.tex
\documentclass{article}

\usepackage{arXiv, times}
\usepackage[colorlinks,allcolors=blue]{hyperref} % optional

\usepackage[utf8]{inputenc} % allow utf-8 input
\usepackage[T1]{fontenc}    % use 8-bit T1 fonts
\usepackage{hyperref}       % hyperlinks
\usepackage{url}            % simple URL typesetting
\usepackage{booktabs}       % professional-quality tables
\usepackage{amsfonts}       % blackboard math symbols
\usepackage{nicefrac}       % compact symbols for 1/2, etc.
\usepackage{microtype}      % microtypography
\usepackage{xcolor}         % colors

\usepackage{graphicx}
\usepackage{subcaption}
\usepackage{tabularx}

\input{math_commands.tex}

\usepackage[capitalise, nameinlink]{cleveref}
\usepackage{lipsum}
\usepackage{amsthm, amssymb}
\usepackage{algorithm}
\usepackage{algpseudocode}
\usepackage{comment}
\usepackage{multirow}
\usepackage{anyfontsize}
\theoremstyle{definition}
\newtheorem{definition}{Definition}[section]
\newtheorem{theorem}{Theorem}[section]
\newtheorem{lemma}{lemma}[section]
\newtheorem{corollary}{corollary}[section]
\newtheorem*{conjecture}{Conjecture}

\usepackage{thm-restate}

\usepackage{footnote}
\makesavenoteenv{tabular}

\newcommand{\definitionref}[1]{\cref{#1}}
\newcommand{\figureref}[1]{\cref{#1}}
\newcommand{\sectionref}[1]{\cref{#1}}
\newcommand{\lemmaref}[1]{\cref{#1}}
\newcommand{\theoremref}[1]{\cref{#1}}
\newcommand{\equationref}[1]{\cref{#1}}
\newcommand{\corollaryref}[1]{\cref{#1}}
\newcommand{\tableref}[1]{\cref{#1}}
\newcommand{\algorithmref}[1]{\cref{#1}}

\DeclareCaptionFormat{subcaptionformat}{\fontsize{8.5}{9.5}\selectfont#1#2#3}

\title{Dissecting Hessian: Understanding Common Structure of Hessian in Neural Networks}

\author{
  Yikai Wu\thanks{Equal contribution, listed in alphabetical order}
  \thanks{Work done while Yikai Wu was undergraduate student in Duke University}\\
  Princeton University\\
  \texttt{yikai.wu@cs.princeton.edu}
  \and
  Xingyu Zhu$^*$\\
  Duke University\\
  \texttt{xingyu.zhu@duke.edu}
  \and
  Chenwei Wu\\
  Duke University\\
  \texttt{cwwu@cs.duke.edu}
  \and
  Annie Wang\\
  Duke University\\
  \texttt{annie.wang029@duke.edu}
  \and
  Rong Ge\\
  Duke University\\
  \texttt{rongge@cs.duke.edu}
}

\begin{document}

\maketitle

\begin{abstract}
\input{abs}
\end{abstract}

\section{Introduction}
\input{intro}

\subsection{Related Works}
\input{related}

\section{Preliminaries and Notations}
\input{preliminaries}

\section{Decoupling Conjecture and Implications on the Structures of Hessian}
\input{hessian_structure}

\section{Hessian Structure for Infinite Width Two-Layer ReLU Neural Network}
\label{sec:theoretical}
\input{theoretical_analysis}

\section{Empirical Observation and Verification}
\input{hessian_empirical}

\section{Tighter PAC-Bayes Bound with Hessian Information}
\input{pac_bayes}

\section{Limitations and Conclusions}
\input{conclusions}

\bibliography{VOLDNN_arXiv}
\bibliographystyle{icml2021}

%%%%%%%%%%%%%%%%%%%%%%%%%%%%%%%%%%%%%%%%%%%%%%%%%%%%%%%%%%%%

\appendix

\include{appendix_colt}

\end{document}

%% file: math_commands.tex
\usepackage{amsmath}
\usepackage{amssymb}
\usepackage{bm}
\usepackage{bbm}

\let\hat\widehat
\let\tilde\widetilde

\newcommand{\vb}{\mathbf{v}}

\newcommand{\xb}{\mathbf{x}}

\newcommand{\T}{\mathsf{T}}

\newcommand{\cI}{\mathcal{I}}

\newcommand{\cL}{\mathcal{L}}

\newcommand{\cN}{\mathcal{N}}

\newcommand{\cW}{\mathcal{W}}

\newcommand{\E}{\mathbb{E}}

\newcommand{\NN}{\mathbb{N}}
\newcommand{\PP}{\mathbb{P}}

\newcommand{\R}{\mathbb{R}}

\newcommand{\bSigma}{\bm{\Sigma}}

\newcommand{\tr}{\mathop{\mathrm{tr}}}

\DeclareMathOperator{\ind}{\mathbbm{I}}  % Indicator

\newcommand{\diag}{{\rm diag}}

\newcommand{\norm}[1]{\lVert#1\rVert}
\newcommand{\fn}[1]{\norm{#1}_F}
\newcommand{\ns}[1]{\norm{#1}^2}
\newcommand{\fns}[1]{\norm{#1}^2_F}
\newcommand{\lnorm}[1]{\left\lVert#1\right\rVert}

\newcommand{\lfn}[1]{\lnorm{#1}_F}

\newcommand{\abs}[1]{\lvert#1\rvert}
\newcommand{\labs}[1]{\left\lvert#1\right\rvert}

\newcommand{\ex}[1]{\mathbb{E}\left[#1\right]}
\newcommand{\exs}[1]{\mathbb{E}[#1]}
\newcommand{\exop}[2]{\mathop{\mathbb{E}}_{#1}\left[#2\right]}
\newcommand{\pr}[1]{\Pr\left[#1\right]}
\newcommand{\prop}[2]{\mathop{\Pr}_{#1}\left[#2\right]}

\newcommand{\innerf}[1]{\left\langle #1 \right\rangle_{F}}
\newcommand{\rbr}[1]{\left(#1\right)}
\newcommand{\sbr}[1]{\left[#1\right]}

\newcommand{\overbar}[1]{\mkern 1.5mu\overline{\mkern-1.5mu#1\mkern-1.5mu}\mkern 1.5mu}

\makeatletter
\newcommand{\raisemath}[1]{\mathpalette{\raisem@th{#1}}}
\newcommand{\raisem@th}[3]{\raisebox{#1}{$#2#3$}}
\makeatother

\let\oldabstract\abstract
\let\oldendabstract\endabstract
\makeatletter
\renewenvironment{abstract}
{%
               {\list{}{\addtolength{\leftmargin}{-0.5em} % change this value to add or remove length to the the default
                       \listparindent 0em%
                        \itemindent    \listparindent%
                        \rightmargin   \leftmargin%
                        \parsep        \z@ \@plus\p@}%
                \item\relax}%
               {\endlist}%
\oldabstract}
{\oldendabstract}
\makeatother

\usepackage{enumerate}
\usepackage[OT1]{fontenc}
\usepackage{amsmath,amssymb}

\def\tr{\mathop{\text{tr}}}

\long\def\comment#1{}

\def\vec{\mathop{\text{vec}}}

\providecommand{\norm}[1]{\vvvert#1\vvvert}
\providecommand{\norml}[1]{\Vert#1\Vert}

\let\emptyset\varnothing

\usepackage{mathrsfs}

\def\beq{\begin{equation} }
\def\eeq{\end{equation} }

\def\1{\mathbf{1}}
\def\ry{{\textnormal{y}}}

\def\rvu{{\mathbf{i}}}

\def\rvp{{\mathbf{p}}}

\def\rvs{{\mathbf{s}}}

\def\rvu{{\mathbf{u}}}
\def\rvv{{\mathbf{v}}}

\def\rvx{{\mathbf{x}}}
\def\rvy{{\mathbf{y}}}
\def\rvz{{\mathbf{z}}}

\def\ervs{{\textnormal{s}}}
\def\ervt{{\textnormal{t}}}

\def\ervx{{\textnormal{x}}}
\def\ervy{{\textnormal{y}}}

\def\rmA{{\mathbf{A}}}

\def\rmD{{\mathbf{D}}}

\def\rmM{{\mathbf{M}}}

\def\rmQ{{\mathbf{Q}}}

\def\rmW{{\mathbf{W}}}
\def\rmX{{\mathbf{X}}}

\def\vtheta{{\bm{\theta}}}
\def\vphi{{\bm{\phi}}}
\def\vxi{{\bm{\xi}}}
\def\vvarsigma{{\bm{\varsigma}}}
\def\va{{\bm{a}}}
\def\vb{{\bm{b}}}

\def\ve{{\bm{e}}}

\def\vg{{\bm{g}}}
\def\vh{{\bm{h}}}

\def\vm{{\bm{m}}}

\def\vp{{\bm{p}}}

\def\vr{{\bm{r}}}
\def\vs{{\bm{s}}}
\def\vt{{\bm{t}}}
\def\vu{{\bm{u}}}
\def\vv{{\bm{v}}}
\def\vw{{\bm{w}}}
\def\vx{{\bm{x}}}
\def\vy{{\bm{y}}}
\def\vz{{\bm{z}}}

\def\evz{{z}}

\def\mA{{\bm{A}}}
\def\mB{{\bm{B}}}

\def\mD{{\bm{D}}}

\def\mF{{\bm{F}}}
\def\mG{{\bm{G}}}
\def\mH{{\bm{H}}}
\def\mI{{\bm{I}}}

\def\mM{{\bm{M}}}
\def\mN{{\bm{N}}}

\def\mQ{{\bm{Q}}}

\def\mS{{\bm{S}}}

\def\mU{{\bm{U}}}
\def\mV{{\bm{V}}}
\def\mW{{\bm{W}}}
\def\mX{{\bm{X}}}

\def\mZ{{\bm{Z}}}

\def\mSigma{{\bm{\Sigma}}}

\def\mDelta{{\bm{\Delta}}}
\def\mGamma{{\bm{\Gamma}}}

\DeclareMathAlphabet{\mathsfit}{\encodingdefault}{\sfdefault}{m}{sl}
\SetMathAlphabet{\mathsfit}{bold}{\encodingdefault}{\sfdefault}{bx}{n}
\newcommand{\tens}[1]{\bm{\mathsfit{#1}}}

\def\tW{{\tens{W}}}
\def\tX{{\tens{X}}}

\def\tZ{{\tens{Z}}}

\def\gH{{\mathcal{H}}}

\def\gL{{\mathcal{L}}}

\def\gN{{\mathcal{N}}}

\def\gR{{\mathcal{R}}}

\newcommand{\etens}[1]{\mathsfit{#1}}

\newcommand{\hE}{\hat{\mathbb{E}}}

\newcommand{\N}{\mathcal{N}}
\newcommand{\Ls}{\mathcal{L}}

\newcommand{\softmax}{\mathrm{softmax}}

\newcommand{\KL}{D_{\mathrm{KL}}}
\newcommand{\BRE}{B_{\mathrm{RE}}}
\DeclareMathOperator{\vect}{vec}
\DeclareMathOperator{\Mat}{Mat}
\DeclareMathOperator{\Overlap}{Overlap}

\DeclareMathOperator{\abso}{abs}
\DeclareMathOperator{\Corr}{Corr}

\def\Gx{{\mG_\vx}}

\def\Ax{{\mA_\vx}}
\def\Qx{{\mQ_\vx}}
\def\Mx{{\mM_\vx}}

\def\zx{{\vz_\vx}}

\def\tmM{{\tilde{\mM}}}
\def\tmN{{\tilde{\mN}}}
\def\HessL{{\mH_{\mathcal{L}}}}
\def\Hessl{{\mH_{\ell}}}
\def\diag{\text{diag}}
\newcommand{\algind}{\hspace*{\algorithmicindent}}
\def\eps{\epsilon}
\def\rectNormal{{\cN^{\text{R}}}}
        {\hspace*{\fill}$\Box$\par\vspace{4mm}}
\newenvironment{proofof}[1]{\smallskip\noindent{\bf Proof of #1.}}%
        {\hspace*{\fill}$\Box$\par}
\def\converged{\xrightarrow{d}}
\def\tmH{\breve{\mH}}
\def\trmX{\breve{\rmX}}
\def\trmA{\tilde{\rmA}}
\def\bmW{\overline{\mW}}
\def\bmV{\overline{\mV}}
\def\vtheta{\theta}

\def\mSigma{{\bm{\Sigma}}}

\def\mDelta{{\bm{\Delta}}}
\def\mGamma{{\bm{\Gamma}}}

\def\shownotes{1}  \ifnum\shownotes=1
\newcommand{\authnote}[2]{{[#1: #2]}}
\else
\newcommand{\authnote}[2]{}
\fi

%% file: abs.tex
Hessian captures important properties of the deep neural network loss landscape. Previous works have observed low rank structure in the Hessians of neural networks. In this paper, we propose a decoupling conjecture that decomposes the layer-wise Hessians of a network as the Kronecker product of two smaller matrices. We can analyze the properties of these smaller matrices and prove the structure of top eigenspace random 2-layer networks. %, which implies that the layer-wise Hessian is indeed low rank. 
The decoupling conjecture has several other interesting implications \--- top eigenspaces for different models have surprisingly high overlap, and top eigenvectors form low rank matrices when they are reshaped into the same shape as the corresponding weight matrix. All of these can be verified empirically for deeper networks. Finally, we use the structure of layer-wise Hessian to get better explicit generalization bounds for neural networks.

%% file: intro.tex
\label{sec:intro}
%Despite a lot of recent research, understanding the optimization and generalization for neural networks remain challenging problems: why can simple algorithm such as stochastic gradient descent optimize the complicated nonconvex objective, and why would the result generalize to unseen data? A common conjecture is that both problems are related to the loss landscape of neural networks. In this paper, we focus on an important aspect of the loss landscape \--- the layer-wise Hessian for neural networks along its training trajectory.

%Neural network objectives are complicated and non-convex. However, in practice neural networks can be trained by simple algorithms and they perform well on test data. A common explanation is that neural network objectives have good loss landscapes for optimization and generalization. In this paper we study the structure of Hessians for neural network objectives.

The loss landscape for neural networks is crucial for understanding training and generalization. In this paper we focus on the structure of Hessians, which capture important properties of the loss landscape. For optimization, Hessian information is used explicitly in second order algorithms, and even for gradient-based algorithms properties of the Hessian are often leveraged in analysis \citep{sra2012optimization}. For generalization, the Hessian captures the local structure of the loss function near a local minimum, which is believed to be related to generalization gaps \citep{keskar2016large}. %which is related to the intuition of flat vs. sharp local minimum, and can be incorporated into generalization bounds using ideas like PAC-Bayes bounds.

%Are the Hessians of neural networks similar to a random matrix or highly structured, and does the structure of Hessians change with different architecture or training algorithms? 
Several previous results including \citet{sagun2017empirical, papyan2018full} observed interesting structures in Hessians for neural networks \--- it often has around $c$ large eigenvalues where $c$ is the number of classes. In this paper we ask:
\begin{center}
    \emph{Why does the Hessian of neural networks have special structures in its top eigenspace?}
\end{center}

A rigorous analysis of the Hessian structure would potentially allow us to understand what the top eigenspace of the Hessian depends on (e.g., the weight matrices or data distribution), as well as predicting the behavior of the Hessian when the architecture changes. 

Towards this goal, we focus on the layer-wise Hessians in this paper. One difficulty in analyzing the layer-wise Hessian lies in its size \--- for a fully-connected layer with a $n\times n'$ weight matrix, the layer-wise Hessian is a $nn'\times nn'$ matrix. We propose a {\em decoupling conjecture} that approximates this matrix by the Kronecker product of two smaller matrices \--- a $n\times n$ input autocorrelation matrix and a $n'\times n'$ output Hessian matrix. We then study the properties of these two smaller matrices, which together with the decoupling conjecture give an explanation of why there are just a few large eigenvalues, as well as a heuristic formula to efficiently compute the top eigenspace. We prove the decoupling conjecture and structure of the output Hessian matrix for a simple model of 2-layer network. We then empirically verify that these results extend to much more general settings.

\subsection{Outline}

\textbf{Understanding Hessian Structure using Kronecker Factorization:} %We show that both of these new properties of layer-wise Hessians can be explained by a Kronecker Factorization. 
In \sectionref{sec:hessian} We first formalize a decoupling conjecture that states  the layer-wise Hessian can be approximated by the Kronecker product of the output Hessian and input auto-correlation. 

The auto-correlation of the input is often very close to a rank 1 matrix, because the inputs for most layers have a nonzero expectation. We show that when the input auto-correlation component is approximately rank 1, top eigenspace of the layerwise Hessian is very similar to that of the output Hessian. %the layer-wise Hessians indeed have high overlap at the dimension of the layer's output, and the spectrum of the layer-wise Hessian is similar to the spectrum of the output Hessian. 
On the contrary, when inputs have mean 0 (e.g., when the model is trained with batch normalization), the input auto-correlation matrix is much farther from rank 1 and the layer-wise Hessian often does not have the same low rank structure.

In \sectionref{sec:theoretical} we prove that in an over-parametrized two-layer neural network on random data, the output Hessian is approximately rank $c-1$. Further, we can compute the top $c-1$ eigenspace directly from weight matrices. We show a similar low rank result for %can also prove the decoupling conjecture in this setting, showing that 
the layer-wise Hessian. % is indeed low rank.

%This Kronecker approximation directly implies that the eigenvectors of the layer-wise Hessian should be approximately rank 1 when viewed as a matrix. Moreover, under stronger assumptions, we can generalize the approximation for the top eigenvalues and eigenvectors of the full Hessian.

\textbf{Implication on the Structure of Top Eigenspace for Hessians:} The decoupling conjecture, together with our characterizations of its two components, have surprising implications to the structure of top-eigenspace for layer-wise Hessians. Since the eigenvector of a Kronecker product is just the outer product of eigenvectors of its components, if we express the top eigenvectors of a layer-wise Hessian as a matrix with the same dimensions as the weight matrix, then the matrix is approximately rank 1. In \figureref{fig:intro_figs}.a we show the singular values of several such reshaped eigenvectors. Another more surprising phenomenon considers the overlap between top eigenspaces for different models.

% \begin{figure}[th]
% \vspace{-6pt}
%     \centering
% \captionsetup[sub]{format=subcaptionformat}
%     \begin{subfigure}[b]{0.58\textwidth}
%         \centering
%         \captionsetup{justification=centering}
%         \includegraphics[width=0.45\textwidth]{Figures/SubspaceOverlap/NLeNet5_multi_hyperparam/DimOverlap_CIFAR10_LeNet5_normnew_fixlr0.001_X_LeNet5_normnew_fixlr0.01_X_LeNet5_normnew_fixlr0.01_momentum_fc1.pdf}
%         \includegraphics[width=0.45\textwidth]{Figures/SubspaceOverlap/ResNets/DimOverlap_CIFAR100_Resnet18W64New_nobn_fixlr0.01_layer3.0.conv2.pdf}
%         \caption{Overlap between dominate eigenspace of layer-wise Hessian at different minima for fc1:LeNet5 (\textbf{left}) with output dimension 120 and conv11:ResNet18-W64 (\textbf{right}) with output dimension 64.}
%         \label{fig:intro_overlap}
%     \end{subfigure}%
%     \begin{subfigure}[b]{0.38\textwidth}
%         \centering
%         \captionsetup{justification=centering}
%         \includegraphics[width=0.9\textwidth]{Figures/Eigenvec_single/Top_Eigenvector_sigs_CIFAR10_Exp1_LeNet5_fixlr0.01R1_E-1fc1.pdf}
%         \caption{Top 10 singular values of the top 4\linebreak eigenvectors of the layer-wise Hessian of\linebreak fc1:LeNet5 after reshaped as matrix.}
%         \label{fig:intro_lowrank}
%     \end{subfigure}%
%     \caption{Some interesting observations on the structure of layer-wise Hessians.  The eigenspace overlap is defined in \cref{def:overlap} and the reshape operation is defined in \cref{def:matricization}}
%     \label{fig:intro_figs}
%     \vspace{-6pt}
% \end{figure}
\begin{figure}[th]
\vspace{-6pt}
    \centering
\captionsetup[sub]{format=subcaptionformat}
    \begin{subfigure}[b]{0.58\textwidth}
        \centering
        \captionsetup{justification=centering}
        \includegraphics[width=0.45\textwidth]{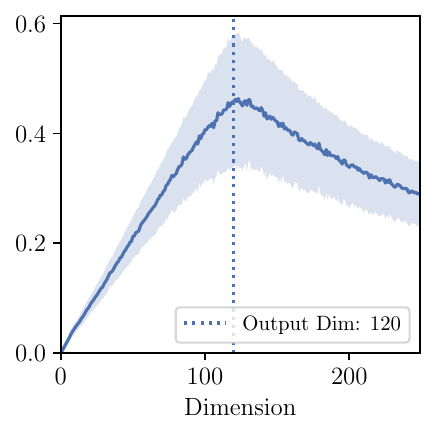}
        \includegraphics[width=0.45\textwidth]{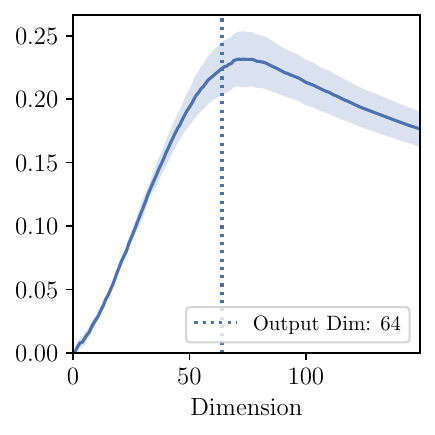}
        \caption{Overlap between dominate eigenspace of layer-wise Hessian at different minima for fc1:LeNet5 (\textbf{left}) with output dimension 120 and conv11:ResNet18-W64 (\textbf{right}) with output dimension 64.}
        \label{fig:intro_overlap}
    \end{subfigure}%
    \begin{subfigure}[b]{0.38\textwidth}
        \centering
        \captionsetup{justification=centering}
        \includegraphics[width=0.9\textwidth]{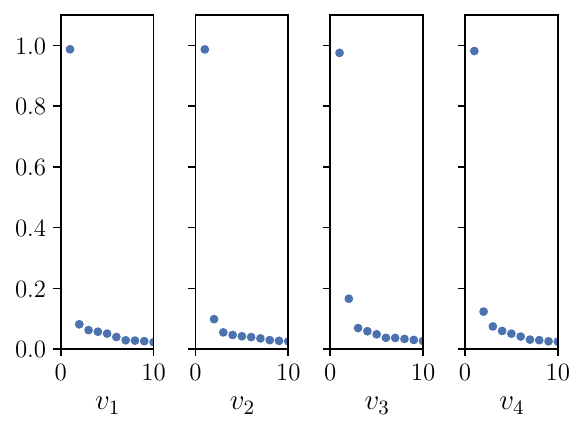}
        \caption{Top 10 singular values of the top 4\linebreak eigenvectors of the layer-wise Hessian of\linebreak fc1:LeNet5 after reshaped as matrix.}
        \label{fig:intro_lowrank}
    \end{subfigure}%
    \caption{Some interesting observations on the structure of layer-wise Hessians.  The eigenspace overlap is defined in \cref{def:overlap} and the reshape operation is defined in \cref{def:matricization}}
    \label{fig:intro_figs}
    \vspace{-6pt}
\end{figure}
Consider two neural networks trained with different random initializations and potentially different hyper-parameters; their weights are usually nearly orthogonal. One might expect that the top eigenspace of their layer-wise Hessians are also very different. However, empirically one observe that the top eigenspace of the layer-wise Hessians have a very high overlap, and the overlap peaks at the dimension of the layer's output (see \figureref{fig:intro_overlap}). This is a direct consequence of the Kronecker product and the fact that the input auto-correlation matrix is close to rank 1.

%\znote{Added some content regarding the full hessian approximation}

%\textbf{Leveraging Hessian Structure:}
%Better understanding of the Hessian can give more insights into optimization and generalization. 
\textbf{Applications:} As a direct application of our results, in \sectionref{sec:pac} we show that the Hessian structure can be used to improve the PAC-Bayes bound computed in \citet{dziugaite2017computing}. %; we also show that the general structure of the Hessian suggests a simple way of incorporating second-order information in optimization without large overhead.

%% file: related.tex
\textbf{Hessian-based analysis for neural networks (NNs):} Hessian matrices for NNs reflect the second order information about the loss landscape, which is important in characterizing SGD dynamics \citep{jastrzebski2018relation} and related to generalization \citep{li2020hessian}, robustness to adversaries \citep{yao2018hessian} and interpretation of NNs \citep{singla2019understanding}. People have empirically observed several interesting phenomena of the Hessian, e.g., the gradient during training converges to the top eigenspace of Hessian \citep{gur2018gradient, ghorbani2019investigation}, and the eigenspectrum of Hessian contains a ``spike" which has about $c-1$ large eigenvalues and a continuous ``bulk" \citep{sagun2016eigenvalues, sagun2017empirical, papyan2018full}. 
People have developed different frameworks to explain the low rank structure of the Hessians including hierarchical clustering of logit gradients \citep{papyan2019measurements, papyan2020traces}, independent Gaussian model for logit gradients \citep{fort2019emergent}, and Neural Tangent Kernel \citep{jacot2019asymptotic}.
A distinguishing feature of this work is that we are able to characterize the top eigenspace of the Hessian directly by the weight matrices of the network.

\textbf{Layer-wise Kronecker factorization (K-FAC) for training NNs:} 
The idea of using Kronecker product to approximate Hessian-like matrices is not new. \citet{heskes2000natural} uses this idea to approximate Fisher Information Matrix (FIM).
%The idea of approximating the FIM using Kronecker factorization can be dated back to . More recently 
\citet{martens2015optimizing} proposed Kronecker-factored approximate curvature which approximates the inverse of FIM using layer-wise Kronecker product. Kronecker factored eigenbasis has also been utilized in training \citep{george2018fast}.
% K-FAC has been generalized to convolutional and recurrent NNs \citep{grosse2016kronecker,martens2018kronecker}, Bayesian deep learning \citep{zhang2018noisy}, and structured pruning \citep{wang2019eigendamage}.
Our paper focuses on a different application with different matrix (Hessian vs. inverse FIM) and different ends of the spectrum (top vs. bottom eigenspace).
%Unlike these previous works which leverages the lower end of the FIM eigenspectrum for training acceleration, in this paper we use Kronecker factorization to explain the structures of the top eigenspace of layer-wise Hessians.

\textbf{Theoretical Analysis for Hessians Eigenstructure:} \citet{karakida2019pathological} showed that the largest $c$ eigenvalues of the FIM for a randomly initialized neural network are much larger than the others. Their results rely on the eigenvalue spectrum analysis in \citet{karakida2019universal, karakida2019normalization}, which assumes the weights used during forward propagation are drawn independently from the weights used in back propagation \citep{schoenholz2016deep}. More recently, \citet{singh2021analytic} provided a Hessian rank formula for linear networks and \citet{liao2021hessian} provided a characterization on the eigenspace structure of G-GLM models (including 1-layer NN). To our best knowledge, theoretical analysis on the Hessians of nonlinear deeper neural networks is still vacant. %The weights used during forward propagation are drawn independently from the weights used in backpropagation.

% that arise in the top eigenspace of the layer-wise Hessians.
% \znote{some can be removed, such as the ones on distrbuted learning}

\textbf{PAC-Bayes generalization bounds:} People have established generalization bounds for neural networks under PAC-Bayes framework by \citet{mcallester1999some}.
% This bound was further tightened by \citet{langford2001bounds}, and \citet{catoni2007pac} proposed a faster-rate version. %Under this framework, generalization bounds for linear regression \citep{germain2016pac} and SVM \citep{rivasplata2018pac} are provided. 
For neural networks, \citet{dziugaite2017computing} proposed the first non-vacuous generalization bound, which used PAC-Bayesian approach with optimization to bound the generalization error for a stochastic neural network.
% Their bound was then extended to ImageNet scale by \citet{zhou2019non} using compression.
% \znote{probably should add the new results on training process that envolves minimizing the bound.}
%The PAC-Bayes framework was also used with other techniques to provide different forms of bounds, e.g., spectrally-normalized Margin bounds \citep{neyshabur2017pac}, deterministic bounds \citep{nagarajan2019deterministic}, or making connections to sharpness \citep{neyshabur2017exploring}, differential privacy \citep{dziugaite2018data}, and Entropy-SGD \citep{dziugaite2018entropy}.

%% file: preliminaries.tex
\label{sec:prelim}
%\znote{Modification on Prelim: Done (We adapt the nabla expression of Hessian as suggested by the official template)}
\textbf{Basic Notations:} In this paper, we generally follow the default notation suggested by \citet{goodfellow2016deep}. Additionally, for a matrix $\mM$, let $\|\mM\|_F$ denote its Frobenius norm and $\|\mM\|$ denote its spectral norm. For two matrices $\mM \in \R^{a_1\times b_1}, \mN\in \R^{a_2\times b_2}$, let $\mM \otimes \mN\in\R^{(a_1a_2)\times (b_1b_2)}$ be their Kronecker product such that $[\mM \otimes \mN]_{(i_1-1)\times a_2 + i_2, (j_1-1)\times b_2 + j_2} = \mM_{i_1,i_2} \mN_{j_1,j_2}$.

\textbf{Neural Networks:}
% In this paper, we consider classification problems with cross-entropy loss.
For a $c$-class classification problem with training samples $S = \{(\vx_i, \vy_i)\}_{i=1}^N$ where $(\vx_i,\vy_i)\in\R^d\times \{0,1\}^c$ for all $i\in[N]$, assume $S$ is i.i.d. sampled from the underlying data distribution $\mathcal{D}$. Consider an $L$-layer fully connected ReLU neural network $f_\vtheta: \R^d\to \R^c$. With $\sigma(x) = x\1_{x \geq 0}$ as the Rectified Linear Unit (ReLU) function, the output of this network is a series of logits $\vz \in \R^c$ computed recursively as $ \vz^{(p)} := \mW^{(p)}\vx^{(p)} + \vb^{(p)}$ and $\vx^{(p)} := \sigma(\vz^{(p)})$

% \begin{align}
%     \vz^{(p)} := \mW^{(p)}\vx^{(p)} + \vb^{(p)}.\qquad \vx^{(p)} := \sigma(\vz^{(p)}).
% \end{align}
Here we denote the input and output of the $p$-th layer as $\vx^{(p)}$ and $\vz^{(p)}$, and set $\vx^{(1)}=\vx$, $\vz := f_\vtheta(\vx)=\vz^{(L)}$.
%\begin{equation}
%    \vz := f_\vtheta(\vx)= \mW^{(L)}\sigma(\mW^{(L-1)}\sigma(\cdots\mW^{(2)}\sigma(\mW^{(1)}\vx+\vb^{(1)})+\vb^{(2)}\cdots)+\vb^{(L-1)})+\vb^{(L)}
%\end{equation}
We denote $\vtheta := (\vw^{(1)}, \vb^{(1)}, \vw^{(2)}, \vb^{(2)},\cdots, \vw^{(L)}, \vb^{(L)})\in\R^P$ the parameters of the network. For the $i$-th layer, $\vw^{(i)}$ is the flattened weight matrix $\mW^{(i)}$ and $\vb^{(i)}$ is its corresponding bias vector. For convolutional networks, a similar framework is introduced in \sectionref{sec:appendix_conv}.

For a single input $\vx\in\R^d$ with one-hot label $\vy$ and logit output $\vz$, let $n^{(p)}$ and $m^{(p)}$ be the lengths of $\vx^{(p)}$ and $\vz^{(p)}$.  For convolutional layers, we consider the number of output channels as $m^{(p)}$ and width of unfolded input as $n^{(p)}$. Note that $\vx^{(1)}=\vx,\vz^{(L)}=\vz = f_\vtheta(\vx)$. We denote $\vp := \softmax(\vz) = e^{\vz}/\sum_{i=1}^ce^{\vz_i}$ as the output confidence.
With the cross-entropy loss function $\ell(\vp, \vy) = -\sum_{i=1}^c\vy_i\log(\vp_i)\in\R^+$,
% being the  loss between the softmax of logits $\vz = f_\vtheta(\vx_i)\in\R^c$ and the one-hot label $\vy\in\R^c$,
the training process optimizes parameter $\vtheta$ to minimize the empirical training loss $\mathcal{L}(\vtheta):=\mathop{\E}_{(\vx, \vy)\in S}\left[\ell\left(\vz, \vy\right)\right].$
% \begin{align}
%     \mathcal{L}(\vtheta) = \frac{1}{N}\sum_{i=1}^N\ell(f_\vtheta(\vx_i), \vy_i) = \mathop{\E}_{(\vx, \vy)\in S}\left[\ell\left(\vz, \vy\right)\right].
% \end{align}

\textbf{Hessians:} Fixing the parameter $\vtheta$, we use $\mH_{\ell}(\vv,\vx) = \nabla^2_\vv \ell(f_\vtheta(\vx), \vy) = \nabla^2_\vv \ell(\vz, \vy)$ to denote the Hessian of some vector $\vv$ with respect to scalar loss function $\ell$ at input $\vx$.  
% \begin{equation}
%     \mH_\ell(\vv, \vx) = \nabla^2_\vv \ell(f_\vtheta(\vx), \vy) = \nabla^2_\vv \ell(\vz, \vy).
% \end{equation}
Note that $\vv$ can be any vector. For example, the full parameter Hessian is $\mH_\ell(\vtheta, \vx)$ where we take $\vv = \vtheta$, and the layer-wise weight Hessian of the $p$-th layer is $\mH_\ell(\vw^{(p)}, \vx)$ where we take $\vv = \vw^{(p)}$.

For simplicity, define $\E$ as the empirical expectation operator over the training sample $S$ unless explicitly stated otherwise.
We mainly focus on the layer-wise weight Hessians $\HessL(\vw^{(p)}) = \E[\Hessl(\vw^{(p)}, \vx)]$ with respect to loss, which are diagonal blocks in the full Hessian $\HessL(\vtheta) = \E[\Hessl(\vtheta, \vx)]$ corresponding to the cross terms between the weight coefficients of the same layer.
We define $\mM_{\vx}^{(p)} := \Hessl(\vz^{(p)}, \vx)$
as the Hessian of output $\vz^{(p)}$ with respect to empirical loss.
%\begin{align}
%    \mM_{\vx}^{(p)} := \Hessl(\zx^{(p)}, \vx) = \left(\frac{\partial \zx}{\partial \zx^{(p)}}\right)^\T\left(\diag(\px)- \px\px^\T\right)\left(\frac{\partial \zx}{\partial \zx^{(p)}}\right).
%\end{align}
%\begin{align}
%    \Gx^{(p)} = \frac{\partial \zx}{\partial \zx^{(p)}},\qquad
%    \Ax = \frac{\partial^2 \ell(\zx, \vy)}{\partial \px^2} = \diag(\px)- \px\px^\T.
%\end{align}
%We further define $\Bx := (\mI - \1^\T\px)\diag(\sqrt{\px})$ (such that $\Bx\Bx^\T = \Ax$) from \citet{papyan2019measurements}.
%\ynote{Probably move this to the Appendix if we do not use it in Section 4}
With the notations defined above, we have the $p$-th layer-wise Hessian for a single input as
%\begin{align}
%    \Hessl(\vw^{(p)}, \vx) = \nabla^2_{\vw^{(p)}} \ell(\zx, \vy) = (\mI_{m^{(p)}}\otimes\vx^{(p)})\mM_{\vx}^{(p)}(\mI_{m^{(p)}}\otimes\vx^{(p)})^\T.\label{eqn:decomp_raw}
%\end{align}
%For fully-connected layers, since $\vx^{(p)}$ is an $n^{(p)}\times 1$ matrix , we can further decompose its weight hessian to two separate matrices.
\begin{align}
    \Hessl(\vw^{(p)}, \vx) = \nabla^2_{\vw^{(p)}} \ell(\vz, \vy) = \mM_{\vx}^{(p)}\otimes (\vx^{(p)}\vx^{(p)T}).
\end{align}
It follows that
\begin{equation}
\HessL(\vw^{(p)}) = \E\left[\mM^{(p)}_{\vx}\otimes \vx^{(p)}\vx^{(p)T}\right] = \E\left[\mM \otimes\vx\vx^\T\right].\label{eqn:decomp}
\end{equation}
The subscription $\vx$ and the superscription $(p)$ will be omitted when there is no confusion, as our analysis primarily focuses on the same layer unless otherwise stated.
% Hessian: \frac{1}{N}\sum_{i=1}^N\frac{\partial^2\ell(f_\vtheta(\vx_i), \vy_i)}{\partial \vtheta^2}
We also define subspace overlap and layer-wise eigenvector matricization for our analysis.
\begin{definition}[Subspace Overlap]
For $k$-dimensional subspaces $\mU,\mV$ in $\R^d$ ($d\geq k$) where the basis vectors $\vu_i$'s and $\vv_i$'s are column vectors, with $\vphi$ as the size $k$ vector of canonical angles between $\mU$ and $\mV$, we define the subspace overlap of $\mU$ and $\mV$ as
$
    \Overlap(\mU, \mV) := \|\mU^\T\mV\|^2_F/k =\|\cos\vphi\|_2^2/k.
    $%\label{eqn:overlap}
%Note that when $k=1$, the overlap is equivalent to the squared dot product between the two vectors.
\label{def:overlap}
\end{definition}
\begin{definition}[Layer-wise Eigenvector Matricization] Consider a layer with input dimension $n$ and output dimension $m$. For an eigenvector $\vh\in\R^{mn}$ of its layer-wise Hessian, the matricized form of $\vh$ is $\Mat(\vh)\in\R^{m\times n}$ where $\Mat(\vh)_{i,j} = \vh_{(i-1)m+j}$.
\label{def:matricization}
\end{definition}

%% file: hessian_structure.tex
\label{sec:hessian}
%\rnote{Title of this section is a bit long and don't have a real focus. We might consider break it into multiple sections? One possibility is that we talk about the decomposition, then talk about properties of $xx^\T$ (both zero-mean and non-zero-mean case) and briefly talk about properties of $M$ (mostly empirical); in the next section we can then use this decomposition to justify (a) Hessian overlap; (b) low rank structure of the top eigenvectors; (c) rank C-1 Hessian}
The fact that layer-wise Hessian for a single sample can be decomposed into Kronecker product of two components naturally leads to the following informal conjecture:
\begin{conjecture}[Decoupling Conjecture] The layer-wise Hessian can be approximated by a Kronecker product of the expectation of its two components, that is
\begin{equation}
    \HessL(\vw^{(p)}) = \E[\mM\otimes \vx\vx^\T] \approx \E[\mM] \otimes \E[\vx\vx^\T].
    \label{eqn:decouple}
\end{equation}
More specifically, we conjecture that 
%\begin{equation}
%    \frac{\|\E[\mM] \otimes \E[\vx\vx^\T]-\HessL(\vw^{(p)})\|}{\|\HessL(\vw^{(p)})\|} = 
$\frac{\|\E[\mM] \otimes \E[\vx\vx^\T]-\E[\mM\otimes \vx\vx^\T]\|}{\|\E[\mM\otimes \vx\vx^\T]\|} \leq \epsilon$,
%\end{equation}
where $\epsilon$ is a small constant.
\end{conjecture}
%The fact that layer-wise Hessian for fully connected layers can be decomposed into the expectation of Kronecker products as in \cref{eqn:decomp} poses a natural question: Can the Hessian be approximated using the Kronecker product of the expectations? That is, whether 
%\begin{equation}
%    \HessL^{(p)}(\vtheta) = \E\left[\mM\otimes \vx\vx^\T\right] \approx \E[\mM] \otimes \E[\vx\vx^\T]. 
%\end{equation}
%We call the above conjecture ``the decoupling conjecture''. Equivalently, we conjecture that the two matrices $\mM$ and $\vx\vx^\T$ are approximately statistically independent.
Note that this conjecture is certainly true when $\mM$ and $\vx\vx^\T$ are approximately statistically independent. One immediate implication is that the top eigenvalues and eigenspace of $\HessL(\vw^{(p)})$ is close to those of $\E[\mM] \otimes \E[\vx\vx^\T]$. In \sectionref{sec:theoretical} we prove that the eigenspaces are indeed close for a simple setting, and in \sectionref{subsec:approx} we show that this conjecture is empirically true in practice.
%\rnote{Maybe we should call the above a conjecture with a name (something like ``the decoupling conjecture''?), then we say that empirically we show the conjecture is true, and this has a lot of benefits.}

% \rnote{We should move the experiments later. This section should contain subsection 3, 4, and a new subsection that discusses the impliction (subsections 1,2 from the next section.)}
% \rnote{Between this and next section we should have a theoretical section that explains the rigorous theorems and the proof ideas, which we can get from appendix B.}

Assuming the decoupling conjecture, we can analyze the layer-wise Hessian by analyzing the two components separately. Note that $\E[\mM]$ is the Hessian of the layer-wise output with respect to empirical loss, and $\E[\vx\vx^\T]$ is the auto-correlation matrix of the layer-wise inputs. For simplicity we call $\E[\mM]$ the output Hessian and $\E[\vx\vx^\T]$ the input auto-correlation. For convolutional layers we can a similar factorization $\E[\mM]\otimes \E[\vx\vx^\T]$ for the layer-wise Hessian, but with a different $\mM$ motivated by \citet{grosse2016kronecker}. (See \sectionref{sec:appendix_conv})
We note that the off-diagonal blocks of the full Hessian can also be decomposed similarly, which in turn allows %We can then approximate each block using the Kronecker factorization, and when the input auto-correlation matrices are close to rank 1, this allows 
us to approximate the eigenvalues and eigenvectors of the full parameter Hessian. The details of this approximation is stated in \sectionref{sec:appendix_full_hessian}.

\subsection{Structure of Input Auto-correlation Matrix \texorpdfstring{$\E[\vx\vx^\T]$}{ExxT} and output Hessian \texorpdfstring{$\E[\mM]$}{M}}
\label{sec:xxT} 
For the  auto-correlation matrix, one can decompose it as $\E[\vx\vx^\T] = \E[\vx]\E[\vx]^\T + \mbox{Var}[\vx]$. A key observation is that the input $\vx$ for most layers are outputs of a ReLU, hence it is nonnegative. For large networks the mean component $\E[\vx]\E[\vx]^\T$ will dominate the variance, making $\E[\vx\vx^\T]$ approximately rank-1 with top eigenvector being very close to $\E[\vx]$. %, which is positive if it is ReLU activated. 
% We can decompose the auto-correlation matrix as $\E[\vx\vx^\T] = \E[\vx]\E[\vx]^\T + \mSigma_\vx,$
% where $\mSigma_\vx:=\E[(\vx-\E[\vx])(\vx-\E[\vx])^\T]$ is the auto-covariance matrix.
% As every sample $\vx$ is nonnegative, the expectation $\E[\vx]\E[\vx]^\T$ has a large norm and usually dominates the covariance matrix $\mSigma_\vx$.
We empirically verified this phenomenon on a variety of networks and datasets (see \sectionref{sec:appendix_xxT}).
% (squared dot product mean: 0.997, range: 0.964-1.000). Meanwhile $\|\E[\vx]\E[\vx]^\T\|$ is significantly larger than $\|\mSigma_\vx\|$ in our experiments ($\|\E[\vx]\E[\vx]^\T\|/\|\mSigma_\vx\|$ mean: 12.08, range: 2.28-30.03). 
% This suggests that $\E[\vx]\E[\vx]^\T$ is approximately equal to $\E[\vx\vx^\T]$ and dominates the covariance $\mSigma_\vx$. Similar phenomenon also exists for convolution layers. The complete experiment results are provided in \sectionref{sec:appendix_xxT}. We also observe the $\E[\vx\vx^\T]$ matrices are all close to rank 1 throughout the training trajectory as shown in \sectionref{sec:appendix_training_traj}.

For the output Hessian, we observe that $\E[\mM]$ is approximately rank $c-1$ (with $c-1$ significantly large eigenvalues) in most cases. In \sectionref{sec:theoretical}, we show this is indeed the case in a simplified setting, and give a formula for computing the top $c-1$ eigenspace using rows of weight matrices. % can prove that all other eigenvalues of $\E[\mM]$ converges to 0 in a simplified setting. In addition, we can express its top $c-1$ eigenvectors using the rows of weight matrices.

% Then, several predictions on the structure of the eigenspectrum and eigenspace of the Hessians. We will describe them below in this section, show proofs for these structures on a simplified model in \sectionref{sec:theoretical}, and empirically verify these predictions in \sectionref{sec:empirical}.
%

\subsection{Implications on the eigenspectrum and eigenvectors of layer-wise Hessian}
\label{sec:conjecture-implication}
The eigenvectors of a Kronecker product is the tensor product of eigenvectors of its components. As a result, let $\vh_i$ be the $i$-th eigenvector of a layer-wise Hessian $\mH$, if we matricize it as defined in %the decoupling conjecture implies that with the matricization operation defined in 
\definitionref{def:matricization},
$\Mat(\vh_i)$ would be approximately rank 1. Since $\E[\vx\vx^\T]$ is close to rank 1, by the decoupling conjecture, the top eigenvalues of layer-wise Hessian can be approximated as the top eigenvalues of $\E[\mM]$ multiplied by the first eigenvalue of $\E[\vx\vx^\T]$. %Thus, the top eigenvalues of Hessians should have the same relative ratios as the top eigenvalues of their corresponding $\E[\mM]$'s. Therefore, if the decoupling conjecture holds, 
The low rank structure of the layer-wise Hessian $\mH$ is due to the low rank structure of $\E[\mM]$.

%$\vu \otimes \vv$, where $\vu$ is some eigenvector of $\E[\mM]$ and $\vv$ is some eigenvector of $\E[\vx\vx^\T]$. Thus, we would expect $\Mat(\vh_i) \approx \vu\vv^\T$ to be close to rank 1.

%Previous works observed the gap in Hessian eigenspectrum around the number of classes $c$. 
%One characteristic of Hessian that has been mentioned by many is the outliers in the spectrum of eigenvalues. \citet{sagun2017empirical} suggests that there is a gap in Hessian eigenvalue distribution around the number of classes $C$ in most cases, where $C=10$ in our case. \citet{papyan2019measurements} attempted further explanation for the $C$ outliers using class clustering. 
Another implication is related to eigenspace overlap for different models. Even though the output Hessians of two randomly trained models may be very different, the top eigenspace of the Hessian will be close to $\E[\vx]\otimes I$, so the top eigenspace of the two models will have a high overlap that peaks at the output dimension. See \cref{sec:models} for more details.

%% file: theoretical_analysis.tex
% In this section, we provide the proof sketch of the decoupling conjecture for the top $c-1$ eigenspace of the layer-wise parameter Hessian of a simplified setting. In particular, we consider an infinite width neural network over random data at initialization.

In this section, %we analyze the top eigenspace of the Hessian for a special network: 
we show that for a simple setting of 2-layer networks, % 2-layer infinite width neural network over random data at initialization, 
the layer-wise parameter Hessian has $c-1$ large eigenvalues and its top $c-1$ eigenspace is close to the top $c-1$ eigenspace of the Kronecker product approximation.
%\znote{not sure about the wording here} \ynote{changed wording}

\paragraph{Problem Setting and Notations}
\label{sec:theory:setting}
%In this section we will generally follow the notation defined in \sectionref{sec:prelim}. Moreover, 
Let bold non-italic letters such as $\rvv, \rmM$ denote random vectors (lowercase) and matrices (uppercase).
Consider a two layer fully connected ReLU activated neural network with input dimension $d$, hidden layer dimension $n$ and output dimension $c$. In particular, let $d=n^{1+\alpha}$ for some constant $\alpha>0$. Let the network has positive input from a rectified Gaussian $\rvx\sim \rectNormal(0, \mI_d)$ where every entry is identically distributed as $\max\{\hat{\rvx},0\}$ for $\hat{\rvx}\sim \mathcal{N}(0,1)$. %which, for a single entry, has density $f_\rectNormal(x) =\frac12\delta(x)+\frac{1}{\sqrt{2\pi}}\exp(-\frac{x^2}{2})\ind\sbr{x> 0}$ where $\delta(x)$ is the Dirac delta function. 
Let $\mW^{(1)}\in\R^{n\times d}$ and $\mW^{(2)}\in\R^{c\times n}$ be the weight matrices. In this problem we consider a random Gaussian initialization that $\mW^{(1)}\sim \cN(0,\frac1d\mI_{dn})$ and $\mW^{(2)}\sim \cN(0,\frac1n\mI_{nc})$. Both weight matrices has expected row norm of 1. Let the loss objective be cross entropy $\ell$. Training labels are irrelevant as they are independent from the Hessian at initialization. %We do not need to specify the training labels as they are independent from the Hessian at initialization.

Denote the output of the first and second layer as $\rvy$ and $\rvz$ respectively. We have $\rvy = \sigma(\mW^{(1)}\rvx)$ and $\rvz = \mW^{(2)}\rvy.$ Here $\sigma$ is the element-wise ReLU function.
Let $\rmD\triangleq\diag(\ind\sbr{\rvy\geq 0})\in\R^{n\times n}$ denote the 0/1 diagonal matrix representing the activation of $\sigma$ that $\rvy=\rmD\mW^{(1)}\rvx$.
Let $\rvp=\mbox{softmax}(\rvz)$ and let $\rmA\triangleq\diag(\rvp)-\rvp\rvp^\T$. Note $\rmA$ is rank $c-1$ with the null space of the all one vector.  We give full details about our settings in \sectionref{sec:proof-prelim}.
By simple matrix calculus (see \sectionref{sec:appendix_derivation}), the output Hessian of $\mM^{(1)}$ and the full layer-wise Hessian has closed-form
\begin{equation}
    \mM^{(1)} = \exop{\rvx\sim \rectNormal(0, \mI_d)}{\rmD\mW^{(2)\T}\rmA\mW^{(2)}\rmD}, \mH^{(1)} = \exop{\rvx\sim \rectNormal(0, \mI_d)}{\rmD\mW^{(2)\T}\rmA\mW^{(2)}\rmD\otimes \rvx\rvx^\T}.
\end{equation}
Following the decoupling conjecture, the Kronecker approximation of the layer-wise Hessian is \begin{equation}
    \hat\mH^{(1)} \triangleq \exop{\rvx\sim \rectNormal(0, \mI_d)}{\rmD\mW^{(2)\T}\rmA\mW^{(2)}\rmD}\otimes \exop{\rvx\sim \rectNormal(0, \mI_d)}{\rvx\rvx^\T}.
\end{equation}
Since we are always taking the expectation over the input $\rvx$, we will neglect the subscript and use $\E$ for expectation. Now we are ready to state our main theorem.

\begin{theorem}
\label{thm:main-full}
For an infinite width two-layer ReLU activated neural network with Gaussian initialization as defined above, let $V_1$ and $V_2$ be the top $c-1$ eigenspaces of $\mH^{(1)}$ and $\hat\mH^{(1)}$ respectively, for all $\eps>0$, 
$
    \lim_{n\to\infty}\mathop{\Pr}_{\mW^{(1)}\sim\gN(0,\frac{1}{d}\mI_{nd}), \mW^{(2)}\sim\gN(0,\frac{1}{n}\mI_{cn})}\left[\Overlap\left(V_1,V_2\right)>1-\eps\right] = 1.
$
Moreover $\mH^{(1)}$ has $c-1$ large eigenvalues that, \begin{equation}
    \lim_{n\to\infty}\mathop{\Pr}_{\mW^{(1)}\sim\gN(0,\frac{1}{d}\mI_{nd}), \mW^{(2)}\sim\gN(0,\frac{1}{n}\mI_{cn})}\left[\left(\left.\frac{\lambda_c(\mH^{(1)})}{\lambda_{c-1}(\mH^{(1)})}\right|_{\mW^{(1)}, \mW^{(2)}}\right) < \eps\right] = 1.
\end{equation}
\end{theorem}

Instead of directly working on the layer-wise Hessian, we first show a similar theorem for the output Hessian $\mM^{(1)}$. We will then show that the proof technique of the following theorem can be easily generalized to prove our main theorem.

\begin{theorem}
\label{thm:main-out}
For the same network as in \theoremref{thm:main-full}, let $\mM^*\triangleq \ex{\rmD'\mW^{(2)\T}\rmA\mW^{(2)}\rmD'}$ where $\rmD'$ is an independent copy of $\rmD$ and is independent of $\rmA$. Let $S_1$ and $S_2$ be the top $c-1$ eigenspaces of $\mM^{(1)}$ and $\mM^*$ respectively, $S_2$ is approximately $\gR\{\mW_i\}_{i=1}^c\backslash\{\textbf{1}^\T\mW\}$ where $\gR$ is the row span, and for all $\eps>0$,
$\lim_{n\to\infty}\mathop{\Pr}_{\mW^{(1)}\sim\gN(0,\frac{1}{d}\mI_{nd}), \mW^{(2)}\sim\gN(0,\frac{1}{n}\mI_{cn})}\left[\Overlap\left(S_1,S_2\right)>1-\eps\right] = 1.$
Moreover, $\mM$ has $c-1$ large eigenvalues that \begin{equation}
    \lim_{n\to\infty}\mathop{\Pr}_{\mW^{(1)}\sim\gN(0,\frac{1}{d}\mI_{nd}), \mW^{(2)}\sim\gN(0,\frac{1}{n}\mI_{cn})}\left[\left(\left.\frac{\lambda_c(\mM^{(1)})}{\lambda_{c-1}(\mM^{(1)})}\right|_{\mW^{(1)}, \mW^{(2)}}\right) < \eps\right] = 1.
\end{equation}
\end{theorem}

\textbf{Remark.}
The closed form approximating of $S_1$ in \cref{thm:main-out} can be heuristically extended to the case with multiple layers, that the top eigenspace of the output Hessian of the $k$-layer would be approximately $\gR(\mS^{(k)})\setminus\{\textbf{1}^\T\mS^{(k)}\}$
where $\mS^{(k)} = \mW^{(n)}\mW^{(n-1)}\cdots\mW^{(k+1)}$ and $\gR(\mS^{(k)})$ is the row space of $\mS^{(k)}$.
Though our result was only proven for random initialization and random data, we observe that this subspace also has high overlap with the top eigenspace of output Hessian at the minima of models trained with real datasets. The corresponding empirical results are shown in \cref{sec:app_outhessian_exp}. 

\paragraph{Proof Sketch for \theoremref{thm:main-out}}
For simplicity of notations, in this section we will use $\mW$ to denote $\mW^{(2)}$ and $\mM$ to denote $\mM^{(1)}$ unless specified otherwise.
Our proof of \theoremref{thm:main-out} mainly consists of three parts. First we analyze the structure of $\mM^*$ and show that it is approximately rank $c-1$. Then we show that $\mM^*$ and $\mM$ are roughly equivalent via an approximate independence between $\rmD$ and $\rmA$. Finally, by projecting both $\mM$ and $\mM^*$ onto a $c\times c$ matrix using $\mW$, we can apply the approximate independence and prove that the top $c-1$ eigenspace of $\mM^*$ is approximately that of $\mM$, which concludes the proof.

\paragraph{(1) Structure of $\mM^*$} When $n\to\infty$, the output of the second layer $\rvy$ converges to a multivariate Gaussian (\lemmaref{lemma:y-gaussian}), hence we can consider each diagonal entry of $\rmD$ as a $p=\frac12$ Bernoulli random variable. Since we assumed that $\rmD'$ and $\rmA$ are independent, by some simple calculation,
\begin{equation}
    \mM^*=\frac14\left(\mW^\T\E[\rmA]\mW+\text{diag}(\mW^\T\E[\rmA]\mW)\right).
\end{equation}
Here $\E[\rmA]$ is rank $c-1$ with the $(c-1)$-th eigenvalue bounded below from 0 (\lemmaref{lemma:A-rank-c-1}).
Since the two terms in the sum has the same trace while $\mW^\T\E[\rmA]\mW$ is rank $c-1$ compared to rank $n$ of $\diag(\mW^\T\E[\rmA]\mW)$, we can show that the top eigenspace is dominated by the eigenspace of $\mW^\T\E[\rmA]\mW$, which is approximately $\gR\{\mW_i\}_{i=1}^c\backslash\{\textbf{1}^\T\mW\}$. 

\paragraph{(2) Approximate Independence Between $\rmA$ and $\rmD$} Intuitively, if $\rmD$ and $\rmA$ are independent, then $\mM = \mM^*$. However, this is clearly not true - if the activations align with a row of $\mW$ then the corresponding output is going to be large, which changes $\rmA$ significantly. To address this problem, we observe that the formula for $\mM$ is only of degree 2 in $\rmD$, so one can focus on conditioning on two of the activations \--- a negligible fraction in the limit. More precisely, if one expand out the expression of each element squared in $\mM$, it is an homogeneous polynomial of the form
$p(\rmA,\rmD,\bar\rmA,\bar\rmD) = \sum_{i,j,k,l=1}^c\sum_{p,q=1}^n c_{ijklpq}\rmA_{ij}\bar\rmA_{kl}\rmD_{pp}\bar\rmD_{qq},
$
where $(\bar\rmA,\bar\rmD)$ are independent copies of $(\rmA, \rmD)$. The same element squared in $\mM^*$ is just going to be $p(\rmA,\rmD',\bar\rmA,\bar\rmD')$.
By nice properties of the Gaussian initialized weight matrix, we show that as $n\to\infty$, $\rmA$ is invariant when conditioning on two entries of $\rmD$ (\lemmaref{lemma:z-invariant}). Therefore, in the limit we have $\lim_{n\to\infty}\ex{p(\rmA, \rmD, \bar\rmA, \bar\rmD)} = \ex{p(\rmA, \rmD', \bar\rmA, \bar\rmD')}$ (detailed proof in Appendix). 
%Hence by the portmanteau theorem we have the following key lemma:
%\begin{lemma}
%\label{lemma:polynomial-maintext} (informal)
%Let $p(\rmA,\rmD,\bar\rmA,\bar\rmD)$ be a homogeneous polynomial as defined above with an upper bounded $\ell_1$ norm of the coefficients, then
%\begin{equation}
%    \lim_{n\to\infty}\ex{p(\rmA, \rmD, \bar\rmA, \bar\rmD)} = \ex{p(\rmA, \rmD', \bar\rmA, \bar\rmD')}.
%\end{equation}
%\end{lemma}

\paragraph{(3) Equivalence between $\mM^*$ and $\mM$} Since the size of $\mM$ also goes to infinity as we take the limit on $n$, it is technically difficult to directly compare their eigenspaces. In this case we utilize the fact that $\mW$ has approximately orthogonal rows, and project $\mM$ onto $\mW\mM\mW^\T$. In particular, by expanding out the Frobenious norms as polynomials and bounding the $\ell_1$ norm of the coefficients, using \lemmaref{lemma:z-invariant} we are able to show that $\fns{\mM}\approx\fns{\mW\mM\mW^\T}\approx \fns{\mW\mM^*\mW^\T}\approx \fns{\mM^*}$ (\lemmaref{lemma:M-proj-preserve-f-norm}- \lemmaref{lemma:F-norm-equal}).
This result tells us that the projection does not lose information, and hence indirectly gives us the dominating eigenspace of $\mM$. This concludes our proof for \theoremref{thm:main-out}

\paragraph{Proving \theoremref{thm:main-full} and Beyond}
To prove Theorem~\ref{thm:main-full}, we use a very similar strategy. We consider a re-scaled Hessian $\tmH\triangleq\frac1d\mH$ and show that in the independent setting $\tmH^* = \frac1d\E[\rmD'\mW\rmA\mW\rmD'\otimes \rvx''\rvx''^\T] = \mM^*\otimes\frac1d \E[\rvx''\rvx''^\T].$ We then generalize the conditioning technique to involve conditioning on two entries of $x$.

%% file: hessian_empirical.tex
\label{sec:empirical}
In this section, we present some empirical observations that either verifies, or are induced by the decoupling conjecture.
We conduct experiments on the CIFAR-10, CIFAR-100  \citep{Krizhevsky09learningmultiple}, and MNIST  \citep{lecun1998gradient} datasets as well as their random labeled versions, namely MNIST-R and CIFAR10-R. 
% We use PyTorch \citep{NEURIPS2019_9015} framework for all experiments.
We used different fully connected (fc) networks (a fc network with $m$ hidden layers and $n$ neurons each hidden layer is denoted as F-$n^m$), several variations of LeNet \citep{lecun1998gradient}, VGG11 \citep{simonyan2014very}, and ResNet18 \citep{kaiming2015}.
% The results shown in the main text are from variants of VGG11 and ResNet18 trained on CIFAR100, variants of LeNet5 trained on CIFAR10, and F-$200^2$ trained on MNIST.
% The eigenvalues and eigenvectors of the exact layer-wise Hessians are approximated using a modified Lanczos algorithm \citep{hessian-eigenthings} which is described in detail in \sectionref{sec:appendix_eigencomp}. 
We use ``layer:network'' to denote a layer of a particular network. For example, conv2:LeNet5 refers to the second convolutional layer in LeNet5.
More empirical results are included in \sectionref{sec:appendix_exp_res}.

\subsection{Kronecker Approximation of Layer-wise Hessian and Full Hessian}\label{subsec:approx}

To verify the decoupling conjecture in practical settings, we compare the top eigenvalues and eigenspaces of the approximated Hessian and the true Hessian. We use subspace overlap (\definitionref{def:overlap}) to measure the similarity between top eigenspaces. As shown in \figureref{fig:eigeninfo_approx}, this approximation works reasonably well on the top eigenspace.

\begin{figure}[ht]
    \centering
\begin{subfigure}[b]{0.24\columnwidth}
    \captionsetup{justification=centering}
    \includegraphics[width=\columnwidth]{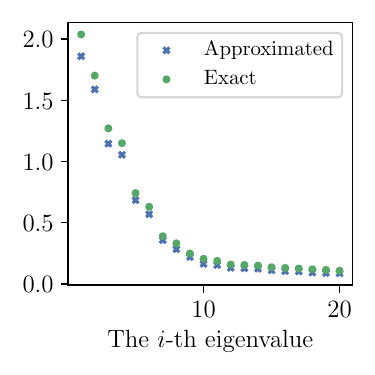}
    \vspace{-0.2in}
    \caption{Top eigenvalues of layer-wise Hessian of fc1}
    % \caption{Eigenvalues}
    \label{fig:eigenval_approx}
\end{subfigure}%
\begin{subfigure}[b]{0.24\columnwidth}
    \captionsetup{justification=centering}
    \includegraphics[width=\columnwidth]{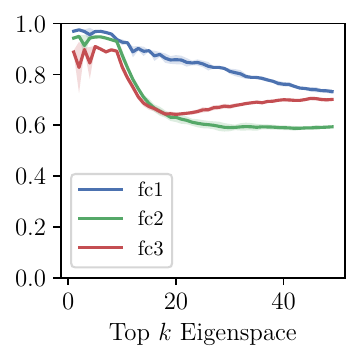}
    \vspace{-0.2in}
    \caption{Top eigenspace of layer-wise Hessians}
    % \caption{Eigenspace overlap}
    \label{fig:overlap_approx}
\end{subfigure}
\begin{subfigure}[b]{0.24\columnwidth}
    \captionsetup{justification=centering}
    \includegraphics[width=\columnwidth]{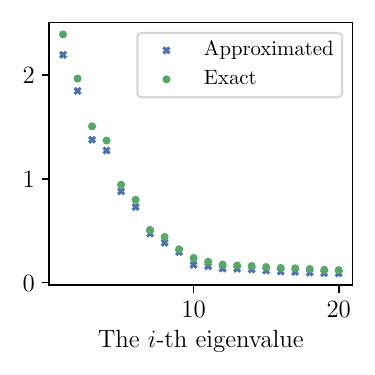}
    \vspace{-0.2in}
    \caption{Top eigenvalues of the full Hessian}
    % \caption{Eigenspace overlap}
    \label{fig:eigenval_approx_full}
\end{subfigure}
\begin{subfigure}[b]{0.24\columnwidth}
    \captionsetup{justification=centering}
    \includegraphics[width=\columnwidth]{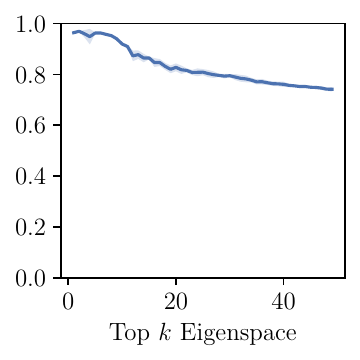}
    \vspace{-0.2in}
    \caption{Top eigenspace of the full Hessian}
    % \caption{Eigenspace overlap}
    \label{fig:overlap_approx_full}
\end{subfigure}
% \vspace{-0.2in}
\caption{Comparison between the approximated and true layer-wise Hessian of F-$200^2$.}
\vspace{-4pt}
\label{fig:eigeninfo_approx}
\end{figure}
%for the top eigenvalues and eigenspaces of both layer-wise weight Hessians and the full parameter Hessian.
%In this section we leverage the Kronecker factorization to understand structures of the dominating eigenspace of layerwise Hessians.
\subsection{Low Rank Structure of \texorpdfstring{$\E[\mM]$}{EM} and \texorpdfstring{$\mH$}{H}}
\label{sec:emp_outlier}
Another way to empirically verify the decoupling conjecture is to show the similarity between the outliers in eigenspectrum of the layer-wise Hessian $\E[\mM]$ and the output Hessian $\mH_\cL$.
%We can also conjecture that the outliers also appear in $\E[\mM]$.
\figureref{fig:UTAU_H_spec} shows the similarity of eigenvalue spectrum between $\E[\mM]$ and layer-wise Hessians in different situations, which agrees with our prediction. For (a) and  (b) we are also seeing the eigengap at $c-1$, which is consistent with our analysis and previous observations \citep{sagun2017empirical,papyan2019measurements}. However, the eigengap does not appear at minimum for random labeled data with a under-parameterized network, meaning that our theory may not generalize to all settings.

\begin{figure*}[h]
\resizebox{\textwidth}{!}{%
    \centering
    \captionsetup[sub]{format=subcaptionformat}
    \begin{subfigure}[b]{0.3\textwidth}
        \centering
        \captionsetup{justification=centering}
        \includegraphics[width=\textwidth]{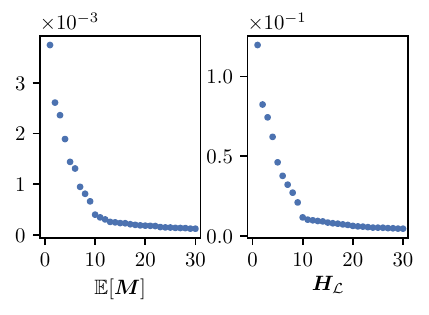}
        \caption{fc1:LeNet5 at initialization (CIFAR10).}
        \label{fig:UTAU_H_spec_FC2}
    \end{subfigure}%
    \begin{subfigure}[b]{0.3\textwidth}
        \centering
        \captionsetup{justification=centering}
        \includegraphics[width=\textwidth]{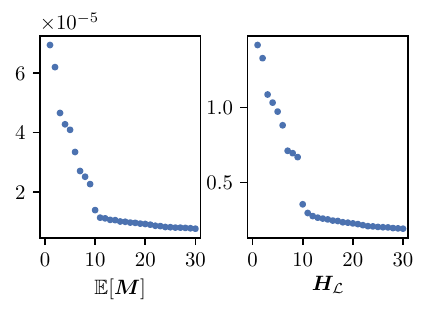}
        \caption{fc1:LeNet5 at minimum\\ (CIFAR10).}
        \label{fig:UTAU_H_spec_Lenet}
    \end{subfigure}%
    \begin{subfigure}[b]{0.3\textwidth}
        \centering
        \captionsetup{justification=centering}
        \includegraphics[width=\textwidth]{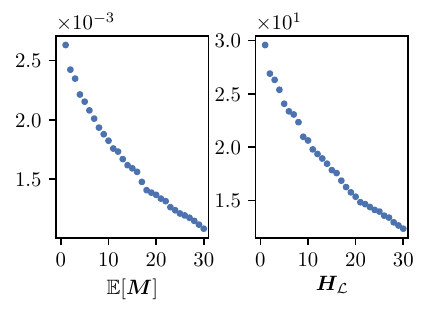}
        \caption{fc1:LeNet5 at minimum\\ (CIFAR10-R).}
        \label{fig:UTAU_H_spec_RL}
    \end{subfigure}%
    }
     %\captionsetup{justification=centering}
    
    \caption{Eigenspectrum of the layer-wise output Hessian $\E[\mM]$ and the layer-wise weight Hessian $\mH_\Ls(\vw^{(p)})$. The vertical axes denote the eigenvalues. Similarity between the two eigenspectra is a direct consequence of a low rank $\E[\vx\vx^T]$ and the decoupling conjecture.}
    \label{fig:UTAU_H_spec}
    \vskip -0.05in
\end{figure*}

\subsection{Eigenspace Overlap of Different Models}
\label{sec:models}
Apart from the phenomena that are direct consequences of the decoupling conjecture, we observe another nontrivial phenomenon involving different minima.
Consider models with the same structure, trained on the same dataset, but using different random initializations, despite no obvious correlation between their parameters, we observe surpisingly high overlap between the dominating eigenspace of some of their layer-wise Hessians.
% \begin{figure}[h]
% \centering
% % \vspace{-1em}
% \subfigure[\small{conv12:ResNet18} ]{\includegraphics[width=0.25\linewidth]{Figures/SubspaceOverlap/NeurIPS/Overlap_ResNet_conv12.conv1.pdf}}
% \quad
% \subfigure[\small{conv6:VGG11} ]{\includegraphics[width=0.25\linewidth]{Figures/SubspaceOverlap/NeurIPS/VGG11_conv6.pdf}}
% \quad
% \subfigure[\small{fc1:LeNet5} ]{\includegraphics[width=0.25\linewidth]{Figures/SubspaceOverlap/NeurIPS/LeNetVarying.pdf}}
% \caption{Overlap between the top $k$ dominating eigenspace of different independently trained models. In each figure, we vary the number of output neuron/channels $m$. We includes 4 variants of LeNet5 trained on CIFAR10 (a), 3 variants of ResNet18 trained on CIFAR100 (b), and 3 different of VGG11 trained on CIFAR100 (c). For each structural variant, 5 models are trained from independent random initializations. We plot the average pairwise overlap between the top eigenspaces of those models' layer-wise Hessians. The overlap peaks at the output dimension $m$.}
% \label{fig:overlap}
% \vspace{-1em}
% \end{figure}

\begin{figure}[H]
    \captionsetup[sub]{format=subcaptionformat}
    \centering
    \begin{subfigure}[h]{0.32\columnwidth}
        \centering
        \captionsetup{justification=centering}
        \includegraphics[width=\textwidth]{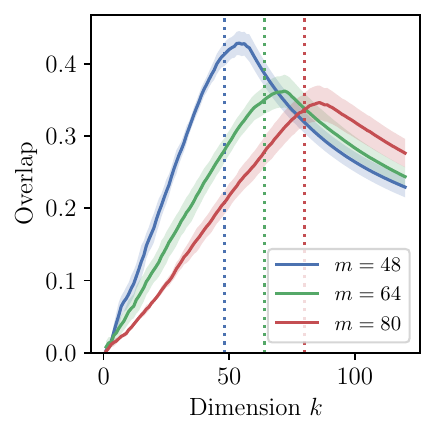}
        \vspace{-0.2in}
        \caption{conv12:ResNet18 (CIFAR100) with 48/64/80 output channels}

        \label{fig:Overlap_resnet_conv1}
    \end{subfigure}
    \begin{subfigure}[h]{0.32\columnwidth}
        \centering
        \captionsetup{justification=centering}
        \includegraphics[width=\textwidth]{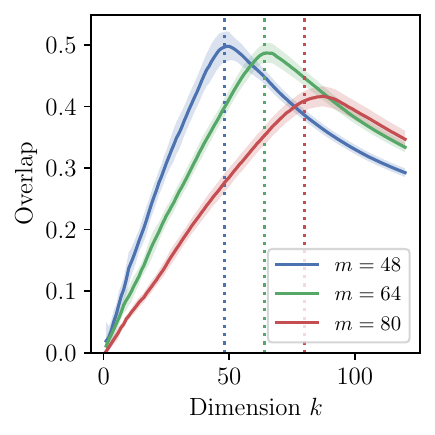}
        \vspace{-0.2in}
        \caption{conv6:VGG11 (CIFAR100) with 48/64/80 output channels}

        \label{fig:Overlap_resnet_conv2}
    \end{subfigure}
    \begin{subfigure}[h]{0.32\columnwidth}
        \centering
        \captionsetup{justification=centering}
        \includegraphics[width=\textwidth]{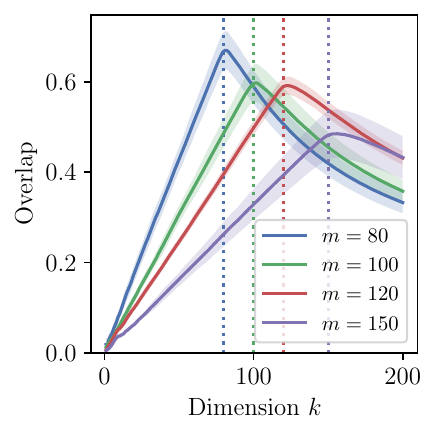}
        \vspace{-0.2in}
        \caption{fc1:LeNet5 (CIFAR10) with 80/100/120/150 output neurons}

        \label{fig:Overlap_fc1}
    \end{subfigure}
    %\captionsetup{justification=centering}

    \caption{Overlap between the top $k$ dominating eigenspace of different independently trained models. The overlap peaks at the output dimension $m$. The eigenspace overlap is defined in \cref{def:overlap}.}
    \label{fig:overlap}
    \vskip -0.1in
\end{figure}

It turns out that the nontrivial overlap is also a consequence of the decoupling conjecture, which arises when the output Hessian and autocorrelation are related in the following way: When the small eigenvalues of $\E[\mM]\in\R^{m\times m}$ approaches 0 slower than the small eigenvalues of $\E[\vx\vx^\T]$, the top $m$ eigenspace will then be approximately spanned by $\mI_m\otimes \E[\vx]^\T$ by the decoupling conjecture.
Now suppose we have two different models with $\hE[\vx]_1$ and $\hE[\vx]_2$ respectively. Their top-$m$ eigenspaces are approximately $\mI_m \otimes \hE[\vx]_1$ and $\mI_m \otimes \hE[\vx]_2$. Thus the overlap at dimension $m$ is approximately $(\hE[\vx]_1^\T \hE[\vx]_2)^2$, which is large since $\hE[\vx]_1$ and $\hE[\vx]_2$ are the same for the input layer and all non-negative for other layers.
While this particular relation between $\E[\mM]$ and $\E[\vx\vx^\T]$ are true in many shallow networks and in later layers of deeper networks, they are not satisfied for earlier layers of deeper networks.  In \sectionref{sec:appendix_model_overlap} we explain how one can still understand the overlap using correspondence matrices when the above simplified argument does not hold. %why the overlap before rank-$m$ grows linearly. We also make a more general explanation and account for some cases where this argument does not hold. 

%% file: pac_bayes.tex
\label{sec:pac}
The PAC-Bayes bound is a commonly used bound for the generalization gap of neural networks. In this section we show how we can obtain tighter PAC-Bayes bounds using the Kronecker approximation of Hessian eigenbasis.
\begin{theorem}[PAC-Bayes Bound]
\citep{mcallester1999some,langford2001bounds}
With the hypothesis space $\gH$ parametrized by model parameters.
For any prior distribution $P$ in $\gH$ that is chosen independently from the training set $S$, and any posterior distribution $Q$ in $\gH$ whose choice may inference $S$, with probability $1-\delta$, $\KL\left(\hat{e}(Q)||e(Q)\right)\leq \frac{1}{|S|-1}\left[\KL(Q||P) + \log\frac{|S|}{\delta}\right]$. Where $e(Q)$ is the expected classification error for the posterior over the underlying data distribution and $\hat{e}(Q)$ is the classification error for the posterior over the training set.
\end{theorem}

Intuitively, if one can find a posterior $Q$ that has low loss on the training set, and is close to the prior $P$, then the generalization error on $Q$ must be small. \citet{dziugaite2017computing} uses optimization techniques to find an optimal posterior in the family of Gaussians with diagonal covariance. They showed that the bound can be nonvacuous for several neural network models.

We follow \citet{dziugaite2017computing} to set the prior $P$ to be a multi-variant Gaussian. The covariance is invariant with respect to the change of basis since it is a multiple of identity. Thus, 
For the posterior, when the variance in one direction is larger, the distance with the prior decreases; however this also has the risk of increasing the empirical loss over the posterior. In general, one would expect the variance to be larger along a flatter direction in the loss landscape and smaller along a sharper direction.
However, since the covariance matrix of $Q$ is fixed to be diagonal in \citet{dziugaite2017computing}, the search of optimal deviation happens in standard basis vectors which are not aligned with the local loss landscape.
Using the Kronecker factorization as in Equation~\ref{eqn:decouple}, we can approximate the layer-wise Hessian's eigenspace. We set $Q$ to be a Gaussian whose covariance is diagonal in the approximated eigenbasis of the layer-wise Hessians.
Under this posterior change of basis, we can obtain tighter bounds compared to \citet{dziugaite2017computing}. In our experiments, the final posterior variance $\vs'$ is smaller along the direction of eigenvectors with larger eigenvalues (see \figureref{fig:app_PAC}). This agrees with our presumption that the alignment of sharp and flat directions will result in a better optimized posterior $Q$ and thus a tighter bound on classification error.

Detailed algorithm description, experiment results, and plots are shown in \sectionref{sec:appendix_pac}.

\begin{table}[h]
% \vspace{-1em}
\centering
\caption{Optimized PAC-Bayes bounds using different methods. T-$n^m$ and R-$n^m$ represents network F-$n^m$ trained with true/random labels. \textsc{TestEr.} gives the empirical generalization gap. \textsc{Base} represents the bound given by %the algorithm proposed by
\citet{dziugaite2017computing}. \textsc{Ours} represents the bound we get.} %given by our algorithms.}
% \vspace{-1em}
\label{tab:pac}
\begin{center}
\begin{small}
\begin{tabular}{cccccccc}
\toprule
\multicolumn{1}{l}{Model} & \multicolumn{1}{l}{T-$600$}     & \multicolumn{1}{l}{T-$1200$}    & \multicolumn{1}{l}{T-$300^2$}   & \multicolumn{1}{l}{T-$600^2$}   & \multicolumn{1}{l}{R-$600$}     & \multicolumn{1}{l}{T-$600_{10}$} & \multicolumn{1}{l}{T-$200_{10}^2$} \\
\midrule
TestEr.   & 0.015  & 0.016  & 0.015  & 0.015  & 0.493  & 0.018   & 0.021     \\
\textsc{Base}      & 0.154  & 0.175  & 0.169  & 0.192  & 0.605  & 0.287   & 0.417     \\
% \textsc{Appr} (ours)      & 0.146  & 0.173  & 0.142  & 0.171  & 0.565  & 0.242   & 0.273     \\
\textsc{Ours}      & \textbf{0.120} & \textbf{0.142} & \textbf{0.125} & \textbf{0.146} & \textbf{0.568}  & \textbf{0.213}  & \textbf{0.215}    \\
% \textsc{Iter.M} (ours)    & 0.126  & 0.149  & 0.131  & 0.150  & \textbf{0.562} & 0.223   & 0.273\\
\bottomrule
\end{tabular}
\end{small}
\end{center}
% \vspace{-2em}
\end{table}

%T-$n^m$ and R-$n^m$ represents network F-$n^m$ trained with true/random labels. \textsc{TestEr.} gives the empirical generalization gap. \textsc{Base} represents the bound given by the algorithm proposed by \citet{dziugaite2017computing}. \textsc{Appr, Iter}, and \textsc{Iter.M} represents the bound given by our algorithms.

% We also plotted the final posterior variance, $\vs$ for network T-$200^2_{10}$ in \figureref{fig:PAC}. For our algorithms, \textsc{Appr, Iter}, and \textsc{Iter.M}, we can see that direction associated with larger eigenvalue has a smaller variance. This agrees with our presumption that top eigenvectors are aligned with sharper directions and should have smaller variance after optimization. 
% Detailed algorithm description and experiment results are shown in \sectionref{sec:appendix_pac}. %\ynote{Shall we keep the figure or refer to the appendix?}

% \begin{figure*}[ht]
%     \centering
%     \includegraphics[width=0.9\textwidth]{Figures/PacBayes/FC2_10cls/sigma_post_compare_iter_iterH_Sigma_post_MNIST_Exp1FC2_fixlr0.01_fc1.weight.pdf}%
  
%     %\captionsetup{justification=centering}

%     \caption{Optimized posterior variance $\vs$ using different algorithms (fc1:T-$200^2$ trained on MNIST). The horizontal axis denotes the eigenbasis ordered with decreasing eigenvalues. The abbreviation of algorithms are the same as in \cref{tab:pac}.}
%     \label{fig:PAC}
% \vspace{-0.1in}
% \end{figure*}

%% file: conclusions.tex
\label{sec:conclusion}
In this paper we proposed the decoupling conjecture which helps in understanding many different structures for the top eigenspace of layer-wise Hessian. Our theory only applies to the initialization for a 2-layer network. How the property can be maintained throughout training is a major open problem. %Although we can only prove the decoupling conjecture and the top eigenspace of output Hessian for a simple 2-layer model, the conjecture and 
However, the implications of the decoupling conjecture can be verified empirically. Having such a conjecture allows us to predict how the structure of the Hessian changes based on architecture/training method (such as batch normalization), and has potential applications in understanding training and generalization (as we demonstrated by the new generalization bounds in Section~\ref{sec:pac}). We hope this work would be a starting point towards formally proving the structures of neural network Hessians.

%% file: appendix_colt.tex
\section{Detailed Derivations}
\label{sec:appendix_derivation_main}
\subsection{Derivation of Hessian}
\input{Appendix_Sections/derivations}
\subsection{Approximating Weight Hessian of Convolutional Layers}
\input{Appendix_Sections/conv_approximation}
\newpage
\section{Main Proof}
\label{sec:main-proof-full}
\input{MainProof/main_proof}
\newpage
\input{Appendix_Sections/full_hessian}
\input{Appendix_Sections/hessian_calc}
\input{Appendix_Sections/experiment_setup}
\input{Appendix_Sections/additional_exp}
% % \newpage
% \newpage
% \input{Appendix_Sections/hessian_calc}
% \input{Appendix_Sections/experiment_setup}
% \input{Appendix_Sections/additional_exp}
% \newpage
\newpage
\input{Appendix_Sections/additional_explanation}

\newpage
\section{Computing PAC-Bayes Bounds with Hessian Approximation}
\input{Appendix_Sections/pac_bayes}

%% file: Appendix_Sections/derivations.tex
\label{sec:appendix_derivation}
For an input $\vx$ with label $\vy$, we define the Hessian of single input loss with respect to vector $\vv$ as
\begin{equation}
    \mH_\ell(\vv, \vx) = \nabla^2_\vv \ell(f_\vtheta(\vx), \vy) = \nabla^2_\vv \ell(\zx, \vy).
\end{equation}
We define the Hessian of loss with respect to $\vv$ for the entire training sample as
\begin{equation}
   \HessL(\vv) = \nabla^2_\vv\Ls(\vtheta) = \sum_{i=1}^N \nabla^2_\vv \ell(f_\vtheta(\vx_i), \vy_i)= \sum_{i=1}^N \mH_\ell(\vv, \vx_i) = \E\left[ \mH_\ell(\vv, \vx)\right].
\end{equation}
We now derive the Hessian for a fixed input label pair ($\vx, \vy$). Following the definition and notations in \sectionref{sec:prelim}, we also denote output as $\vz = f_\vtheta(\vx)$. We fix a layer $p$ for the layer-wise Hessian. Here the layer-wise weight Hessian is $\mH_\ell(\vw^{(p)}, \vx)$. We also have the output for the layer as $\vz^{(p)}$. Since $\vw^{(p)}$ only appear in the layer but not the subsequent layers, we can consider $ \vz = f_\vtheta(\vx) = g_\vtheta(\vz^{(p)}(\vw,\vx))$ where $g_\vtheta$ only contains the layers after the $p$-th layer and does not depend on $\vw^{(p)}$. Thus, using the Hessian Chain rule \citep{skorski2019chain}, we have 
\begin{equation}
    \mH_\ell(\vw^{(p)}, \vx) = \left(\frac{\partial \vz^{(p)}}{\partial\vw^{(p)}}\right)^\T\mH_\ell(\vz^{(p)}, \vx)\left(\frac{\partial \vz^{(p)}}{\partial\vw^{(p)}}\right) + \sum_{i=1}^{m^{(p)}} \frac{\partial \ell(\vz, \vy)}{\partial\evz_i^{(p)}} \nabla^2_{\vw^{(p)}} \evz_i^{(p)},
\end{equation}
where $\evz_i^{(p)}$ is the $i$th entry of $\vz^{(p)}$ and $m^{(p)}$ is the number of neurons in $p$-th layer (size of $\vz^{(p)}$).

Since $\vz^{(p)} = \mW^{(p)}\vx^{(p)} + \vb^{(p)}$ and $\vw^{(p)} = \vect(\mW^{(p)})$ we have
\begin{equation}
    \frac{\partial \vz^{(p)}}{\partial\vw^{(p)}} = \mI_{m^{(p)}} \otimes \vx^{(p)\T}.
\end{equation}
Since $\frac{\partial \vz^{(p)}}{\partial\vw^{(p)}}$ does not depend on $\vw^{(p)}$, for all $i$ we have $\nabla^2_{\vw^{(p)}} \evz_i^{(p)} = 0$.
Thus, \begin{equation}
    \mH_\ell(\vw^{(p)}, \vx) = \left(\mI_{m^{(p)}} \otimes \vx^{(p)}\right)\mH_\ell(\vz^{(p)}, \vx)\left(\mI_{m^{(p)}} \otimes \vx^{(p)\T}\right).
\end{equation}
We define $\mM^{(p)}_\vx = \mH_\ell(\vz^{(p)}, \vx)$ as in \sectionref{sec:prelim} so that \begin{equation}
    \mH_\ell(\vw^{(p)}, \vx) = \left(\mI_{m^{(p)}} \otimes \vx^{(p)}\right)\mM^{(p)}_\vx\left(\mI_{m^{(p)}} \otimes \vx^{(p)\T}\right) = \mM_\vx^{(p)} \otimes \vx^{(p)}\vx^{(p)\T}.
    \label{eqn:appendix_closeformhessian}
\end{equation}

We now look into $\mM_x^{(p)} = \mH_\ell(\vz^{(p)}, \vx)$. Again we have $\vz = g_\vtheta(\vz^{(p)})$ and can use chain rule here,
\begin{equation}
    \mH_\ell(\vz^{(p)}, \vx) = \left(\frac{\partial \vz}{\partial\vz^{(p)}}\right)^\T\mH_\ell(\vz, \vx)\left(\frac{\partial \vz}{\partial\vz^{(p)}}\right) + \sum_{i=1}^{c} \frac{\partial \ell(\vz, \vy)}{\partial\evz_i} \nabla^2_{\vz^{(p)}} \evz_i
\end{equation}
By letting $\vp := \softmax(\vz)$ be the output confidence vector, we define the Hessian with respect to output logit $\vz$ as $\mA_\vx$ and have
\begin{equation}
    \label{eqn:hessian_decomp_general}
    \mA_\vx := \mH_\ell(\vz, \vx) = \nabla^2_\vz l(\vz,\vy)= \diag(\vp)- \vp\vp^\T,
\end{equation}
according to \citet{singla2019understanding}.

We also define the Jacobian of $\vz$ with respect to $\vz^{(p)}$ (informally logit gradient for layer $p$) as $\mG^{(p)}_\vx := \frac{\partial \vz}{\partial\vz^{(p)}}$.
For FC layers with ReLUs, we can consider ReLU after the $p$-th layer as multiplying $\vz^{(p)}$ by an indicator function $\1_{\vz^{(p)} > 0}$. To use matrix multiplication, we can turn the indicator function into a diagonal matrix and define it as $\mD^{(p)}$ where
\begin{equation}
    \mD^{(p)} := \diag\left(\1_{\vz^{(p)} > 0}\right).
\end{equation}
Thus, we have the input of the next layer as $\vx^{(p+1)} = \mD^{(p)}\vz^{(p)}$.
The FC layers can then be considered as a sequential matrix multiplication and we have the final output as
\begin{equation}
    \vz = \mW^{(L)}\mD^{(L-1)}\mW^{(L-1)}\mD^{(L-2)}\cdots \mD^{(p)}\vz^{(p)}.
\end{equation}
Thus, \begin{equation}
    \mG_\vx^{(p)} = \frac{\partial \vz}{\partial\vz^{(p)}} = \mW^{(L)}\mD^{(L-1)}\mW^{(L-1)}\mD^{(L-2)}\cdots \mD^{(p)}.
\end{equation}
Since $\mG_\vx^{(p)}$ is independent of $\vz^{(p)}$, we have
\begin{equation}
    \nabla^2_{\vz^{(p)}} \evz_i = 0, \forall i.
    \label{eqn:appendix_zero_logit_output_hessian}
\end{equation}
Thus, \begin{equation}
    \mM_\vx^{(p)} = \mH_\ell(\vz^{(p)}, \vx) = \mG_\vx^{(p)\T}\mA_\vx\mG_\vx^{(p)}.
\end{equation}
Moreover, loss Hessian with respect to the bias term $\vb^{(p)}$ equals to that with respect to the output of that layer $\vz^{(p)}$. We thus have
\begin{equation}
    \mH_\ell(\vb^{(p)}, \vx) = \mM_\vx^{(p)} = \mG_\vx^{(p)\T}\mA_\vx\mG_\vx^{(p)}.
\end{equation}

The Hessians of loss for the entire training sample are simply the empirical expectations of the Hessian for single input. We have the formula as the following:
\begin{align}
    \HessL(\vw^{(p)}) &= \E\left[\mH_\ell(\vw^{(p)}, \vx)\right] = \E\left[\mM^{(p)}_\vx \otimes \vx^{(p)}\vx^{(p)\T}\right],\label{eqn:app_layerwise_approx}\\
    \HessL(\vb^{(p)}) &= \HessL(\vz^{(p)}) = \E\left[\mM^{(p)}_\vx\right] = \E\left[\mG_\vx^{(p)\T}\mA_\vx\mG_\vx^{(p)}\right].
\end{align}

Note that we can further decompose $\mA_\vx = \mQ_\vx^\T\mQ_\vx$, where 
\begin{equation}
    \Qx = \diag\left(\sqrt{\vp}\right)\left(\mI_c-\1_c\vp^\T\right),
    \label{eqn:app_qx}
\end{equation}
with $\1_c$ is a all one vector of size $c$, proved in \citet{papyan2019measurements}.

We can further extend the close form expression to off diagonal blocks and the bias entries to get the full Gauss-Newton term of Hessian. Let
\begin{align}
    \mF^\T_\vx = \begin{pmatrix}
    \Gx^{(1)\T} \otimes \vx^{(1)}\\
    \Gx^{(1)\T}\\
    \Gx^{(2)\T} \otimes \vx^{(2)}\\
    \Gx^{(2)\T}\\
    \vdots\\
    \Gx^{(L)\T} \otimes \vx^{(n)}\\
    \Gx^{(L)\T}
    \end{pmatrix}.
\end{align}
The full Hessian is given by
\begin{equation}
    \HessL(\vtheta) = \E \left[\mF^\T_{\vx}\mA_\vx\mF_{\vx}\right] + \E\left[\sum_{i=1}^c \frac{\partial \ell(\vz,\vy)}{\evz_i} \nabla^2_\vtheta \evz_i \right].
\label{eqn:app_full_hessian}
\end{equation}

%% file: Appendix_Sections/conv_approximation.tex
\label{sec:appendix_conv}
The approximation of weight Hessian of convolutional layer is a trivial extension from the approximation of Fisher information matrix of convolutional layer by \citet{grosse2016kronecker}.

Consider a two dimensional convolutional layer of neural network with $m$ input channels and $n$ output channels. Let its input feature map $\tX$ be of shape $(n, X_1, X_2)$ and output feature map $\tZ$ be of shape $(m, P_1, P_2)$. Let its convolution kernel be of size $K_1\times K_2$. Then the weight $\tW$ is of shape $(m, n, K_1, K_2)$, and the bias $\vb$ is of shape $(m)$. Let $P$ be the number of patches slide over by the convolution kernel, we have $P=P_1P_2$.

Follow  \citet{dangel2020modular}, we define $\mZ\in\R^{m\times P}$ as the reshaped matrix of $\tZ$ and $\mW\in\R^{m \times nK_1K_2}$ as the reshaped matrix of $\tW.$
Define $\mB\in\R^{m\times P}$ by broadcasting $\vb$ to $P$ dimensions. Let $\mX\in\R^{nK_1K_2 \times P}$ be the unfolded $\tX$ with respect to the convolutional layer. The unfold operation \citep{NEURIPS2019_9015} is commonly used in computation to model convolution as matrix operations.

After the above transformation, we have the linear expression of the $p$-th convolutional layer similar to FC layers:\begin{equation}
    \label{eqn:conv_linear}
    \mZ^{(p)} = \mW^{(p)}\mX^{(p)} + \mB^{(p)}
\end{equation}
We still omit superscription of $(p)$ for dimensions for simplicity. We also denote $\vz^{(p)}$ as the vector form of $\mZ^{(p)}$ and has size $mP$.
Similar to fully connected layer, we have analogue of \equationref{eqn:appendix_closeformhessian} for convolutional layer as
\begin{equation}
    \label{eqn:appendix_closeformhessian_conv}
    \mH_\ell(\vw^{(p)}, \mX) = \left(\mI_{m} \otimes \mX^{(p)}\right)\mM_\vx^{(p)}\left(\mI_{m} \otimes \mX^{(p)\T}\right),
\end{equation}
where $\mM_\vx^{(p)} = \mH_\ell(\vz^{(p)}, \mX)$ and is a $mP \times mP$ matrix.
Also, since convolutional layers can also be considered as linear operations (matrix multiplication with reshape) together with FC layers and ReLUs, \equationref{eqn:appendix_zero_logit_output_hessian} still holds. Thus, we still have  \begin{equation}
    \mH_\ell(\vz^{(p)}, \mX) = \mM_\vx^{(p)} = \mG_\vx^{(p)\T}\mA_\vx\mG_\vx^{(p)},
\end{equation}
where $\mG_\vx^{(p)} = \frac{\partial \vz}{\partial\vz^{(p)}}$ and has dimension $c \times mP$, although is cannot be further decomposed as direct multiplication of weight matrices as in the FC layers.

However, for convolutional layers, $\mX^{(p)}$ is a matrix instead of a vector. Thus, we cannot make \equationref{eqn:appendix_closeformhessian_conv} into the form of a Kronecker product as in \equationref{eqn:appendix_closeformhessian}.

Despite this, it is still possible to have a Kronecker factorization of the weight Hessian in the form
\begin{equation}
   \mH_\ell(\vw^{(p)}, \mX) \approx \tmM^{(p)}_\vx \otimes \mX^{(p)}\mX^{(p)\T},
\end{equation}
using further approximation motivated by \cite{grosse2016kronecker}.
Note that $\tmM^{(p)}_\vx$ need to have a different shape ($m\times m$) from $\mM^{(p)}_\vx$ ($mP\times mP$), since $\mH_\ell(\vw^{(p)},\mX)$ is $mnK1K2 \times mnK1K2$ and $\mX^{(p)}\mX^{(p)\T}$ is $nK1K2 \times nK1K2$.

Since we can further decompose $\Ax = \Qx^\T\Qx$, we then have
\begin{equation}
    \Mx^{(p)} = \mG_\vx^{(p)\T}\Ax\mG_\vx^{(p)} = \left(\Qx\Gx^{(p)}\right)^\T\left(\Qx\Gx^{(p)}\right).
\end{equation}
We define $\mN_\vx^{(p)} =\Qx\Gx^{(p)}$. Here $\Qx$ is $c\times c$ and $\Gx^{(p)}$ is $c\times mP$ so that $\mN_\vx^{(p)}$ is $c \times mP$. We can reshape $\mN_\vx^{(p)}$ into a $cP\times m$ matrix $\tmN_\vx^{(p)}$. We then reduce $\mM^{(p)}_\vx$ ($mP\times mP$) into a $m\times m$ matrix as 
\begin{equation}
    \tmM^{(p)}_\vx = \frac{1}{P}\tmN_\vx^{(p)\T}\tmN_\vx^{(p)}.
\end{equation}
The scalar $\frac{1}{P}$ is a normalization factor since we squeeze a dimension of size $P$ into size 1.

Thus, we can have similar Kronecker factorization approximation as
\begin{align}
    \HessL(\vw^{(p)}) &= \E\left[\mH_\ell(\vw^{(p)}, \mX)\right] = 
    \E\left[\left(\mI_{m} \otimes \mX^{(p)}\right)\mM_\vx^{(p)}\left(\mI_{m} \otimes \mX^{(p)\T}\right)\right] \\ &\approx \E\left[\tmM^{(p)}_\vx \otimes \mX^{(p)}\mX^{(p)\T}\right] \approx \E\left[\tmM^{(p)}_\vx\right] \otimes \E\left[\mX^{(p)}\mX^{(p)\T}\right].
\end{align}

%% file: MainProof/main_proof.tex
\input{MainProof/prelims_drafts}
\input{MainProof/detailed_proofs}

%% file: MainProof/prelims_drafts.tex
This is the complete proof for the two main theorems sketched in \sectionref{sec:theoretical}.
\subsection{Preliminaries}
\label{sec:proof-prelim}
\subsubsection{Notations}
In this section, we generally follow the notation standard by \citet{goodfellow2016deep}. We will use bold italic lowercase letters ($\vv$) to denote vectors, bold non-italic lowercase letters to denote random vectors ($\rvv$), bold italic uppercase letters ($\mA$) to denote matrices, and bold italic uppercase letters ($\rmA$) to denote random matrices.

Moreover, we use $[n]$ for positive integer $n$ to denote the set $\{1,\cdots,n\}$, and $\norm{\mM}$ to denote the spectral norm of a matrix $\mM$. We use $\innerf{\mA, \mB}$ to denote the Frobenius inner product of two matrices $\mA$ and $\mB$, namely $\innerf{\mA, \mB} \triangleq\sum_{i,j}\mA_{i,j}\mB_{i,j}$. We use $\tr(\mM)$ to denote the trace of a matrix $\mM$, and we use $\textbf{1}_c$ to denote the all-one vector of dimension $c$ (the subscript may be omitted when it's clear from the context).

For probability distributions, we use $\rectNormal(\mu,\sigma)$ to denote the rectified Gaussian distribution which has density function\begin{equation}
    f_\rectNormal(x;\mu,\sigma) = \Phi\rbr{\frac\mu\sigma}\delta(x)+\frac{1}{\sqrt{2\pi\sigma^2}}\exp\rbr{-\frac{(x-\mu)^2}{2\sigma^2}}\ind\sbr{x> 0}.
\end{equation}
Here $\Phi$ is the CDF of standard normal distribution, $\delta(x)$ is the Dirac delta function. Note that when $\mu=0$, the density function simplifies to\begin{align}
    f_\rectNormal(x;0,\sigma) = 
    \frac12\delta(x)+\frac{1}{\sqrt{2\pi\sigma^2}}\exp\rbr{-\frac{x^2}{2\sigma^2}}\ind\sbr{x> 0}.
\end{align}
We will use the same notation for multivariate rectified Gaussian distribution, which will be used to characterize the inputs of the network.
% \znote{Move this to somewhere after we define the network?}
\subsubsection{Problem Setting}

Consider a two layer fully connected ReLU activated neural network with input dimension $d$, hidden layer dimension $n$ and output dimension $c$. In particular, $n$ goes to infinity, $d=n^{1+\alpha}$ for some $\alpha>0$, and $c$ is a finite constant. Let network be trained with cross-entropy objective $\gL$. Let $\sigma$ denote the element-wise ReLU activation function which acts as $\sigma(x) = x\cdot\ind_{x\geq 0}$ and the product here is applied element-wise.
Let $\mW^{(1)}\in\R^{n\times d}$ and $\mW^{(2)}\in\R^{c\times n}$ denote the weight matrices of the first and second layer respectively. 
%Let $\vb^{(1)}\in\R^{n}$ and $\vb^{(2)}\in\R^{c}$ denote the weight matrices of the first and second layer respectively.

We consider the case that the neural network has rectified standard Gaussian input $\rvx\sim \rectNormal(0, \mI_d)$. 
Denote the output of the first and second layer as $\rvy$ and $\rvz$ respectively. We have $\rvy = \sigma(\mW^{(1)}\rvx)$ and $\rvz = \mW^{(2)}\rvy.$ Let $\rvp=\mbox{softmax}(\rvz)$ denote the softmax output of the network and let $\rmA\triangleq\diag(\rvp)-\rvp\rvp^\T$.

In this problem, we look into the state of random Gaussian initialization, in which entries of both matrices are i.i.d. sampled from a standard normal distribution, and then re-scaled such that each row of $\mW^{(1)}$ and $\mW^{(2)}$ has norm 1. When taking $n$ and $d$ to infinity, with the concentration of norm in high-dimensional Gaussian random variables, we assume in this problem that entries of $\mW^{(1)}$ are iid sampled from a zero-mean distribution with variance $1/d$, and entries of $\mW^{(2)}$ are iid sampled from a zero-mean distribution with variance $1/n$. This initialization is standard in training neural networks. From the previous analysis of Hessian, the output Hessian corresponding to the first layer has closed form
\begin{equation}
    \mM^{(1)} \triangleq \exop{\rvx\sim \rectNormal(0, \mI_d)}{\rmD\mW^{(2)\\T}\rmA\mW^{(2)}\rmD},
\end{equation}
where $\rmD\triangleq\diag(\ind\sbr{\rvy\geq 0})\in\R^{n\times n}$ is the random 0/1 diagonal matrix representing the activations of ReLU function after the first layer. Note that the output Hessian of the second layer is simply $\mM^{(2)}\triangleq\ex{\rmA}$.

By the Kronecker decomposition, the closed form of the layer-wise Hessians of the first and the second layer are\begin{align*}
    \mH^{(1)} &\triangleq \exop{\rvx\sim \rectNormal(0, \mI_d)}{\rmD\mW^{(2)\T}\rmA\mW^{(2)}\rmD\otimes \rvx\rvx^\T},\\
    \mH^{(2)} &\triangleq \exop{\rvx\sim \rectNormal(0, \mI_d)}{\rmA\otimes \rmD\mW^{(1)}\rvx\rvx^\T\mW^{(1)\T}\rmD}.
\end{align*}
Following the decoupling conjecture, let the Kronecker approximation of the Hessians above be\begin{align*}
    \hat\mH^{(1)} &\triangleq \exop{\rvx\sim \rectNormal(0, \mI_d)}{\rmD\mW^{(2)\T}\rmA\mW^{(2)}\rmD}\otimes\exop{\rvx\sim \rectNormal(0, \mI_d)}{\rvx\rvx^\T},\\
    \hat\mH^{(2)} &\triangleq \exop{\rvx\sim \rectNormal(0, \mI_d)}{\rmA}\otimes\exop{\rvx\sim \rectNormal(0, \mI_d)}{\rmD\mW^{(1)}\rvx\rvx^\T\mW^{(1)\T}\rmD}.
\end{align*}
The decoupling conjecture is then equivalent to $\mH^{(1)}\approx \hat\mH^{(1)}$, $\mH^{(2)}\approx \hat\mH^{(2)}$.

Since our formulae for the Hessians are going to depend on the weight matrices, throughout the section we will condition on the value of $\mW^{(1)}$ and $\mW^{(2)}$ when we take expectation (i.e. the expectation is only taken over the input $\rvx\sim \rectNormal(0, \mI_d)$). We will neglect this under-script of the expectation operator $\E$ as there will be no confusion. When we are discussing the Hessians of a certain layer, we will also neglect the upper-script and just use $\mH$ and $\mM$ when there is no confusion. Moreover, we denote $\rmX\triangleq\ex{\rvx\rvx^\T}$ as the autocorrelation of the input.

Furthermore, for simplicity of notations, we will sometimes use the verbal description ``with probability 1 over $\mW^{(1)}$/$\mW^{(2)}$, event $E$ is true'' to denote
\begin{equation}
    \lim_{n\to\infty}\prop{\mW^{(1)}\sim\gN(0,\frac{1}{d}\mI_{nd}), \mW^{(2)}\sim\gN(0,\frac{1}{n}\mI_{cn})}{E} = 1.
\end{equation}

%% file: MainProof/detailed_proofs.tex
% \newpage
\subsection{Detailed Proof}
\label{sec:detailed-proof}
\input{MainProof/main-theorems}
\input{MainProof/w-properties}
\input{MainProof/approx-independence}
\input{MainProof/A-structure}
\input{MainProof/projection-approx}
\input{MainProof/out-hessian-structure}
\input{MainProof/full-hessian-structure}

%% file: MainProof/main-theorems.tex
First, we restate our main theorems:

\noindent\textbf{\theoremref{thm:main-full} (Decoupling Theorem)}
\emph{
Let $V_1$ and $V_2$ be the top $c-1$ eigenspaces of $\mH^{(1)}$ and $\hat\mH^{(1)}$ respectively, for all $\eps>0$, 
\begin{equation}
    \lim_{n\to\infty}\mathop{\Pr}_{\mW^{(1)}\sim\gN(0,\frac{1}{d}\mI_{nd}), \mW^{(2)}\sim\gN(0,\frac{1}{n}\mI_{cn})}\left[\Overlap\left(V_1,V_2\right)>1-\eps\right] = 1.
\end{equation}
Moreover $\mH^{(1)}$ has $c-1$ large eigenvalues that, \begin{equation}
    \lim_{n\to\infty}\mathop{\Pr}_{\mW^{(1)}\sim\gN(0,\frac{1}{d}\mI_{nd}), \mW^{(2)}\sim\gN(0,\frac{1}{n}\mI_{cn})}\left[\left(\left.\frac{\lambda_c(\mH^{(1)})}{\lambda_{c-1}(\mH^{(1)})}\right|_{\mW^{(1)}, \mW^{(2)}}\right) < \eps\right] = 1.
\end{equation}}

\noindent\textbf{\theoremref{thm:main-out} }
\emph{
Let $\mM^*\triangleq \ex{\rmD'\mW^{(2)\T}\rmA\mW^{(2)}\rmD'}$ where $\rmD'$ is an independent copy of $\rmD$ and is independent of $\rmA$. Let $S_1$ and $S_2$ be the top $c-1$ eigenspaces of $\mM^{(1)}$ and $\mM^*$ respectively, for all $\eps>0$,
\begin{equation}
    \lim_{n\to\infty}\mathop{\Pr}_{\mW^{(1)}\sim\gN(0,\frac{1}{d}\mI_{nd}), \mW^{(2)}\sim\gN(0,\frac{1}{n}\mI_{cn})}\left[\Overlap\left(S_1,S_2\right)>1-\eps\right] = 1.
\end{equation}
Moreover,  \begin{equation}
    \lim_{n\to\infty}\mathop{\Pr}_{\mW^{(1)}\sim\gN(0,\frac{1}{d}\mI_{nd}), \mW^{(2)}\sim\gN(0,\frac{1}{n}\mI_{cn})}\left[\left(\left.\frac{\lambda_c(\mM)}{\lambda_{c-1}(\mM)}\right|_{\mW^{(1)}, \mW^{(2)}}\right) < \eps\right] = 1.
\end{equation}
% THE OLDER VERSION
% For all $\eps > 0$,
% \begin{equation}
%     \lim_{n\to\infty}\mathop{\Pr}_{\mW^{(1)}\sim\gN(0,\frac{1}{d}\mI_{nd}), \mW^{(2)}\sim\gN(0,\frac{1}{n}\mI_{cn})}\left[\left(\left.\frac{\lambda_c(\mM)}{\lambda_{c-1}(\mM)}\right|_{\mW^{(1)}, \mW^{(2)}}\right) < \eps\right] = 1.
% \end{equation}
% Besides, for all $\eps>0$, if we define $S_1$ as the top $c-1$ eigenspace of $\mM$, and $S_2$ as $\gR(\mW)\backslash\{\mW\cdot\textbf{1}\}$ where $\gR(\mW)$ is the row space of $\mW$, then
% \begin{equation}
%     \lim_{n\to\infty}\mathop{\Pr}_{\mW^{(1)}\sim\gN(0,\frac{1}{d}\mI_{nd}), \mW^{(2)}\sim\gN(0,\frac{1}{n}\mI_{cn})}\left[\Overlap\left(S_1,S_2\right)>1-\eps\right] = 1.
% \end{equation}
}

%% file: MainProof/w-properties.tex
\subsubsection{Properties of Infinite Width Weight Matrices}
\label{sec:pf-W-properties}
We will first prove some simple properties of the Gaussian initialized weight matrices $\mW^{(1)}$ and $\mW^{(2)}$ that will facilitate our analysis. Recall that $\mW^{(1)}\in\R^{d\times n}$ and $\mW^{(2)}\in\R^{n\times c}$ where the output dimension $c$ is a finite constant, the hidden layer width $n$ goes to infinity, and the input dimension $d=n^{1+\alpha}$ for some constant $\alpha>0$.
\begin{lemma}
\label{lemma:W-expectation}
 For all $i\in[c]$, for all $\eps>0$,
 \begin{equation}
    \lim_{n\to\infty}\pr{\left|\sum_{j=1}^n\mW_{ij}^{(2)}\right|\geq\eps}=0.
 \end{equation}
\end{lemma}
\begin{proofof}{\lemmaref{lemma:W-expectation}}
Since each entry of $\mW^{(2)}$ is initialized independently from $\gN(0,\frac1n)$, by Central Limit Theorem we have $\sum_{j=1}^n\mW_{ij}^{(2)}\sim \gN(0,\frac{1}{n})$. For any $\eps > 0$, fix $\eps$. By Chebyshev's inequality,
\begin{equation}
    \lim_{n\to\infty}\pr{\left|\sum_{j=1}^n\mW^{(2)}_{ij}\right|\geq\epsilon} < \lim_{n\to\infty}\frac{1}{n\eps^2} = 0. 
\end{equation}
\end{proofof}

\begin{lemma}
\label{lemma:chi2-tail}\citep{laurent2000adaptive} For $X\sim\chi_n^2$,\begin{equation}
    \pr{X-n\geq 2\sqrt{nt}+2t}\leq e^{-t},\qquad \pr{X-n\leq -2\sqrt{nt}}\leq e^{-t}.
\end{equation}

\end{lemma}

\begin{lemma}
\label{lemma:W-norm}
For all $\eps>0$, 
\begin{equation}
    \lim_{n\to\infty}\pr{\abs{\fns{\mW^{(2)}}-c}\geq\eps}=0.
\end{equation}
Beside, for all $i\in[c]$,
\begin{equation}
    \lim_{n\to\infty}\pr{\abs{\ns{\mW_i^{(2)}}-1}\geq\eps}=0.
\end{equation}
\end{lemma}
\begin{proofof}{\lemmaref{lemma:W-norm}}
For simplicity of notations, we will use $\mW$ to denote $\mW^{(2)}$ in this proof.
Since each entry of $\mW$ is initialized independently from $\gN(0,\frac1n)$, we know that $n\fns{\mW}=\sum_{i=1}^c\sum_{j=1}^nn{\mW_{i,j}}^2$ follows a $\chi_{cn}^2$-distribution. From \lemmaref{lemma:chi2-tail} we know that for large enough $n$,
\begin{equation}
\pr{|n\fns{\mW}-cn|\geq n\eps}\geq\pr{|n\fns{\mW}-cn|\geq 2\sqrt{c}n^{3/4}+2n^{1/2}}\leq 2\exp(-n^{1/2}).
\end{equation}
In other words,
\begin{equation}
\lim_{n\to\infty}\pr{|\fns{\mW}-c|\geq\eps} = \lim_{n\to\infty}\pr{|n\fns{\mW}-cn|\geq n\eps} = 0.
\end{equation}
Similarly, for any $i\in[c]$, $n\fns{\mW_i}$ follows a $\chi_n^2$-distribution, so for large enough $n$,
\begin{equation}
\pr{|n\fns{\mW_i}-n|\geq n\eps}\leq\pr{|n\fns{\mW}-n|\geq 2n^{3/4}+2n^{1/2}}\leq 2\exp(-n^{1/2}),
\end{equation}
which indicates that
\begin{equation}
\lim_{n\to\infty}\pr{|\ns{\mW_i}-1|\geq\eps} = \lim_{n\to\infty}\pr{|n\ns{\mW_i}-n|\geq n\eps} = 0.
\end{equation}
\end{proofof}

\begin{lemma}
\label{lemma:w1-col-norm}
Let $\vw_i$ denote the $i$-th column vector of $\mW^{(1)}$.
With probability 1 over $\mW^{(1)}$, \begin{equation}
    \max_{i=1}^d\norm{\vw_i}<5n^{-\frac\alpha2}.
\end{equation}

\end{lemma}
\begin{proofof}{\lemmaref{lemma:w1-col-norm}}
Since entries of $\mW^{(1)}$ are i.i.d. sampled from $\cN(0, \frac1n)$, each $\ns{\vw_i}$ obeys a $\chi_n^2$ scaled by $\frac1d = n^{-(1+\alpha)}$. Thus by the tail bound of \lemmaref{lemma:chi2-tail}, setting $t=n$ we have\begin{equation}
\pr{\ns{\vw_i}\geq 5n^{-\alpha}} =  \pr{d\ns{\vw_i}\geq n+2\sqrt{n^2}+2n}\leq e^{-n}.
\end{equation}
By a Union bound we have \begin{equation}
    \pr{ \max_{i=1}^d\norm{\vw_i}^2\geq 5n^{-\alpha}}\leq \sum_{i=1}^d\pr{ \norm{\vw_i}^2\geq 5n^{-\alpha}} = de^{-n} = n^{1+\alpha}e^{-n}.
\end{equation}
Since $\alpha$ is a constant, RHS converges to 0. Thus with probability 1 over $\mW^{(1)}$, we have \begin{equation}
    \max_{i=1}^d\norm{\vw_i}^2<5n^{-\alpha}.
\end{equation}
Taking square root on both sides completes the proof.
\end{proofof}

\begin{lemma}
\label{lemma:WW-identity}
For any random matrix $\mW$
For all $\eps>0$,
\begin{equation}
    \lim_{n\to\infty}\pr{\norm{\mW^{(1)}\mW^{(1)\T}-\mI_c}\geq\eps}=0.
\end{equation}Besides, for all $i,j\in[c]$, 
\begin{equation}
    \lim_{n\to\infty}\pr{|(\mW^{(1)}\mW^{(1)\T})_{i,j}-\delta_{i,j}|\geq\eps}=0
\end{equation}
Here $\delta$ is the Kronecker delta function, i.e., $\delta_{i,j} = \ind[i=j]$.
\end{lemma}
\begin{proofof}{\lemmaref{lemma:WW-identity}}
To prove this lemma we need the following tail bound:
\begin{lemma}\citep{zhu2012short} If $S$ follows a Wishart distribution $\cW_d(n,C)$, with $r = \emph{tr}(C)/\norm{C}$, for $\theta\geq 0$ the following inequality holds that \begin{equation}
    \pr{\lnorm{\frac1n S-C}\geq \rbr{\sqrt{\frac{2\theta(r+1)}{n}}+\frac{2\theta r}{n}}\norm{C}}\leq 2d\exp(-\theta).
\end{equation}
\label{lemma:W-wishart}
\end{lemma}

Since each entry of $\mW^{(1)}$ is initialized independently from $\gN(0,\frac1d)$, we know that $\mW^{(1)}\mW^{(1)\T}$ follows Wishart distribution $\cW_d(d, \frac1d \mI_n)$. With $r = \tr(\frac1d\mI_n)/\norm{\frac1d\mI_n} = n$ and set $\theta = n^{\frac{\alpha}{2}}$, from \lemmaref{lemma:W-wishart}, for $n\geq 1$ we get
\begin{equation}
\begin{split}
2d\exp(-n^{\frac{\alpha}{2}})&\geq \pr{\lnorm{\frac1d\mW^{(1)}\mW^{(1)\T}-\frac1d\mI_n}\geq \rbr{\sqrt{\frac{2\theta(n+1)}{d}}+\frac{2\theta n}{d}}\lnorm{\frac1d\mI_n}}\\
&=\pr{\lnorm{\frac1d\mW^{(1)}\mW^{(1)\T}-\frac1d\mI_n}\geq \rbr{\sqrt{\frac{2n^{\frac{\alpha}{2}}(2n)}{n^{1+\alpha}}}+\frac{2n^{\frac{\alpha}{2}} n}{n^{1+\alpha}}}\lnorm{\frac1d\mI_n}}\\
&=\pr{\lnorm{\mW^{(1)}\mW^{(1)\T}-\mI_n}\geq 2(n^{-\frac\alpha4}+n^{-\frac\alpha2})}.\\
\end{split}
\end{equation}
Fix any $\eps>0$, we may find $N\in \NN$ such that for all $n>N$, $2(n^{-\frac\alpha4}+n^{-\frac\alpha2})<\eps$. For any $\eps'>0$, we may find $N'$ such that $2d\exp(-n^{\frac\alpha2}) = 2n^{1+\alpha}\exp(-n^{\frac\alpha2}) < \eps'$. Passing $n$ to infinity we get
\begin{equation}
    \lim_{n\to\infty}\pr{\lnorm{\mW^{(1)}\mW^{(1)\T}-\mI_n}>\eps} = 0.
\end{equation}
Then we proceed to analyze the entries. For all $i,j\in[n]$, we have
\begin{equation}
\begin{split}
\pr{|(\mW^{(1)}\mW^{(1)\T})_{i,j}-\delta_{i,j}|\geq\eps} &\leq\pr{\sum_{i,j=1}^n\pr{(\mW^{(1)}\mW^{(1)\T})_{i,j}-\delta_{i,j}}^2\geq\eps^2}\\
&=\pr{\fns{\mW^{(1)}\mW^{(1)\T}-\mI_n}\geq\eps^2}\\
&\leq\pr{\norm{\mW^{(1)}\mW^{(1)\T}-\mI_n}\geq\frac{\eps}{\sqrt{n}}},
\end{split}
\end{equation}
which implies that for all $i,j\in[n]$,
\begin{equation}
\lim_{n\to\infty}\pr{|(\mW^{(1)}\mW^{(1)\T})_{i,j}-\delta_{i,j}|\geq\eps}=0.
\end{equation}
\end{proofof}
For the second weight matrix $\mW^{(2)}$, where $\mW^{(2)}\mW^{(2)\T}\sim \cW_c(n, \frac1n\mI_c)$, we may prove an identical statement as shown in the corollary below. The proof proceeds identical as above since we only need the ratio between the width and the height of $\mW$, which is $n/c$ in this case, to go to infinity.

\begin{corollary}
\label{cor:WW-Identity}
For all $\eps>0$,
\begin{equation}
    \lim_{n\to\infty}\pr{\norm{\mW^{(2)}\mW^{(2)\T}-\mI_n}\geq\eps}=0.
\end{equation}
\end{corollary}

Next we establish the approximate equivalence between the scatter matrix $\mW^{(2)\T}\mW^{(2)}$ and the projection matrix $P_{\mW^{(2)}}$.
\begin{lemma}
\label{lemma:W-projection}
Let $P_{\mW^{(2)}}$ be the projection matrix onto the row space of $\mW^{(2)}$, then for all $\eps>0$,
\begin{equation}
    \lim_{n\to\infty}\Pr\left[\fns{\mW^{(2)\T}\mW^{(2)}-P_{\mW^{(2)}}}>\eps\right]=0.
\end{equation}
\end{lemma}
\begin{proofof}{\lemmaref{lemma:W-projection}}
For simplicity of notations, in this proof we will neglect the layer index superscript and use $\mW$ to denote $\mW^{(2)}$. Recall that $\mW\in\R^{n\times c}$.

Fix $\eps\in(0,1)$ without loss of generality. Let $\mW_i(i\in[c])$ be the $i$-th row of $\mW$, and we will do the Gram–Schmidt process for the rows of $\mW$. Specifically, the Gram–Schmidt process is as following: Assume that the basis $\{\overline{\mW}_i\}_{i=1}^k$ are already normalized, we set $\mW_{k+1}'\triangleq \mW_{k+1} - \sum_{i=1}^k\langle \mW_{k+1}, \overline{\mW}_i\rangle$ and $\overline{\mW}_{k+1}\triangleq\mW_{k+1}'/\norm{\mW_{k+1}'}$. Finally, from the definition of projection matrix, we know that $P_\mW=\overline{\mW}^\T\overline{\mW}$.

From \lemmaref{lemma:W-norm} we have for all $i\in[c]$,
\begin{equation}
    \lim_{n\to\infty}\pr{|\ns{\mW_i}-1|\geq\eps}=0.
\end{equation}
Let $\eps'\triangleq \eps^2/\rbr{c^3\cdot 16^{2c+1}}$, from \lemmaref{lemma:WW-identity} we know that for all $i,j\in[c]$, 
\begin{equation}
    \lim_{n\to\infty}\pr{|\mW_i\mW_j^\T-\delta_{i,j}|\geq\eps'}=0.
\end{equation}

Then we use induction to bound the difference between $\mW$ and $\overline{\mW}$. Specifically, we will show that for all $i\in[c], \norm{\overline{\mW}_i-\mW_i}\leq 8^{i}\eps'$. For simplicity of notations, in the following proof we will not repeat the probability argument and assume that for all $i,j\in[c]$, $|\mW_i\mW_j^\T-\delta_{i,j}|\leq\eps'$ and for all $i\in[c]$, $|\ns{\mW_i}-1|\leq\eps'$. We will only use these inequalities finite times so applying a union bound will give the probability result.

For $i=1$, we know that $\overline{\mW}_1 = \mW_1/\norm{\mW_1}$ and $|\norm{\mW_1}-1|\leq\eps'$, so $\norm{\overline{\mW}_i-\mW_i}\leq \eps'$.

If our inductive hypothesis holds for $i\leq k$, then for $i=k+1$, we have for all $j\leq k$,
\begin{equation}\begin{split}
    |\langle \mW_i, \overline{\mW}_j\rangle|&\leq |\langle \mW_i, \mW_j\rangle| + |\langle \mW_i, \overline{\mW}_j-\mW_j\rangle|\\
    &\leq \eps'+ \norm{\mW_i}\cdot\norm{\overline{\mW}_j-\mW_j}
    \\
    &\leq \eps'+ (1+\eps')8^j\eps'\\
    &\leq (2^{3j+1}+1)\eps'.
\end{split}\end{equation}
Therefore,
\begin{equation}
    \norm{\mW_i'-\mW_i}\leq \sum_{j\in[k]} |\langle \mW_i, \overline{\mW}_j\rangle|\leq \eps'+\sum_{j\in[k]}(2^{3j+1}+1)\eps'\leq (2^{3k+2} - 1)\eps',
\end{equation}
and
\begin{equation}
    |\norm{\mW_i'}-1|\leq |\norm{\mW_i}-1| + \norm{\mW_i'-\mW_i}\leq 2^{3k+2}\eps'.
\end{equation}
Thus,
\begin{equation}\begin{split}
    \norm{\overline{\mW}_i-\mW_i}&\leq \norm{\overline{\mW}_i-\mW_i'} + \norm{\mW_i'-\mW_i}\\
                          &\leq |\norm{\mW_i'}-1| + \norm{\mW_i'-\mW_i}\\
                          &\leq 8^{k+1}\eps',
\end{split}\end{equation}
which finishes the induction and implies that for all $\eps>0$, for all $i\in[c], \norm{\overline{\mW}_i-\mW_i}\leq 8^{i}\eps'$. Thus,
\begin{equation}
\label{eqn:W-equal-W-bar}
    \fns{\overline{\mW}-\mW} = \sum_{i\in[c]} \ns{\overline{\mW}_i-\mW_i}\leq c\cdot 16^c\eps'.
\end{equation}
This means that
\begin{equation}\begin{split}
    \fn{\mW^\T\mW-P_\mW} &= \fn{\mW^\T\mW-\overline{\mW}^\T\overline{\mW}}\\
                  &\leq 2\fn{\mW-\overline{\mW}}\fn{\overline{\mW}} + \fns{\mW-\overline{\mW}}\\
                  &\leq 2c\cdot\sqrt{c}\cdot 8^c\sqrt{\eps'} + c\cdot 16^c\eps'\leq \eps.
\end{split}
\end{equation}
\end{proofof}

For the final property of the weight matrices, we show that the maximum among all entry of the weight matrices are reasonably small with high probability.
\begin{lemma}
\label{lemma:w-not-too-large}Fix any $\alpha>0$, consider $\mW\in\R^{a\times b}$ for some $b>a^{1+\alpha}$ such that each entry is sampled from a zero mean Gaussian $\cN(0, \frac1b)$. The largest entry of $\mW$ is reasonably small with high probability as $b$ goes to infinity, namely,
\begin{equation}
\lim_{b\to\infty}\Pr\left[\max_{(i,j)\in[a]\times[b]}|\mW^{(2)}_{ij}| > 2b^{-\frac13}\right] = 0
\end{equation}
\end{lemma}

\begin{proofof}{\lemmaref{lemma:w-not-too-large}}
For i.i.d. random variables $\rvx_1,\dots, \rvx_b\sim \gN(0,1)$,
by concentration inequality on maximum of Gaussian random variables, for any $t>0$, we have
\begin{equation}
    \Pr\left[\max_{i\in[b]}\rvx_i > \sqrt{2\log (2b)} + t\right] < 2e^{-\frac{t^2}{2}}.
\end{equation}
For any $i,j\in[a]\times [b]$, since $\mW_{ij}$ are i.i.d. sampled from $\gN(0,\frac1b)$, with rescaling of $1/\sqrt{b}$ we may substitute $\rvx_j$ with $\mW_{ij}$. It follows that
\begin{equation}
    \Pr\left[\max_{(i,j)\in[a]\times[b]}\mW^{(2)}_{ij} > \frac{\sqrt{2\log (2ab)} + t}{\sqrt{b}}\right] < 2e^{-\frac{t^2}{2}}.
\end{equation}
Taking $t=b^{\frac16}$, since $a<b$, for large $b$ we have $\sqrt{2\log (2ab)} < \sqrt{2\log (2b^2)} < b^{\frac16}$. Thus for large $b$,
\begin{equation}
\begin{split}
    \Pr\left[\max_{(i,j)\in[a]\times[b]}\mW_{ij} > 2b^{-\frac13}\right] &= \Pr\left[\max_{(i,j)\in[a]\times[b]}\mW_{ij} > \frac{b^{\frac16} + b^{\frac16}}{\sqrt{n}}\right]\\
    & < \Pr\left[\max_{(i,j)\in[a]\times[b]}\mW_{ij} > \frac{\sqrt{2\log (2b)} + b^{\frac16}}{\sqrt{b}}\right] < 2e^{-\frac{b^{\frac13}}{2}}.
\end{split}
\end{equation}
With the same argument, we have
\begin{equation}
    \Pr\left[\min_{(i,j)\in[a]\times[b]}\mW_{ij} < -2b^{-\frac13}\right] < 2e^{-\frac{b^{\frac13}}{2}}.
\end{equation}
Passing $b$ to infinity completes the proof.
\end{proofof}
From the above lemma, we can bound the maximum entry of $\mW^{(1)}$ and $\mW^{(2)}$ as follows:

\begin{corollary}
\label{cor:w-not-too-large}
With probability 1 over $\mW^{(1)}$ and $\mW^{(2)}$,
\begin{equation}
\begin{split}
    \lim_{n\to\infty}\Pr\left[\max_{(i,j)\in[n]\times[d]}|\mW^{(1)}_{ij}| > 2d^{-\frac13}\right] &= 0,\\
    \lim_{n\to\infty}\Pr\left[\max_{(i,j)\in[c]\times[n]}|\mW^{(2)}_{ij}| > 2n^{-\frac13}\right] &= 0.\\
\end{split}
\end{equation}
\end{corollary}

%% file: MainProof/approx-independence.tex
\subsubsection{Approximate Independence Between Layer Inputs and Outputs}
Let us first recall some definitions and notations of the inputs and outputs of layers. The input $\rvx$ follows the $d$-dimensional multivariate rectified Gaussian distribution with identity covariance for the pre-rectified Gaussian, namely $\rvx\sim\rectNormal(0, \mI_d)$. The input propagates through the first layer to $\rvu\triangleq\mW^{(1)}\rvx$, and is multiplied element-wise by the ReLU activation to the input of the second layer $\rvy\triangleq \sigma(\rvu)$. Here we denote that activation of ReLU function by the random matrix $\rmD\triangleq \diag(\ind[\rvu\geq 0])\in\R^{n\times n}$. Finally we get the logit output of the network $\rvz\triangleq \mW^{(2)}\rvy$. The output Hessian of the last layer is $\rmA=\diag(\rvp)-\rvp\rvp^\T\in\R^{c\times c}$.

In this section we will show that when $n$ goes to infinity, both $\rvy$ and $\rvz$ will converge in distribution to rectified Gaussian. Moreover, when we condition on two entries of $\rvx$ and two entries of $\rvy$, the output Hessian $\rmA$ will be invariant in the limiting case.

\begin{lemma}
\label{lemma:y-gaussian}
When $d\to \infty$, with probability 1 over $\mW^{(1)}$,
\begin{equation*}
    \lim_{d\to\infty}\rvy\xrightarrow{d}\cN^{\emph{R}}\left(0,\frac{\pi-1}{2\pi}\mI_n\right).
\end{equation*}
\end{lemma}

\begin{proofof}{\lemmaref{lemma:y-gaussian}}
We will prove this lemma using the multivariate Lindeberg-Feller CLT.
Given that $\ervx_i$'s are i.i.d. sampled from $\rectNormal(0,1)$ with bounded moments:
\begin{equation}
     \E[\ervx_i] = \frac{1}{\sqrt{2\pi}},\qquad \E[(\ervx_i - \E[\ervx_i])^2] = \frac{\pi-1}{2\pi},\qquad \E[(\ervx_i - \E[\ervx_i])^4] = \frac{6\pi^2-10\pi-3}{4\pi^2} < 1.
\end{equation}
For each $i\in[d]$, let $\vw_i^{(1)}\in\R^d$ denote the $i$-th column vector of $\mW^{(1)}$. Let 
$\rvs_i = \vw_i^{(1)}(\ervx_i - \E[\ervx_i])$,
then we have \begin{equation}
\label{eqn:z-expression-w1}
    \rvy = \sum_{i=1}^d\vw_i^{(1)}\ervx_i = \sum_{i=1}^d\rvs_i + \sum_{i=1}^d\E[\ervx_i]\vw_i^{(1)} = \sum_{i=1}^d\rvs_i + \frac{1}{\sqrt{2\pi}}\sum_{i=1}^d\vw_i^{(1)}.
\end{equation}
It follows that 
\begin{equation}
   Var[\rvs_i] = Var[\vw_i^{(1)}\ervx_i] = \frac{\pi-1}{2\pi}\vw_i^{(1)}\vw_i^{(1)^\T}.
\end{equation}
Let $\mS = \sum_{i=1}^d Var[\rvs_i]$,
\begin{equation}
    \mS = \frac{\pi-1}{2\pi}\sum_{i=1}^d\vw_i^{(1)}\vw_i^{(1)^\T} = \frac{\pi-1}{2\pi}\mW^{(1)}\mW^{(1)\T}.
\end{equation}
As $d\to\infty$, from \corollaryref{cor:WW-Identity} we have $\mW^{(1)}\mW^{(1)\T}\to \mI_n$ in probability, therefore
\begin{equation}
\lim_{d\to\infty}\mS = \frac{\pi-1}{2\pi}\mI_n.
\end{equation}

We now verify the Lindeberg condition of independent random vectors $\{\rvs_1,\ldots, \rvs_n\}$.
First observe that the fourth moments of the $\rvs_i$'s are sufficiently small.
\begin{equation}
\begin{split}
    \lim_{d\to\infty}\sum_{i=1}^d\E\left[\left\Vert\rvs_i\right\Vert^4\right] &= \lim_{d\to\infty}\sum_{i=1}^d\E\left[\left(\sum_{j=1}^n \left(\mW^{(1)}_{ji}\left(\ervx_i - \E[\ervx_i]\right)\right)^2\right)^2\right]\\
    & \leq\lim_{d\to\infty} \sum_{i=1}^d\E\left[c^2\left(\left(\max_{j\in[n]}\mW^{(1)}_{ji}\right)^2\left(\ervx_i - \E[\ervx_i]\right)^2\right)^2\right]\\
    & \leq\lim_{d\to\infty} c^2\left(\max_{i\in[d],j\in[n]}\mW^{(1)}_{ji}\right)^4\sum_{i=1}^d\E\left[\left(\ervx_i - \E[\ervx_i]\right)^4\right].
\end{split}
\end{equation}
Since $\E[(\ervx_i - \E[\ervx_i])^4] < 1$ and $\max_{i\in[d],j\in[n]}|\mW^{(1)}_{ji}| < 2d^{-\frac13}$, with probability 1 over $\mW^{(1)}$ from \lemmaref{lemma:w-not-too-large}, it follows that
\begin{equation}
\begin{split}
    \lim_{d\to\infty}\sum_{i=1}^d\E\left[\left\Vert\rvs_i\right\Vert^4\right] \leq c^2\lim_{d\to\infty}\left(2d^{-\frac13}\right)^4\sum_{i=1}^d 1 = c^2 \lim_{d\to\infty}16d^{-\frac43}n =  16c^2 \lim_{d\to\infty}d^{-\frac13} =0.
\end{split}
\end{equation}
For any $\epsilon > 0$, since $\left\Vert\rvs_{i}\right\Vert > \epsilon$ in the domain of integration (when $\ind[\norm{\rvs_i} > \eps]$),
\begin{equation}
\begin{split}
\lim_{d\to\infty}\sum_{i=1}^d\E\left[\left\Vert\rvs_i\right\Vert^2\ind\left[\left\Vert\rvs_{i}\right\Vert > \epsilon\right]\right] & < \lim_{d\to\infty}\sum_{i=1}^d\E\left[\frac{\left\Vert\rvs_i\right\Vert^2}{\epsilon^2}\left\Vert\rvs_i\right\Vert^2\ind\left[\left\Vert\rvs_{i}\right\Vert > \epsilon\right]\right]\\
&\leq \frac{1}{\epsilon^2}\lim_{d\to\infty}\sum_{i=1}^d\E\left[\left\Vert\rvs_i\right\Vert^4\right] = 0.
\end{split}
\end{equation}
As the Lindeberg Condition is satisfied, with $\lim_{d\to\infty}\mS = \frac{\pi-1}{2\pi}\mI_n$ we have
\begin{equation}
\label{eqn:lemma:sumv-convergence-w1}
    \lim_{d\to\infty}\sum_{i=1}^d\rvs_i\xrightarrow{d}\gN\left(0,\frac{\pi-1}{2\pi}\mI_n\right).
\end{equation}

By \lemmaref{lemma:W-expectation}, we have $\lim_{d\to\infty}\vw_i^{(1)} = \overrightarrow{0}$ with probability 1 over $\mW^{(1)}$, therefore plugging \equationref{eqn:lemma:sumv-convergence-w1} into \equationref{eqn:z-expression-w1} we have\begin{equation}
    \lim_{d\to\infty} \rvy\xrightarrow{d}\gN\left(0,\frac{\pi-1}{2\pi}\mI_n\right).
\end{equation}
Which completes the proof.
\end{proofof}

\begin{lemma}
\label{lemma:z-gaussian}
$\lim_{n\to\infty}\rvz\xrightarrow{d} \gN(0, \frac{(\pi-1)^2}{4\pi^2}\mI_c)$ with probability 1 over $\mW^{(2)}$.
\end{lemma}

\begin{proofof}{\lemmaref{lemma:z-gaussian}}
The proof technique for $\rvz$ is identical to that of $\rvy$. For completeness we will redo it for $\mW^{(2)}$.
From \lemmaref{lemma:y-gaussian}, $\ervy_i$'s are i.i.d. from $\rectNormal(0,\frac{\pi-1}{2\pi})$ with bounded moments:
\begin{equation}
    \E[\ervy_i] = \frac{\sqrt{\pi-1}}{2\pi},\ 
    \E[(\ervy_i - \E[\ervy_i])^2] = \frac{(\pi-1)^2}{4\pi^2},\ \E[(\ervy_i - \E[\ervy_i])^4] = \frac{(6\pi^2-10\pi-3)(\pi-1)}{8\pi^3} < 1.
\end{equation}
For each $i\in[n]$, let $\vw_i^{(2)}\in\R^c$ denote the $i$-th column vector of $\mW^{(2)}$. Let 
$\rvv_i = \vw_i^{(2)}(\ervy_i - \E[\ervy_i])$,
then we have \begin{equation}
\label{eqn:z-expression}
    \rvz = \sum_{i=1}^n\vw_i^{(2)}\ervy_i = \sum_{i=1}^n\rvv_i + \sum_{i=1}^n\E[\ervy_i]\vw_i^{(2)} = \sum_{i=1}^n\rvv_i + \frac{\sqrt{\pi-1}}{2\pi}\sum_{i=1}^n\vw_i^{(2)}.
\end{equation}
It follows that 
\begin{equation}
   Var[\rvv_i] = Var[\vw_i^{(2)}\ervy_i] = \frac{(\pi-1)^2}{4\pi^2}\vw_i^{(2)}\vw_i^{(2)\T}.
\end{equation}
Let $\mV = \sum_{i=1}^n Var[\rvv_i]$,
\begin{equation}
    \mV = \frac{(\pi-1)^2}{4\pi^2}\sum_{i=1}^n\vw_i^{(2)}\vw_i^{(2)\T} = \frac{(\pi-1)^2}{4\pi^2}\mW^{(2)}\mW^{(2)\T}.
\end{equation}
As $n\to\infty$, from \corollaryref{cor:WW-Identity} we have $\mW^{(2)}\mW^{(2)\T}\to \mI_c$ in probability, therefore
\begin{equation}
\lim_{n\to\infty}\mV = \frac{(\pi-1)^2}{4\pi^2}\mI_c.
\end{equation}

We now verify the Lindeberg condition of independent random vectors $\{\rvv_1,\ldots, \rvv_n\}$.
First observe that the fourth moments of the $\rvv_i$'s are sufficiently small.
\begin{equation}
\begin{split}
    \lim_{n\to\infty}\sum_{i=1}^n\E\left[\left\Vert\rvv_i\right\Vert^4\right] &= \lim_{n\to\infty}\sum_{i=1}^n\E\left[\left(\sum_{j=1}^c \left(\mW^{(2)}_{ji}\left(\ervy_i - \E[\ervy_i]\right)\right)^2\right)^2\right]\\
    & \leq\lim_{n\to\infty} \sum_{i=1}^n\E\left[c^2\left(\left(\max_{j\in[c]}\mW^{(2)}_{ji}\right)^2\left(\ervy_i - \E[\ervy_i]\right)^2\right)^2\right]\\
    & \leq\lim_{n\to\infty} c^2\left(\max_{i\in[n],j\in[c]}\mW^{(2)}_{ji}\right)^4\sum_{i=1}^n\E\left[\left(\ervy_i - \E[\ervy_i]\right)^4\right].
\end{split}
\end{equation}
Since $\E[(\ervy_i - \E[\ervy_i])^4] < 1$ and $\max_{i\in[n],j\in[c]}|\mW^{(2)}_{ji}| < 2n^{-\frac13}$ with probability 1 from \corollaryref{cor:w-not-too-large}, it follows that
\begin{equation}
\begin{split}
    \lim_{n\to\infty}\sum_{i=1}^n\E\left[\left\Vert\rvv_i\right\Vert^4\right] \leq c^2\lim_{n\to\infty}\left(2n^{-\frac13}\right)^4\sum_{i=1}^n 1 = c^2 \lim_{n\to\infty}16n^{-\frac43}n =  16c^2 \lim_{n\to\infty}n^{-\frac13} =0.
\end{split}
\end{equation}
For any $\epsilon > 0$, since $\left\Vert\rvv_{i}\right\Vert > \epsilon$ in the domain of integration (when $\ind[\norm{\rvv_i} > \eps]$),
\begin{equation}
\begin{split}
\lim_{n\to\infty}\sum_{i=1}^n\E\left[\left\Vert\rvv_i\right\Vert^2\ind\left[\left\Vert\rvv_{i}\right\Vert > \epsilon\right]\right] & < \lim_{n\to\infty}\sum_{i=1}^n\E\left[\frac{\left\Vert\rvv_i\right\Vert^2}{\epsilon^2}\left\Vert\rvv_i\right\Vert^2\ind\left[\left\Vert\rvv_{i}\right\Vert > \epsilon\right]\right]\\
&\leq \frac{1}{\epsilon^2}\lim_{n\to\infty}\sum_{i=1}^n\E\left[\left\Vert\rvv_i\right\Vert^4\right] = 0.
\end{split}
\end{equation}
As the Lindeberg Condition is satisfied, with $\lim_{n\to\infty}\mV = \frac{(\pi-1)^2}{4\pi^2}\mI_c$ we have
\begin{equation}
\label{eqn:lemma:sumv-convergence}
    \lim_{n\to\infty}\sum_{i=1}^n\rvv_i\xrightarrow{d}\gN\left(0,\frac{(\pi-1)^2}{4\pi^2}\mI_c\right).
\end{equation}

By \lemmaref{lemma:W-expectation}, we have $\lim_{n\to\infty}\vw_i^{(2)} = \overrightarrow{0}$ with probability 1 over $\mW^{(2)}$, therefore plugging \equationref{eqn:lemma:sumv-convergence} into \equationref{eqn:z-expression} we have\begin{equation}
    \lim_{n\to\infty} \rvz\xrightarrow{d}\gN\left(0,\frac{(\pi-1)^2}{4\pi^2}\mI_c\right).
\end{equation}
Which completes the proof.
\end{proofof}

Now we will show a key lemma for proving the main theorem, which suggests that when reasonably conditioning on two entries of the input $\rvx$ and two entries of the activation $\rmD$, the distribution of $\rvz$ converges in distribution to $\rvz$ without conditioning as $n\to\infty$.

\begin{lemma}
\label{lemma:z-invariant}
With probability 1 over $\mW^{(1)}$ and $\mW^{(2)}$, fix any $\beta < \frac\alpha2$ (recall that $d=n^{1+\alpha}$), fix any $a,b\in (-n^{\beta}, n^{\beta})$, for any $p,q\in[n]$ and $k,l\in[d]$, we have the following convergence in distribution\begin{equation}
    \rvz\vert(\rmD_{pp}=1, \rmD_{qq}=1, \rvx_k=a, \rvx_l=b)\xrightarrow{d} \rvz.
\end{equation}
\end{lemma}

\begin{proofof}{\lemmaref{lemma:z-invariant}} For simplicity of notation, we will use subscript $|_\rvx$ and $|_\rmD$ to denote the conditions we impose. For example, we will denote $\rvz\vert(\rmD_{pp}=1, \rmD_{qq}=1, \rvx_k=a, \rvx_l=b)$ by $\rvz|_{\rmD,\rvx}$, and denote $\rvx|(\rvx_k=a, \rvx_l=b)$ by $\rvx|_\rvx$ etc.

First claim that with probability 1 over $\mW^{(1)}$, $\rvu$ is invariant upon the conditioning on $\rvx$. Let $\ve^{(i)}\in\R^d$ be the standard basis vector such that $\ve^{(i)}_j=\ind[i=j]$. Then \begin{equation}
\begin{split}
    \norm{\rvu-\rvu|_\rvx} &= \norm{\mW^{(1)}\rvx - \mW^{(1)}\rvx|_\rvx}\\
    &= \norm{\mW^{(1)}(\rvx -\rvx|_\rvx)}\\
    &= \norm{\mW^{(1)}((\rvx_k-a)\ve^{(k)} + (\rvx_l-b)\ve^{(l)})}\\
    &= \norm{\vw^{(1)}_k}|\rvx_k-a| + \norm{\vw^{(1)}_l}|\rvx_l-b|\\
    &\leq 5n^{-\frac\alpha2}(|\rvx_k|+|\rvx_l|+|a|+|b|)\\
    &\leq 5n^{-\frac\alpha2}(\rvx_k+\rvx_l) + 10n^{-\frac\alpha2}n^{\frac\beta2}.
\end{split}
\end{equation}
The norms of $\vw_k^{(1)}$ and $\vw_l^{(1)}$ are bounded from \lemmaref{lemma:w1-col-norm}.
Note that as $n\to\infty$ we have $n^{-\frac\alpha2}$ and $n^{-\frac{\alpha-\beta}{2}}$ converging to 0 as we set $\beta < \alpha$.
Since $\rvx$ is of bounded expectation and variance, $5n^{-\frac\alpha2}(\rvx_k+\rvx_l)$ converges in distribution to $0$. Therefore $\norm{\rvu-\rvu|_\rvx}\converged 0$ and hence $\rvy|_\rvx\converged\rvy$. Since $\rvz$ is determined by $\rvy$, to prove $\rvz|_{\rmD,\rvx}\converged \rvz$, we now only ne
ed to show $\rvz|_{\rmD}\converged \rvz$.

Note that conditioning on $\rmD_{pp}=\rmD_{qq}=1$ is equivalent to conditioning on $\rvu_p>0$ and $\rvu_q>0$. Which is again equivalent to conditioning on $\rvy_p$ and $\rvy_q$ to be a half Gaussian distribution truncated at 0 instead of the rectified Gaussian. Recall that $\rvz=\mW^{(2)}\rvy = \sum_{i=1}^n\vw^{(2)}_i\rvy_i.$ Since only $\rvy_p$ and $\rvy_q$ are affected by conditioning on $\rmD$, we have\begin{equation}
\begin{split}
\norm{\rvz - \rvz|_\rmD} &=\lnorm{ \sum_{i=1}^n\vw^{(2)}_i\rvy_i - \sum_{i=1}^n\vw^{(2)}_i(\rvy|_\rmD)_i}\\
&= \lnorm{\vw^{(2)}_p(\rvy_p-(\rvy|_\rmD)_p) + \vw^{(2)}_q(\rvy_q-(\rvy|_\rmD)_q)}\\
&\leq \norm{\vw^{(2)}_p}|\rvy_p-(\rvy|_\rmD)_p| + \norm{\vw^{(2)}_q}|\rvy_q-(\rvy|_\rmD)_q|.
\end{split}
\end{equation}
Note that $\rvy_p-(\rvy|_\rmD)_p$ and $\rvy_q-(\rvy|_\rmD)_q$ are difference between a rectified Gaussian with finite variance and its corresponding truncated Gaussian, both are of bounded expectation and variance. Meanwhile, by \lemmaref{cor:w-not-too-large}, for all $i\in[n]$ we have that with probability 1 over $\mW^{(2)}$,
\begin{equation}
    \left\Vert\vw^{(2)}_i\right\Vert \leq \sqrt{c\left(\max_{i\in[c],j\in[n]}\mW^{(2)}_{ij}\right)^2} < \sqrt{4cn^{-\frac{2}{3}}}.
\end{equation}
Since $\lim_{n\to\infty}\sqrt{4cn^{-\frac{2}{3}}}=0$, as $n$ goes to infinity we have \begin{equation}
\norm{\vw^{(2)}_p}|\rvy_p-(\rvy|_\rmD)_p| + \norm{\vw^{(2)}_q}|\rvy_q-(\rvy|_\rmD)_q|\converged \overrightarrow{0}.
\end{equation}
Therefore $\rvz|_\rmD\converged\rvz$, and hence
\begin{equation}
    \rvz\vert(\rmD_{pp}=1, \rmD_{qq}=1, \rvx_k=a, \rvx_l=b)\xrightarrow{d} \rvz.
\end{equation}
\end{proofof}

Given that $\rvp=\softmax(\rvz)$ and $\rmA=\diag(\rvp)-\rvp\rvp^\T$, the mapping from $\rvz$ to $\rmA$ is bounded and continuous. Thus by the Portmanteau Theorem, we have the following corollary,
\begin{corollary}
\label{cor:A-invariant}
For any $\epsilon>0$, with probability 1 over $\mW^{(1)}$ and $\mW^{(2)}$, fix any $\beta < \frac\alpha2$ (recall that $d=n^{1+\alpha}$), fix any $a,b\in (-n^{\beta}, n^{\beta})$, for any $p,q\in[n]$, $k,l\in[d]$, and $i,j\in[c]$,  we have \begin{equation}
    |\ex{\rmA_{ij}\vert(\rmD_{pp}=1, \rmD_{qq}=1, \rvx_k=a, \rvx_l=b)} - \ex{\rmA_{ij}}| < \eps.
\end{equation}
By the proof of \lemmaref{lemma:z-invariant}, this property holds when dropping the conditioning on $\rmD$ or $\rvx$. 
\end{corollary}

%% file: MainProof/A-structure.tex
\subsubsection{Structure of $\rmA$}
\label{sec:proof-A-structure}
In this section we will analyze properties of the second output Hessian $\rmA$, which, despite being a $\R^{c\times c}$ ``small'' matrix, provides many important properties to the first output Hessian and the full layer-wise Hessians.

\begin{lemma}
\label{lemma:A-rank-c-1}
With probability 1 over $\mW^{(1)}$ and $\rm^{(2)}$, $\tilde{\rmA}\triangleq\lim_{n\to\infty}\E[\rmA]$ exist and is rank-$(c-1)$ .
\end{lemma}

\begin{proofof}{\lemmaref{lemma:A-rank-c-1}}
Note that each entry of $\rmA$ is a quadratic function of $\rvp$, and $\rvp$ is a continuous function of $\rvz$. Therefore, we consider $\rmA$ as a function of $\rvz$ and write $\rmA(\rvz)$ when necessary. From \lemmaref{lemma:z-gaussian} we know that $\lim_{n\to\infty}\rvz$ follows a standard normal distribution $\mathcal{N}(0,\gamma \mI_c)$ with probability 1 over $\mW^{(1)}$ and $\mW^{(2)}$, where $\gamma$ is some absolute constant. Therefore, $\tilde{\rmA}\triangleq\lim_{n\to\infty}\E[\rmA]$ exist and it equals $\E[\rmA(\lim_{n\to\infty} \rvz)]=\E_{\rvz\sim\mathcal{N}(0,\gamma \mI_c)}[\rmA(\rvz)]$. For simplicity of notations, we will omit the statement ``with probability 1 over $\mW^{(1)}$ and $\mW^{(2)}$'' when there is no confusion.

From the definition of $\rmA$ we know that $\rmA\triangleq \text{diag}(\rvp)-\rvp\rvp^\T$ where $\rvp$ is the vector obtained by applying softmax to $\rvz$, so $\sum_{i=1}^c \rvp_i=1$ and for all $i\in[c], \rvp_i\in(0,1)$. Therefore, for any vector $\rvp$ satisfying the previous conditions, we have
\begin{equation}
\textbf{1}^\T\rmA\textbf{1} = \sum_{i=1}^c\rbr{\rvp_i-\sum_{j=1}^c\rvp_i\rvp_j} = \sum_{i=1}^c(\rvp_i-\rvp_i) =0,
\end{equation}
where $\textbf{1}$ is the all-one vector. Therefore, we know that $\rmA$ has an eigenvalue 0 with eigenvector $c^{-\frac12}\textbf{1}$. This means that $\E[\rmA]$ also has an eigenvalue 0 with eigenvector $c^{-\frac12}\textbf{1}$. Thus, $\E[\rmA]$ is at most of rank $(c-1)$.

Then we analyze the other $(c-1)$ eigenvalues of $\tilde{\rmA}$. Since $\rmA=\rmQ\rmQ^\T$ where $\rmQ=\text{diag}(\sqrt{\rvp})(\mI_c-\textbf{1}\rvp^\T)$, we know that $\rmA$ is always a positive semi-definite (PSD) matrix, which indicates that $E[\rmA]$ must also be PSD. Assume the $c$ eigenvalues of $\tilde{\rmA}$ are $\lambda_1\geq\lambda_2\geq\cdots\geq\lambda_{c-1}\geq\lambda_c=0$. Therefore, by definition, we have
\begin{equation}
\lambda_{c-1} = \min_{\vv\in S, \norm{\vv}=1}\vv^\T\tilde{\rmA}\vv = \exop{\rvz\sim\mathcal{N}(0,\gamma \mI_c)}{\min_{\vv\in S, \norm{\vv}=1}\vv^\T\rmA\vv},
\end{equation}
where $S\triangleq\R^c\backslash\gR\{\textbf{1}^\T\}$ is the orthogonal subspace of the span of $\textbf{1}$. $\vv\in S$ implies that $\vv\perp\textbf{1}$, i.e., $\sum_{i=1}^c \vv_i = 0$.

Direct computation gives us
\begin{equation}
\vv^\T\rmA\vv = \sum_{i=1}^c\vv_i^2\rvp_i-\rbr{\sum_{i=1}^c\vv_i\rvp_i}^2.
\end{equation}
Define two vectors $\va,\vb\in\R^c$ as for all $i\in[c]$, with $\va_i\triangleq \vv_i\sqrt{\rvp_i}, \vb_i\triangleq\sqrt{\rvp_i}$, then $\ns{\vb}=\sum_{i=1}^c\rvp_i=1$ and
\begin{equation}
\vv^\T\rmA\vv = \ns{\va}-\langle\va,\vb\rangle^2 = \ns{\va}\cdot\ns{\vb}-\langle \va,\vb\rangle^2.
\end{equation}
Therefore,
\begin{equation}
\vv^\T\rmA\vv \geq \ns{\va}\ns{\vb}\sin^2\theta(\va,\vb),
\end{equation}
where $\theta(\va,\vb)$ is the angle between $\va$ and $\vb$, i.e., $\theta(\va,\vb)\triangleq\arccos\frac{\langle \va,\vb\rangle}{\norm{\va}\norm{\vb}}$.
Define $\rvp_0\triangleq\min_{i\in[c]}\rvp_i$, then
\begin{equation}
\ns{\va} = \sum_{i=1}^c\vv_i^2\rvp_i \geq \sum_{i=1}^c\vv_i^2\rvp_0 = \rvp_0\ns{\vv} = \rvp_0.
\end{equation}
Since $\norm{\vb}=1$, we have
\begin{equation}
\sin^2\theta(\va,\vb) = \frac{\ns{\va-\langle \va,\vb\rangle\cdot \vb}}{\ns{\va}}.
\end{equation}
Besides,
\begin{equation}\begin{split}
\ns{\va-\langle \va,\vb\rangle\cdot \vb} &= \sum_{i=1}^c\rbr{\vv_i\sqrt{\rvp_i}-\rbr{\sum_{j=1}^c\vv_j\rvp_j}\sqrt{\rvp_i}}^2\\
                                 &= \sum_{i=1}^c\rvp_i\rbr{\vv_i-\sum_{j=1}^c\vv_j\rvp_j}^2\\
                                 &\geq \rvp_0\sum_{i=1}^c\rbr{\vv_i-\sum_{j=1}^c\vv_j\rvp_j}^2.
\end{split}\end{equation}
Define $s\triangleq\arg\max_{i\in[c]}{\vv_i}$ and $t\triangleq\arg\min_{i\in[c]}{\vv_i}$, then
\begin{equation}
\sum_{i=1}^c\rbr{\vv_i-\sum_{j=1}^c\vv_j\rvp_j}^2\geq \rbr{\vv_s-\sum_{j=1}^c\vv_j\rvp_j}^2 + \rbr{\vv_t-\sum_{j=1}^c\vv_j\rvp_j}^2 \geq \frac{(\vv_s-\vv_t)^2}{2}.
\end{equation}
From $\norm{\vv}=1$ we know that $\max_{i\in[c]}|\vv_i|\geq c^{-\frac12}$. Besides, since $\sum_{i=1}^c\vv_i=0$, we have $\vv_s>0>\vv_t$. Therefore, $\vv_s-\vv_t>\max_{i\in[c]}|\vv_i|\geq c^{-\frac12}$. As a result,
\begin{equation}
\ns{\va-\langle \va,\vb\rangle\cdot \vb} \geq \rvp_0\cdot \frac{(\vv_s-\vv_t)^2}{2} > \frac{\rvp_0}{2c}.
\end{equation}
Moreover,
\begin{equation}
\ns{\va} = \sum_{i=1}^c \vv_i^2\rvp_i \leq \sum_{i=1}^c \rvp_i = 1.
\end{equation}
Thus,
\begin{equation}
\sin^2\theta(\va,\vb) \geq \frac{\frac{\rvp_0}{2c}}{1} = \frac{\rvp_0}{2c},
\end{equation}
which means that
\begin{equation}
\vv^\T\rmA\vv \geq \rvp_0\cdot 1\cdot\frac{\rvp_0}{2c} = \frac{\rvp_0^2}{2c}.
\end{equation}
Now we analyze the distribution of $\rvp_0$. Since $\rvz$ follows a spherical Gaussian distribution $\mathcal{N}(0,\gamma \mI_c)$, we know that the entries of $\rvz$ are totally independent. Besides, for each entry $\rvz_i(i\in[c])$, we have $|\rvz_i|<\gamma$ with probability $\xi$, where $\xi\approx 0.68$ is an absolute constant. Therefore, with probability $\xi^c$, forall entries $\rvz_i(i\in[c])$, we have $|\rvz_i|<\gamma$. In this case,
\begin{equation}
\rvp_0 = \frac{\exp(\min_{i\in[c]}\rvz_i)}{\sum_{i=1}^c\exp(\rvz_i)}\geq\frac{\exp(-\gamma)}{c\exp(\gamma)}.
\end{equation}
In other cases, we know that $\rvp_0>0$. Thus,
\begin{equation}
\label{lower-bound-c-1-eigenvalue-of-A}
\lambda_{c-1} = \E_{\rvz\sim\mathcal{N}(0,\gamma \mI_c)}\left[\min_{\vv\in S, \norm{v}=1}\vv^\T\rmA\vv\right]\geq \xi^c\cdot\frac{\rbr{\frac{\exp(-\gamma)}{c\exp(\gamma)}}^2}{2c}.
\end{equation}
The right hand side is independent of $n$. Therefore, $\lambda_{c-1}>0$, which means that $\tilde{\rmA}$ has exactly $(c-1)$ positive eigenvalues and a $0$ eigenvalue, and the eigenvalue gap between the smallest positive eigenvalue and 0 is independent of $n$.

Hence we complete the proof.

\end{proofof}

%% file: MainProof/projection-approx.tex
\subsubsection{Projecting Hessians onto Finite Dimensions}
\label{sec:proof-project}
In this section we will develop some technical tools for analyzing the eigenvalues and eigenvectors of the output Hessians and the full layer-wise Hessians. In particular, we will project both infinite dimensional matrices to $c\times c$ matrices.

First, we prove a technical lemma that will be very useful when we bound the Frobenius norm of the difference between infinite size matrices.
\begin{lemma}
\label{lemma:polynomial}
Let $p(\rmA,\rmD,\rvx)$ be a homogeneous polynomial of $\rmA$, $\rmD$, and $\rvx$ and is degree 1 in $\rmA$, degree 2 in $\rmD$, and degree 2 in $\rvx$. Suppose the coefficients in $p$ are upper bounded in $\ell_1$-norm by an absolute constant $\mu$. Also let $\rmD'$ be an independent copy of $\rmD$ and $\rvx''$ be an independent copy of $\rvx$ independent to $\rmD$ and $\rmA$. Then with probability 1 over $\mW^{(1)}$ and $\mW^{(2)}$, we have
\begin{equation}
    \lim_{n\to\infty}\ex{p(\rmA, \rmD, \rvx)} = \ex{p(\rmA, \rmD', \rvx'')}
\end{equation}
\end{lemma}
\begin{proofof}{\lemmaref{lemma:polynomial}}
Fix any $\eps>0$.
Assume that the homogeneous polynomial is of the form
\begin{equation}
p(\rmA,\rmD,\rvx)=\sum_{i=1}^m c_i\rmA_{s(i),t(i)}\rmD_{u(i),u(i)}\rmD_{v(i),v(i)}\rvx_{p(i)}\rvx_{q(i)},
\end{equation}
for coefficients $c_i$, then from linearity of expectation we know
    
\begin{equation}
    \E[p(\rmA,\rmD,\rvx)] = \sum_{i=1}^m c_i\E[\rmA_{s(i),t(i)}\rmD_{u(i),u(i)}\rmD_{v(i),v(i)}\rvx_{p(i)}\rvx_{q(i)}].
\end{equation}
Hence \begin{equation}
\begin{split}
&|\ex{p(\rmA, \rmD, \rvx)} - \ex{p(\rmA, \rmD', \rvx'')}|\\
\leq&\sum_{i=1}^m c_i |\E[\rmA_{s(i),t(i)}\rmD_{u(i),u(i)}\rmD_{v(i),v(i)}\rvx_{p(i)}\rvx_{q(i)}] - \E[\rmA_{s(i),t(i)}\rmD_{u(i),u(i)}\rmD_{v(i),v(i)}\rvx_{p(i)}\rvx_{q(i)}]|
\end{split}
\end{equation}

Since the entries of $\rmD$ can only be $0$ or $1$, we have
\begin{equation}
\begin{split}
&\E[\rmA_{s(i),t(i)}\rmD_{u(i),u(i)}\rmD_{v(i),v(i)}\rvx_{p(i)}\rvx_{q(i)}]\\
=& \pr{\rmD_{u(i),u(i)}=\rmD_{v(i),v(i)}=1} \E[\rmA_{s(i),t(i)}\rvx_{p(i)}\rvx_{q(i)}|\rmD_{u(i),u(i)}=\rmD_{v(i),v(i)}=1]\\
=&\ \frac14\E[\rmA_{s(i),t(i)}\rvx_{p(i)}\rvx_{q(i)}|\rmD_{u(i),u(i)}=\rmD_{v(i),v(i)}=1].
\end{split}
\end{equation}
The last equality holds since $\rvu$ converges in distribution to a spherical Gaussian, and its entry-wise activations $\rmD$ follows a $p=\frac12$ Bernoulli distribution.
Assume $\sum_{i=1}^m|c_i|\geq \mu$, that the $\ell_1$ norm of the coefficients is upper bounded by some constant $\mu$. Set $\eps'=\frac{\eps}{\mu}$. To prove this lemma it is sufficient to prove that each term of the polynomial are sufficiently small, namely, for any index,
\begin{equation}
\begin{split}
&|\E[\rmA_{s(i),t(i)}\rvx_{p(i)}\rvx_{q(i)}|\rmD_{u(i),u(i)}=\rmD_{v(i),v(i)}=1]\\&\quad-\E[\rmA_{s(i),t(i)}\rvx''_{p(i)}\rvx''_{q(i)}|\rmD'_{u(i),u(i)}=\rmD'_{v(i),v(i)}=1]| 
\\
&|\E[\rmA_{s(i),t(i)}\rvx_{p(i)}\rvx_{q(i)}|\rmD_{u(i),u(i)}=\rmD_{v(i),v(i)}=1]-\E[\rmA_{s(i),t(i)}\rvx''_{p(i)}\rvx''_{q(i)}]|<\eps'.
\end{split}
\end{equation}

Fix a set of index $s,t,p,q,u,v$, for simplicity of notation, we use the abbreviation $\ex{\rmA_{st}\rvx_p\rvx_q|_\rmD}$ to denote $\E[\rmA_{s(i),t(i)}\rvx_{p(i)}\rvx_{q(i)}|\rmD_{u(i),u(i)}=\rmD_{v(i),v(i)}=1]$. Since $\rvx$ is of rectified Gaussian with the covariance of the initial Gaussian distribution being the identity, $\rvx_p$ and $\rvx_q$ shares the same density function when $x>0$, namely $f(x)=\frac{1}{\sqrt{2\pi}}\exp(-x^2/2)$.
Note that \begin{equation}
\label{eqn:proof-xxexp}
    \iint\limits_{\R^+\times\R^+}xy\ f(x)f(y)\ dx\ dy = \ex{\rvx_i\rvx_j}=\ex{\rvx_i}\ex{\rvx_j} = \frac{1}{2\pi}.
\end{equation}

Fix some $\beta < \frac\alpha2$, we have
\begin{equation}
\label{eqn:poly-bound}
\begin{split}
    &|\ex{\rmA_{st}\rvx_p\rvx_q|_\rmD}-\ex{\rmA_{st}\rvx''_p\rvx''_q}|\\
    =&\labs{\quad\iint\limits_{\R^+\times\R^+} \ex{\rmA_{st}|_{\rmD, \rvx_p=x,\rvx_q=y}}xy\ f(x)f(y)\ dx\ dy - \iint\limits_{\R^+\times\R^+} \ex{\rmA_{st}}xy\ f(x)f(y)\ dx\ dy}\\
    \leq& \iint\limits_{\R^+\times\R^+} |\ex{\rmA_{st}|_{\rmD, \rvx_p=x,\rvx_q=y}}-\ex{\rmA_{st}}|xy\ f(x)f(y)\ dx\ dy\\
    =&\iint\limits_{[0,n^\beta]\times [0,n^\beta]} |\ex{\rmA_{st}|_{\rmD, \rvx_p=x,\rvx_q=y}}-\ex{\rmA_{st}}|xy\ f(x)f(y)\ dx\ dy\\
    &+ \iint\limits_{\R^+\times\R^+\backslash\rbr{[0,n^\beta]\times [0,n^\beta]}} |\ex{\rmA_{st}|_{\rmD, \rvx_p=x,\rvx_q=y}}-\ex{\rmA_{st}}|xy\ f(x)f(y)\ dx\ dy.
\end{split}
\end{equation}
From \corollaryref{cor:A-invariant} we have, for any indices $s,t$, for sufficiently large $n$, for any $(x,y)\in[0,n^\beta]\times [0,n^\beta]$,
\begin{equation}
    |\ex{\rmA_{st}|_{\rmD, \rvx_p=x,\rvx_q=y}}-\ex{\rmA_{st}}| < \eps'
\end{equation}
Thus
\begin{equation}
\begin{split}
    &\iint\limits_{[0,n^\beta]\times [0,n^\beta]} |\ex{\rmA_{st}|_{\rmD, \rvx_p=x,\rvx_q=y}}-\ex{\rmA_{st}}|xy\ f(x)f(y)\ dx\ dy\\
    \leq &\ \eps'\iint\limits_{[0,n^\beta]\times [0,n^\beta]} xy\ f(x)f(y)\ dx\ dy = \frac{\eps'}{2\pi}.
\end{split}
\end{equation}
Now we consider the other integral. First note that since $\rmA_{st}$ is either $\rvp_i-\rvp_i^2$ or $-\rvp_i\rvp_j$ for some $i,j$, and $\rvp_i,\rvp_j, \rvp_i+\rvp_j\in(0,1)$ as it is the output of the softmax function, we have $\rmA_{st}\in(-\frac14, \frac14)$. It follows that $|\ex{\rmA_{st}|_{\rmD, \rvx_p=x,\rvx_q=y}}-\ex{\rmA_{st}}|\leq\frac12$. Therefore
\begin{equation}
\begin{split}
&\iint\limits_{\R^+\times\R^+\backslash\rbr{[0,n^\beta]\times [0,n^\beta]}} |\ex{\rmA_{st}|_{\rmD, \rvx_p=x,\rvx_q=y}}-\ex{\rmA_{st}}|xy\ f(x)f(y)\ dx\ dy\\
\leq &\ \frac12\iint\limits_{\R^+\times\R^+\backslash\rbr{[0,n^\beta]\times [0,n^\beta]}}xy\ \frac{e^{-x^2/2}}{\sqrt{2\pi}}\frac{e^{-y^2/2}}{\sqrt{2\pi}}\ dx\ dy\\
\leq&\ \frac12\cdot\frac{1}{2\pi}\int_{n^\beta}^\infty e^{-x^2/2}x\ dx\int_{\R^+}e^{-y^2/2}y\ dy + \frac12\cdot\frac{1}{2\pi}\int_{\R^+} e^{-x^2/2}x\ dx\int_{n^\beta}^{\infty}e^{-y^2/2}y\ dy\\
=&\ \frac{1}{2\pi}e^{-n^{2\beta}},
\end{split}
\end{equation}
which decreases below $\eps'/2$ for sufficiently large $n$. As both terms in \equationref{eqn:poly-bound} are less than $\eps'/2$ as $n\to\infty$, we have $|\ex{\rmA_{st}\rvx_p\rvx_q|_\rmD}-\ex{\rmA_{st}\rvx''_p\rvx''_q}|<\eps'$. Which completes the proof of this lemma.
\end{proofof}

We then generalize this lemma for a degree ten homogeneous polynomial, in which the monomials are roughly multiplied with an independent copy of itself (except for $\rmA$).
\begin{corollary}
\label{cor:polynomial}
Let $p(\rmA,\rmD,\rvx,\bar\rmA,\bar\rmD,\bar\rvx)$ be a homogeneous polynomial of $\rmA,\rmD,\rvx,\bar\rmA,\bar\rmD$, and $\bar\rvx$. Let it be degree 1 in $\rmA$, $\bar\rmA$, degree 2 in $\rmD$, $\bar\rmD$, and degree 2 in $\rvx$,$\bar\rvx$. Suppose the coefficients in $p$ are upper bounded in $\ell_1$-norm by an absolute constant $\mu$. Also let $\rmD'$ be an independent copy of $\rmD$ and $\rvx''$ be an independent copy of $\rvx$ independent to $\rmD$ and $\rmA$. Morever let $(\bar\rmA,\bar\rmD,\bar\rvx,\bar\rmD',\bar\rvx'')$ be an independent copy of $(\rmA,\rmD,\rvx,\rmD',\rvx'')$. Then with probability 1 over $\mW^{(1)}$ and $\mW^{(2)}$, we have
\begin{equation}
    \lim_{n\to\infty}\ex{p(\rmA, \rmD, \rvx, \bar\rmA, \bar\rmD, \bar\rvx)} = \ex{p(\rmA, \rmD', \rvx'', \bar\rmA, \bar\rmD', \bar\rvx'')}.
\end{equation}
\end{corollary}

\begin{proofof}{\corollaryref{cor:polynomial}}
For simplicity of notations, denote $\ervs_{ijuvrs} = \rmA_{ij}\rmD_{vv}\rmD_{ww}\rvx_r\rvx_s$, $\ervs'_{ijuvrs} = \rmA_{ij}\rmD'_{vv}\rmD'_{ww}\rvx''_r\rvx''_s$. Similarly, denote $\ervt_{klpqtu}=\bar\rmA_{kl}\bar\rmD_{pp}\bar\rmD_{qq}\bar\rvx_t\bar\rvx_u$ and $\ervt'_{klpqtu}=\bar\rmA_{kl}\bar\rmD'_{pp}\bar\rmD'_{qq}\bar\rvx''_t\bar\rvx''_u$. As there is no confusion on indexing, we will also omit the subscripts and use $\ervs, \ervt$.

Fix any $\eps>0$, Following the argument of the proof of \lemmaref{lemma:polynomial}, it is sufficient to prove this corollary by showing for any indexing, \begin{equation}
\begin{split}
&|\E[\rmA_{ij}\bar\rmA_{kl}\rmD_{vv}\rmD_{ww}\bar\rmD_{pp}\bar\rmD_{qq}\rvx_r\rvx_s\bar\rvx_t\bar\rvx_u] - \E[\rmA_{ij}\bar\rmA_{kl}\rmD_{vv}'\rmD_{ww}'\bar\rmD_{pp}'\bar\rmD_{qq}'\rvx_r''\rvx_s''\bar\rvx_t''\bar\rvx_u'']|\\
=&\ |\E[\ervs\ervt] - \E[\ervs'\ervt']|< \frac\eps\mu.
\end{split}
\end{equation}
First note that since $|\rmA_{ij}| < \frac14$ and $|\rmD_{ii}| \leq 1$ for all $i,j$, we have \begin{equation}
\begin{split}
|\E[\ervs]| &= |\E[\rmA_{ij}\rmD_{vv}\rmD_{ww}\rvx_r\rvx_s]|
\leq \frac14|\E[\rvx_r\rvx_s]| = \frac{1}{8\pi}.
\end{split}
\end{equation} The same argument also applies to $\ervs', \ervt$, and $\ervt'$.
Also, by \lemmaref{lemma:polynomial}, for sufficiently large $n$ we have $|\E[\ervs] - \E[\ervs']|<\eps'$ and $|\E[\ervt] - \E[\ervt']|<\eps'$. Since by construction $\ervs$ and $\ervt$ are independent, we have
\begin{equation}
\begin{split}
|\E[\ervs\ervt] - \E[\ervs'\ervt']|&=|\E[\ervs]\E[\ervt] - \E[\ervs']\E[\ervt']|\\ &=|\E[\ervs]\E[\ervt] - \E[\ervs]\E[\ervt'] + \E[\ervs]\E[\ervt'] - \E[\ervs']\E[\ervt']|\\
&\leq |\E[\ervs]||\E[\ervt] - \E[\ervt']| + |\E[\ervt']||\E[\ervs] - \E[\ervs']|\\
&\leq \frac{1}{8\pi}\eps' + \frac{1}{8\pi}\eps' < \eps',
\end{split}
\end{equation}
which completes the proof of \corollaryref{cor:polynomial}.
\end{proofof}

Now we formally begin our analysis. We will start from $\mM^{(1)}=\ex{\rmD\mW^{(2)\T}\rmA\mW^{(2)}\rmD}$, the output Hessian of the first layer. The output Hessian of the second layer is just $\ex{\rmA}$, which had been analyzed in \sectionref{sec:proof-A-structure}. In this section we will neglect the superscript for $\mM^{(1)}$ and use $\mM$ as there is no confusion. Also, we use $\mW$ to denote $\mW^{(2)}$ unless specified otherwise. We first state our main lemma of projecting $\mM$. 

\begin{lemma}
\label{lemma:M-proj-preserve-f-norm}
With probability 1 over $\mW^{(1)}$ and $\mW^{(2)}$, \begin{equation}
    \lim_{n\to\infty}\frac{\fns{\mW\mM\mW^\T}}{\fns{\mM}}=1.
\end{equation}
\end{lemma}
\begin{proofof}{\lemmaref{lemma:M-proj-preserve-f-norm}}
To prove the equivalence between $\fns{\mW\mM\mW^\T}$ and $\fns{\mM}$, we need to introduce a bridging term 
\begin{equation}
    \mM^*\triangleq\E[\rmD'\mW^{(2)\T}\rmA\mW^{(2)}\rmD']
\end{equation}
where $\rmD'$ is an independent copy of $\rmD$ and also independent of $\rmA$. Essentially $\mM^*$ is the matrix which has the same expression as $\mM$ except that we assume $\rmD$ is independent of $\rmA$ in $\mM^*$. Informally, the proof strategy of \lemmaref{lemma:M-proj-preserve-f-norm} is\begin{equation}
    \fns{\mW\mM\mW^\T}\approx \fns{\mW\mM^*\mW^\T}\approx \fns{\mM^*}\approx \fns{\mM}.
\end{equation}
We now formally establish this equivalence.

Then we look into the structures of the bridging matrix $\mM^*$. It is simple to analyze as we assumed the independence between $\rmA$ and $\rmD'$. Formally,
\begin{lemma}
\label{lemma:M-star}
With probability 1 over $\mW^{(1)}$ and $\mW^{(2)}$,
\begin{equation}
\label{eqn:proof-M-star}
\mM^*=\frac14\rbr{\mW^\T\E[\rmA]\mW + \emph{diag}(\mW^\T\E[\rmA]\mW)}.
\end{equation}
Moreover, $\norm{\mM^*}$ and $\fns{\mM^*}$ are bounded below by some nonzero constant and bounded above by some constant.
\end{lemma}
\begin{proofof}{\lemmaref{lemma:M-star}}
First note that since $\rmD'$ is the activation of $\rvu'$, which converges to a spherical Gaussian with probability 1 over $\mW^{(1)}$ and is independent with $\rmA$, each diagonal entry of $\rmD$ is a Bernoulli random variable with $p=\frac12$.
For $i,j\in [n]$, when $i\neq j$, we have\begin{equation}
\begin{split}
\mM^*_{ij} &= \E[\rmD'_{ii}(\mW^\T\rmA\mW)_{ij}\rmD'_{jj}]\\
&= \E[\rmD'_{ii}]\E[\rmD'_{jj}]\E[(\mW^\T\rmA\mW)_{ij}]\\
&= \frac14(\mW^\T\E[\rmA]\mW)_{ij}.
\end{split}
\end{equation}
When $i=j$,
\begin{equation}
\begin{split}
\mM^*_{i,i} &= \E[\rmD'_{ii}(\mW^\T\rmA\mW)_{ii}\rmD'_{ii}]\\
&= \E[\rmD'_{ii}]\E[(\mW^\T\rmA\mW)_{ii}]\\
&= \frac12(\mW^\T\E[\rmA]\mW)_{i,j}.
\end{split}
\end{equation}
Thus\begin{equation}
    \mM^*=\frac14\rbr{\mW^\T\E[\rmA]\mW + \diag(\mW^\T\E[\rmA]\mW)}.
\end{equation}
Now we show the lower bound and upper bound on norms of $\mM^*$.

Since $\langle\E[\mW^\T\rmA\mW],\text{diag}(\E[\mW^\T\rmA\mW])\rangle\geq0$, we have
\begin{equation}
    \fn{\mM^*}\geq\fn{\E[\mW^\T\rmA\mW]} = \fn{\mW^\T\tilde{\rmA}\mW}.
\end{equation}
Since $\mW\mW^\T$ converges to $\mI_c$ in spectral norm from \lemmaref{lemma:WW-identity}, we have for sufficiently large $n$, the smallest singular value of $\mW$ is larger than $\frac12$. Moreover, since $\E[\rmA]$ admits an eigenvalue that is bounded below by some constants $\eta\triangleq \xi^c\cdot\rbr{\frac{\exp(-\gamma)}{c\exp(\gamma)}}^2/2c$ where $\xi\approx 0.68$ is an absolute constant and $\gamma=\frac{(\pi-1)^2}{4\pi^2}$ as shown in \lemmaref{lemma:A-rank-c-1}, there exists an eigenvalue of $\mM^* = \mW^\T\E[\rmA]\mW$ that is larger than $\frac{\eta}{4}$. Hence for large $n$, $\norm{\mM^*}$ is bounded from below by $\frac{\eta}{4}$, and hence $\fns{\mM^*}$.

Besides, since $\rmD$ is a diagonal matrix with 0/1 entries, and the absolute value of each entry of $\rmA$ is bounded by 1, we have
\begin{equation}
    \fn{\mM}=\fn{\E[\rmD\mW^\T\rmA\mW \rmD]}\leq\fn{\E[\mW^\T\rmA\mW]}\leq\fns{\mW}\fn{\rmA}\leq c\fns{\mW}.
\end{equation}
From \lemmaref{lemma:W-norm}, we know that with probability 1, $\fns{\mW}\leq 2c$, therefore, $\norm{\mM}_F$ is upper bounded by $2c^2$, which is independent of $n$.
\end{proofof}

\begin{lemma}
\label{lemma:M-equivalence}
With probability 1 over $\mW^{(1)}$ and $\mW^{(2)}$,
\[
\lim_{n\to\infty}\frac{\fns{\mM}}{\fns{\mM^*}}=1.
\]\end{lemma}
\begin{proofof}{\lemmaref{lemma:M-equivalence}}

Recall that $\mM^*\triangleq\E[\rmD'\mW^{(2)\T}\rmA\mW^{(2)}\rmD']$ where $\rmD'$ is an independent copy of $\rmD$ and also independent of $\rmA$.
Since we will only explicitly use $\mW^{(2)}$ in this proof, for simplicity of notation, we will omit its superscript and use $\mW$.
Let $(\bar{\rmD},\bar{\rmA})$ be an independent copy of $(\rmD,\rmA)$, then
\begin{equation}
\begin{split}
    \fns{\mM} &=\fns{\E[\rmD\mW^\T\rmA\mW\rmD]}\\
    &=\E\left[\langle \rmD\mW^\T\rmA\mW\rmD,\bar{\rmD}\mW^\T\bar{\rmA}\mW\bar{\rmD}\rangle\right]\\
    &=\E\left[\tr\left(\rmD\mW^\T\rmA\mW\rmD\bar{\rmD}\mW^\T\bar{\rmA}\mW\bar{\rmD}\right)\right]\\
    &=\E\left[\tr\left(\mW\bar{\rmD}\rmD\mW^\T\rmA\mW\rmD\bar{\rmD}\mW^\T\bar{\rmA}\right)\right].
\end{split}
\end{equation}
Expressing the term inside the expectation as a polynomial of entries of $\rmA$, $\rmD$, $\bar{\rmA}$ and $\bar{\rmD}$, we get
\begin{equation}
\label{eqn:proof-Mfnorm-polyexpression}
\begin{split}
     &\tr\left(\mW\bar{\rmD}\rmD\mW^\T\rmA\mW\rmD\bar{\rmD}\mW^\T\bar{\rmA}\right) \\
    =&\sum_{i=1}^c\left(\mW\bar{\rmD}\rmD\mW^\T\rmA\mW\rmD\bar{\rmD}\mW^\T\bar{\rmA}\right)_{i,i}\\
    =&\sum_{i,j=1}^c\rbr{\mW\bar{\rmD}\rmD\mW^\T\rmA}_{i,j}\rbr{\mW\rmD\bar{\rmD}\mW^\T\bar{\rmA}}_{j,i}\\
    =&\sum_{i,j=1}^c\rbr{\sum_{k=1}^c\sum_{l=1}^n\mW_{i,l}\mW_{k,l}\rmD_{l,l}\rmD_{l,l}\rmA_{k,j}}\rbr{\sum_{s=1}^c\sum_{t=1}^n\mW_{j,t}\mW_{s,t}\bar{\rmD}_{t,t}\bar{\rmD}_{t,t}\rmA_{s,i}}\\
    =&\sum_{i,j,k,s=1}^c\sum_{l,t=1}^n \mW_{i,l}\mW_{k,l}\mW_{j,t}\mW_{s,t}\bar{\rmA}_{k,j}\rmA_{s,i}\bar{\rmD}_{l,l}\rmD_{l,l}\bar{\rmD}_{t,t}\rmD_{t,t}.
\end{split}
\end{equation}
The monomials are $\bar{\rmA}_{k,j}\rmA_{s,i}\bar{\rmD}_{l,l}\rmD_{l,l}\bar{\rmD}_{t,t}\rmD_{t,t}$, and the corresponding coefficients are $\mW_{i,l}\mW_{k,l}\mW_{j,t}\mW_{s,t}$.
Now we can bound the $\ell_1$ norm of the coefficient of this polynomial as follows:

\begin{equation}
\label{eqn:proof-Mfnorm-poly-l1bound}
\begin{split}
     &\lnorm{\sum_{i,j,k,s=1}^c\sum_{l,t=1}^n \mW_{i,l}\mW_{k,l}\mW_{j,t}\mW_{s,t}}_1\\
    \leq &\sum_{i,j,k,s=1}^c\sum_{l,t=1}^n |\mW_{i,l}|\cdot|\mW_{k,l}|\cdot|\mW_{j,t}|\cdot|\mW_{s,t}|\\
    =&\rbr{\sum_{i,k=1}^c\sum_{l=1}^n|\mW_{i,l}|\cdot|\mW_{k,l}|}\rbr{\sum_{j,s=1}^c\sum_{t=1}^n|\mW_{j,t}|\cdot|\mW_{s,t}|}\\
    \leq &\rbr{\sum_{i,k=1}^c\sum_{l=1}^n\frac{\mW_{i,l}^2+\mW_{k,l}^2}{2}}\rbr{\sum_{j,s=1}^c\sum_{t=1}^n\frac{\mW_{j,t}^2+\mW_{s,t}^2}{2}}\\
    =&\rbr{\sum_{i,k=1}^c\frac{\ns{\mW_i}+\ns{\mW_k}}{2}}\rbr{\sum_{j,s=1}^c\frac{\ns{\mW_j}+\ns{\mW_s}}{2}}\\
    =&(c\fns{\mW})^2=c^2\fn{\mW}^4.
\end{split}
\end{equation}
From \lemmaref{lemma:W-norm} we know that $\fn{\mW}^2=O(c)$ with probability 1 over $\mW$, so the coefficient of this polynomial is $\ell_1$-norm bounded.

For any $\eps>0$, fix $\eps$. Note that $\fns{\mM^*}$ is just substituting $\rmD,\bar\rmD$ by $\rmD',\bar\rmD'$ in the polynomial characterized by \equationref{eqn:proof-Mfnorm-polyexpression}. From \corollaryref{cor:polynomial} we have the convergence of the difference of the expectation of the two polynomials, namely $|\fns{\mM} - \fns{\mM^*}| < \eps$ for sufficiently large $n$.
Since the spectral norm of $\mM^*$ is on the order of constant from \lemmaref{lemma:M-star}, we have $\lim_{n\to\infty} \fns{\mM}/\fns{\mM^*}=1.$
\end{proofof}

\begin{lemma}
\label{lemma:WMW-equivalence}
For all $i,j\in[c], \lim_{n\to\infty}((\mW\mM\mW^\T)_{i,j}-(\mW\mM^*\mW^\T)_{i,j})=0$. Thus, \[\lim_{n\to\infty}\frac{\fns{\mW\mM\mW^\T}}{\fns{\mW\mM^*\mW^\T}}=1.\]
\end{lemma}
\begin{proofof}{\lemmaref{lemma:WMW-equivalence}}
This proof is very similar to that of \lemmaref{lemma:M-equivalence}. First, we focus on a single entry of the matrix $\mW\mM\mW^\T$ and express it as a polynomial of entries of $\rmA$ and $\rmD$:
\begin{equation}
\label{eqn:proof-WMW-equiv-poly}
\begin{split}
(\mW\mM\mW^\T)_{i,j} &= \E[(\mW\rmD\mW^\T\rmA\mW\rmD\mW^\T)_{i,j}]\\
              &= \E\left[\sum_{k=1}^c(\mW\rmD\mW^\T\rmA)_{i,k}(\mW\rmD\mW^\T)_{k,j}\right]\\
              &= \E\left[\sum_{k=1}^c\rbr{\sum_{s=1}^c\sum_{l=1}^n\mW_{i,l}\mW_{s,l}\rmD_{l,l}\rmA_{s,k}}\rbr{\sum_{t=1}^n\mW_{k,j}\mW_{j,t}\rmD_{t,t}}\right]\\
              &= \E\left[\sum_{k,s=1}^c\sum_{l,t=1}^n\mW_{i,l}\mW_{s,l}\mW_{k,t}\mW_{j,t}\rmA_{s,k}\rmD_{l,l}\rmD_{t,t}\right].
\end{split}
\end{equation}
Then we bound the $\ell_1$ norm of the coefficients of this polynomial as follows:
\begin{equation}
\label{eqn:proof-WMW-equiv-l1bound}
\begin{split}
    &\lnorm{\sum_{k,s=1}^c\sum_{l,t=1}^n\mW_{i,l}\mW_{s,l}\mW_{k,t}\mW_{j,t}}_1\\
    \leq &\sum_{k,s=1}^c\sum_{l,t=1}^n|\mW_{i,l}|\cdot|\mW_{s,l}|\cdot|\mW_{k,t}|\cdot|\mW_{j,t}|\\
    =    &\rbr{\sum_{s=1}^c\sum_{l=1}^n|\mW_{i,l}|\cdot|\mW_{s,l}|}\rbr{\sum_{k=1}^c\sum_{t=1}^n|\mW_{k,t}|\cdot|\mW_{j,t}|}\\
    \leq &\rbr{\sum_{s=1}^c\sum_{l=1}^n\frac{\mW_{i,l}^2+\mW_{s,l}^2}{2}}\rbr{\sum_{k=1}^c\sum_{t=1}^n\frac{\mW_{k,t}^2+\mW_{j,t}^2}{2}}\\
    =    &\rbr{c\ns{\mW_i}+\fns{\mW}}\rbr{c\ns{\mW_j}+\fns{\mW}}\\
    \leq &(2c\fns{\mW})^2=4c^2\fn{\mW}^4.
\end{split}
\end{equation}
Similar to \lemmaref{lemma:M-equivalence}, this coefficient is $\ell_1$-norm bounded. 
The expression of each entry of $\mW\mM^*\mW^\T$ is just substituting $\rmD,\bar\rmD$ by $\rmD',\bar\rmD'$ in the polynomial characterized by \equationref{eqn:proof-WMW-equiv-poly}.
Therefore, using \lemmaref{lemma:polynomial}, we have with probability 1 over $\mW$, for all $i,j\in[c]$, 
\begin{equation}
    \lim_{n\to\infty}((\mW\mM\mW^\T)_{i,j}-(\mW\mM^*\mW^\T)_{i,j})=0.
\end{equation}
This completes the proof of the lemma as $\mW\mM\mW^\T$ is of constant size.
\end{proofof}

\begin{lemma}
\label{lemma:F-norm-equal}
With probability 1 over $\rmW^{(1)}$ and $\rmW^{(2)}$,
\begin{equation}
\lim_{n\to\infty}\frac{\fns{\mW\mM^*\mW^\T}}{\fns{\mM^*}}=1.
\end{equation}

\end{lemma}
\begin{proofof}{\lemmaref{lemma:F-norm-equal}}
The proof of this lemma will be divided into two parts. In the first part, we will estimate the Frobenius norm of $\mM^*$, and in the second part we do the same thing for $\mW\mM^*\mW^\T$.

\textbf{Part 1:} From \lemmaref{lemma:M-star} we know that
\begin{equation}
\mM^*=\frac14\rbr{\mW^\T\E[\rmA]\mW + \emph{diag}(\mW^\T\E[\rmA]\mW)}.
\end{equation}
Denote $\tilde{\rmA}\triangleq\E[\rmA]$, then
\begin{equation}
\E[\mW^\T\rmA\mW] = \mW^\T\E[\rmA]\mW = \mW^\T\tilde{\rmA}\mW.
\end{equation}
From \lemmaref{lemma:WW-identity}, for all $\eps'>0$, with probability 1 over $\rmW$ we have $\norm{\mW\mW^\T-\mI_c}\leq\eps'$. Besides, from \cite{kleinman1968design} we know that for positive semi-definite matrices $\mA$ and $\mB$ we have $\lambda_{\min}(\mA)\tr(\mB)\leq \tr(\mA\mB)\leq \lambda_{\max}(\mA)\tr(\mB)$, so
\begin{equation}
\begin{split}
    \bigg|\fns{\mW^\T\tilde{\rmA}\mW} - \fns{\tilde{\rmA}}\bigg|
    &=\Big|\tr(\mW^\T\tilde{\rmA}\mW\mW^\T\tilde{\rmA}\mW)-\tr(\tilde{\rmA}\tilde{\rmA})\Big|\\
    &=\Big|\tr(\mW\mW^\T\tilde{\rmA}\mW\mW^\T\tilde{\rmA})-\tr(\tilde{\rmA}\tilde{\rmA})\Big|\\
    &\leq\Big|\rbr{\norml{\mW\mW^\T-\mI_c}+1}\tr(\tilde{\rmA}\mW\mW^\T\tilde{\rmA})-\tr(\tilde{\rmA}\tilde{\rmA})\Big|\\
    &=\Big|\rbr{\norml{\mW\mW^\T-\mI_c}+1}\tr(\mW\mW^\T\tilde{\rmA}\tilde{\rmA})-\tr(\tilde{\rmA}\tilde{\rmA})\Big|\\
    &\leq\Big|\rbr{\norml{\mW\mW^\T-\mI_c}+1}^2\tr(\tilde{\rmA}\tilde{\rmA})-\tr(\tilde{\rmA}\tilde{\rmA})\Big|\\
    &\leq\norml{\mW\mW^\T-\mI_c}^2\fns{\tilde{A}} + 2\norml{\mW\mW^\T-\mI_c}\fns{\tilde{A}}.
\end{split}
\end{equation}
For any $\eps>0$, set $\eps'=\min\{\frac{\eps}{4},\frac{\sqrt{\eps}}{2}\}$ gives us with probability 1,
\begin{equation}
    \lim_{n\to\infty}\frac{\bigg|\fns{\mW^\T\tilde{\rmA}\mW} - \fns{\tilde{\rmA}}\bigg|}{\fns{\tilde{\rmA}}}=0,
\end{equation}
i.e.,
\begin{equation}
    \lim_{n\to\infty}\frac{\fns{\mW^\T\tilde{\rmA}\mW}}{\fns{\tilde{\rmA}}}=1.
\end{equation}
Besides, if we denote the $i$-th column of $\mW$ by $\vw_i$, then
\begin{equation}
\begin{split}
\fns{\diag(\E[\mW^\T\rmA\mW])} &= \sum_{i=1}^n (\vw_i^\T\tilde{\rmA}\vw_i)^2\\
&\leq \sum_{i=1}^n \rbr{\ns{\vw_i}\cdot\norml{\tilde{\rmA}}}^2\\
&= \ns{\tilde{\rmA}}\sum_{i=1}^n \norml{\vw_i}^4.
\end{split}
\end{equation}
Since $\E[n^2\norml{\vw_i}^4]=c^2+2c$, by the additive form of Chernoff bound we get
\begin{equation}
\pr{\sum_{i=1}^n\norml{\vw_i}^4\geq \frac{c^2+3c}{n}}=\pr{\frac{\sum_{i=1}^nn^2\norml{\vw_i}^4}{n}-(c^2+2c)\geq c}\leq e^{-2nc^2}.
\end{equation}
Therefore, when $n\to\infty$, with probability 1 over $\rmW$ we have
\begin{equation}
\fns{\diag(\E[\mW^\T\rmA\mW])}\leq \ns{\tilde{\rmA}}\sum_{i=1}^n \norml{\vw_i}^4\leq \ns{\tilde{\rmA}}\cdot\frac{c^2+3c}{n}.
\end{equation}
Thus, with probability 1 over $\rmW$,
\begin{equation}
\label{lemma:W-diag-neg}
\lim_{n\to\infty}\frac{\fns{\diag\left(\E\left[\mW^\T\rmA\mW\right]\right)}}{\fns{\mW^\T\tilde{\rmA}\mW}} = 0,
\end{equation}
i.e.,
\begin{equation}
\lim_{n\to\infty}\frac{\frac{1}{16}\fns{\tilde{\rmA}}}{\fns{\mM^*}} = 1.
\end{equation}

\textbf{Part 2:} We now estimate the norm of $\mW\mM^*\mW$. Plug equation \equationref{eqn:proof-M-star} into $\mW\mM^*\mW$ and we get
\begin{equation}
\mW\mM^*\mW = \frac14\rbr{\E[\mW\mW^\T\rmA\mW\mW^\T]+\E[\mW \diag(\mW^\T\rmA\mW)\mW^\T]}.
\end{equation}
Similar to \textbf{Part 1}, when $n\to\infty$, with probability 1, we have
\begin{equation}
\lim_{n\to\infty}\frac{\fns{\E[\mW\mW^\T\rmA\mW\mW^\T]}}{\fns{\tilde{\rmA}}} = 1.
\end{equation}
Besides, when $n\to\infty$, with probability 1 we have
\begin{equation}
\fns{\mW\diag(\E[\mW^\T\rmA\mW])\mW^\T}\leq \fns{\mW}\ns{\tilde{\rmA}}\sum_{i=1}^n \norml{\vw_i}^4\leq \ns{\tilde{\rmA}}\cdot\frac{c^2+3c}{n}\fns{\mW}.
\end{equation}
As a result, with probability 1,
\begin{equation}
\lim_{n\to\infty}\frac{\fns{\mW\diag(\E[\mW^\T\rmA\mW])\mW^\T}}{\fns{\mW\mW^\T\tilde{\rmA}\mW\mW^\T}} = 0,
\end{equation}
i.e.,
\begin{equation}
\lim_{n\to\infty}\frac{\frac{1}{16}\fns{\tilde{\rmA}}}{\fns{\mW\mM^*\mW^\T}} = 1.
\end{equation}
Combining the results of \textbf{Part 1} and \textbf{Part 2} proves this lemma.

\end{proofof}
Combining \lemmaref{lemma:M-equivalence}, \lemmaref{lemma:WMW-equivalence}, and \lemmaref{lemma:F-norm-equal} directly finishes the proof of \lemmaref{lemma:M-proj-preserve-f-norm}.
\end{proofof}

After establishing the projection of $\mM$ onto a $c\times c$ matrix, we may project the full layer-wise Hessian of the first layer, namely $\mH^{(1)}=\exs{\rmD\mW^{(2)\T}\rmA\mW^{(2)}\rmD\otimes \rvx\rvx^\T}$ onto a $c\times c$ matrix using very similar techniques. For simplicity of notation, we will denote $\mW^{(2)}$ by $\mW$ and $\mH^{(1)}$ by $\mH$ unless explicitly stated otherwise.

Since the autocorrelation matrix $\rvx\rvx$
has unbounded Frobenious norm, we will consider a re-scaled version $\tmH\triangleq \mH/d^2$ for our analysis. Let $\mU\triangleq\frac{1}{\sqrt{d}}\1_d^\T\in\R^{1\times d}$ be an all-1 matrix scaled by $\frac{1}{\sqrt{d}}$, we have $\mU\mU^\T=1$. Let $\mV\triangleq \mW\otimes\mU$ be our projection matrix for $\tmH$, we may then state our main lemma for full layer-wise Hessian.
\begin{lemma}
\label{lemma:H-proj-preserve-f-norm}
With probability 1 over $\mW^{(1)}$ and $\mW^{(2)}$, \begin{equation}
    \lim_{n\to\infty}\frac{\fns{\mV\tmH\mV^\T}}{\fns{\tmH}}=1.
\end{equation}
\end{lemma}

\begin{proofof}{\lemmaref{lemma:H-proj-preserve-f-norm}}
Similar to the proof for the output Hessian, we will introduce a ``bridging term''\begin{equation}
    \tmH^*\triangleq\frac{1}{d}\E[\rmD'\mW^{(2)\T}\rmA\mW^{(2)}\rmD'\otimes\rvx''\rvx''^\T]
\end{equation}
where $\rmD'$ is an independent copy of $\rmD$ and also independent of $\rmA$, and $\rvx''$ is an independent copy of $\rvx$ which is independent to both $\rmD'$ and $\rmA$.
Informally, we will show \begin{equation}
    \fns{\mV\tmH\mV^\T}\approx \fns{\mV\tmH^*\mV^\T}\approx \fns{\tmH^*}\approx \fns{\tmH}.
\end{equation}
We first look into the structures of $\tmH^*$.
\begin{lemma}
\label{lemma:H-star}
With probability 1 over $\mW^{(1)}$ and $\mW^{(2)}$,
\begin{equation}
\label{eqn:proof-H-star}
\tmH^*=\frac{1}{4d}\rbr{\mW^\T\E[\rmA]\mW + \emph{diag}(\mW^\T\E[\rmA]\mW)}\otimes \rbr{\frac{1}{2\pi}\1_d\1_d^\T+\frac{\pi - 1}{2\pi}\mI_d}.
\end{equation}
Moreover, for large $n$, $\eta/32<\fn{\tmH^*}<2c^2$. 
\end{lemma}
\begin{proofof}{\lemmaref{lemma:H-star}}
By independence in construction, we have $\tmH^* = \mM^*\otimes(\frac1d\E[\rvx''\rvx''^\T])$. Thus we only need to look into $\E[\rvx''\rvx''^\T]$.
For $i=j$, we have $\E[\rvx''\rvx''^\T]_{ii}=\E[\rvx_i\rvx_i]=\frac{1}{2}$ while for $i\neq j$, $\E[\rvx''\rvx''^\T]_{ij}=\E[\rvx_i\rvx_j]=\frac{1}{2\pi}$. Thus
\begin{equation}
\label{eqn:xxT-structure}
\E[\rvx''\rvx''^\T]=\frac{1}{2\pi}\1_d\1_d^\T+\frac{\pi - 1}{2\pi}\mI_d.
\end{equation}
It follows that\begin{equation}
\begin{split}
    \lim_{d\to\infty}\frac1d\fn{\E[\rvx\rvx^\T]} &= \lim_{d\to\infty}\frac1d\sqrt{d^2\frac{1}{4\pi^2} + d\frac{(\pi-1)^2}{4\pi^2}} = \frac{1}{2\pi} > \frac{1}{8}.
\end{split}
\end{equation}
Thus for large $n$ we have $\frac18<\frac1d\fn{\E[\rvx\rvx^\T]}<1$.
Since $\fn{\tmH^*} = \frac1d\fn{\mM^*\otimes \E[\rvx\rvx^\T]} = \fn{\mM^*}\cdot\frac1d\fn{\E[\rvx\rvx^\T]}$ and we know that $\frac{\eta}{4}<\fn{\tmH^*}<2c^2$ from \lemmaref{lemma:M-star}. We can conclude that for large $n$, $\eta/32<\fn{\tmH^*}<2c^2$.
\end{proofof}

\begin{lemma}
\label{lemma:H-equivalence} With probability 1 over $\mW^{(1)}$ and $\mW^{(2)}$,
\[
\lim_{n\to\infty}\frac{\fns{\tmH}}{\fns{\tmH^*}}=1.
\]\end{lemma}
\begin{proofof}{\lemmaref{lemma:H-equivalence}}
Unsurprisingly, this proof will be very similar to the proof of \lemmaref{lemma:M-equivalence}.
Recall that $\tmH^*\triangleq\frac1d\E[\rmD'\mW^\T\rmA\mW\rmD'\otimes\rvx''\rvx''^\T]$. Let $(\bar{\rmD},\bar{\rmA},\bar{\rvx})$ be an independent copy of $(\rmD,\rmA,\rvx)$,
\begin{equation}
\begin{split}
    \fns{\tmH} &=\lnorm{\frac1d\E[\rmD\mW^\T\rmA\mW\rmD\otimes \rvx\rvx^\T]}_F^2\\
    &=\E\left[\frac{1}{d^2}\langle \rmD\mW^\T\rmA\mW\rmD\otimes \rvx\rvx^\T,\bar{\rmD}\mW^\T\bar{\rmA}\mW\bar{\rmD}\otimes \bar\rvx\bar\rvx^\T\rangle\right]\\
    &=\ex{\frac{1}{d^2}\tr\left(\rbr{\rmD\mW^\T\rmA\mW\rmD\otimes\rvx\rvx^\T}\rbr{\bar{\rmD}\mW^\T\bar{\rmA}\mW\bar{\rmD}\otimes\bar\rvx\bar\rvx^\T}\right)}\\
    &=\ex{\frac{1}{d^2}\tr\rbr{\rmD\mW^\T\rmA\mW\rmD\bar{\rmD}\mW^\T\bar{\rmA}\mW\bar{\rmD}}\tr\rbr{\rvx\rvx^\T\bar\rvx\bar\rvx^\T}}\\
    &=\ex{\frac{1}{d^2}(\rvx^\T\bar\rvx\bar\rvx^\T\rvx)\tr\left(\mW\bar{\rmD}\rmD\mW^\T\rmA\mW\rmD\bar{\rmD}\mW^\T\bar{\rmA}\right)}.
\end{split}
\end{equation}
Expressing the term inside the expectation as a polynomial of entries of $\rmA$, $\rmD$, $\bar{\rmA}$ and $\bar{\rmD}$, we get
\begin{equation}
\label{eqn:proof-Mfnorm-polyexpression}
\begin{split}
     &\frac{1}{d^2}(\rvx^\T\bar\rvx\bar\rvx^\T\rvx)\tr\left(\mW\bar{\rmD}\rmD\mW^\T\rmA\mW\rmD\bar{\rmD}\mW^\T\bar{\rmA}\right) \\
    =&\frac{1}{d^2}\sum_{p,q=1}^d\rvx_p\bar\rvx_p\rvx_q\bar\rvx_q\rbr{\sum_{i=1}^c\left(\mW\bar{\rmD}\rmD\mW^\T\rmA\mW\rmD\bar{\rmD}\mW^\T\bar{\rmA}\right)_{i,i}}\\
    =&\frac{1}{d^2}\sum_{p,q=1}^d\sum_{i,j,k,s=1}^c\sum_{l,t=1}^n \mW_{i,l}\mW_{k,l}\mW_{j,t}\mW_{s,t}\bar{\rmA}_{k,j}\rmA_{s,i}\bar{\rmD}_{l,l}\rmD_{l,l}\bar{\rmD}_{t,t}\rmD_{t,t}\rvx_p\bar\rvx_p\rvx_q\bar\rvx_q.
\end{split}
\end{equation}
We skipped some derivations as they are identical to \equationref{eqn:proof-Mfnorm-polyexpression}.
The monomials are\\ $\bar{\rmA}_{k,j}\rmA_{s,i}\bar{\rmD}_{l,l}\rmD_{l,l}\bar{\rmD}_{t,t}\rmD_{t,t}\rvx_p\bar\rvx_p\rvx_q\bar\rvx_q$, and the corresponding coefficients are $\mW_{i,l}\mW_{k,l}\mW_{j,t}\mW_{s,t}$.
The $\ell_1$ norm of the coefficients is
\begin{equation}
\begin{split}
    \lnorm{\frac{1}{d^2}\sum_{p,q=1}^d\sum_{i,j,k,s=1}^c\sum_{l,t=1}^n \mW_{i,l}\mW_{k,l}\mW_{j,t}\mW_{s,t}}_1 = \lnorm{\sum_{i,j,k,s=1}^c\sum_{l,t=1}^n \mW_{i,l}\mW_{k,l}\mW_{j,t}\mW_{s,t}}_1,
\end{split}
\end{equation}
which we know is upper bounded by some constant with probability 1 over $\mW$ from \equationref{eqn:proof-Mfnorm-poly-l1bound}.

For any $\eps>0$, fix $\eps$. Note that $\fns{\tmH^*}$ is just substituting $(\rmD,\bar\rmD,\rvx,\bar\rvx)$ by $(\rmD',\bar\rmD',\rvx'',\bar\rvx'')$ in the polynomial characterized by \equationref{eqn:proof-Mfnorm-polyexpression}. From \corollaryref{cor:polynomial} we have the convergence of the difference of the expectation of the two polynomials, namely $|\fns{\tmH} - \fns{\tmH^*}| < \eps$ for sufficiently large $n$.
Since the spectral norm of $\tmH^*$ is bounded below from 0 by \lemmaref{lemma:M-star}, we have $\lim_{n\to\infty} \fns{\tmH}/\fns{\tmH^*}=1.$
\end{proofof}

\begin{lemma}
\label{lemma:VHV-equivalence}
For all $i,j\in[c], \lim_{n\to\infty}((\mV\tmH\mV^\T)_{i,j}-(\mV\tmH^*\mV^\T)_{i,j})=0$. Thus,
\begin{equation}
\lim_{n\to\infty}\frac{\fns{\mV\tmH\mV^\T}}{\fns{\mV\tmH^*\mV^\T}}=1.
\end{equation}
\end{lemma}
\begin{proofof}{\lemmaref{lemma:VHV-equivalence}}
This proof is very similar to that of \lemmaref{lemma:H-equivalence}. First, we focus on a single entry of the matrix $\mV\tmH\mV^\T$ and express it as a polynomial of entries of $\rmA$ and $\rmD$:
\begin{equation}
\label{eqn:proof-VHV-equiv-poly}
\begin{split}
(\mV\tmH\mV^\T)_{i,j} &= \ex{\rbr{(\mW\otimes \mU)\frac1d(\rmD\mW^\T\rmA\mW\rmD\otimes \rvx\rvx^\T)(\mW\otimes \mU)^\T}_{i,j}}\\
&= \ex{\frac1d\rbr{(\mW\rmD\mW^\T\rmA\mW\rmD\mW^\T)\otimes(\mU\rvx\rvx^\T \mU^\T)}_{i,j}}\\
&= \ex{\frac1d\cdot \frac{1}{d}(\1_d^\T\rvx\rvx^\T \1_d)\rbr{\mW\rmD\mW^\T\rmA\mW\rmD\mW^\T}_{i,j}}\\
&= \ex{\frac{1}{d^2}\rbr{\sum_{p,q=1}^d\rvx_p\rvx_q}\rbr{\sum_{k=1}^c(\mW\rmD\mW^\T\rmA)_{i,k}(\mW\rmD\mW^\T)_{k,j}}}\\
&= \ex{\frac{1}{d^2}\sum_{p,q=1}^d\sum_{k,s=1}^c\sum_{l,t=1}^n\mW_{i,l}\mW_{s,l}\mW_{k,t}\mW_{j,t}\rmA_{s,k}\rmD_{l,l}\rmD_{t,t}\rvx_p\rvx_q}.
\end{split}
\end{equation}
We skipped some derivations as they are identical to \equationref{eqn:proof-WMW-equiv-poly}.
The monomials are $\rmA_{s,k}\rmD_{l,l}\rmD_{t,t}\rvx_p\rvx_q$, and the corresponding coefficients are $\mW_{i,l}\mW_{s,l}\mW_{k,t}\mW_{j,t}$. Observe that the $\ell_1$ norm of the coefficients satisfies
\begin{equation}
    \lnorm{\frac{1}{d^2}\sum_{p,q=1}^d\sum_{k,s=1}^c\sum_{l,t=1}^n\mW_{i,l}\mW_{s,l}\mW_{k,t}\mW_{j,t}}_1 = \lnorm{\sum_{k,s=1}^c\sum_{l,t=1}^n\mW_{i,l}\mW_{s,l}\mW_{k,t}\mW_{j,t}}_1,
\end{equation}
which we know is bounded above by some constant from \equationref{eqn:proof-WMW-equiv-l1bound}.
Note that the expression of each entry of $\mV\tmH^*\mW^\T$ is just substituting $(\rmD,\bar\rmD,\rvx,\bar\rvx)$ by $(\rmD',\bar\rmD',\rvx'',\bar\rvx'')$ in the polynomial characterized by \equationref{eqn:proof-VHV-equiv-poly}.
Therefore, using \lemmaref{lemma:polynomial}, we have with probability 1 over $\mW$, for all $i,j\in[c]$, 
\begin{equation}
    \lim_{n\to\infty}((\mV\tmH\mV^\T)_{i,j}-(\mV\tmH^*\mV^\T)_{i,j})=0.
\end{equation}
This completes the proof of the lemma as $\mV\tmH\mV^\T$ is of constant size.
\end{proofof}

\begin{lemma}
\label{lemma:F-norm-equal-H}
With probability 1 over $\rmW^{(1)}$ and $\rmW^{(2)}$,
\begin{equation}
\lim_{n\to\infty}\frac{\fns{\mV\tmH^*\mV^\T}}{\fns{\tmH^*}}=1.
\end{equation}
\end{lemma}

\begin{proofof}{\lemmaref{lemma:F-norm-equal-H}}
This lemma is a direct corollary of \lemmaref{lemma:F-norm-equal} for the output Hessian. Note that by the independence in construction,\begin{equation}
\begin{split}
\mV\tmH^*\mV^\T &= \frac{1}{d}\rbr{\mW\otimes\mU}\ex{\rmD'\mW^\T\rmA\mW\rmD\otimes\rvx''\rvx''^\T}(\mW^\T\otimes\mU^\T)\\
&= \frac{1}{d}\rbr{\mW\otimes\mU}\rbr{\mM^*\otimes\ex{\rvx''\rvx''^\T}}(\mW^\T\otimes\mU^\T)\\
&= \frac1d \rbr{\mW\mM^*\mW^\T}\otimes\rbr{\mU\ex{\rvx''\rvx''^\T}\mU^\T}\\
&= \rbr{\mW\mM^*\mW^\T}\otimes\rbr{\frac{1}{d^2}\1_d^\T\ex{\rvx''\rvx''^\T}\1_d}\\
&= \frac{1}{d^2}\1_d^\T\ex{\rvx''\rvx''^\T}\1_d\rbr{\mW\mM^*\mW^\T}.
\end{split}
\end{equation}
From \equationref{eqn:xxT-structure} we have
\begin{equation}
    \1_d^\T\ex{\rvx''\rvx''^\T}\1_d = \sum_{i,j=1}^d\ex{\rvx\rvx^\T}_{ij} = \frac{1}{2\pi}d^2 + \frac{\pi-1}{2\pi}d.
\end{equation}
Thus \begin{equation}
\begin{split}
    \lnorm{\mV\tmH^*\mV^\T}_F^2 &= \lnorm{\rbr{\frac{1}{2\pi} + \frac{\pi-1}{2\pi d}}\mW\mM^*\mW}_F^2\\
    &= \rbr{\frac{1}{4\pi^2} + \frac{\pi-1}{2\pi^2 d} + \frac{(\pi-1)^2}{4\pi^2d^2}}\fns{\mW\mM^*\mW}.
\end{split}
\end{equation}
Meanwhile note that \begin{equation}
    \fns{\tmH^*} = \frac{1}{d^2}\fns{\tmM^*\otimes\ex{\rvx''\rvx''^\T}} = \frac{1}{d^2}\fns{\tmM^*}\otimes\fns{\ex{\rvx''\rvx''^\T}},
\end{equation}
where \begin{equation}
    \fns{\ex{\rvx''\rvx''^\T}} = \sum_{i,j=1}^d\ex{\rvx\rvx^\T}_{ij}^2 = \frac{1}{4\pi^2}d^2 + \frac{\pi-1}{2\pi}d.
\end{equation}
Thus\begin{equation}
    \fns{\tmH^*} = \rbr{\frac{1}{4\pi^2} + \frac{\pi-1}{2\pi d}}\fns{\tmM^*}.
\end{equation}
Since $d=n^{1+\alpha}$ for some constant $\alpha>0$, we have 
\begin{equation}
    \lim_{n\to\infty}\frac{\frac{1}{4\pi^2} + \frac{\pi-1}{2\pi^2 d} + \frac{(\pi-1)^2}{4\pi^2d^2}}{\frac{1}{4\pi^2} + \frac{\pi-1}{2\pi d}} = 1.
\end{equation}
Thus combined with the result from \lemmaref{lemma:F-norm-equal}, we have\begin{equation}
    \lim_{n\to\infty}\frac{\fns{\mV\tmH^*\mV^\T}}{\fns{\tmH^*}} = \rbr{\lim_{n\to\infty}\frac{\frac{1}{4\pi^2} + \frac{\pi-1}{2\pi^2 d} + \frac{(\pi-1)^2}{4\pi^2d^2}}{\frac{1}{4\pi^2} + \frac{\pi-1}{2\pi d}}}\rbr{\lim_{n\to\infty}\frac{\fns{\mW\mM^*\mW^\T}}{\fns{\mM^*}}} = 1.
\end{equation}
\end{proofof}

Combining \lemmaref{lemma:H-equivalence}, \lemmaref{lemma:VHV-equivalence}, and \lemmaref{lemma:F-norm-equal-H} completes the proof of \lemmaref{lemma:H-proj-preserve-f-norm}.
\end{proofof}

Now we are done with the lemmas and will proceed to the proof of the main theorems.

%% file: MainProof/out-hessian-structure.tex
\subsubsection{Structure of Output Hessian of the First Layer}
\label{sec:proof-out-hessian}

We first restate \theoremref{thm:main-out} here:

\noindent\textbf{\theoremref{thm:main-out}} \emph{Let $\mM^*\triangleq \ex{\rmD'\mW^{(2)\T}\rmA\mW^{(2)}\rmD'}$ where $\rmD'$ is an independent copy of $\rmD$ and is also independent of $\rmA$. Let $S_1$ and $S_2$ be the top $c-1$ eigenspaces of $\mM^{(1)}$ and $\mM^*$ respectively, for all $\eps>0$,
\begin{equation}
    \lim_{n\to\infty}\mathop{\Pr}_{\mW^{(1)}\sim\gN(0,\frac{1}{d}\mI_{nd}), \mW^{(2)}\sim\gN(0,\frac{1}{n}\mI_{cn})}\left[\Overlap\left(S_1,S_2\right)>1-\eps\right] = 1.
\end{equation}
Moreover, as\begin{equation}
    \lim_{n\to\infty}\mathop{\Pr}_{\mW^{(1)}\sim\gN(0,\frac{1}{d}\mI_{nd}), \mW^{(2)}\sim\gN(0,\frac{1}{n}\mI_{cn})}\left[\left(\left.\frac{\lambda_c(\mM)}{\lambda_{c-1}(\mM)}\right|_{\mW^{(1)}, \mW^{(2)}}\right) < \eps\right] = 1.
\end{equation}}

\begin{proofof}{\theoremref{thm:main-out}}
\end{proofof}

From \lemmaref{lemma:M-proj-preserve-f-norm} we have 
\begin{equation}
    \lim_{n\to\infty}\frac{\fns{\mW\mM\mW^\T}}{\fns{\mM}}=1.
\end{equation}
Then we consider $\fns{\mW^\T\mW\mM \mW^\T\mW}$. Note that
\begin{equation}
\begin{split}
    \fns{\mW^\T\mW\mM\mW^\T\mW}&=\tr(\mW^\T\mW\mM\mW^\T\mW\mW^\T\mW\mM\mW^\T\mW)\\
    &=\tr(\mW\mW^\T\mW\mM\mW^\T\mW\mW^\T\mW\mM\mW^\T).
\end{split}
\end{equation}
From \lemmaref{lemma:WW-identity} we know that for all $\eps'>0$, $\lim_{n\to\infty}\Pr(\norm{\mW\mW^\T-\mI_c}\geq\eps')=0$. For notation simplicity, in this proof we will omit the limit and probability arguments which can be dealt with using union bound. Therefore, we will directly state $\norm{\mW\mW^\T-\mI_c}\leq\eps'$. From \cite{kleinman1968design} we know that for positive semi-definite matrices $\mA$ and $\mB$ we have $\lambda_{\min}(\mA)\tr(\mB)\leq \tr(\mA\mB)\leq \lambda_{\max}(\mA)\tr(\mB)$, so
\begin{equation}\begin{split}
    &|\tr(\mW\mW^\T\cdot \mW\mM \mW^\T\mW\mW^\T\mW\mM \mW^\T) - \tr(\mW\mM \mW^\T\mW\mW^\T\mW\mM \mW^\T)| \\
    \leq& \max\{1-\lambda_{\min}(\mW\mW^\T), \lambda_{\max}(\mW\mW^\T)-1\}\tr(\mW\mM \mW^\T\mW\mW^\T\mW\mM \mW^\T)\\
    \leq& \norm{\mW\mW^\T-\mI_c}\tr(\mW\mM \mW^\T\mW\mW^\T\mW\mM \mW^\T)\leq\eps'\tr(\mW\mM \mW^\T\mW\mW^\T\mW\mM \mW^\T).
\end{split}
\end{equation}
Similarly,
\begin{equation}\begin{split}
    &|\tr(\mW\mM \mW^\T\mW\mW^\T\mW\mM \mW^\T) - \tr(\mW\mM \mW^\T\mW\mM \mW^\T)|\\
    =&|\tr(\mW\mW^\T\cdot \mW\mM \mW^\T\mW\mM \mW^\T) - \tr(\mW\mM \mW^\T\mW\mM \mW^\T)|\\
    \leq&\norm{\mW\mW^\T-\mI_c}\tr(\mW\mM \mW^\T\mW\mM \mW^\T)\leq\eps'\tr(\mW\mM \mW^\T\mW\mM \mW^\T).
\end{split}\end{equation}
Therefore,
\begin{equation}\begin{split}
    &|\fns{\mW^\T\mW\mM \mW^\T\mW}-\fns{\mW\mM \mW^\T}|\\
    =&|\tr(\mW\mW^\T\cdot \mW\mM \mW^\T\mW\mW^\T\mW\mM \mW^\T) - \tr(\mW\mM \mW^\T\mW\mM \mW^\T)| \\
    \leq&|\tr(\mW\mW^\T\cdot \mW\mM \mW^\T\mW\mW^\T\mW\mM \mW^\T) - \tr(\mW\mM \mW^\T\mW\mW^\T\mW\mM \mW^\T)|\\
    &+|\tr(\mW\mM \mW^\T\mW\mW^\T\mW\mM \mW^\T) - \tr(\mW\mM \mW^\T\mW\mM \mW^\T)|\\
    \leq&\eps'\tr(\mW\mM \mW^\T\mW\mW^\T\mW\mM \mW^\T)+\eps'\tr(\mW\mM \mW^\T\mW\mM \mW^\T)\\
    \leq&\eps'(1+\eps')\tr(\mW\mM \mW^\T\mW\mM \mW^\T)+\eps'\tr(\mW\mM \mW^\T\mW\mM \mW^\T)\\
    \leq&(2\eps'+(\eps')^2)\tr(\mW\mM \mW^\T\mW\mM \mW^\T) = (2\eps'+(\eps')^2)\fns{\mW\mM \mW^\T}.
\end{split}\end{equation}
For all $\eps>0$, select $\eps'<\min\{\frac{\sqrt{\eps}}{2},\frac{\eps}{4}\}$, we have
\begin{equation}
    |\fns{\mW^\T\mW\mM \mW^\T\mW}-\fns{\mW\mM \mW^\T}|<\eps \fns{\mW\mM \mW^\T}.
\end{equation}
In other words,
\begin{equation}
    \lim_{n\to\infty}\frac{\fns{\mW^\T\mW\mM \mW^\T\mW}}{\fns{\mW\mM \mW^\T}}=1.
\end{equation}
Hence we get
\begin{equation}
    \lim_{n\to\infty}\frac{\fns{\mW^\T\mW\mM \mW^\T\mW}}{\fns{\mM }}=1.
\end{equation}
Next, consider the orthogonal projection matrix $P_\mW \triangleq \bmW^\T\bmW$ that projects vectors in $\R^n$ into the subspace spanned by all rows of $\mW$. Here $\bmW$ is the orthogonolized $\mW$, which is explicitly defined in \lemmaref{lemma:W-projection}.
We will consider the matrix $P_\mW\mM P_\mW$. Define $\delta\triangleq \mW^\T\mW-P_\mW$, then from \lemmaref{lemma:W-projection} we get $\fns{\delta}\leq\eps'$. Therefore,
\begin{equation}\begin{split}
    &|\fn{\mW^\T\mW\mM \mW^\T\mW} - \fn{P_\mW\mM P_\mW}|\\
    \leq\ &\fn{P_\mW\mM \delta} + \fn{\delta \mM P_\mW} + \fn{\delta \mM \delta}\\
    \leq\ &\fn{\mM }\pr{2\fn{P_\mW}\fn{\delta}+\fns{\delta}}\\
    \leq\ &\fn{\mM }\pr{2\cdot 4c^2\eps'+(\eps')^2}.
\end{split}\end{equation}
For all $\eps>0$, we choose $\eps'<\min\{\frac{\sqrt{\eps}}{2},\frac{\eps}{16c^2}\}$ and have
\begin{equation}
    \frac{|\fn{\mW^\T\mW\mM \mW^\T\mW} - \fn{P_\mW\mM P_\mW}|}{\fn{\mM }} < \eps,
\end{equation}
which means that
\begin{equation}
\begin{split}
    &\lim_{n\to\infty}\frac{|\fn{\mW^\T\mW\mM \mW^\T\mW} - \fn{P_\mW\mM P_\mW}|}{\fn{\mW^\T\mW\mM \mW^\T\mW}}\\
    =\ &\lim_{n\to\infty}\frac{|\fn{\mW^\T\mW\mM \mW^\T\mW} - \fn{P_\mW\mM P_\mW}|}{\fn{\mM }}=0.
\end{split}
\end{equation}
Thus,
\begin{equation}
    \lim_{n\to\infty}\frac{\fn{P_\mW\mM P_\mW}}{\fn{\mM }}=\lim_{n\to\infty}\frac{\fn{P_\mW\mM P_\mW}}{\fn{\mW^\T\mW\mM \mW^\T\mW}}=1.
\end{equation}
Note that $\fns{\mM }=\fns{P_\mW\mM P_\mW}+\fns{P_\mW\mM P_\mW^\perp}+\fns{P_\mW^\perp \mM P_\mW}+\fns{P_\mW^\perp \mM P_\mW^\perp}$. It follows that, 
\begin{equation}
\begin{split}
    &\lim_{n\to\infty}\frac{\fns{P_\mW\mM P_\mW^\perp}+\fns{P_\mW^\perp \mM P_\mW}+\fns{P_\mW^\perp \mM P_\mW^\perp}}{\fns{\mM }}\\
    =\ &\lim_{n\to\infty}\frac{\fns{\mM }-\fns{P_\mW\mM P_\mW}}{\fns{\mM }}=0.
\end{split}
\end{equation}
In other words,
\begin{equation}
    \lim_{n\to\infty}\frac{\fn{P_\mW\mM P_\mW^\perp}}{\fn{\mM }}=\lim_{n\to\infty}\frac{\fn{P_\mW^\perp \mM P_\mW}}{\fn{\mM }}=\lim_{n\to\infty}\frac{\fn{P_\mW^\perp \mM P_\mW^\perp}}{\fn{\mM }}=0.
\end{equation}
From \lemmaref{lemma:M-star} we know that for large $n$, $\lim_{n\to\infty}\fn{\mM }$ is lower bounded by some constant that is independent of $n$, so
\begin{equation}
    \lim_{n\to\infty}\fn{P_\mW\mM P_\mW^\perp} = \lim_{n\to\infty}\fn{P_\mW^\perp \mM P_\mW} = \lim_{n\to\infty}\fn{P_\mW^\perp \mM P_\mW^\perp} = 0.
\end{equation}
Note that
\begin{equation}
    \mM  = P_\mW\mM P_\mW + P_\mW\mM P_\mW^\perp + P_\mW^\perp \mM P_\mW + P_\mW^\perp \mM P_\mW^\perp.
\end{equation}
Thus,
\begin{equation}
    \lim_{n\to\infty}\fn{\mM -P_\mW\mM P_\mW}=0.
\end{equation}
For any $\eps>0$, set $\delta<\min\{\frac{\eps\eta}{8c^2},\frac{\sqrt{\eps\eta}}{2c}\}$, from \lemmaref{lemma:W-projection}, we know that with probability 1, $\fn{P_\mW-\mW^\T\mW}\leq\delta$. Therefore,
\begin{equation}\begin{split}
    &\fn{P_\mW\mM P_\mW-\mW^\T\mW\mM \mW^\T\mW}\\ \leq& \fns{P_\mW-\mW^\T\mW}\fn{\mM } + 2\fn{P_\mW-\mW^\T\mW}\fn{\mM }\fn{P_\mW}\\
                           \leq& \delta^2\cdot 2c^2 + 2\delta\cdot 2c^2\\
                           <& \eps.
\end{split}\end{equation}
In other words, 
\begin{equation}
    \lim_{n\to\infty}\fn{P_\mW\mM P_\mW-\mW^\T\mW\mM \mW^\T\mW}=0.
\end{equation}
Now we conclude that
\begin{equation}
    \lim_{n\to\infty}\fn{\mM -\mW^\T\mW\mM \mW^\T\mW} = 0.
\end{equation}
From \lemmaref{lemma:WMW-equivalence} we know that 
\begin{equation}
    \lim_{n\to\infty}\fn{\mW\mM \mW^\T-\mW\mM^*\mW^\T}=0.
\end{equation}
Since
\begin{equation}
    \fn{\mW^\T\mW\mM \mW^\T\mW-\mW^\T\mW\mM^*\mW^\T\mW}\leq \fns{\mW}\fn{\mW\mM \mW^\T-\mW\mM^*\mW^\T},
\end{equation}
from \lemmaref{lemma:W-norm} which bounds the Frobenius norm of $\mW$ we know that
\begin{equation}
    \lim_{n\to\infty}\fn{\mW^\T\mW\mM \mW^\T\mW-\mW^\T\mW\mM^*\mW^\T\mW} = 0.
\end{equation}
Thus,
\begin{equation}
\label{eqn:M-equal-WTWMWTW}
    \lim_{n\to\infty}\fn{\mM-\mW^\T\mW\mM^*\mW^\T\mW} = 0.
\end{equation}
Note that $\mM^*=\frac14\left(\E[\mW^\T\rmA\mW]+\text{diag}(\E[\mW^\T\rmA\mW])\right)$, so
\begin{equation}
\label{eqn:M-approx-complex}
    4\mW^\T\mW\mM^*\mW^\T\mW = \mW^\T\mW\mW^\T\tilde{\rmA}\mW\mW^\T\mW + \mW^\T\mW\text{diag}(\E[\mW^\T\rmA\mW])\mW^\T\mW.
\end{equation}
We will first analyze the second term on the RHS of equation \equationref{eqn:M-approx-complex}. For all $\eps>0$, set $\eps'=\frac{\eps}{\sqrt{c}}$, and from \lemmaref{lemma:WW-identity} we know that $\norm{\mW\mW^\T-\mI_c}<\eps'$ with probability 1, which means that $|\fn{\mW\mW^\T}-c|<\eps$ with probability 1. Set $\eps=c$, we know that $\fn{\mW\mW^\T}<2c$ with probability 1. Note that
\begin{equation}\begin{split}
    \fn{\mW^\T\mW\text{diag}(\E[\mW^\T\rmA\mW])\mW^\T\mW} &\leq\fns{\mW^\T\mW}\fn{\text{diag}(\E[\mW^\T\rmA\mW])}\\
                                        &=\fns{\mW\mW^\T}\fn{\text{diag}(\E[\mW^\T\rmA\mW])}\\
                                        &\leq 4c^2\fn{\text{diag}(\E[\mW^\T\rmA\mW])}.
\end{split}\end{equation}
Combine this with equation \equationref{lemma:W-diag-neg} and we have
\begin{equation}
    \lim_{n\to\infty}\frac{\fn{\mW^\T\mW\text{diag}(\E[\mW^\T\rmA\mW])\mW^\T\mW}}{\fn{\mW^\T\tilde{\rmA}\mW}} = 0.
\end{equation}
From \lemmaref{lemma:M-star} we know that $\fn{\mW^\T\tilde{\rmA}\mW}\geq\frac{\eta}{4}$ with probability 1, so
\begin{equation}
\label{eqn:WTWMWTW-equal-WTWWTAWWTW}
    \lim_{n\to\infty}\fn{4\mW^\T\mW\mM^*\mW^\T\mW - \mW^\T\mW\mW^\T\tilde{\rmA}\mW\mW^\T\mW} = 0.
\end{equation}
Similarly, define $\delta\triangleq \mW\mW^\T-\mI_c$, then
\begin{equation}\begin{split}
    &\fn{\mW^\T\mW\mW^\T\tilde{\rmA}\mW\mW^\T\mW - \mW^\T\tilde{\rmA}\mW}\\
\leq&\fn{\mW^\T\delta\tilde{\rmA}\delta \mW}+2\fn{\mW^\T\tilde{\rmA}\delta}\\
\leq&\fns{\mW}\fns{\delta}\fn{\tilde{\rmA}} + 2\fn{\mW}\fn{\delta}\fn{\tilde{\rmA}}.
\end{split}\end{equation}
Set $\eps'<\min\{\frac{\eps}{8c^2},\sqrt{\frac{\eps}{8c^3}}\}$, then from \lemmaref{lemma:WW-identity} we know that $\fn{\delta}<\eps'$ with probability 1, and from \lemmaref{lemma:W-norm} we have $\fn{\mW}\leq 2c$ with probability 1. We also have $\fn{\tilde{\rmA}}\leq c$ since each entry of $\rmA$ is bounded by 1 in absolute value.
Therefore,
\begin{equation}
\fn{\mW^\T\mW\mW^\T\tilde{\rmA}\mW\mW^\T\mW - \mW^\T\tilde{\rmA}\mW} \leq 4c^2(\eps')^2\cdot c + 2\cdot 2c\eps'\cdot c < \frac{\eps}{2} + \frac{\eps}{2} = \eps,
\end{equation}
which means that
\begin{equation}
\label{eqn:WTWWTAWWTW-equal-WTAW}
    \lim_{n\to\infty}\fn{\mW^\T\mW\mW^\T\tilde{\rmA}\mW\mW^\T\mW - \mW^\T\tilde{\rmA}\mW} = 0.
\end{equation}
From \equationref{eqn:WTWMWTW-equal-WTWWTAWWTW} and \equationref{eqn:WTWWTAWWTW-equal-WTAW} we get
\begin{equation}
\label{eqn:WTWMWTW-equal-WTAW}
    \lim_{n\to\infty}\fn{\frac14\mW^\T\tilde{\rmA}\mW - \mW^\T\mW\mM^*\mW^\T\mW} = 0.
\end{equation}
Combining with \equationref{eqn:M-equal-WTWMWTW} we have
\begin{equation}
    \lim_{n\to\infty}\fn{\mM - \frac14\mW^\T\tilde{\rmA}\mW} = 0.
\end{equation}
Besides, from equation \equationref{eqn:W-equal-W-bar} in \lemmaref{lemma:W-projection} we know that for any $\eps'>0$,
\begin{equation}
    \fns{\bmW-\mW} = \sum_{i\in[c]} \ns{\bmW_i-\mW_i}<\eps',
\end{equation}
where $\bmW$ is the orthogonal version of $\mW$, i.e., we run the Gram-Schmidt process for the rows of $\mW$. Define $\delta\triangleq \bmW-\mW$, for any $\eps>0$, set $\eps'=\min\{\frac{\eps}{8c^2},\sqrt{\frac{\eps}{2c}}\}$, we have with probability 1,
\begin{equation}\begin{split}
    \fn{\mW^\T\tilde{\rmA}\mW - \bmW^\T\tilde{\rmA}\bmW} &\leq 2\fn{\delta}\fn{\tilde{\rmA}}\fn{\mW} + \fns{\delta}\fn{\tilde{\rmA}}\\&\leq 4c^2\eps'+c(\eps')^2 < \eps.
\end{split}\end{equation}
Therefore,
\begin{equation}
    \lim_{n\to\infty}\fn{\mW^\T\tilde{\rmA}\mW - \bmW^\T\tilde{\rmA}\bmW} = 0,
\end{equation}
which implies
\begin{equation}
    \lim_{n\to\infty}\fn{\mM - \frac14\bmW^\T\tilde{\rmA}\bmW} = 0.
\end{equation}
From \lemmaref{lemma:A-rank-c-1} we know that with probability 1, $\tilde{\rmA}$ is of rank $(c-1)$. Since $\rmA\cdot\textbf{1}=0$ is always true, the top $(c-1)$ eigenspace of $\tilde{\rmA}$ is $\R^c\backslash\{\textbf{1}\}$. Note that the rows in $\bmW$ are of unit norm and orthogonal to each other, we conclude that $\bmW^\T\tilde{\rmA}\bmW$ is of rank $(c-1)$ and the corresponding eigenspace is $\gR\{\bmW_i\}_{i=1}^c\backslash\{\textbf{1}^\T\bmW\}$. Moreover, the minimum positive eigenvalue of $\bmW^\T\tilde{\rmA}\bmW$ is lower bounded by $\frac{\eta}{4}$.

As for the top $c-1$ eigenvectors of $\mM$, define $\delta\triangleq\mM-\frac14\bmW^\T\tilde{\rmA}\bmW$, then $\mM = \frac14\bmW^\T\tilde{\rmA}\bmW + \delta$. Define $S_1$ as the top $c-1$ eigenspaces for $\mM$, and $S_2$ to be the top $c-1$ eigenspaces for $\frac14\bmW^\T\tilde{\rmA}\bmW$. Then from Davis-Kahan Theorem we know that
\begin{equation}
    \fn{\sin\Theta(S_1,S_2)}\leq\frac{\fn{\delta}}{\lambda_{c-1}(\frac14\bmW^\T\tilde{\rmA}\bmW)}.
\end{equation}
Here $\Theta(S_1,S_2)$ is a $(c-1)\times(c-1)$ diagonal matrix whose $i$-th diagonal entry is the $i$-th canonical angle between $S_1$ and $S_2$. Since $\lim_{n\to\infty}\fn{\delta}=0$, and with probability 1, $\lambda_{c-1}(\frac14\bmW^\T\tilde{\rmA}\bmW)\geq\eta$ which is independent of $n$, we have with probability 1,
\begin{equation}
    \lim_{n\to\infty}\fn{\sin\Theta(S_1,S_2)} = 0,
\end{equation}
which indicates that the top $c-1$ eigenspaces for $\mM$ and $\frac14\bmW^\T\tilde{\rmA}\bmW$ are the same when $n\to\infty$.

Here we note that the top $c-1$ eigenspace of $\bmW^\T\tilde{\rmA}\bmW$ is $\gR\{\bmW_i\}_{i=1}^c\backslash\{\textbf{1}^\T\bmW\}$ since $\rmA$ has its null space spanned by the all-one vector, so $\mM$ will also have the same top $c-1$ eigenspaces. Besides, from equation \lemmaref{eqn:W-equal-W-bar} we know that $\lim_{n\to\infty}\fn{\mW-\bmW}=0$, so $\gR\{\bmW_i\}_{i=1}^c\backslash\{\textbf{1}^\T\bmW\}$ are the same as $\gR\{\mW_i\}_{i=1}^c\backslash\{\textbf{1}^\T\mW\}$. This completes the proof of this theorem.

%% file: MainProof/full-hessian-structure.tex
\subsubsection{Structure of Full Hessian of the First Layer}
\label{sec:proof-full-hessian}
We first restate \theoremref{thm:main-full} here:

\noindent\textbf{\theoremref{thm:main-full}:}\emph{
Let $V_1$ and $V_2$ be the top $c-1$ eigenspaces of $\mH$ and $\hat\mH$ respectively, for all $\eps>0$, 
\begin{equation}
    \lim_{n\to\infty}\mathop{\Pr}_{\mW^{(1)}\sim\gN(0,\frac{1}{d}\mI_{nd}), \mW^{(2)}\sim\gN(0,\frac{1}{n}\mI_{cn})}\left[\Overlap\left(V_1,V_2\right)>1-\eps\right] = 1.
\end{equation}
Moreover, \begin{equation}
    \lim_{n\to\infty}\mathop{\Pr}_{\mW^{(1)}\sim\gN(0,\frac{1}{d}\mI_{nd}), \mW^{(2)}\sim\gN(0,\frac{1}{n}\mI_{cn})}\left[\left(\left.\frac{\lambda_c(\mH)}{\lambda_{c-1}(\mH)}\right|_{\mW^{(1)}, \mW^{(2)}}\right) < \eps\right] = 1.
\end{equation}
}

Before proceeding to the main theorem, we will first look into the eigenspectrum of the scaled auto-correlation matrix $\trmX\triangleq \frac1d\ex{\rvx\rvx^\T}$ and the top eigenspace of $\hat\mH$. Also recall some useful notations including $\mU=\frac{1}{\sqrt{d}}\1_d^\T$ and $\mV\triangleq \mW\otimes\mU$.

\begin{lemma}
\label{lemma:X-structure}
$\lambda_1(\trmX)=\frac{1}{2\pi}+\frac{\pi - 1}{2\pi d}$ with eigenvector $\frac{1}{\sqrt{d}}\1_d$. $\lambda_2(\trmX)=\dots=\lambda_d(\trmX)=\frac{\pi - 1}{2\pi d}.$
\end{lemma}

\begin{proofof}{\lemmaref{lemma:X-structure}}
From \equationref{eqn:xxT-structure} we know that \begin{equation}
    \trmX=\frac1d\ex{\rvx\rvx^\T}=\frac{1}{2\pi d}\1_d\1_d^\T+\frac{\pi - 1}{2\pi d}\mI_d.
\end{equation}
For unit vector $\vv=\frac{1}{\sqrt{d}}\1_d$, it satisfies
\begin{equation}
\begin{split}
    \trmX\vv = \frac{1}{2\pi d}\1_d\1_d^\T\frac{1}{\sqrt{d}}\1_d +\frac{\pi - 1}{2\pi d}\mI_d \frac{1}{\sqrt{d}}\1_d = \rbr{\frac{1}{2\pi}+\frac{\pi - 1}{2\pi d}} \frac{1}{\sqrt{d}}\1_d.
\end{split}
\end{equation}
Hence the all one vector has eigenvalue $\frac{1}{2\pi}+\frac{\pi - 1}{2\pi d}$.
For any unit vector $\vv\perp \1_d$, it satisfies
\begin{equation}
\begin{split}
    \trmX\vv = \frac{1}{2\pi d}\1_d\1_d^\T \vv +\frac{\pi - 1}{2\pi d}\mI_d\vv = \frac{\pi - 1}{2\pi d}\vv.
\end{split}
\end{equation}
Which means $\lambda_2=\lambda_3=\dots=\lambda_d=\frac{\pi - 1}{2\pi d}$.
\end{proofof}

\begin{corollary}
\label{cor:H-hat-eigenspace}
With probability 1 over $\mW^{(1)}$ and $\mW^{(2)}$, the overlap between the top $c-1$ eigenspace of $\hat\mH$ and $\gR\{\mV_i\}_{i=1}^c\backslash\{\mV\cdot\emph{\textbf{1}}\}$ converges to 1 as $n\to\infty$.
\end{corollary}
\begin{proofof}{\corollaryref{cor:H-hat-eigenspace}}
First note that by simple linear algebra,
\begin{equation}
\gR\{\mV_i\}_{i=1}^c\backslash\{\textbf{1}^\T\mV\}=(\gR\{\mW_i\}_{i=1}^c\backslash\{\mW\cdot\1\})\otimes \mU.
\end{equation}
While from \theoremref{thm:main-out} we know the overlap between the top $c-1$ eigenspace of $\mM$ and $\gR\{\mW_i\}_{i=1}^c\backslash\{\textbf{1}^\T\mW\}$ converges to 1. Thus for proving this corollary it is sufficient to show that the top $c-1$ eigenspace of $\frac1d\hat\mH = \mM\otimes \trmX$ is the Kronecker product of the top $c-1$ eigenspace of $\mM$ and $\mU$.

From \theoremref{thm:main-out} we know, with probability 1 over $\mW^{(1)}$ and $\mW^{(2)}$, for large $n$, $\lambda_{c-1}(\mM)>\eta/4$ and $\lambda_1(\mM)<2c^2$ where $\eta$ and $c$ are absolute constants. Thus for large $n$ we have \begin{equation}
\begin{split}
\lim_{n\to\infty}\lambda_1(\trmX)\lambda_{c-1}(\mM) &= \lim_{d\to\infty}\rbr{\frac{1}{2\pi}+\frac{\pi-1}{2\pi d}}\lambda_{c-1}(\mM)\geq \frac{1}{2\pi}\frac{\eta}{4}.
\end{split}
\end{equation}
while
\begin{equation}
\begin{split}
\lim_{n\to\infty}\lambda_2(\trmX)\lambda_1(\mM) &= \lim_{d\to\infty}\frac{\pi-1}{2\pi d}\lambda_{c-1}(\mM)\leq \lim_{d\to\infty}\frac{\pi-1}{2\pi d}2c^2=0.
\end{split}
\end{equation}
Since for large $n$, $\lambda_1(\trmX)\lambda_{c-1}(\mM) > \lambda_2(\trmX)\lambda_1(\mM)$, the top $c-1$ eigenspace of $\frac1d\hat\mH$ is the top $c-1$ eigenspace of $\mM$ Kronecker with the first eigenvector of $\trmX$, which is exactly $\mU$ from \lemmaref{lemma:X-structure}. This completes the proof of this corollary.

\end{proofof}

Now we proceed to prove the main theorem

\begin{proofof}{\theoremref{thm:main-full}}
We will conduct the proof on $\tmH=\frac{1}{d}\mH$ as the properties to be proved are invariant to scalar multiplication. From \corollaryref{cor:H-hat-eigenspace} we know the overlap between the top $c-1$ eigenspace of $\hat\mH$ and $\gR\{\mV_i\}_{i=1}^c\backslash\{\textbf{1}^\T\mV\}$ converges to 1. Thus we only need to show the overlap between the top $c-1$ eigenspace of $\tmH$ and $\gR\{\mV_i\}_{i=1}^c\backslash\{\textbf{1}^\T\mV\}$ converges to 1.

The proof strategy for the full layerwise Hessian is exactly the same as the proof for the output Hessian in \sectionref{sec:proof-A-structure}. In particular, the proof is nearly identical when we change the projection matrix from $P_\mW$ to $P_\mV$ where $\mV\triangleq \mW\otimes\mU$.

Therefore, instead of rewriting the entire proof, we may neglect some repeating arguments by verifying the equivalent lemmas for the full layer-wise Hessian.
With $\mV$ as defined, we have $\fn{\mV} = \fn{\mW}$, $\mV\mV^\T = \mW\mW^\T$, and $\fn{\mV^\T\mV} = \fn{\mW^\T\mW}$, so we can directly apply the exact same result of the two norm bounds (\lemmaref{lemma:W-norm}, \lemmaref{cor:WW-Identity}) on $\mV$. Now we prove \lemmaref{lemma:V-projection} as the equivalent of \lemmaref{lemma:W-projection}.

\begin{lemma}
\label{lemma:V-projection}
Let $\bmV\triangleq \bmW\otimes \mU$, then $P_\mV\triangleq \bmV^\T\bmV$ is the projection matrix from $\R^{nd}$ onto the subspace spanned by all rows of $\mV=\mW\otimes \mU$. Moreover, for all $\eps>0$, \begin{equation}
    \lim_{n\to\infty}\pr{\fns{\mV^\T\mV-P_\mV}>\eps} = 0.
\end{equation}
\end{lemma}

\begin{proofof}{\lemmaref{lemma:V-projection}}
Since Kronecker product with the constant $1\times d$ matrix $\mU$ preserves the orthogonality of vectors, doing Gram-Schmit on $\mV$ is equivalent to doing Gram-Schmit on $\mW$ then Kronecker with $\mU$, which results in $\bmV$ by construction. Therefore $P_\mV$ is a valid projection matrix.

Also note that for any $\mW$, \begin{equation}
\begin{split}
\fns{\mV^\T\mV-P_\mV}&=\fns{(\mW\otimes\mU)^\T(\mW\otimes\mU)-\bmV^\T\bmV}\\
&=\fns{(\mW^\T\mW)\otimes(\mU^\T\mU)-(\bmW^\T\bmW)\otimes(\mU^\T\mU)}\\
&=\fns{\mW^\T\mW-\bmW^\T\bmW}\fns{\mU^\T\mU}\\
&=\fns{\mW^\T\mW-\bmW^\T\bmW}\fns{\frac1d\1_d\1_d^\T}\\
&=\fns{\mW^\T\mW-P_\mW}.
\end{split}
\end{equation}
From \lemmaref{lemma:W-projection} we have
\begin{equation}
    \lim_{n\to\infty}\pr{\fns{\mV^\T\mV-P_\mV}>\eps} = \lim_{n\to\infty}\pr{\fns{\mW^\T\mW-P_\mW}>\eps}= 0.
\end{equation}
\end{proofof}

Note that the equivalent lemmas of \lemmaref{lemma:M-proj-preserve-f-norm}- \lemmaref{lemma:F-norm-equal} for $\tmH$ are also established in \lemmaref{lemma:H-proj-preserve-f-norm} - \lemmaref{lemma:F-norm-equal-H} in \sectionref{sec:proof-project}, substituting $(P_\mW, \mM, \mM^*)$ by $(P_\mV, \tmH, \tmH^*)$, we may follow the argument in \sectionref{sec:proof-out-hessian} up to \equationref{eqn:M-equal-WTWMWTW} and conclude that
\begin{equation}
\label{eqn:H-equal-VTVHVTV}
    \lim_{n\to\infty}\fn{\tmH-\mV^\T\mV\tmH^*\mV^\T\mV} = 0.
\end{equation}
Now claim an equivalent argument of \equationref{eqn:WTWMWTW-equal-WTAW}, that
\begin{equation}
\label{eqn:VTVHVTV-equal-VTAV}
    \lim_{n\to\infty}\fn{\frac{1}{8\pi}\mV^\T\tilde{\rmA}\mV - \mV^\T\mV\tmH^*\mV^\T\mV} = 0.
\end{equation}
Observe that 
\begin{equation}
\begin{split}
\mV^\T\trmA\mV
=(\mW\otimes\mU)^\T\trmA(\mW\otimes\mU) =  \mW^\T\trmA\mW\otimes\mU^\T\mU
\end{split}
\end{equation}
and
\begin{equation}
\begin{split}
\mV^\T\mV\tmH^*\mV^\T\mV &= (\mW\otimes\mU)^\T(\mW\otimes\mU)(\mM^*\otimes\trmX)(\mW\otimes\mU)^\T(\mW\otimes\mU)\\
&=\mW^\T\mW\mM^*\mW^\T\mW\otimes\mU^\T\mU\trmX\mU^\T\mU.
\end{split}
\end{equation}
We have\begin{equation}
\label{eqn:proof-VTVHVTV-equal-VTAV}
\begin{split}
    & \fn{\frac{1}{8\pi}\mV^\T\tilde{\rmA}\mV - \mV^\T\mV\tmH^*\mV^\T\mV}\\
    =\ &\fn{\frac{1}{8\pi}\mW^\T\trmA\mW\otimes\mU^\T\mU - \mW^\T\mW\mM^*\mW^\T\mW\otimes\mU^\T\mU\trmX\mU^\T\mU}\\
    \leq\ &\fn{\frac{1}{8\pi}\mW^\T\trmA\mW\otimes\mU^\T\mU - \frac{1}{2\pi}\mW^\T\mW\mM^*\mW^\T\mW\otimes\mU^\T\mU}\\
    &\ +\fn{\frac{1}{2\pi}\mW^\T\mW\mM^*\mW^\T\mW\otimes\mU^\T\mU - \mW^\T\mW\mM^*\mW^\T\mW\otimes\mU^\T\mU\trmX\mU^\T\mU}\\
    =\ &\frac{1}{2\pi}\fn{\frac14\mW^\T\trmA\mW-\mW^\T\mW\mM^*\mW^\T\mW}\fn{\mU^\T\mU\trmX\mU^\T\mU}\\
    &\ +\fn{\mW^\T\mW\mM^*\mW^\T\mW}\fn{\frac{1}{2\pi}\mU^\T\mU - \mU^\T\mU\trmX\mU^\T\mU}.
\end{split}
\end{equation}
Let's first consider the second term. Note that from \lemmaref{lemma:X-structure}, \begin{equation}
\begin{split}
    \fn{\frac{1}{2\pi}\mU^\T\mU - \mU^\T\mU\trmX\mU^\T\mU} &= \lfn{\frac{1}{2\pi d}\1_d\1_d^\T - \frac{1}{d^2}\rbr{\sum_{i,j=1}^d\frac{1}{d}\ex{\rvx\rvx^\T}_{ij}}\1_d\1_d^\T}\\
    &= \labs{\frac{1}{2\pi d} - \frac{1}{d^3}\rbr{\frac{2\pi d^2}{2\pi}+\frac{(\pi-1)d}{2\pi}}}\lfn{\1_d\1_d^\T}\\
    &= \frac{\pi-1}{2\pi d^2}d = \frac{\pi-1}{2\pi d}.
\end{split}
\end{equation}
Which converges to $0$ as $n\to\infty$ (since $d = n^{1+\alpha}$).
Since $\fn{\mW^\T\mW\mM^*\mW^\T\mW}$ is bounded above from \lemmaref{lemma:W-norm} and \lemmaref{lemma:M-star}. We have 
\begin{equation}
\label{eqn:VTVHVTV-equal-VTAV-sub1}
    \lim_{n\to\infty} \fn{\mW^\T\mW\mM^*\mW^\T\mW}\fn{\frac{1}{2\pi}\mU^\T\mU - \mU^\T\mU\trmX\mU^\T\mU} = 0.
\end{equation}
For the first term, since for all $d$,
\begin{equation}
\begin{split}
\fn{\mU^\T\mU\trmX\mU^\T\mU} &= \lfn{\frac{1}{d^2}\rbr{\sum_{i,j=1}^d\frac{1}{d}\ex{\rvx\rvx^\T}_{ij}}\1_d\1_d^\T}\\
&= \frac{1}{d^2}\rbr{\frac{2\pi d^2}{2\pi}+\frac{(\pi-1)d}{2\pi}} < \frac{1}{2},
\end{split}
\end{equation}
Combined with \equationref{eqn:WTWMWTW-equal-WTAW} we have\begin{equation}
\label{eqn:VTVHVTV-equal-VTAV-sub2}
    \lim_{n\to\infty}\frac{1}{2\pi}\fn{\frac14\mW^\T\trmA\mW-\mW^\T\mW\mM^*\mW^\T\mW}\fn{\mU^\T\mU\trmX\mU^\T\mU} = 0.
\end{equation}
Plug \equationref{eqn:VTVHVTV-equal-VTAV-sub1} and \equationref{eqn:VTVHVTV-equal-VTAV-sub2} into \equationref{eqn:proof-VTVHVTV-equal-VTAV} gives us \equationref{eqn:VTVHVTV-equal-VTAV}.

Now substitute $\frac{1}{4}\mW^\T\mA\mW$ in \sectionref{sec:proof-out-hessian} to $\frac{1}{8\pi}\mV^\T\mA\mV$, following the arguments after \equationref{eqn:WTWMWTW-equal-WTAW} completes the remaining proof for this theorem.
\end{proofof}

%% file: Appendix_Sections/full_hessian.tex
\section{Structure of Dominating Eigenvectors of the Full Hessian.}

% \znote{should we move thiis section to an earlier place?}
\label{sec:appendix_full_hessian}
Although it is not possible to apply Kronecker factorization to the full Hessian directly, we can construct an approximation of the top eigenvectors and eigenspace using similar ideas and our findings.In this section, we will always have superscript $(p)$ for all layer-wise matrices and vectors in order to distinguish them from the full versions.
As shown in \equationref{eqn:app_full_hessian} of \sectionref{sec:appendix_derivation}, we have the full Hessian of fully connected networks as 
\begin{equation}
    \HessL(\vtheta) = \E \left[\mF^\T_{\vx}\mA_\vx\mF_{\vx}\right] + \E\left[\sum_{i=1}^c \frac{\partial \ell(\vz,\vy)}{\evz_i} \nabla^2_\vtheta \evz_i \right],
\label{eqn:app_full_hessian_2}
\end{equation}
where
\begin{align}
    \mF^\T_\vx = \begin{pmatrix}
    \Gx^{(1)\T} \otimes \vx^{(1)}\\
    \Gx^{(1)\T}\\
    \Gx^{(2)\T} \otimes \vx^{(2)}\\
    \Gx^{(2)\T}\\
    \vdots\\
    \Gx^{(L)\T} \otimes \vx^{(n)}\\
    \Gx^{(L)\T}
    \end{pmatrix}.
\end{align}
In order to simplify the formula, we define \begin{equation}
    \tilde{\vx}^{(p)} = \begin{pmatrix}
    \vx^{(p)}\\
    1
    \end{pmatrix}
\end{equation} to be the extended input of the $p$-th layer. Thus, the terms in the Hessian attributed to the bias can be included in the Kronecker product with the extended input, and $\mF_\vx^\T$ can be simplified as
\begin{align}
\label{eqn:app_f_extened}
    \mF^\T_\vx = \begin{pmatrix}
    \Gx^{(1)\T} \otimes \tilde{\vx}^{(1)}\\
    \Gx^{(2)\T} \otimes \tilde{\vx}^{(2)}\\
    \vdots\\
    \Gx^{(L)\T} \otimes \tilde{\vx}^{(n)}\\
    \end{pmatrix}.
\end{align}

As discussed in several previous works \citep{sagun2016eigenvalues, papyan2018full, papyan2019measurements, fort2019emergent}, the full Hessian can be decomposed in to the G-term and the H-term. Specifically, the G-term is $\E \left[\mF^\T_{\vx}\mA_\vx\mF_{\vx}\right]$, and the H-term is $\E\left[\sum_{i=1}^c \frac{\partial \ell(\vz,\vy)}{\evz_i} \nabla^2_\vtheta \evz_i \right]$ in \equationref{eqn:app_full_hessian_2}.

Empirically, the G-term usually dominates the H-term, and the top eigenvalues and eigenspace of the Hessian are mainly attributed to the G-term. Since we focus on the top eigenspace, we can approximate our full Hessian using the G-term, as
\begin{equation}
     \HessL(\vtheta) \approx  \E\left[\mF^\T_{\vx}\mA_\vx\mF_{\vx}\right].
\end{equation}

In our approximation of the layer-wise Hessian $\HessL(\vw^{(p)})$ \equationref{eqn:decomp}, the two parts of the Kronecker factorization are the layer-wise output Hessian $\E[\mM^{(p)}_\vx]$ and the auto-correlation matrix of the input $\E[\vx^{(p)}\vx^{(p)\T}]$. Although we cannot apply Kronecker factorization to $\E\left[\mF^\T_{\vx}\mA_\vx\mF_{\vx}\right]$, we can still approximate its eigenspace using the eigenspace of the full output Hessian.

Note here that the full output Hessian is not a common definition. Let $\hat{m} = \sum_{p=1}^Lm^{(p)}$ be the sum of output dimension of each layer. We define a full output vector $\tilde{\vz}\in\R^{\hat{m}}$ by concatenating all the layerwise outputs together,
\begin{equation}
    \tilde{\vz} := \begin{pmatrix}
    \vz^{(1)}\\
    \vz^{(2)}\\
    \vdots\\
    \vz^{(L)}
    \end{pmatrix}.
\end{equation} 
We then define the full output Hessian is the Hessian w.r.t. $\tilde{\vz}$. Let the full output Hessian for a single input $\vx$ be $\mM_\vx\in\R^{\hat{m}\times \hat{m}}$. Similar to \equationref{eqn:app_layerwise_approx}, it can be expressed as
\begin{equation}
    \mM_\vx := \mH_\ell(\tilde{\vz}, \vx) = \mG_\vx^{\T}\mA_\vx\mG_\vx,
\end{equation}
where
\begin{equation}
    \mG^\T_\vx = \begin{pmatrix}
    \mG_\vx^{(1)\T}\\
    \mG_\vx^{(2)\T}\\
    \vdots\\
    \mG_\vx^{(L)\T}
    \end{pmatrix}
\end{equation}
similar to \equationref{eqn:app_f_extened}.
The full output Hessian for the entire training sample is thus
\begin{equation}
    \HessL(\tilde{\vz}) = \E[\mM_\vx] = \E[\mG_\vx^{\T}\mA_\vx\mG_\vx].
\end{equation}

We can then approximate the eigenvectors of the full Hessian $\HessL(\vtheta)$ using the eigenvectors of $\E[\mM_\vx]$. Let the $i$-th eigenvector of $\HessL(\vtheta)$ be $\vv_i$ and that of $\E[\mM_\vx]$ be $\vu_i$. We may then break up $\vu_i$ into segments corresponding to different layers as in
\begin{equation}
    \vu_i = \begin{pmatrix}
    \vu_i^{(1)}\\
    \vu_i^{(2)}\\
    \vdots\\
    \vu_i^{(L)}
    \end{pmatrix},
\end{equation}
where for all layer $p$, $\vu_i^{(p)}\in\R^{m^{(p)}}$.
Motivated by the relation between $\mG_\vx$ and $\mF_\vx$, the $i$-th eigenvector of $\HessL(\vtheta)$ can be approximated as the following. Let
\begin{equation}
    \vw_i = \begin{pmatrix}
    \vu_i^{(1)} \otimes \E[\tilde{\vx^{(1)}}]\\
    \vu_i^{(2)}\otimes \E[\tilde{\vx^{(2)}}]\\
    \vdots\\
    \vu_i^{(L)}\otimes \E[\tilde{\vx^{(L)}}]
    \end{pmatrix}.
\end{equation}
We then have
\begin{equation}
    \vv_i \approx \frac{\vw_i}{\|\vw_i\|}
\end{equation}
We can then use the Gram–Schmidt process to get the basis vectors of the approximated eigenspace.

Another reason for this approximation is that the expectation is the input of each layer $\E[\vx^{(p)}]$ dominates its covariance as shown in \sectionref{sec:appendix_xxT}. Thus, the approximate is accurate for top eigenvectors and also top eigenspace. For latter eigenvectors, the approximation would not be as accurate since this approximate loses all information in the covariance of the inputs.

We also approximated the eigenvalues using this approximation. Let the $i$-th eigenvalue of $\HessL(\vtheta)$ be $\lambda_i$ and that of $\E[\mM_\vx]$ be $\sigma_i$. We have
\begin{equation}
    \lambda_i \approx \sigma_i\|\vw_i\|^2.
\end{equation}

Below we show the approximation of the eigenvalues top eigenspace using this method. The eigenspace overlap is defined as in \definitionref{def:overlap}. We experimented on several fully connected networks, the results shown below are for F-$200^2$ (same as  \figureref{fig:eigeninfo_approx}(c)(d) in the main text), F-$200^4$, F-$600^4$, and F-$600^8$, all with dimension 50. The approximations are reasonably accurate.

\begin{figure}[H]
    \centering
    \begin{subfigure}[t]{0.5\textwidth}
        \centering
        \captionsetup{justification=centering}
        \includegraphics[width=\textwidth]{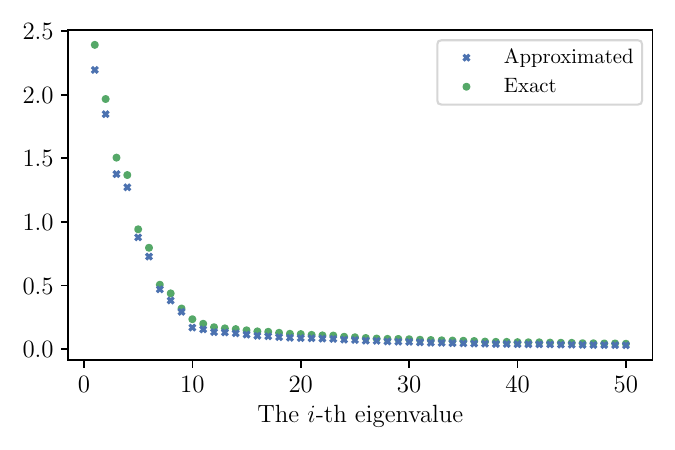}
        \caption{Eigenvalues for F-$200^2$}
    \end{subfigure}%
    \begin{subfigure}[t]{0.5\textwidth}
        \centering
        \captionsetup{justification=centering}
        \includegraphics[width=\textwidth]{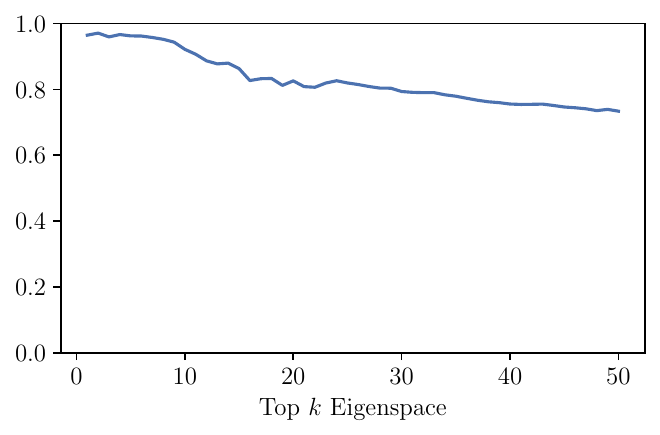}
        \caption{Eigenspace overlap for F-$200^2$}
    \end{subfigure}\\
    \begin{subfigure}[t]{0.5\textwidth}
        \centering
        \captionsetup{justification=centering}
        \includegraphics[width=\textwidth]{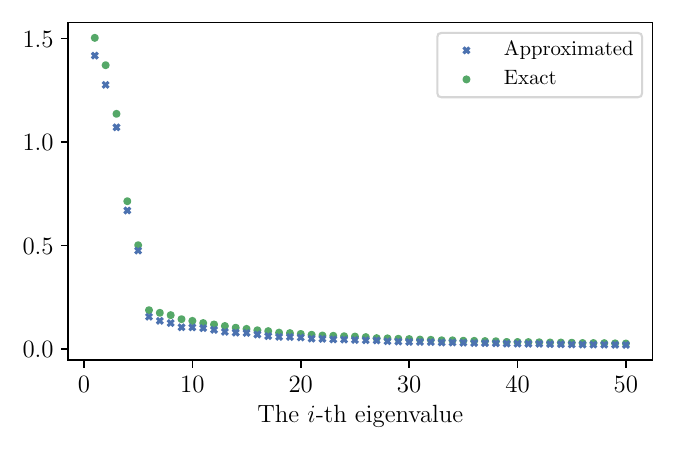}
        \caption{Eigenvalues for F-$200^4$}
    \end{subfigure}%
    \begin{subfigure}[t]{0.5\textwidth}
        \centering
        \captionsetup{justification=centering}
        \includegraphics[width=\textwidth]{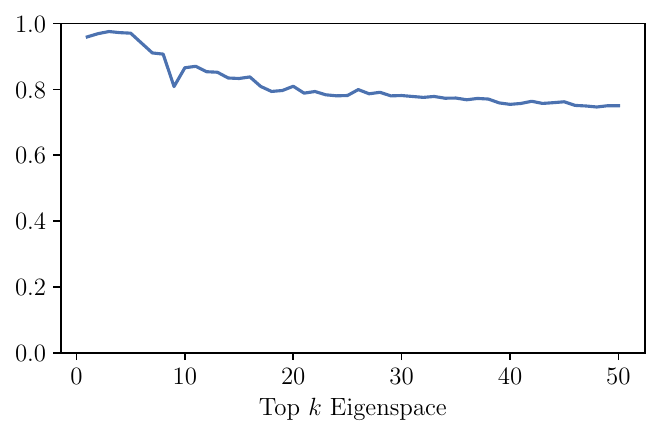}
        \caption{Eigenspace overlap for F-$200^4$}
    \end{subfigure}\\
    \begin{subfigure}[t]{0.5\textwidth}
        \centering
        \captionsetup{justification=centering}
        \includegraphics[width=\textwidth]{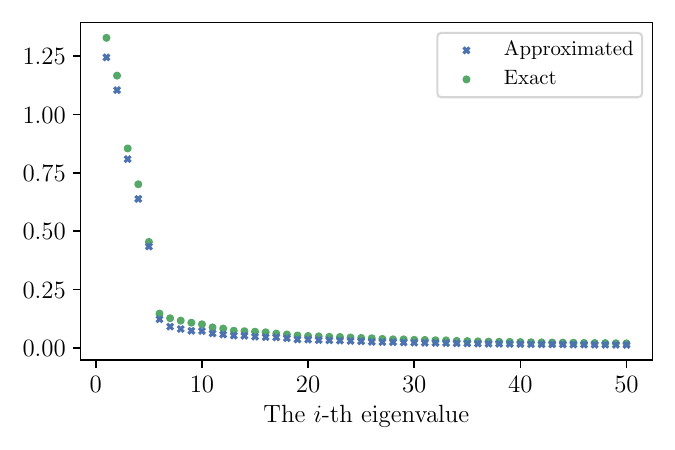}
        \caption{Eigenvalues for F-$600^4$}
    \end{subfigure}%
    \begin{subfigure}[t]{0.5\textwidth}
        \centering
        \captionsetup{justification=centering}
        \includegraphics[width=\textwidth]{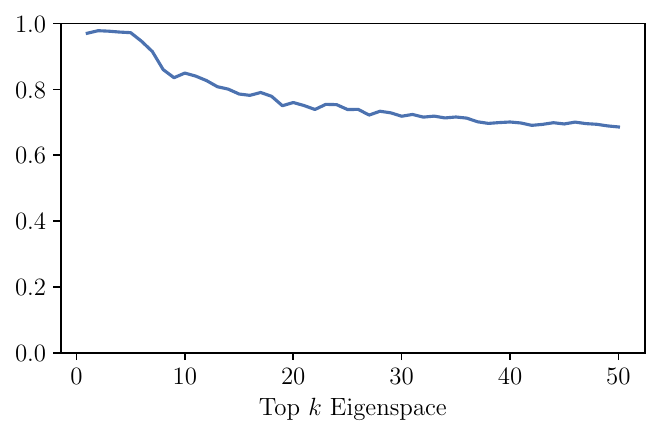}
        \caption{Eigenspace overlap for F-$600^4$}
    \end{subfigure}\\
    \begin{subfigure}[t]{0.5\textwidth}
        \centering
        \captionsetup{justification=centering}
        \includegraphics[width=\textwidth]{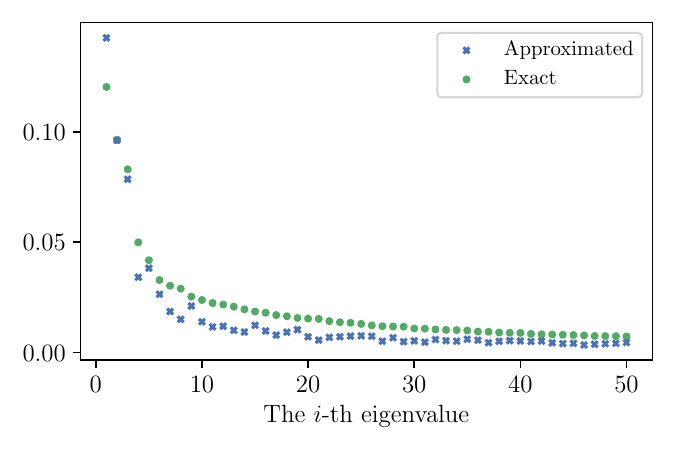}
        \caption{Eigenvalues for F-$600^8$}
    \end{subfigure}%
    \begin{subfigure}[t]{0.5\textwidth}
        \centering
        \captionsetup{justification=centering}
        \includegraphics[width=\textwidth]{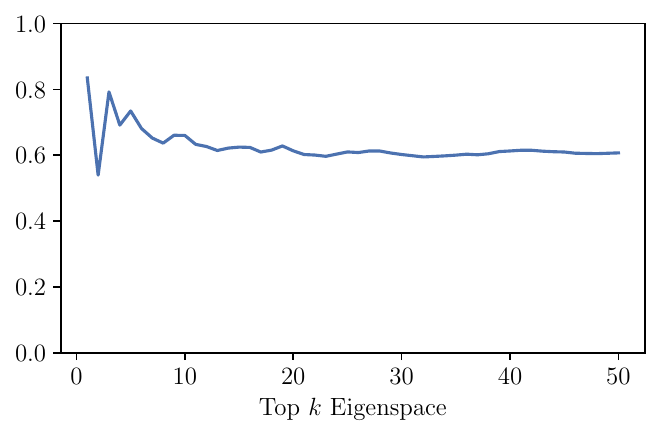}
        \caption{Eigenspace overlap for F-$600^8$}
    \end{subfigure}\\
    \caption{Top 50 Eigenvalues and Eigenspace approximation for full Hessian}
    \label{fig:Corrfc11}
\end{figure}
% \newpage
% \begin{figure}[h]
%     \centering
%     \vspace{-1em}
%     \subfigure[\small{Eigenvalues for F-$200^2$}]{\includegraphics[width=0.42\linewidth]{Appendix_Figures/Full_hessian/newplots/eigenval_fc2_200.pdf}}
%     \subfigure[\small{Eigenspace overlap for F-$200^2$}]{\includegraphics[width=0.42\linewidth]{Appendix_Figures/Full_hessian/newplots/eigenvec_fc2_200.pdf}}\\
%     \subfigure[\small{Eigenvalues for F-$200^4$}]{\includegraphics[width=0.42\linewidth]{Appendix_Figures/Full_hessian/newplots/eigenval_fc4_200.pdf}}
%     \subfigure[\small{Eigenspace overlap for F-$200^4$}]{\includegraphics[width=0.42\linewidth]{Appendix_Figures/Full_hessian/newplots/eigenvec_fc4_200.pdf}}\\
%     \subfigure[\small{Eigenvalues for F-$600^4$}]{\includegraphics[width=0.42\linewidth]{Appendix_Figures/Full_hessian/newplots/eigenval_fc4_600.pdf}}
%     \subfigure[\small{Eigenspace overlap for F-$600^4$}]{\includegraphics[width=0.42\linewidth]{Appendix_Figures/Full_hessian/newplots/eigenvec_fc4_600.pdf}}\\
%     \subfigure[\small{Eigenvalues for F-$600^8$}]{\includegraphics[width=0.42\linewidth]{Appendix_Figures/Full_hessian/newplots/eigenval_fc8_600.pdf}}
%     \subfigure[\small{Eigenspace overlap for F-$600^8$}]{\includegraphics[width=0.42\linewidth]{Appendix_Figures/Full_hessian/newplots/eigenvec_fc8_600.pdf}}\\
%     \caption{Top 50 Eigenvalues and Eigenspace approximation for full Hessian}
%     \label{fig:Corrfc11}
%     \vspace{-1em}
% \end{figure}
\newpage

%% file: Appendix_Sections/hessian_calc.tex
\section{Computation of Hessian Eigenvalues and Eigenvectors}
\label{sec:appendix_eigencomp}
For Hessian approximated using Kronecker factorization, we compute $\E[\mM]$ and $\E[\vx\vx^T]$ explicitly. Let $\vm$ and $\vv$ be an eigenvector of $\E[\mM]$ and $\E[\vx\vx^T]$ respectively, with corresponding eigenvalues $\lambda_\vm$ and $\lambda_\vv$. Since both matrices are positive semi-definite, $\vm \otimes \vv$ is an eigenvector of $\E[\mM] \otimes \E[\vx\vx^T]$ with eigenvalue $\lambda_\vm\lambda_\vv$. In this way, since $\E[\mM]$ has $m$ eigenvectors and $\E[\vx\vx^T]$ has $n$ eigenvectors, we can approximate all $mn$ eigenvectors for the layer-wise Hessian. All these calculation can be done directly.

However, it is almost prohibitive to calculate the true Hessian explicitly. Thus, we use numerical methods with automatic differentiation \citep{paszke2017automatic} to calculate them. The packages we use is \citet{hessian-eigenthings} and we use the Lanczos method in most of the calculations. We also use package in \citet{yao2019pyhessian} as a reference.

For layer-wise Hessian, we modified the \citet{hessian-eigenthings} package. In particular, the package relies on the calculation of Hessian-vector product $\mH\vv$, where $\vv$ is a vector with the same size as parameter $\theta$. To calculate eigenvalues and eigenvectors for layer-wise Hessian at the $p$-th layer, we cut the $\vv$ into different layers. Then, we only leave the part corresponding to weights of the $p$-th layer and set all other entries to 0. Note that the dimension does not change. We let the new vector be $\vv^{(p)}$ and get the value of $\vu = \mH\vv^{(p)}$ using auto differentiation. Then, we do the same operation to $\vu$ and get $\vu^{(p)}$.

%% file: Appendix_Sections/experiment_setup.tex
\newpage
\section{Detailed Experiment Setup}
\label{sec:appendix_exp_setup}
\subsection{Datasets}
\label{sec:appendix_exp_dataset}
We conduct experiment on CIFAR-10, CIFAR-100 (MIT) \citep{Krizhevsky09learningmultiple} (\url{https://www.cs.toronto.edu/~kriz/cifar.html}), and MNIST (CC BY-SA 3.0) \citep{lecun1998gradient} (\url{http://yann.lecun.com/exdb/mnist/}). The datasets are downloaded through torchvision \citep{NEURIPS2019_9015} (\url{https://pytorch.org/vision/stable/index.html}). We used their default splitting of training and testing set.

To compare our work on PAC-Bayes bound with the work of \citet{dziugaite2017computing}, we created a custom dataset MNIST-2 by setting the label of images 0-4 to 0 and 5-9 to 1.
We also created random-labeled datasets MNIST-R and CIFAR10-R by randomly labeling the images from the training set of MNIST and CIFAR10.
The dataset information is summarized in \tableref{tab:appendix_dataset}
\begin{table}[h]
\small
  \centering
  \caption{Datasets}
  \vskip 0.1in
    \begin{center}

    \begin{tabular}{lccccc}
    \toprule
    &   \multicolumn{2}{c}{\# Data Points}    &    & &     \\
    Dataset & Train & Test & Input Size & \# Classes & Label \\
    \midrule
    CIFAR10 & 50000 & 10000 & $3\times32\times32$ & 10 & True \\
    CIFAR10-R & 50000 & 10000 & $3\times32\times32$ & 10 & Random \\
    CIFAR100 & 50000 & 10000 & $3\times32\times32$ & 100 & True \\
    MNIST & 60000 & 10000 & $28\times28$ & 10 & True \\
    MNIST-2 & 60000 & 10000 & $28\times28$ & 2 & True \\
    MNIST-R & 60000 & 10000 & $28\times28$ & 10 & Random \\\bottomrule
    \end{tabular}
\end{center}
  \label{tab:appendix_dataset}%
\end{table}%

All the datasets (MNIST, CIFAR-10, and CIFAR-100) we used are publicly available. According to their descriptions on the contents and collection methods, they should not contain any personal information or offensive content. MNIST is a remix of datasets from the National Institute of Standards and Technology (NIST), which obtained consent for collecting the data. However, we also note that CIFAR-10 and CIFAR-100 are subsets of the dataset 80 Million Tiny Image \citep{torralba2007tiny} (\url{http://groups.csail.mit.edu/vision/TinyImages/}), which used automatic collection and includes some offensive images.

\subsection{Network Structures}
\label{appendix_exp_nn}
\paragraph{Fully Connected Network:}
We used several different fully connected networks varying in the number of hidden layers and the number of neurons for each hidden layer. The output of all layers except the last layer are passed into ReLU before feeding into the subsequent layer.  As described in \sectionref{subsec:approx}, we denote a fully connected network with $m$ hidden layers and $n$ neurons each hidden layer by F-$n^m$. For networks without uniform layer width, we denote them by a sequence of numbers (e.g. for a network with three hidden layers, where the first two layers has 200 neurons each and the third has 100 neurons, we denote it as F-$200^2$-$100$). For example, the structure of F-$200^2$ is shown in \tableref{tab:appendix_fc_struct}.

\begin{table}[h]
\small
  \centering
  \caption{Structure of F-$200^2$ on MNIST}
  \vskip 0.1in
    \begin{center}
    \begin{tabular}{rllcc}
    \toprule
    \# & Name & Module & In Shape & Out Shape\\
    \midrule
    1 & & Flatten & (28,28) & 784\\
    2 & fc1 & Linear(784, 200) & 784 & 200\\
    3 & & ReLU & 200 & 200\\
    4 & fc2 & Linear(200, 200) & 200 & 200\\
    5 & & ReLU & 200 & 200\\
    6 & fc3 & Linear(200, 10) & 200 & 10\\
    \multicolumn{5}{c}{\emph{output}}\\\bottomrule
    \end{tabular}%
\end{center}

  \label{tab:appendix_fc_struct}%
\end{table}%

\paragraph{LeNet5:} We adopted the LeNet5 structure proposed by \citet{lecun1998gradient} for MNIST, and slightly modified the input convolutional layers to adapt the input of CIFAR-10 dataset. The standard LeNet5 structure we used in the experiments is shown in \tableref{tab:appendix_lenet_struct}. We further modified the dimension of fc1 and conv2 to create several variants for the experiment in \sectionref{sec:models}. Take the model whose first fully connected layer is adjusted to have 80 neurons as an example, we denote it as LeNet5-(fc1-80).

\begin{table}[h]
\small
  \centering
  \caption{Structure of LeNet5 on CIFAR-10}
  \vskip 0.1in
    \begin{center}
    \begin{tabular}{rllcc}
    \toprule
    \# & Name & Module & In Shape & Out Shape\\\midrule
    1 & conv1 & Conv2D(3, 6, 5, 5) & (3, 32, 32) & (6, 28, 28)\\
    2 & & ReLU & (6, 28, 28) & (6, 28, 28)\\
    3 & maxpool1 & MaxPooling2D(2,2) & (6, 28, 28) & (6, 14, 14)\\
    4 & conv2 & Conv2D(6, 16, 5, 5) & (6, 14, 14) & (16, 10, 10)\\
    5 & & ReLU & (16, 10, 10) & (16, 10, 10)\\
    6 & maxpool2 & MaxPooling2D(2,2) & (16, 10, 10) & (16, 5, 5)\\
    7 & & Flatten & (16, 5, 5) & 400\\
    8 & fc1 & Linear(400, 120) & 400 & 120\\
    9 & & ReLU & 120 & 120\\
    10 & fc2 & Linear(120, 84) & 120 & 84\\
    11 & & ReLU & 84 & 84\\
    12 & fc3 & Linear(84, 10) & 84 & 10\\
    \multicolumn{5}{c}{\emph{output}} \\\bottomrule
    \end{tabular}%
\end{center}
  \label{tab:appendix_lenet_struct}%
\end{table}%

\paragraph{Networks with Batch Normalization:} In \sectionref{sec:appendix_batchnorm} we conducted several experiments regarding the effect of batch normalization on our results. For those experiments, we use the existing structures and add batch normalization layer for each intermediate output after it passes the ReLU module. In order for the Hessian to be well-defined, we fix the running statistics of batch normalization and treat it as a linear layer during inference. We also turn off the learnable parameters $\theta$ and $\beta$ \citep{ioffe2015batch} for simplicity. For network structure X, we denote the variant with batch normalization after all hidden layers X-BN.
For example, the detailed structure LeNet5-BN is shown in \tableref{tab:appendix_lenetBN_struct}.

\begin{table}[h]
\small
  \centering
  \caption{Structure of LeNet5-BN on CIFAR-10}
  \vskip 0.1in
  \begin{center}
    \begin{tabular}{rllcc}
    \toprule
    \# & Name & Module & In Shape & Out Shape\\\midrule
    1 & conv1 & Conv2D(3, 6, 5, 5) & (3, 32, 32) & (6, 28, 28)\\
    2 & & ReLU & (6, 28, 28) & (6, 28, 28)\\
    3 & & BatchNorm2D & (6, 28, 28) & (6, 28, 28)\\
    4 & maxpool1 & MaxPooling2D(2,2) & (6, 28, 28) & (6, 14, 14)\\
    5 & conv2 & Conv2D(6, 16, 5, 5) & (6, 14, 14) & (16, 10, 10)\\
    6 & & ReLU & (16, 10, 10) & (16, 10, 10)\\
    7 & & BatchNorm2D & (16, 10, 10) & (16, 10, 10)\\
    8 & maxpool2 & MaxPooling2D(2,2) & (16, 10, 10) & (16, 5, 5)\\
    9 & & Flatten & (16, 5, 5) & 400\\
    10 & fc1 & Linear(400, 120) & 400 & 120\\
    11 & & ReLU & 120 & 120\\
    12 & & BatchNorm1D & 120 & 120\\
    13 & fc2 & Linear(120, 84) & 120 & 84\\
    14 & & ReLU & 84 & 84\\
    15 & & BatchNorm1D & 84 & 84\\
    16 & fc3 & Linear(84, 10) & 84 & 10\\
    \multicolumn{5}{c}{\emph{output}} \\\bottomrule
    \end{tabular}%
\end{center}
  \label{tab:appendix_lenetBN_struct}%
\end{table}%

\paragraph{Variants of VGG11:} To verify that our results apply to larger networks, we trained a number of variant of VGG11 (originally named VGG-A in the paper, but commonly refered as VGG11) proposed by \citet{simonyan2014very}. For simplicity, we removed the dropout regularization in the original network. To adapt the structure, which is originally designed for the $3\times224\times224$ input of ImageNet, to $3\times32\times32$ input of CIFAR-10.

Since the original VGG11 network is too large for computing the top eigenspace up to hundreds of dimensions, we reduce the number of output channels of each convolution layer in the network to 32, 48, 64, 80, and  200. We denote the small size variants as VGG11-W32, VGG11-W48, VGG11-W64, VGG11-W80, and VGG11-W200 respectively. We use conv1 - conv8 and fc1 to denote the layers of VGG11 where conv1 is closest to the input feature and fc1 is the classification layer.

\paragraph{Variants of ResNet18:} We also trained a number of variant of ResNet18 proposed by \citet{kaiming2015}. As batch normalization will change the low rank structure of the auto correlation matrix and reduce the overlap, we removed all batch normalization operations.
Following the adaptation of ResNet to CIFAR dataset as in \url{https://github.com/kuangliu/pytorch-cifar}, we changed the input size to $3\times32\times32$ and added a 1x1 convolutional layer for each shortcut after the first block.

Similar to VGG11, we reduce the number of output channels of each convolution layer in the network to 48, 64, 80. We denote the small size variants as ResNet18-W48, ResNet18-W64, and ResNet18-W80 respectively.
We use conv1 - conv17 and fc1 to denote the layers of the ResNet18 backbone where conv1 is closest to the input feature and fc1 is the classification layer. For the 1x1 convolutional layers in the shortcut, we denote them by sc-conv1 - sc-conv3. where sc-conv1 is the convolutional layer on the shortcut of the second ResNet block and  sc-conv3 is the convolutional layer on the shortcut of the fourth ResNet block.

\subsection{Training Process and Hyperparameter Configuration}
\label{sec:appendix_exp_train}
For all datasets, we used the default splitting of training and testing set. All models (except explicitly stated otherwise) are trained using batched stochastic gradient descent (SGD) with batch-size 128 and fixed learning rate 0.01 for 1000 epochs. No momentum and weight decay regularization were used. The loss objective converges by the end of training, so we may assume that the final models are at local minima. For generality we also used a training scheme with fixed learning rate at 0.001, and a training scheme with fixed learning rate at 0.01 with momentum of 0.9 and weight-decay factor of 0.0005. Models trained with these settings will be explicitly stated. Otherwise we assume they were trained with the default scheme mentioned above.

Follow the default initialization scheme of PyTorch\citep{NEURIPS2019_9015}, the  weights of linear layers and convolutional layers are initialized using the Xavier method  \citep{glorot2010understanding}, and bias of each layer are initialized to be zero.

%% file: Appendix_Sections/additional_exp.tex
\newpage
\section{Additional Empirical Results}
\label{sec:appendix_exp_res}

\input{Appendix_Sections/appendix_exps/xxT}

\input{Appendix_Sections/appendix_exps/overlap_results}

\input{Appendix_Sections/appendix_exps/correspondance}

\input{Appendix_Sections/appendix_exps/structure_traj}

%% file: Appendix_Sections/appendix_exps/xxT.tex
\subsection{Low Rank Structure of Auto-correlation Matrix \texorpdfstring{$\E[\vx\vx^\T]$}{ExxT}}
\label{sec:appendix_xxT}

We have briefly discussed about the autocorrelation matrix $\E[\vx\vx^\T]$ being approximately rank 1 in \sectionref{sec:xxT} in the main text. In particular, we claimed that the mean of layer input dominate the covariance, that $\E[\vx\vx^\T]\approx \E[\vx]\E[\vx^\T]$. In this section we provide some additional empirical results supporting that claim.

We use two metrics to quantify the quality of this approximation: the squared dot product between normalized $\E[\vx]$ and the first eigenvector of $\E[\vx\vx^\T]$ and the ratio between the first and second eigenvalue of $\E[\vx\vx^\T]$. Intuitively if the first quantity is close to 1 and the second quantity is large, then the approximation is accurate.
Formally, for fully connected layers, define $\hE[\vx]$ as the normalized expectation of the layer input $\vx$, namely ${\E[\vx]}/{\|\E[\vx]\|}$.
For convolutional layers, following the notations in \sectionref{sec:appendix_conv}, define $\hE[\vx]$ as the first left singular vector of $\E[\mX]$ where $\hE[\vx]\in \R^{nK_1K_2}$. Abusing notations for simplicity, we use $\E[\vx\vx^\T]$ to denote the $nK_1K_2\times nK_1K_2$ matrix $\E[\mX\mX^\T]$. In this section we consider the squared dot product between $\hE[\vx]$ and the first eigenvector $\vv_1$ of $\E[\vx\vx^\T]$, namely $(\vv_1^\T\hE[\vx])^2$.

For the spectral ratio, let $\lambda_1$ be the first eigenvalue of $\E[\vx\vx^\T]$ and $\lambda_2$ be the second. We have
\begin{equation}
    \frac{\lambda_1}{\lambda_2} \geq \frac{\|\E[\vx]\E[\vx]^\T\|-\|\mSigma_\vx\| }{\|\mSigma_\vx\|} = \frac{\|\E[\vx]\E[\vx]^\T\|}{\|\mSigma_\vx\|} -1,
\end{equation}
where $\mSigma_\vx$ is the covariance of $\vx$. Thus, the spectral norm of $\E[\vx]\E[\vx]^\T$ divided by that of $\mSigma_\vx$ gives a lower bound to ${\lambda_1}/{\lambda_2}$. In our experiments, we usually have
${\lambda_1}/{\lambda_2} \geq {\|\E[\vx]\E[\vx]^\T\|}/{\|\mSigma_\vx\|}$.

As we can see from \tableref{tab:appendix_xxT_spec_fc} and \tableref{tab:appendix_xxT_spec_conv}, in a variety of settings, $\E[\vx]\E[\vx]^\T$ indeed dominated the autocorrelation matrix $\E[\vx\vx^\T]$ for fully connected layers. Similar phenomenon also holds for convolutional layers in the modern architectures, but the spectral gap are generally smaller compared to that of the fully connected layers.
% The values for F-$200^2$(MNIST) and LeNet5 (CIFAR10) are average of 5 different models. From \cref{tab:appendix_xxT_spec} and \cref{fig:appedix_xxT_lowrank} we can see that the $\mathbb{E}[\vx\vx^\T]$ matrix is close to rank 1 for all the models (without batch normalization) we experiment on, and the principle component is approximately $\hE[x]$ from \cref{tab:appendix_xxT_inner}.
 
\newpage

\begin{table}[h]
\small
  \centering
  \caption{Squared dot product $(\vv_1^\T\hE[\vx])^2$ and spectral ratio $\lambda_1/\lambda_2$ for fully connected layers in a selection of network structures and datasets. We independently trained 5 runs for each instance and compute the mean, minimum, and maximum of the two quantities over all layers (except the first layer which takes the input with mean-zero) in all runs.}
  \vskip 0.1in
    \begin{center}
    \begin{tabular}{llccccccc}
    \toprule
         &                  &              & \multicolumn{3}{c}{$(\vv_1^\T\hE[\vx])^2$} & \multicolumn{3}{c}{$\lambda_1/\lambda_2$} \\
Dataset  & Network          & \# fc & mean        & min         & max         & mean         & min          & max         \\
         \midrule
MNIST    & F-$200^2$        & 2            & 1.000       & 1.000       & 1.000       & 12.29        & 9.65         & 16.16       \\
         & F-$600^2$        & 2            & 0.999       & 0.999       & 0.999       & 12.00        & 11.42        & 13.00       \\
         & F-$600^4$        & 4            & 1.000       & 0.999       & 1.000       & 17.81        & 7.33         & 28.00       \\
         & F-$600^8$        & 8            & 0.991       & 0.965       & 1.000       & 6.63         & 2.28         & 11.15       \\
\midrule
CIFAR10  & F-$600^2$        & 2            & 0.999       & 0.998       & 1.000       & 9.24         & 4.74         & 13.74       \\
         & F-$1500^3$       & 3            & 0.999       & 0.997       & 1.000       & 13.27        & 6.10         & 18.41       \\
         & LeNet5           & 3            & 0.998       & 0.997       & 0.999       & 7.21         & 5.88         & 9.02        \\
         & LeNet5-(fc1-80)  & 3            & 0.998       & 0.996       & 0.999       & 7.80         & 6.77         & 11.01       \\
         & LeNet5-(fc1-100) & 3            & 0.997       & 0.995       & 0.999       & 7.42         & 6.20         & 9.10        \\
         & LeNet5-(fc1-150) & 3            & 0.998       & 0.992       & 0.999       & 7.35         & 5.34         & 9.62        \\
         & VGG11-W32        & 1            & 0.990       & 0.988       & 0.993       & 6.02         & 5.57         & 6.51        \\
         & VGG11-W64        & 1            & 0.996       & 0.993       & 0.999       & 5.87         & 5.32         & 6.26        \\
         & VGG11-W64        & 1            & 0.995       & 0.993       & 0.996       & 6.24         & 5.97         & 6.70        \\
\midrule
CIFAR100 & VGG11-W48        & 1            & 0.999       & 0.999       & 0.999       & 17.861       & 15.456       & 20.491      \\
         & VGG11-W64        & 1            & 0.999       & 0.999       & 1.000       & 19.185       & 18.358       & 20.410      \\
         & VGG11-W80        & 1            & 0.999       & 0.999       & 1.000       & 19.455       & 18.120       & 21.450      \\
         & ResNet18-W48     & 1            & 1.000       & 1.000       & 1.000       & 28.23        & 27.37        & 29.27       \\
         & ResNet18-W64     & 1            & 1.000       & 1.000       & 1.000       & 27.07        & 25.72        & 29.50       \\
         & ResNet18-W80     & 1            & 1.000       & 1.000       & 1.000       & 28.23        & 25.98        & 30.03      \\\bottomrule
\end{tabular}
\end{center}
  \label{tab:appendix_xxT_spec_fc}%
\end{table}
\newpage
\begin{table}[h]
\small
  \centering
  \caption{Squared dot product $(\vv_1^\T\hE[\vx])^2$ and spectral ratio $\lambda_1/\lambda_2$ for convolutional layers in the selection of network structures and datasets in \tableref{tab:appendix_xxT_spec_fc}.}
  \vskip 0.1in
    \begin{center}
    \begin{tabular}{llccccccc}
\toprule
         &                  &         & \multicolumn{3}{c}{$(\vv_1^\T\hE[\vx])^2$} & \multicolumn{3}{c}{$\lambda_1/\lambda_2$} \\
Dataset  & Network          & \# conv & mean        & min         & max         & mean         & min          & max         \\
\midrule
CIFAR10  & LeNet5           & 1       & 0.999       & 0.998       & 0.999       & 15.87        & 11.15        & 27.20       \\
         & LeNet5-(fc1-80)  & 1       & 0.998       & 0.998       & 0.999       & 12.36        & 9.53         & 13.36       \\
         & LeNet5-(fc1-100) & 1       & 0.999       & 0.999       & 0.999       & 19.49        & 16.69        & 21.92       \\
         & LeNet5-(fc1-150) & 1       & 0.999       & 0.998       & 0.999       & 12.86        & 7.65         & 16.34       \\
         & VGG11-W32        & 7       & 0.995       & 0.991       & 0.999       & 5.31         & 2.39         & 9.09        \\
         & VGG11-W64        & 7       & 0.997       & 0.993       & 1.000       & 5.76         & 2.50         & 9.98        \\
         & VGG11-W64        & 7       & 0.998       & 0.995       & 1.000       & 5.81         & 2.53         & 10.62       \\
\midrule
CIFAR100 & VGG11-W48        & 7       & 0.996       & 0.991       & 0.999       & 5.72         & 2.46         & 9.90        \\
         & VGG11-W64        & 7       & 0.995       & 0.991       & 0.999       & 5.66         & 2.50         & 10.79       \\
         & VGG11-W80        & 7       & 0.994       & 0.988       & 0.998       & 5.18         & 2.50         & 8.45        \\
         & ResNet18-W48     & 19      & 0.981       & 0.917       & 0.998       & 3.79         & 1.89         & 7.56        \\
         & ResNet18-W64     & 19      & 0.985       & 0.910       & 0.998       & 3.96         & 1.81         & 7.53        \\
         & ResNet18-W80     & 19      & 0.987       & 0.954       & 0.997       & 4.16         & 2.11         & 7.04        \\ \bottomrule
\end{tabular}
\vspace{-0.15in}
\end{center}
  \label{tab:appendix_xxT_spec_conv}%
\end{table}

% \begin{figure}[ht]
%     \centering
%     \begin{subfigure}[b]{0.27\textwidth}
%         \centering
%         \captionsetup{justification=centering}
%         \includegraphics[width=\textwidth]{Figures/Eigenspectrum/xxT/xxT_sigval_d20_MNIST_Exp1_FC2_fixlr0.01R1_E-1_fc1fc2.pdf}

%         \caption{fc1:F-$200^2$ (MNIST)}
%         \label{fig:xxT_sig_lenet_init}%
%     \end{subfigure}
%     \hfill
%     \begin{subfigure}[b]{0.27\textwidth}
%         \centering
%         \captionsetup{justification=centering}
%         \includegraphics[width=\textwidth]{Figures/Eigenspectrum/xxT/xxT_sigval_d20_CIFAR10_Exp1_LeNet5_fixlr0.01R1_E-1_fc1fc2.pdf}

%         \caption{fc1:LeNet5 (CIFAR10)}
%         \label{fig:xxT_sig_lenet}
%     \end{subfigure}%
%     \hfill
%     \begin{subfigure}[b]{0.27\textwidth}
%         \centering
%         \captionsetup{justification=centering}
%         \includegraphics[width=\textwidth]{Figures/Eigenspectrum/xxT/xxTConv_sigval_d20_CIFAR10_Exp1_LeNet5_fixlr0.01_E-1.pdf}

%         \caption{conv1:LeNet5 (CIFAR10)}
%         \label{fig:xxT_sig_lenet_random}
%     \end{subfigure}
%     \captionsetup{justification=centering}

%     \caption{Eigenspectrum of $\E[\vx\vx^\T]$ for different layers in different models. All are close to rank 1.}
%     \label{fig:appedix_xxT_lowrank}
% \end{figure}
\newpage

%% file: Appendix_Sections/appendix_exps/overlap_results.tex
\subsection{Eigenspace Overlap Between Different Models}
\label{sec:appendix_expres_ovlp}
The non trivial overlap between top eigenspaces of layer-wise Hessians is one of our interesting observations that had been discusses in \sectionref{sec:models}. Here we provide more related empirical results. Some will further verify our claim in \sectionref{sec:models} and some will appear to be challenge that. Both results will be explained discussed more extensively in \sectionref{sec:appendix_explanation}.

\subsubsection{Overlap preserved when varying hyper-parameters:}
We first verify that the overlap also exists for a set of models trained with the different hyper-parameters. Using the LeNet5 (defined in \tableref{tab:appendix_lenet_struct}) as the network structure. We train 6 models using the default training scheme (SGD, lr=0.01, momentum=0), 5 models using a smaller learning rate (SGD, lr=0.001, momentum=0), and 5 models using a combination of optimization tricks (SGD, lr=0.01, momentum=0.9, weight decay=0.0005).
With these 16 models, we compute the pairwise eigenspace overlap of their layer-wise Hessians (120 pairs in total) and plot their average in \figureref{fig:app_overlap_different_hyperparam}. The shade areas in the figure represents the standard deviation. The pattern of overlap is clearly preserved, and the position of the peak roughly agrees with the output dimension $m$, demonstrating that the phenomenon is caused by a common structure instead of similarities in training process.

\begin{figure}[H]
    \centering
    % \vspace{-1em}
    \includegraphics[width=\textwidth]{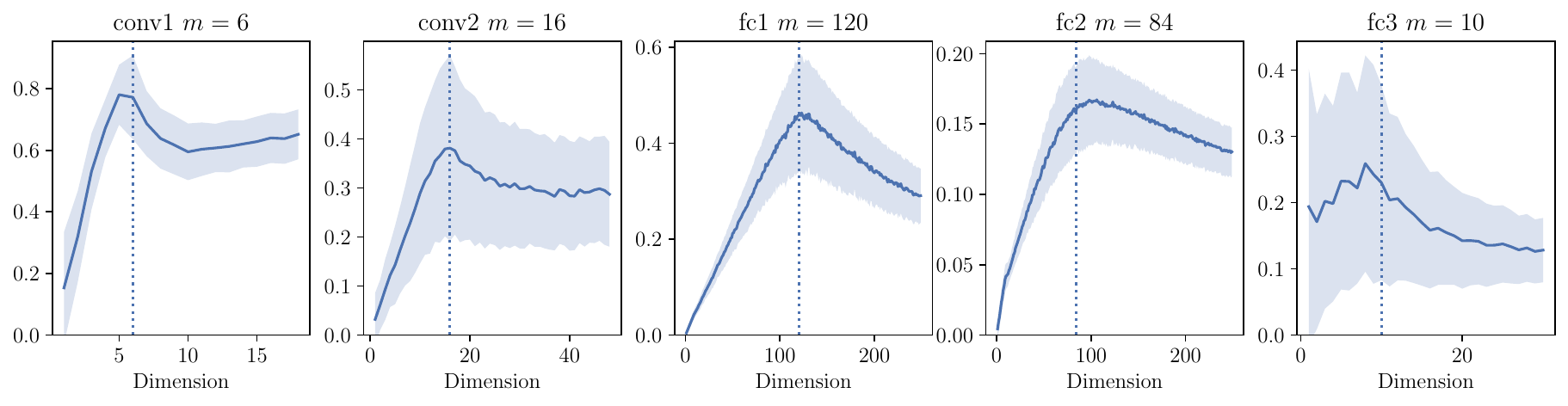}
    \vspace{-1em}
    \caption{Eigenspace overlap of different models of LeNet5 trained with different hyperparameters.}
    \vspace{-1em}
    \label{fig:app_overlap_different_hyperparam}
\end{figure}

Note that for fc3 (the final output layer), we are not observing a linear growth starting from 0 like other layers. This can be explained by the lack of neuron permutation. Related details will be discussed along with the reason for the linear growth pattern for other layers in \sectionref{sec:appendix_model_overlap}.

\subsubsection{Eigenspace overlap for convolutional layers in large models:}
Even though the exact Kroneckor Factorization for layer-wise Hessians is only well-defined for fully connected layers, we also observe similar nontrivial eigenspace overlap for convolutional layers in larger and deeper networks including variants of VGG11 and ResNet18 on datasets CIFAR10 and CIFAR100. Some representative results are shown in \figureref{fig:app_adexp_vgg} and \figureref{fig:app_adexp_resnet}. For each model on each dataset, we independently train 5 models and compute the average pairwise eigenspace overlap. The shade areas represents the standard deviation.
 
For most of the convolutional layers, the eigenspace overlap peaks around the dimension which is equal to the number of output channels of that layer, which is similar to the layers in LeNet5 as in \figureref{fig:app_overlap_different_hyperparam}. The eigenspace overlap of the final fully connected-layer also behaves similar to fc3:LeNet5, which remains around a constant then drops after exceeding the dimension of final output.
However, there are also layers whose overlap does not peak around the output dimensions, (e.g. conv2 of \figureref{fig:app_adexp_vgg}(a) and conv7 of \figureref{fig:app_adexp_resnet}(a)). We will discuss these special cases in the following paragraph.

% \begin{figure}[H]
%     \centering
%     \subfigure[\centering\small{VGG11-W32 (CIFAR10)}]{
%     \includegraphics[width=0.48\textwidth]{Appendix_Figures/Overlap_large_model/overlap_raw/DimOverlap_CIFAR10_VGG11W32_fixlr0.01_appendix_vertical_3col.pdf}}
%     \subfigure[\centering\small{VGG11-W200 (CIFAR10)}]{
%     \includegraphics[width=0.48\textwidth]{Appendix_Figures/Overlap_large_model/overlap_raw/DimOverlap_CIFAR10_VGG11W200_fxlr0.01_appendix_vertical_3col.pdf}}\\
%     \subfigure[\centering\small{VGG11-W200 (CIFAR10)}]{
%     \includegraphics[width=0.48\textwidth]{Appendix_Figures/Overlap_large_model/overlap_raw/DimOverlap_CIFAR100_VGG11W48New_nobn_fixlr0.01_appendix_vertical_3col.pdf}}
%     \subfigure[\centering\small{VGG11-W80 (CIFAR100)}]{
%     \includegraphics[width=0.48\textwidth]{Appendix_Figures/Overlap_large_model/overlap_raw/DimOverlap_CIFAR100_VGG11W80New_nobn_fixlr0.01_appendix_vertical_3col.pdf}}
%     \caption{Top Eigenspace overlap for varients of VGG11 on CIFAR10 and CIFAR100.}
%     \label{fig:app_adexp_vgg}
% \end{figure}

\begin{figure}[H]
    \centering
    \begin{subfigure}[b]{0.48\textwidth}
        \centering
        \captionsetup{justification=centering}
        \includegraphics[width=\textwidth]{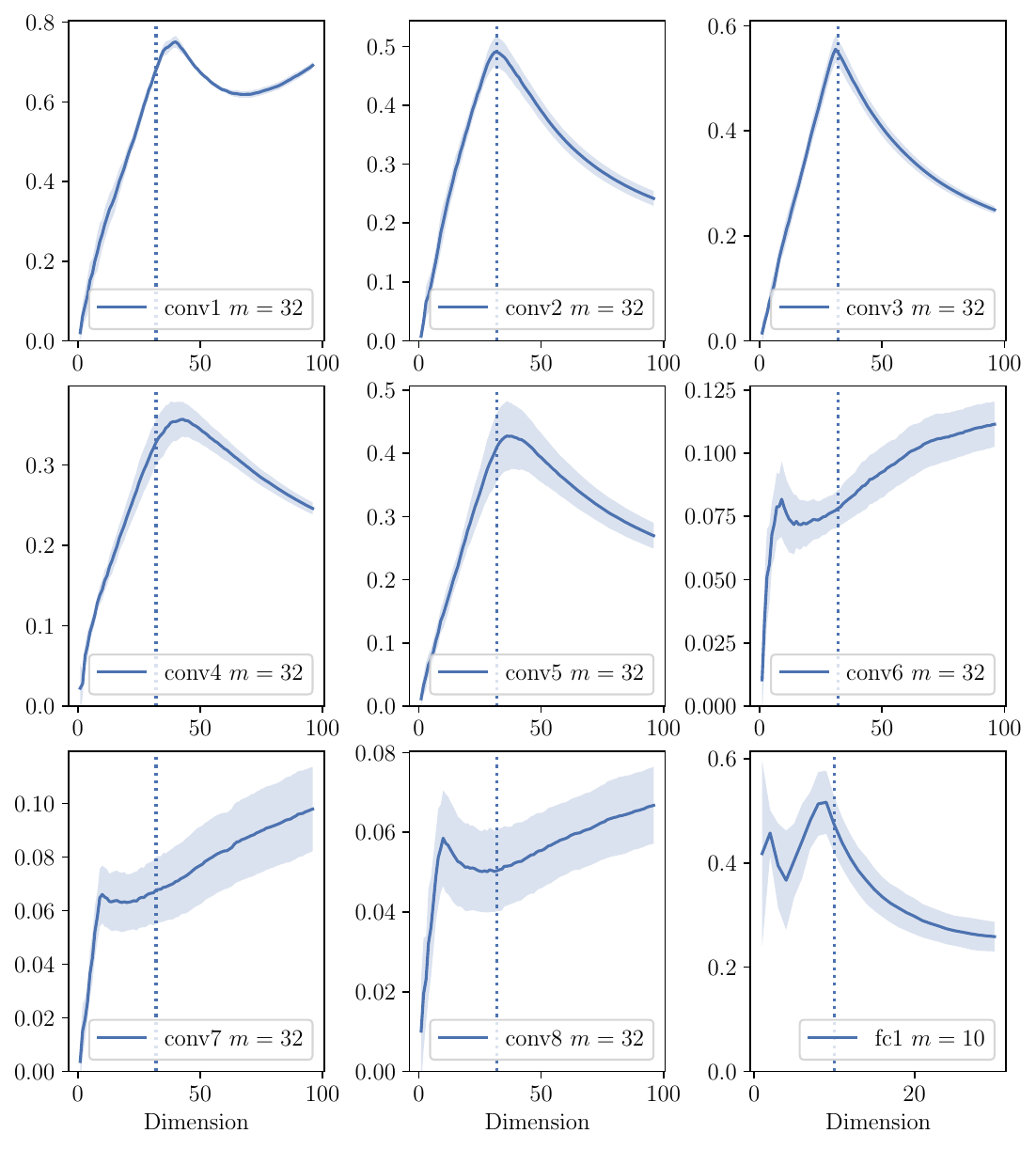}
        \caption{VGG11-W32 (CIFAR10)}
        \label{fig:app_adexp_cifar10_vgg32}
    \end{subfigure}%
    \ \ \ \ \ \
    \begin{subfigure}[b]{0.48\textwidth}
        \centering
        \captionsetup{justification=centering}
        \includegraphics[width=\textwidth]{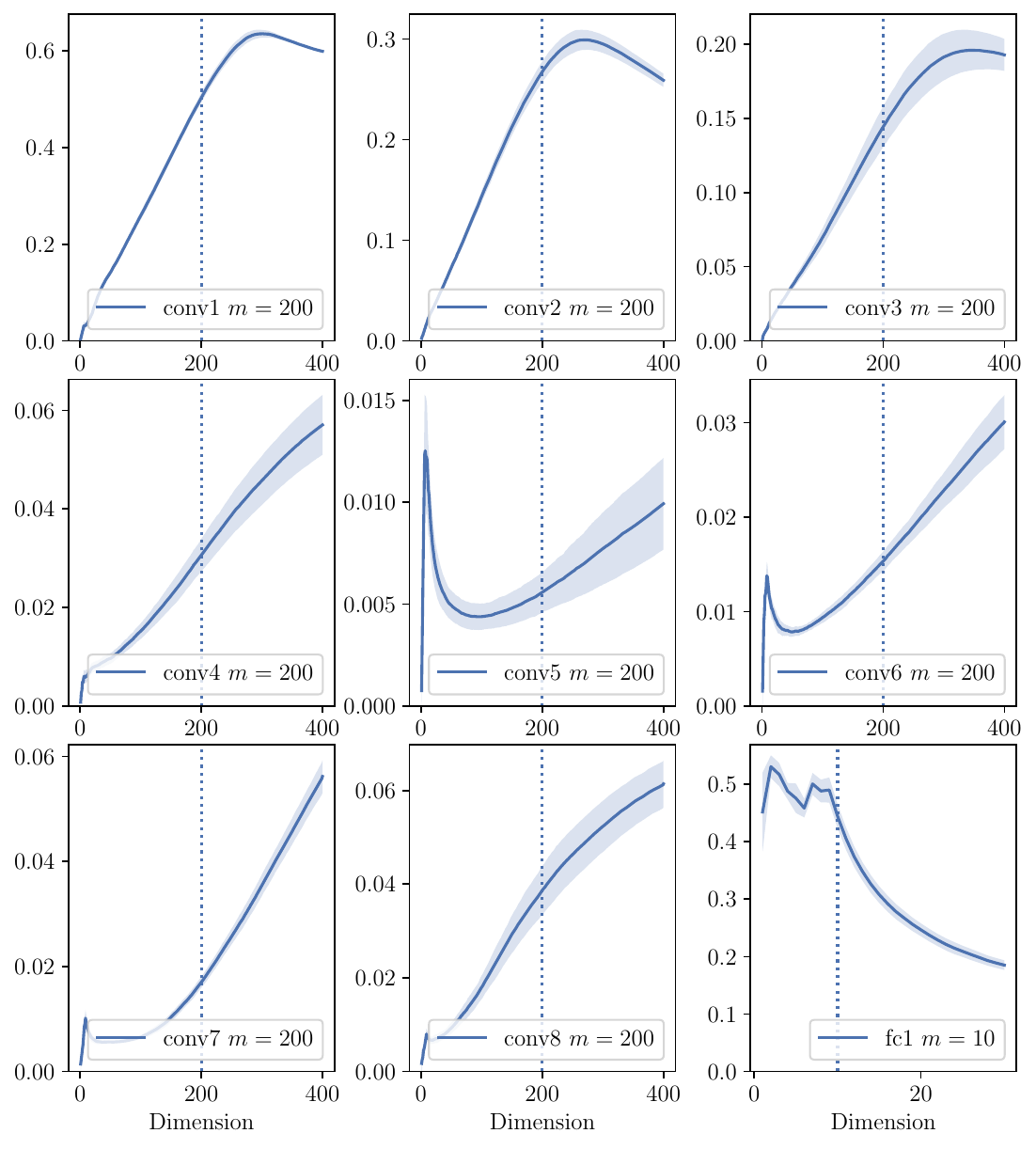}
        \caption{VGG11-W200 (CIFAR10)}
        \label{fig:app_adexp_cifar10_vgg200}
    \end{subfigure}
    \\\vspace{0.1in}
    \begin{subfigure}[b]{0.48\textwidth}
        \centering
        \captionsetup{justification=centering}
        \includegraphics[width=\textwidth]{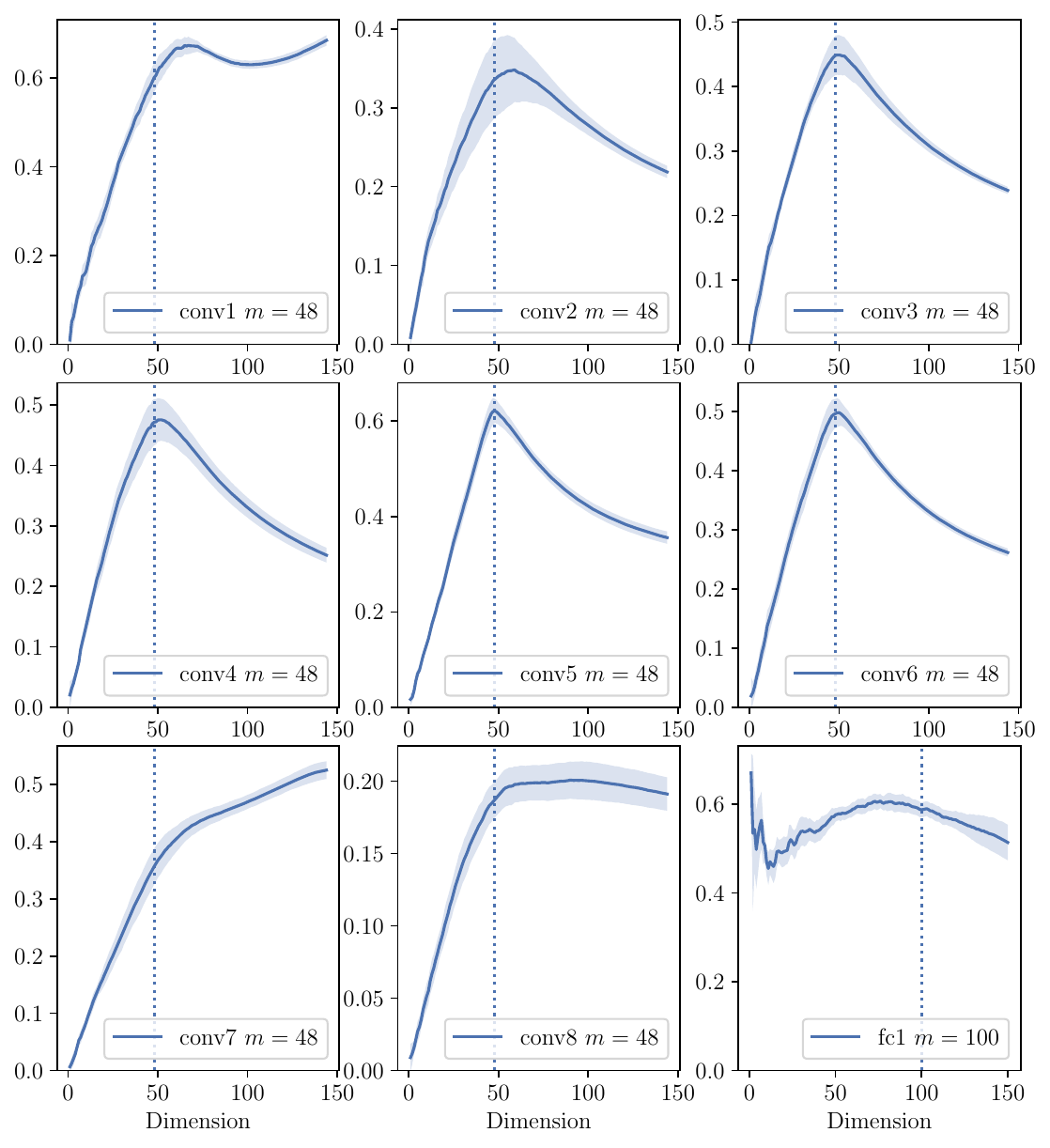}
        \caption{VGG11-W48 (CIFAR100)}
        \label{fig:app_adexp_cifar100_vgg48}
    \end{subfigure}%
    \ \ \ \ \ \
    \begin{subfigure}[b]{0.48\textwidth}
        \centering
        \captionsetup{justification=centering}
        \includegraphics[width=\textwidth]{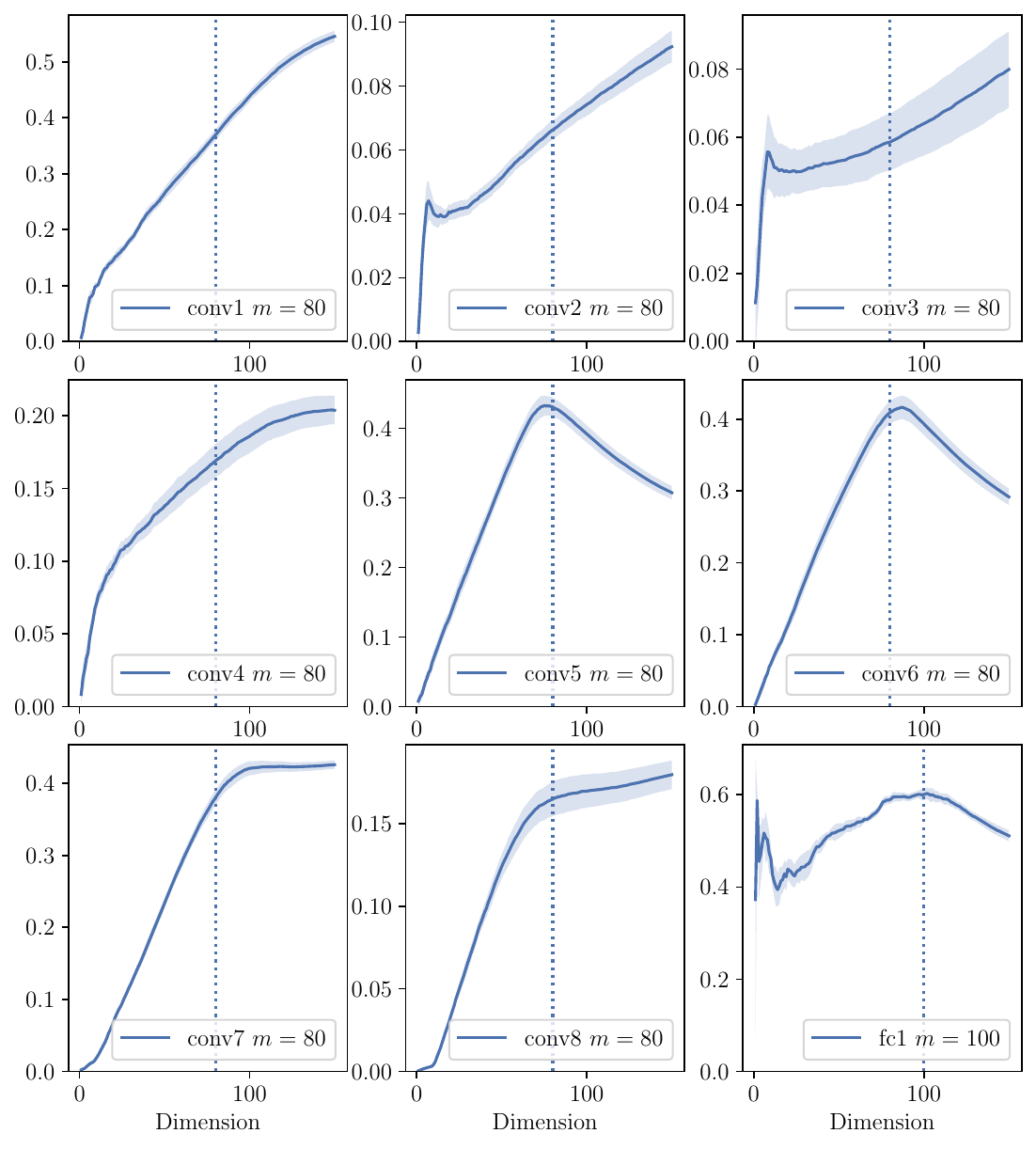}
        \caption{VGG11-W80 (CIFAR100)}
        \label{fig:app_adexp_cifar100_vgg80}
    \end{subfigure}
    \captionsetup{justification=centering}
    \caption{Top Eigenspace overlap for varients of VGG11 on CIFAR10 and CIFAR100}
    \label{fig:app_adexp_vgg}
\end{figure}

% \begin{figure}[H]
%     \centering
%     \subfigure[\centering\small{ResNet18-W48 (CIFAR100)}]{
%     \includegraphics[width=0.8\textwidth]{Appendix_Figures/Overlap_large_model/overlap_raw/ResNet/Resnet18W48New_nobn_fixlr0.01_appendix_vertical_7col.pdf}}\\
%     \subfigure[\centering\small{ResNet18-W64 (CIFAR100)}]{
%     \includegraphics[width=0.8\textwidth]{Appendix_Figures/Overlap_large_model/overlap_raw/ResNet/Resnet18W64New_nobn_fixlr0.01_appendix_vertical_7col.pdf}}\\
%     \subfigure[\centering\small{ResNet18-W80 (CIFAR100)}]{
%     \includegraphics[width=0.8\textwidth]{Appendix_Figures/Overlap_large_model/overlap_raw/ResNet/Resnet18W80_nobn_fixlr0.01_appendix_vertical_7col.pdf}}
%     \caption{Top Eigenspace overlap for variants of ResNet18 on CIFAR100}
%     \label{fig:app_adexp_resnet}
% \end{figure}

\begin{figure}[H]
    \centering
    \begin{subfigure}[b]{\textwidth}
        \centering
        \captionsetup{justification=centering}
        \includegraphics[width=\textwidth]{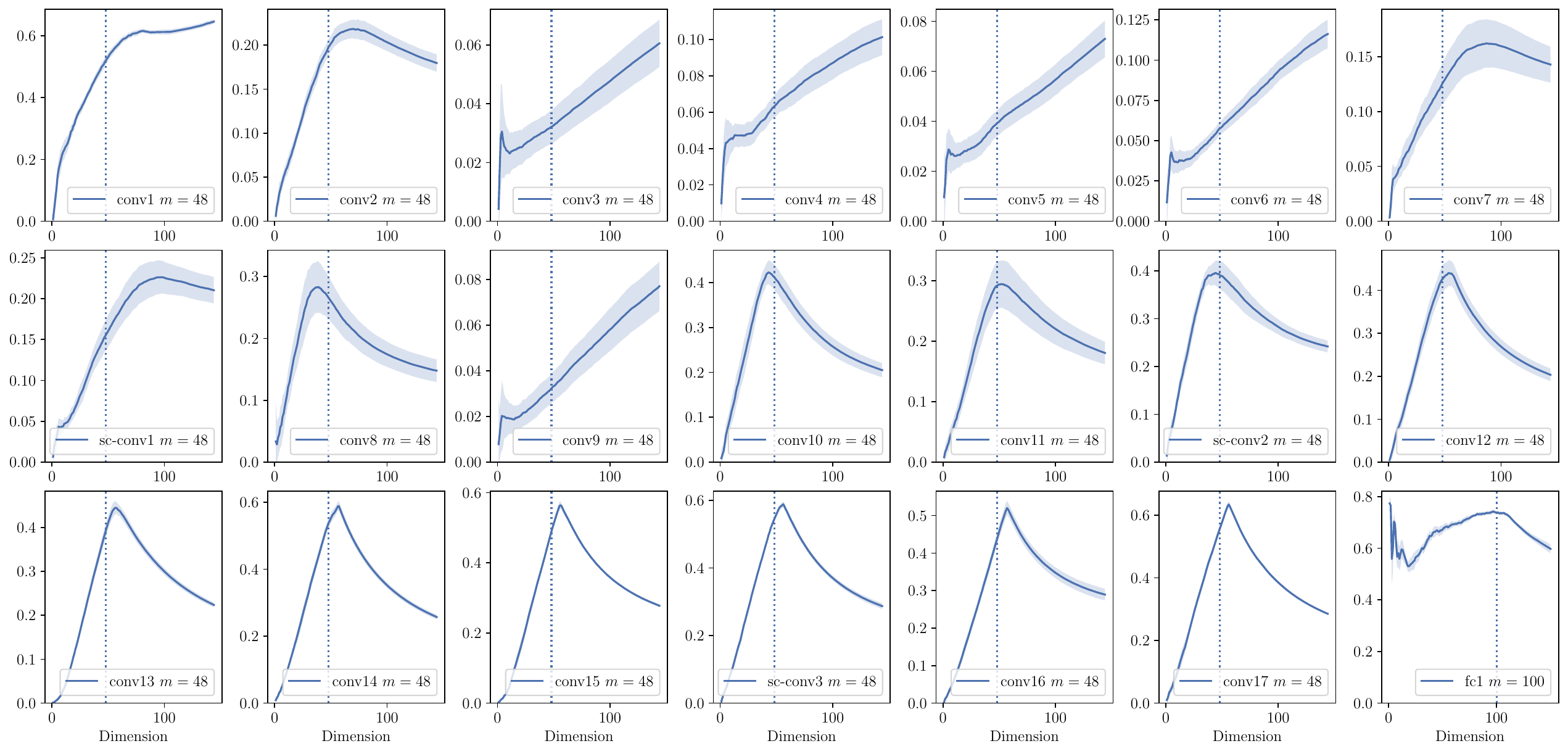}
        \caption{ResNet18-W48 (CIFAR100)}
        \label{fig:app_adexp_cifar100_resnet48}
    \end{subfigure}%
    \\
    \begin{subfigure}[b]{\textwidth}
        \centering
        \captionsetup{justification=centering}
        \includegraphics[width=\textwidth]{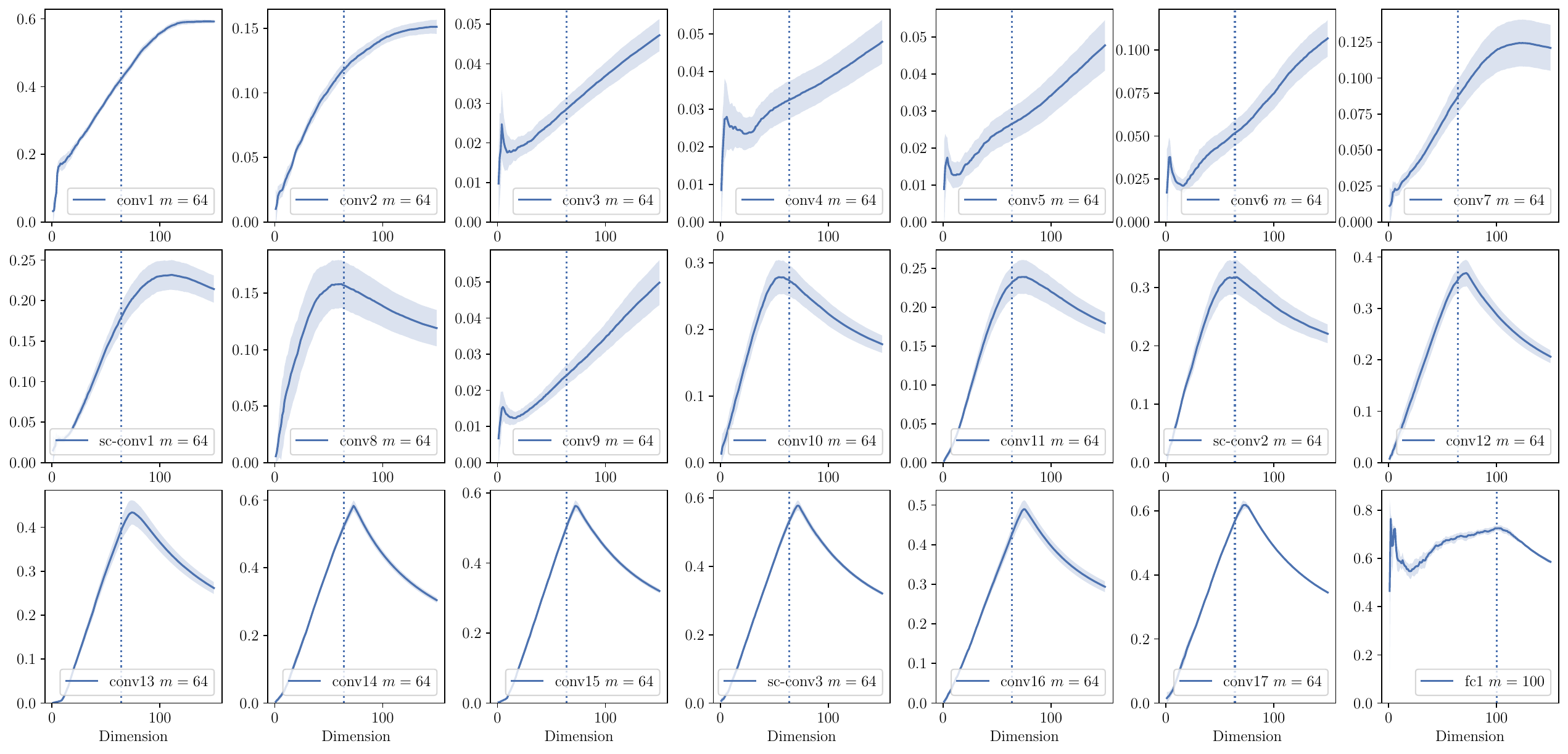}
        \caption{ResNet18-W64 (CIFAR100)}
        \label{fig:app_adexp_cifar100_resnet64}
    \end{subfigure}
    \\
    \begin{subfigure}[b]{\textwidth}
        \centering
        \captionsetup{justification=centering}
        \includegraphics[width=\textwidth]{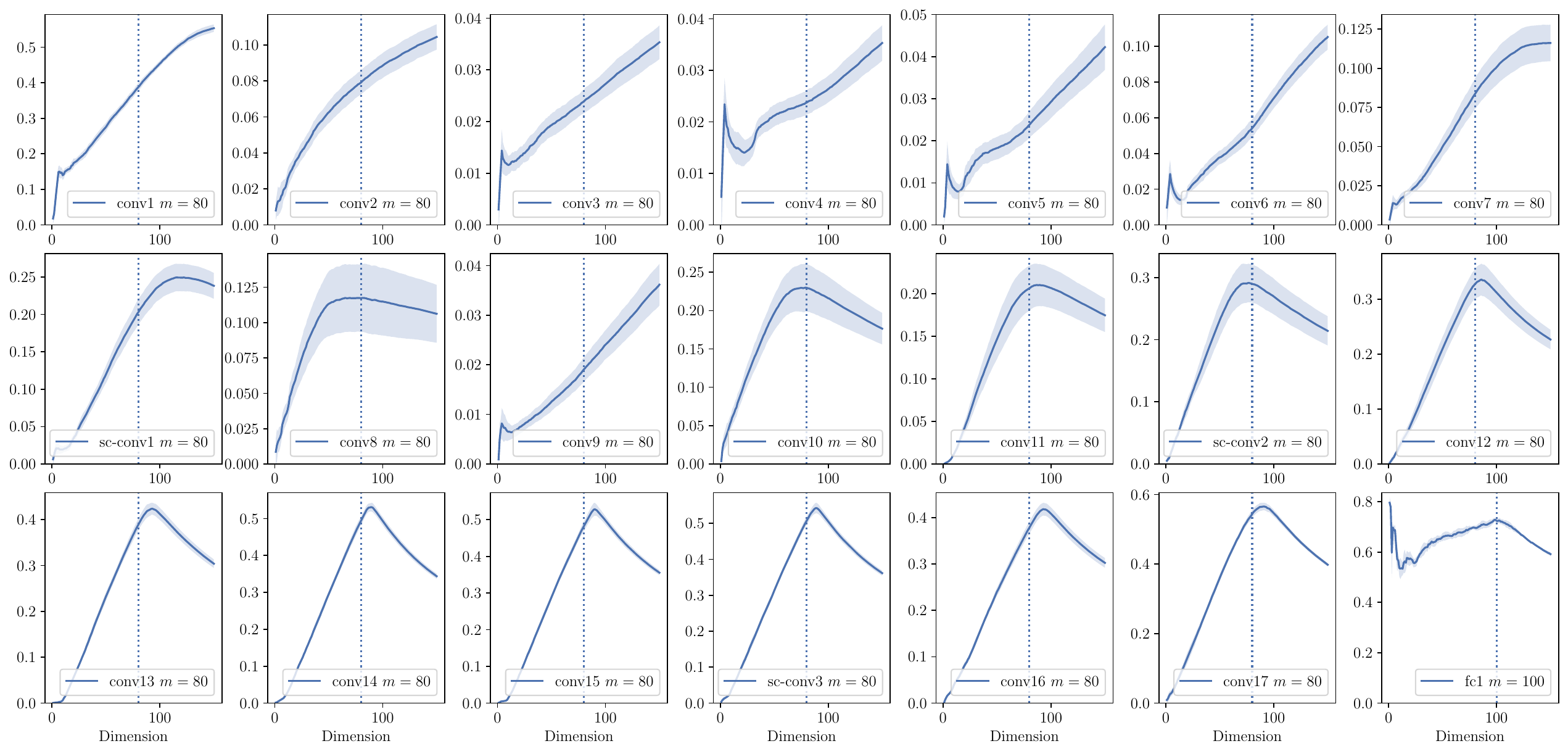}
        \caption{ResNet18-W80 (CIFAR100)}
        \label{fig:app_adexp_cifar100_resnet80}
    \end{subfigure}
    \captionsetup{justification=centering}
    \caption{Top Eigenspace overlap for variants of ResNet18 on CIFAR100}
    \label{fig:app_adexp_resnet}
\end{figure}

\subsubsection{Failed cases for eigenspace overlap}
\label{sec:appendix-failed-exp}
As seen in \figureref{fig:app_adexp_vgg} and \figureref{fig:app_adexp_resnet}, there is a small portion of layers, usually closer to the input, whose eigenspace overlap does peak around the output dimensions. These layers can be clustered into the following two general cases.

\paragraph{Early Peak of Low Overlap}
For layers shown in \figureref{fig:app_adexp_failure_early}. The overlap of dominating eigenspaces are significantly lower than the other layers. Also there exists a small peak at very small dimensions.

% \begin{figure}[H]
%     \centering
%     \subfigure[\centering\small{fc2:F-$200^2$\\(MNIST)}]{
%     \includegraphics[width=0.23\textwidth]{Appendix_Figures/Overlap_large_model/FailCases/early/FC2_fixlr0.01_fc2_zoom_stacked.pdf}}
%     \subfigure[\centering\small{conv5:VGG11-W200\\(CIFAR10)}]{
%     \includegraphics[width=0.23\textwidth]{Appendix_Figures/Overlap_large_model/FailCases/early/CIFAR10_VGG11W200_fxlr0.01_conv5_zoom_stacked.pdf}}
%     \subfigure[\centering\small{conv2:VGG11-W80\\(CIFAR100)}]{
%     \includegraphics[width=0.23\textwidth]{Appendix_Figures/Overlap_large_model/FailCases/early/CIFAR100_VGG11W80New_nobn_fixlr0.01_conv2_zoom_stacked.pdf}}
%     \subfigure[\centering\small{conv5:RN18-W64\\(CIFAR100)}]{
%     \includegraphics[width=0.23\textwidth]{Appendix_Figures/Overlap_large_model/FailCases/early/CIFAR100_ResNet18W80_nobn_fixlr0.01_conv5_zoom_stacked.pdf}}
%     \caption{Top eigenspace overlap for layers with an early low peak. Figures in the second row are the zoomed in versions of the figures in the first row. (RN denotes ResNet in (\emph{d}))}
%     \label{fig:app_adexp_failure_early}
% \end{figure}

\begin{figure}[h]
    \centering
    \begin{subfigure}[b]{0.23\textwidth}
        \centering
        \captionsetup{justification=centering}
        \includegraphics[width=\textwidth]{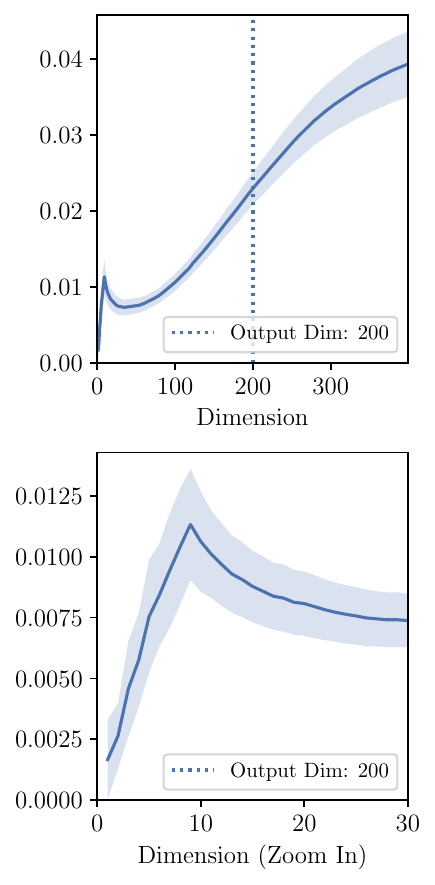}
        \caption{fc2:F-$200^2$\\(MNIST)}
        \label{fig:app_adexp_overlap_early_mnistfc2}
    \end{subfigure}
    \begin{subfigure}[b]{0.23\textwidth}
        \centering
        \captionsetup{justification=centering}
        \includegraphics[width=\textwidth]{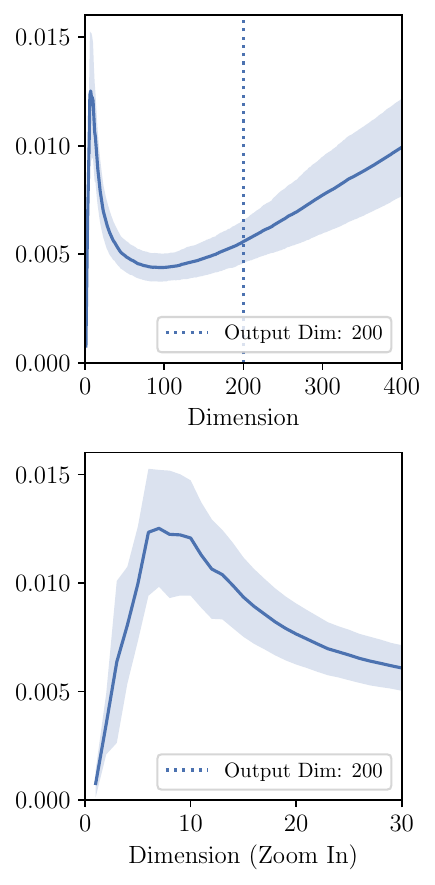}
        \caption{conv5:VGG11-W200\\(CIFAR10)}
        \label{fig:app_adexp_overlap_early_vgg_cifar10}
    \end{subfigure}
    \begin{subfigure}[b]{0.23\textwidth}
        \centering
        \captionsetup{justification=centering}
        \includegraphics[width=\textwidth]{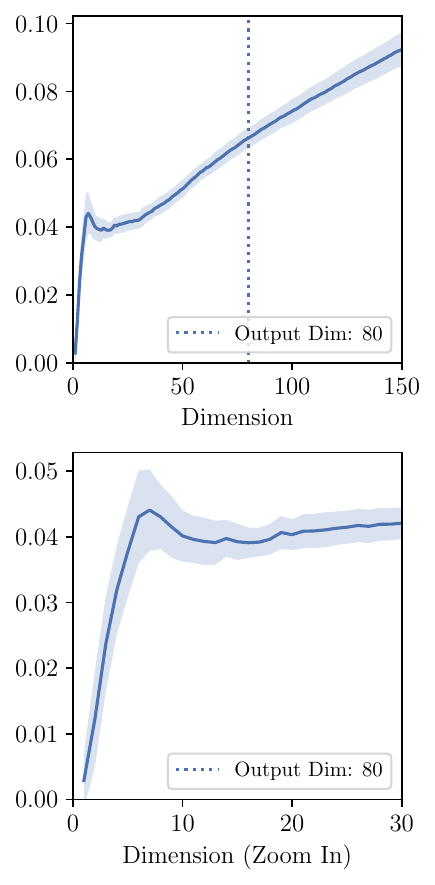}
        \caption{conv2:VGG11-W80\\(CIFAR100)}
        \label{fig:app_adexp_overlap_early_vgg_cifar100}
    \end{subfigure}
    \begin{subfigure}[b]{0.23\textwidth}
        \centering
        \captionsetup{justification=centering}
        \includegraphics[width=\textwidth]{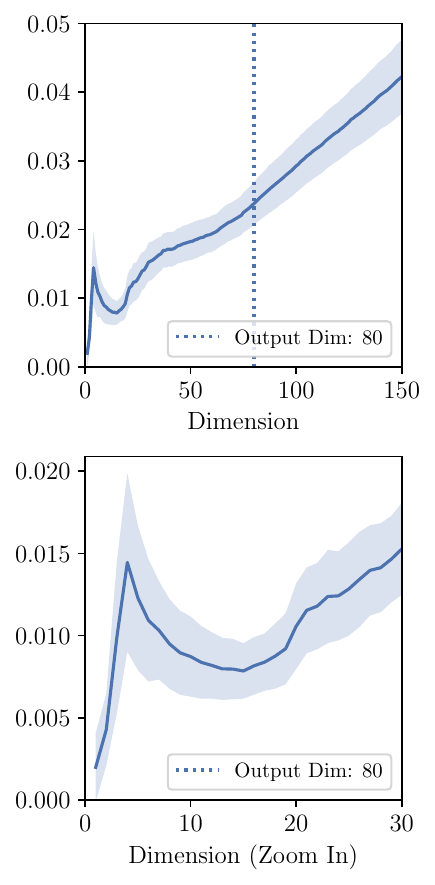}
        \caption{conv5:ResNet18-W64\\(CIFAR100)}
        \label{fig:app_adexp_overlap_early_resnet}
    \end{subfigure}
    \captionsetup{justification=centering}
    \caption{Top eigenspace overlap for layers with an early low peak.\\Figures in the second row are the zoomed in versions of the figures in the first row.}
    \label{fig:app_adexp_failure_early}
\end{figure}

\paragraph{Delayed Peak / Peak Doesn't Decline} 
For layers shown in \figureref{fig:app_adexp_failure_late}, the top eigenspaces has a nontrivial overlap, but the peak dimension is larger than predicted output dimension.

% \begin{figure}[H]
%     \centering
%     \subfigure[\centering\small{conv2:VGG11-W200\\(CIFAR10)}]{
%     \includegraphics[width=0.23\textwidth]{Appendix_Figures/Overlap_large_model/FailCases/late/CIFAR10_VGG11W200_fxlr0.01_conv2.pdf}}
%     \subfigure[\centering\small{conv7:VGG11-W48\\(CIFAR100)}]{
%     \includegraphics[width=0.23\textwidth]{Appendix_Figures/Overlap_large_model/FailCases/late/CIFAR100_VGG11W48New_nobn_fixlr0.01_conv7.pdf}}
%     \subfigure[\centering\small{conv7:RN18-W48\\(CIFAR100)}]{
%     \includegraphics[width=0.23\textwidth]{Appendix_Figures/Overlap_large_model/FailCases/late/CIFAR100_Resnet18W48New_nobn_fixlr0.01_conv7.pdf}}
%     \subfigure[\centering\small{conv7:RN18-W64\\(CIFAR100)}]{
%     \includegraphics[width=0.23\textwidth]{Appendix_Figures/Overlap_large_model/FailCases/late/CIFAR100_Resnet18W64New_nobn_fixlr0.01_conv7.pdf}}
%     \caption{Top eigenspace overlap for layers with a delayed peak (RN denotes ResNet).}
%     \label{fig:app_adexp_failure_late}
% \end{figure}

\begin{figure}[h]
    \centering
    \begin{subfigure}[b]{0.23\textwidth}
        \centering
        \captionsetup{justification=centering}
        \includegraphics[width=\textwidth]{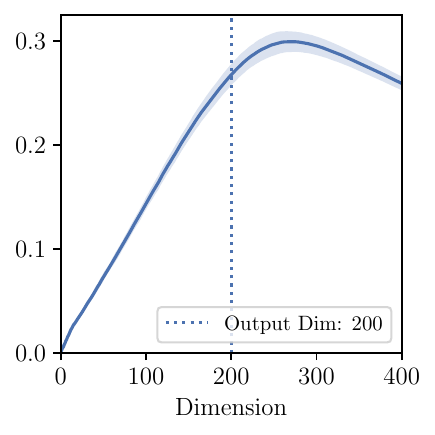}
        \caption{conv2:VGG11-W200\\(CIFAR10)}
        \label{fig:app_adexp_overlap_late_vgg_cifar10}
    \end{subfigure}
    \begin{subfigure}[b]{0.23\textwidth}
        \centering
        \captionsetup{justification=centering}
        \includegraphics[width=\textwidth]{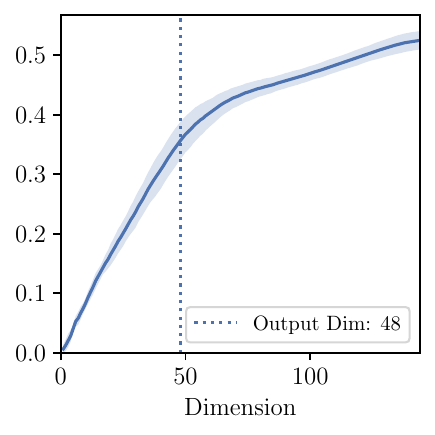}
        \caption{conv7:VGG11-W48\\(CIFAR100)}
        \label{fig:app_adexp_overlap_late_vgg_cifar100}
    \end{subfigure}
    \begin{subfigure}[b]{0.23\textwidth}
        \centering
        \captionsetup{justification=centering}
        \includegraphics[width=\textwidth]{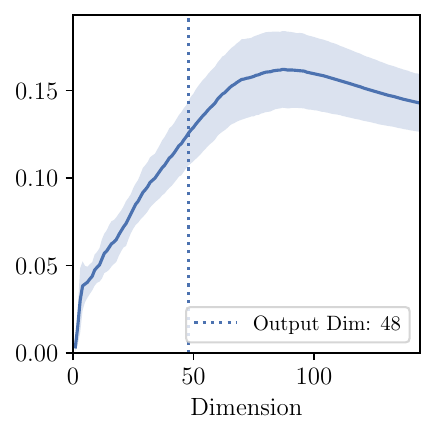}
        \caption{conv7:VGG11-W48\\(CIFAR100)}
        \label{fig:app_adexp_overlap_late_resnet48}
    \end{subfigure}
    \begin{subfigure}[b]{0.23\textwidth}
        \centering
        \captionsetup{justification=centering}
        \includegraphics[width=\textwidth]{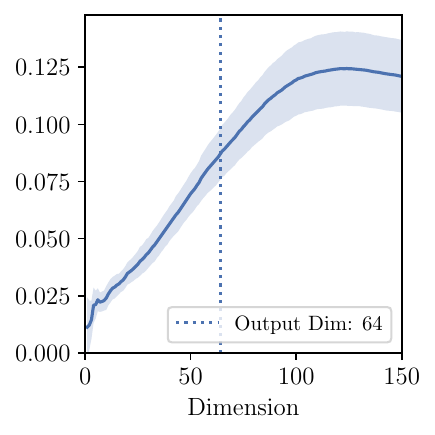}
        \caption{conv7:ResNet18-W64\\(CIFAR100)}
        \label{fig:app_adexp_overlap_late_resnet}
    \end{subfigure}
    \captionsetup{justification=centering}
    \caption{Top eigenspace overlap for layers with a delayed peak.}
    \label{fig:app_adexp_failure_late}
\end{figure}

However, the existence of such failure cases \emph{does not} undermine the theory of Kronecker factorization approximation. In fact, both appear because the top hessian eigenspace is not completely spanned by $\E[\vx]$,  and can be predicted by computing the auto correlation matrices and the output Hessians. The details will also be elaborated in \sectionref{sec:appendix_model_overlap} with the help of correspondence matrices.

% \paragraph{Overlap may exhibits a early peak for some layers:} For the eigenspace overlap plot of F-$200^2$ (\figureref{fig:app_overlap_fail}), we see that the overlap for fc2 is significantly lower than the other layers, and there exists a small peak near dimension $k=10$. This is because the top hessian eigenspace is not completely spanned by $\E[\vx]$. This phenomenon will also be elaborated in \sectionref{sec:appendix_model_overlap} with the help of correspondence matrices.
% \begin{figure}[h]
%     \centering
%     \includegraphics[width=0.6\textwidth]{Appendix_Figures/Overlap_different_models/DimOverlap_MNIST_FC2_fixlr0.01_appendix_full.pdf}
%     \caption{Eigenspace overlap of different models of F-$200^2$.}
%     \label{fig:app_overlap_fail}

% \end{figure}

%% file: Appendix_Sections/appendix_exps/correspondance.tex
\subsection{Eigenvector Correspondence}
\label{sec:app_exp_corr}

In this section, we leverage the idea of eigenvector matricization (\definitionref{def:matricization}) and analyze the validity of the decoupling conjecture using a matrix which we defined as the eigenvector corresponding matrix. First let us recall the definition of eigenvector matricization

\paragraph{\definitionref{def:matricization}} \emph{Consider a layer with input dimension $n$ and output dimension $m$. For an eigenvector $\vh\in\R^{mn}$ of its layer-wise Hessian, the matricized form of $\vh$ is $\Mat(\vh)\in\R^{m\times n}$ where $\Mat(\vh)_{i,j} = \vh_{(i-1)m+j}$.}
\vspace{1em}

Suppose the $i$-th eigenvector for $\E[\vx\vx^\T]$ is $\vv_i$ and the $j$-th eigenvector for $\E[\mM]$ is $\vu_j$. Then the Kronecker product $\E[\mM]\otimes \E[\vx\vx^\T]$ has an eigenvector $\vu_j\otimes \vv_i$. Therefore if the decoupling conjecture is true, one would expect that the top eigenvector of the layer-wise Hessian have a clear correspondence with the top eigenvectors of its two components. Note that $\vu\otimes\vv$ is just the flattened matrix $\vu\vv^\T$.

More concretely, to demonstrate the correspondence between the eigenvectors of the layerwise hessian and the eigenvectors of matrix $\E[\mM]$ and $\E[\vx\vx^\T]$, we introduce ``eigenvector correspondence matrices'' as shown in \figureref{fig:Corr_fc}.
\begin{definition}[Eigenvector Correspondence Matrices]
For layer-wise Hessian matrix $\mH\in\R^{mn\times mn}$ with eigenvectors $\vh_1, \cdots, \vh_{mn}$, and its corresponding auto-correlation matrix $\E[\vx\vx^\T]\in\R^{n\times n}$ with eigenvectors $\vv_1, \cdots, \vv_n$. The correspondence between $\vv_i$ and $\vh_j$ can be defined as \begin{equation}
    \Corr(\vv_i,\vh_j) := \|\Mat(\vh_j)\vv_i\|^2.
\end{equation}
For the output Hessian matrix $\E[\mM]\in\R^{m\times m}$ with eigenvectors $\vu_1, \cdots, \vu_m$, we can likewise define correspondence between $\vv_i$ and $\vh_j$ as
\begin{equation}
    \Corr(\vu_i,\vh_j) := \|\Mat(\vh_j)^\T\vu_i\|^2
\end{equation}
We may then define the eigenvector correspondence matrix between $\mH$ and $\E[\vx\vx^\T]$ as a $n\times mn$ matrix whose $i,j$-th entry is $\Corr(\vv_i,\vh_j)$, and the eigenvector correspondence matrix between $\mH$ and $\E[\mM]$ as a $m\times mn$ matrix whose $i,j$-th entry is $\Corr(\vu_i,\vh_j)$.
\label{def:corr_mat}
\end{definition}

Intuitively, if the $i,j$-th entry of the corresponding matrix is close to 1, then the eigenvector $\vh_j$ is likely to be the Kronecker product of $\vv_i$ (or $\vu_i$) with some vector. Note that if the decoupling conjecture holds absolutely, every eigenvector of the layer-wise Hessian (column of the correspondence matrices) should have a perfect correlation of 1 with exactly one of $\vv_i$ and one of $\vu_i$.
In \figureref{fig:Corr_fc} we can see that the correspondence matrices for the true layer-wise Hessian approximately satisfies this property for top eigenvectors. The similarity between the correspondence patterns for the true and approximated Hessian also verifies the validity of the Kronecker approximation for dominating eigenspace.

In \figureref{fig:Corr_fc}, we show the heatmap of Eigenvector Correspondence Matrices for fc1:LeNet5, which has 120 output neurons. Here we take the top left corner of the eigenvector correspondence matrices. We can see that the top 120 eigenvectors of $\E[\mH]$, roughly corresponds to the top 120 eigenvectors of $\E[\mM]$ (as shown by the diagonal patter of (\emph{b})) and the first eigenvector of $\E[\vx\vx^\T]$ (as shown by the horizontal pattern of (\emph{a})). The similarity between the first row and the second row also shows the validity of the Kronecker approximation.
\label{sec:eigen_corr}
\begin{figure}[H]
    \centering
    \begin{subfigure}[t]{0.46\textwidth}
        \centering
        \captionsetup{justification=centering}
        \includegraphics[width=0.9\textwidth]{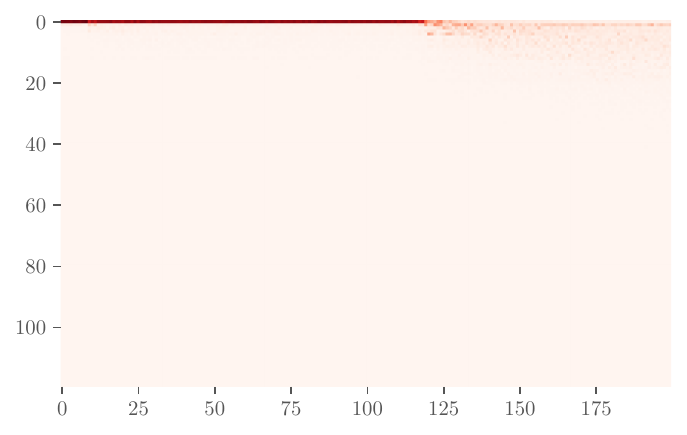}
        \caption{True Hessian with $\E[\vx\vx^\T].$}
        \label{fig:Corr_xxT_True_fc}
    \end{subfigure}%
    \begin{subfigure}[t]{0.46\textwidth}
        \centering
        \captionsetup{justification=centering}
        \includegraphics[width=0.9\textwidth]{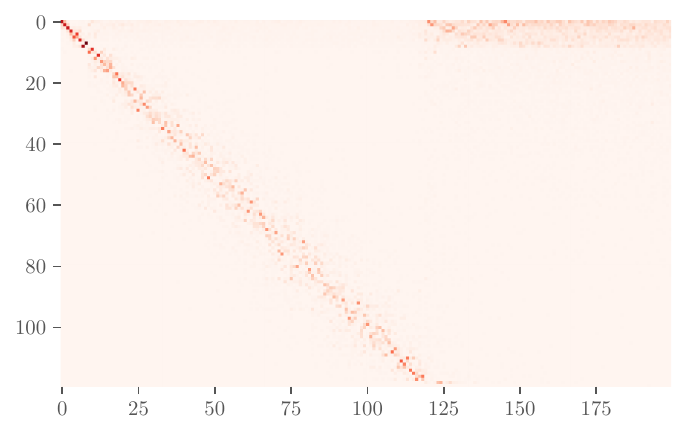}
        \caption{True Hessian with $\E[\mM].$}
        \label{fig:Corr_UTAU_True_fc}
    \end{subfigure}%
    \begin{subfigure}[t]{0.065\textwidth}
        \centering
        \includegraphics[width=\textwidth]{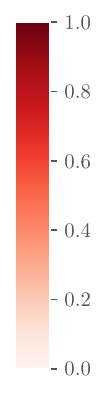}
    \end{subfigure}
    \bigskip
\begin{subfigure}[t]{0.46\textwidth}
        \centering
        \captionsetup{justification=centering}
        \includegraphics[width=0.9\textwidth]{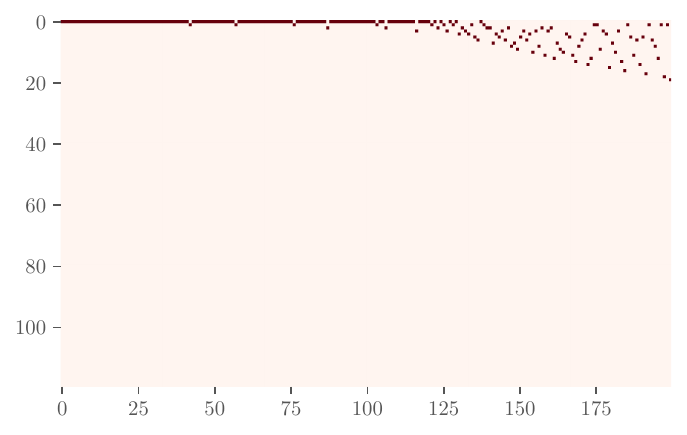}
        \caption{Approximated Hessian with $\E[\vx\vx^\T].$}
        \label{fig:Corr_xxT_Approx_fc}
    \end{subfigure}%
    \begin{subfigure}[t]{0.46\textwidth}
        \centering
        \captionsetup{justification=centering}
        \includegraphics[width=0.9\textwidth]{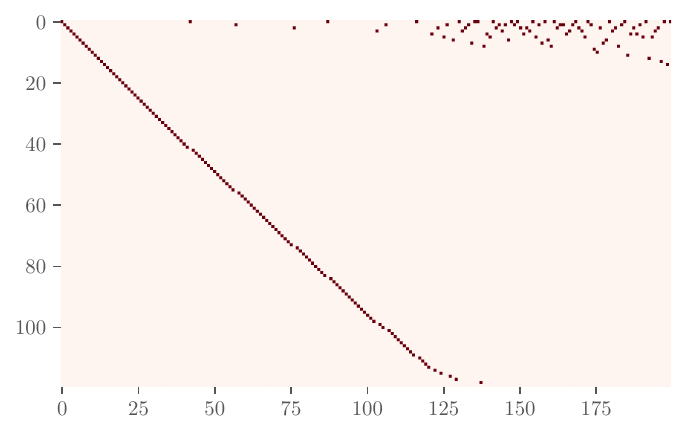}
        \caption{Approximated Hessian with $\E[\vx\vx^\T].$}
        \label{fig:Corr_UTAU_Approx_fc}
    \end{subfigure}%
    \begin{subfigure}[t]{0.065\textwidth}
        \centering
        \includegraphics[width=\textwidth]{Figures/Misc/colorbar.pdf}
    \end{subfigure}
    \captionsetup{justification=centering}
    \caption{Heatmap of Eigenvector Correspondence Matrices for fc1:LeNet5.}
    \label{fig:Corr_fc}
\end{figure}

Here we present the correspondence matrix for fc2, conv1, and conv2 layer of LeNet5. The top eigenvectors for all layers shows a strong correlation with the first eigenvector of $\E[\vx\vx^\T]$ (which is approximately $\hE[\vx]$). For convolutional layers, since the computation of $\mM$ is not exact, the correspondence matrices with $\E[\mM]$ does not exhibit the diagonal pattern. For fc2:LeNet5 as in \figureref{fig:Corrfc22}, the diagonal pattern in (\emph{b}) and the strong correlation with $\E[\vx]$ stops at dimension 9. This fells into one of the ``failed cases'' as described in \sectionref{sec:appendix-failed-exp} case that the small eigenvalues of $\E[\rmM]$ are approaching $0$ faster than $\E[\vx\vx^\T]$. We will discuss this case in more detail in \sectionref{sec:appendix-failure}.

\begin{figure}[H]
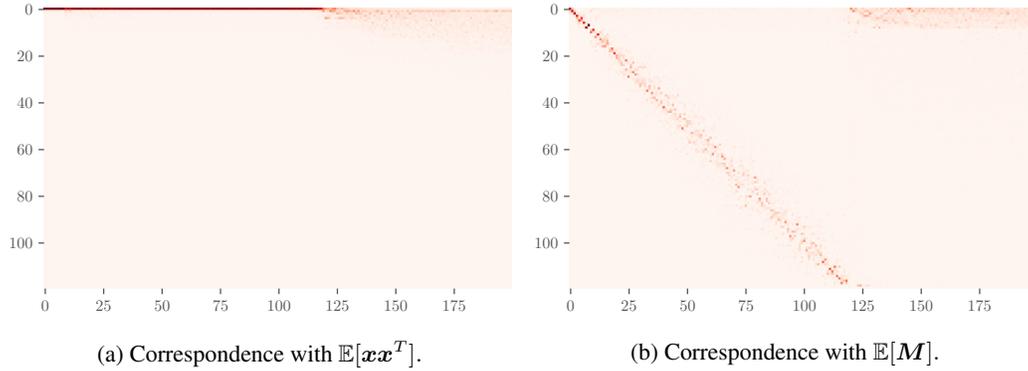

    \centering
    \begin{subfigure}[t]{0.5\textwidth}
        \centering
        \captionsetup{justification=centering}
        \includegraphics[width=\textwidth]{Figures/Correspondence/LeNet5_fixlr0.01/xxT_Trueest_real_corr_expand_t200_CIFAR10_Exp1_LeNet5_fixlr0.01R2_E-1_fc1.pdf}
        \caption{Correspondence with $\E[\vx\vx^T].$}
        \label{fig:Corr_xxT_True_fc1}
    \end{subfigure}%
    \begin{subfigure}[t]{0.5\textwidth}
        \centering
        \captionsetup{justification=centering}
        \includegraphics[width=\textwidth]{Figures/Correspondence/LeNet5_fixlr0.01/UTAU_Trueest_real_corr_expand_t200_CIFAR10_Exp1_LeNet5_fixlr0.01R2_E-1_fc1.pdf}
        \caption{Correspondence with $\E[\mM].$}
        \label{fig:Corr_UTAU_True_fc1}
    \end{subfigure}
    \caption{Eigenvector Correspondence for fc1:LeNet5. ($m$=120)}
    \label{fig:Corrfc11}
\end{figure}

\begin{figure}[H]
    \centering
    \begin{subfigure}[t]{0.5\textwidth}
        \centering
        \captionsetup{justification=centering}
        \includegraphics[width=\textwidth]{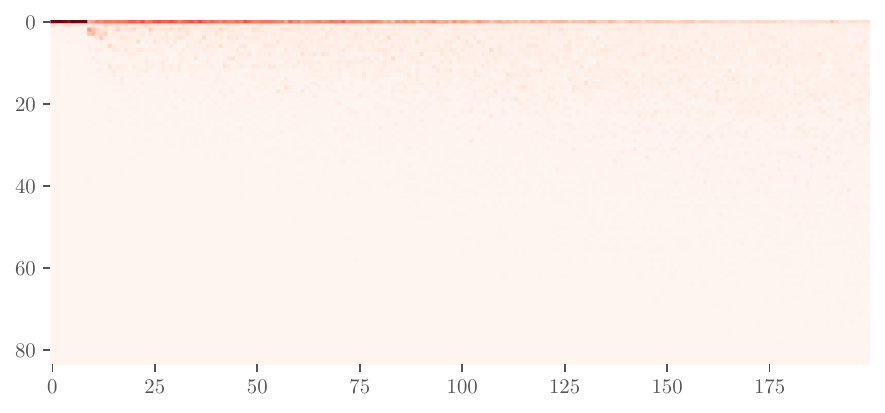}
        \caption{Correspondence with $\E[\vx\vx^T].$}
        \label{fig:Corr_xxT_True_fc2}
    \end{subfigure}%
    \begin{subfigure}[t]{0.5\textwidth}
        \centering
        \captionsetup{justification=centering}
        \includegraphics[width=\textwidth]{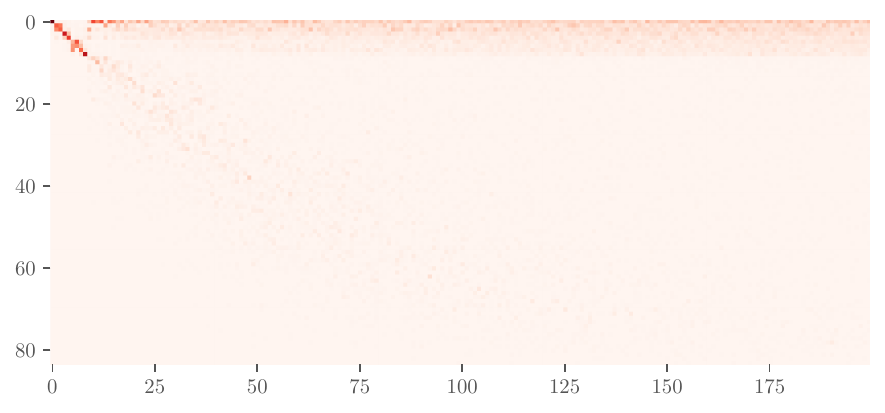}
        \caption{Correspondence with $\E[\mM].$}
        \label{fig:Corr_UTAU_True_fc2}
    \end{subfigure}
    \caption{Eigenvector Correspondence for fc2:LeNet5. ($m$=84)}
    \label{fig:Corrfc22}
\end{figure}

\begin{figure}[H]
    \centering
    \begin{subfigure}[t]{0.5\textwidth}
        \centering
        \captionsetup{justification=centering}
        \includegraphics[width=\textwidth]{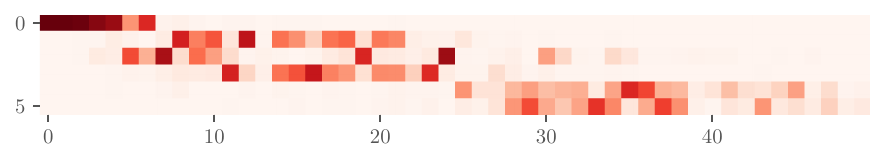}
        \caption{Correspondence with $\E[\vx\vx^T].$}
        \label{fig:Corr_xxT_True_conv1}
    \end{subfigure}%
    \begin{subfigure}[t]{0.5\textwidth}
        \centering
        \captionsetup{justification=centering}
        \includegraphics[width=\textwidth]{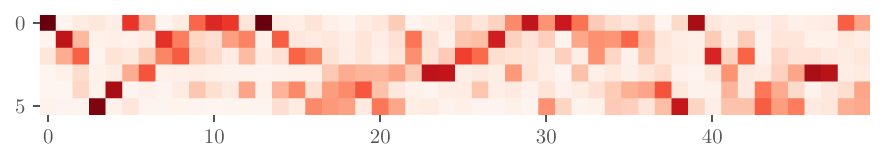}
        \caption{Correspondence with $\E[\mM].$}
        \label{fig:Corr_UTAU_True_conv1}
    \end{subfigure}
    \caption{Eigenvector Correspondence for conv1:LeNet5. ($m$=6)}
    \label{fig:Corr_conv1}
\end{figure}

\begin{figure}[H]
    \centering
    \begin{subfigure}[t]{0.5\textwidth}
        \centering
        \captionsetup{justification=centering}
        \includegraphics[width=\textwidth]{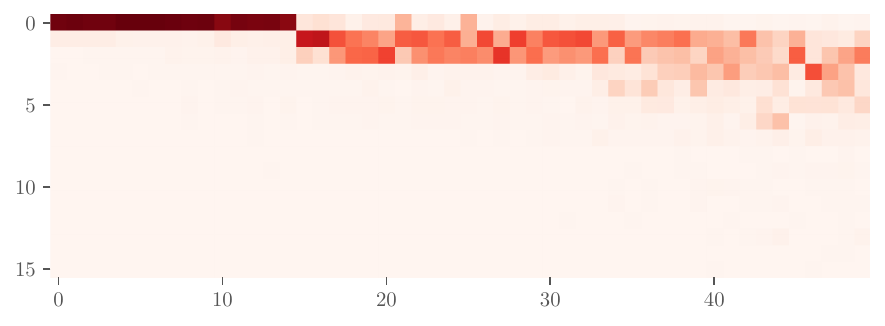}
        \caption{Correspondence with $\E[\vx\vx^T].$}
        \label{fig:Corr_xxT_True_conv2}
    \end{subfigure}%
    \begin{subfigure}[t]{0.5\textwidth}
        \centering
        \captionsetup{justification=centering}
        \includegraphics[width=\textwidth]{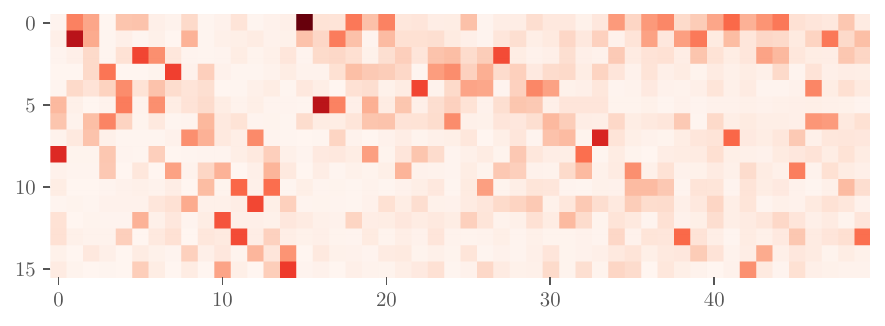}
        \caption{Correspondence with $\E[\mM].$}
        \label{fig:Corr_UTAU_True_conv2}
    \end{subfigure}
    \caption{Eigenvector Correspondence for conv2:LeNet5. ($m$=16)}
    \label{fig:Corr_conv2}
\end{figure}

For VGG11 we also observe a strong correlation with the first eigenvector of $\E[\vx\vx^\T]$.

\begin{figure}[H]
    \centering
    \captionsetup{justification=centering}
    \includegraphics[width=0.8\textwidth]{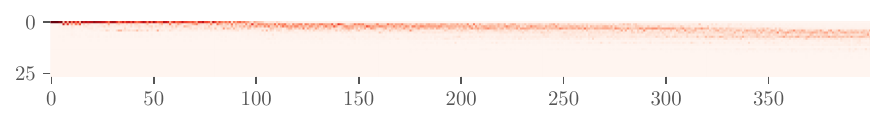}

    \caption{Eigenvector Correspondence with $\E[\vx\vx^\T]$ for conv1:VGG11. ($m$=64)}
    \label{fig:Corr_VGG_conv1}
\end{figure}

\begin{figure}[H]
    \centering
    \captionsetup{justification=centering}
    \includegraphics[width=0.8\textwidth]{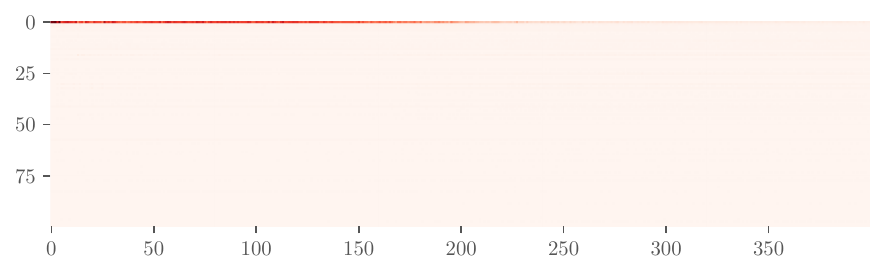}

    \caption{Eigenvector Correspondence with $\E[\vx\vx^\T]$ for conv2:VGG11. ($m$=128)}
    \label{fig:Corr_VGG_conv2}
\end{figure}

\begin{figure}[H]
    \centering
    \captionsetup{justification=centering}
    \includegraphics[width=0.8\textwidth]{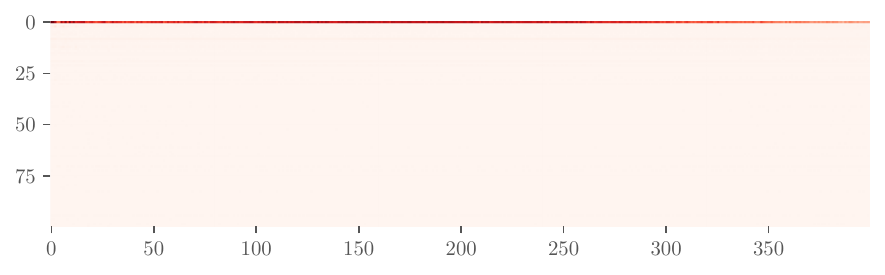}

    \caption{Eigenvector Correspondence with $\E[\vx\vx^\T]$ for conv3:VGG11. ($m$=256)}
    \label{fig:Corr_VGG_conv2}
\end{figure}

%% file: Appendix_Sections/appendix_exps/structure_traj.tex
\subsection{Structure of \texorpdfstring{$\E[\vx\vx^\T]$}{ExxT} and \texorpdfstring{$\E[\mM]$}{EM} During Training}
\label{sec:appendix_training_traj}
We observed the pattern of $\E[\vx\vx^\T]$ matrix and $\E[\mM]$ matrix along the training trajectory (\figureref{fig:traj_xxT_lenet5_fc1}, \figureref{fig:traj_UTAU_lenet5_fc1}). It shows that $\E[\vx\vx^\T]$ is always approximately rank 1, and $\E[\mM]$ always have around $c$ large eigenvalues. According to our analysis, since the nontrivial eigenspace overlap is likely to be a consequence of a approximately rank 1 $\E[\vx\vx^\T]$, we would conjecture that the overlap phenomenon is likely to happen on the training trajectory as well.

\begin{figure}[H]
    \centering
    \includegraphics[height=0.2\textheight]{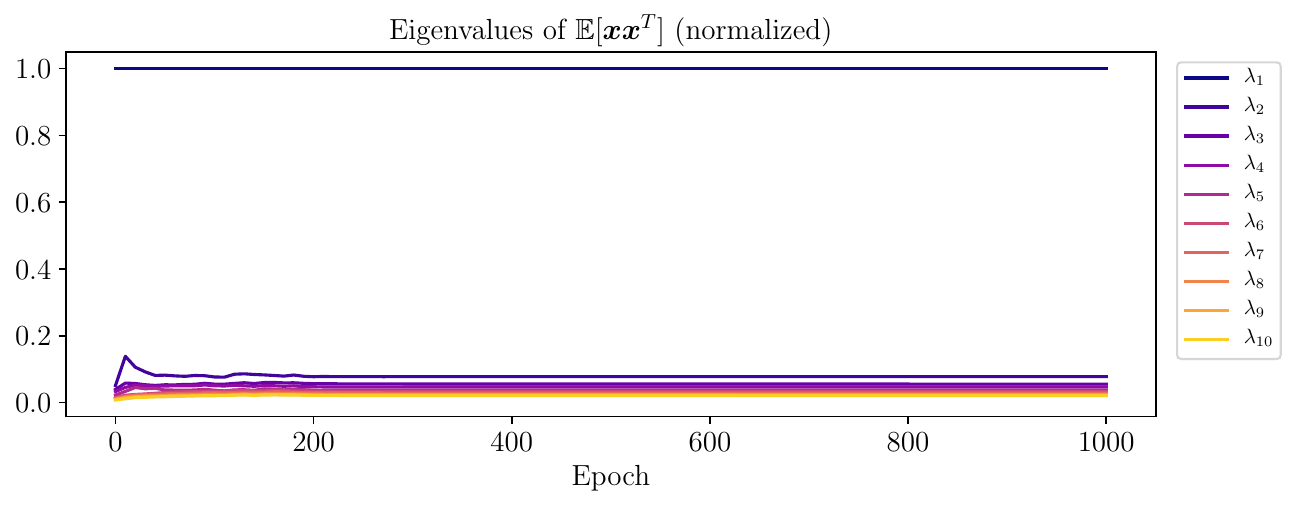}
    % \captionsetup{justification=centering}
    \caption{Top eigenvalues of $\E[\vx\vx^\T]$ along training trajectory. (fc1:LeNet5)}
    \label{fig:traj_xxT_lenet5_fc1}
\end{figure}

\begin{figure}[H]
    \centering
    \includegraphics[height=0.2\textheight]{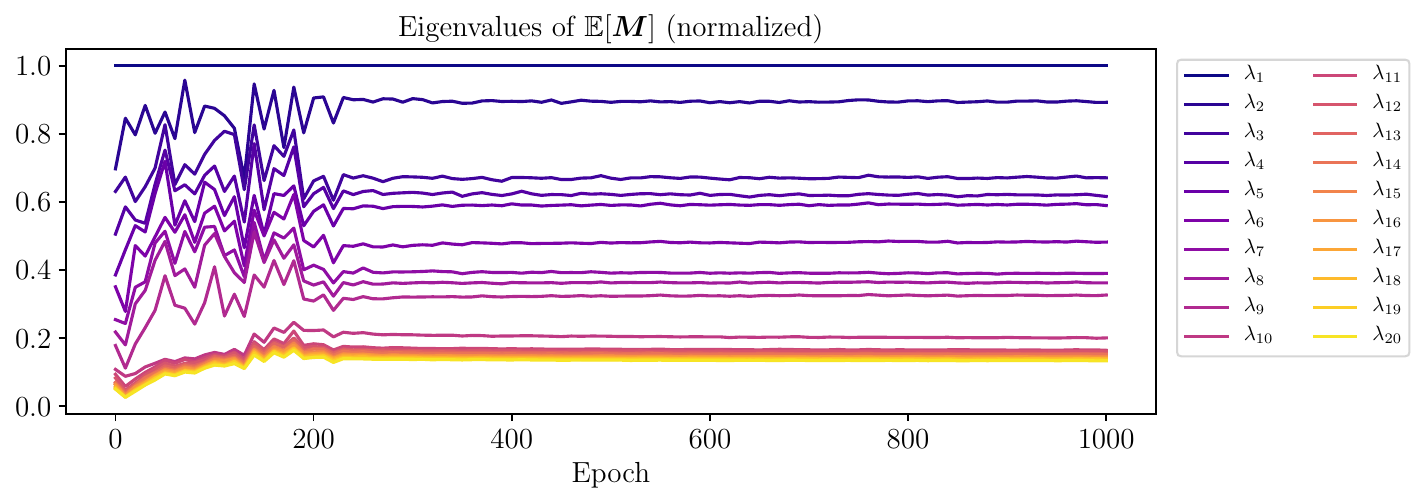}
    % \captionsetup{justification=centering}
    \caption{Top eigenvalues of $\E[\mM]$ along training trajectory. (fc1:LeNet5)}
    \label{fig:traj_UTAU_lenet5_fc1}
\end{figure}

%% file: Appendix_Sections/additional_explanation.tex
\newpage

\section{Additional Explanations}
\label{sec:appendix_explanation}

\input{Appendix_Sections/Output_Hessian_Proof/output_hessian_exp_result.tex}

\input{Appendix_Sections/appendix_exps/low_rank}

\input{Appendix_Sections/appendix_exps/overlap}
\input{Appendix_Sections/appendix_exps/batch_norm}

\input{Appendix_Sections/appendix_exps/M_structure}

%% file: Appendix_Sections/Output_Hessian_Proof/output_hessian_exp_result.tex
\subsection{Heuristic Approximating of the Top Eigenspace of Output Hessians}
\label{sec:app_outhessian_exp}

As briefly mentioned in \cref{sec:theoretical}, the closed form approximating of $S_1$ in \cref{thm:main-out} can be heuristically extended to the case with multiple layers, that the top eigenspace of the output Hessian of the $k$-layer would be approximately $\gR(\mS^{(k)})\setminus\{\textbf{1}^\T\mS^{(k)}\}$
where $\mS^{(k)} = \mW^{(n)}\mW^{(n-1)}\cdots\mW^{(k+1)}$ and $\gR(\mS^{(k)})$ is the row space of $\mS^{(k)}$.

Though our result was only proven for random initialization and random data, we observe that this subspace also has high overlap with the top eigenspace of output Hessian at the minima of models trained with real datasets. In \cref{tab:approx-m}  we show the overlap of $ \gR(\mS^{(k)})\setminus\{\textbf{1}^\T \mS^{(k)}\}$ and the top $c-1$ dimension eigenspace of $\E[\mM^{(k)}]$ of different layers at minima.

\begin{table}[H]
\vskip -0.15in
\caption{Overlap of $ \gR(\mS^{(k)})\setminus\{\textbf{1}^\T \mS^{(k)}\}$ and the top $c-1$ dimension eigenspace of $\E[\mM^{(k)}]$ of different layers.}
\vskip 0.1in
\begin{center}
\begin{small}
% \begin{sc}
\begin{tabular}{ccccccccc}
\toprule
Dataset & \multicolumn{2}{c}{MNIST} & \multicolumn{2}{c}{MNIST-R} & \multicolumn{2}{c}{CIFAR10} & \multicolumn{2}{c}{CIFAR10-R} \\
Network & F-$1500^3$    & LeNet5    & F-$1500^3$     & LeNet5     & F-$1500^3$     & LeNet5     & F-$1500^3$      & LeNet5     \\ \midrule
fc1     & 0.602         & 0.890     & 0.235          & 0.518      & 0.880          & 0.951      & 0.903           & 0.213       \\
fc2     & 0.967         & 0.931     & 0.801          & 0.912      & 0.943          & 0.972      & 0.931           & 0.701       \\
fc3     & 0.982         & 0.999     & 0.998          & 0.999      & 0.993          & 0.999      & 0.996           & 0.999     \\ \bottomrule
\end{tabular}
% \end{sc}
\end{small}
\end{center}
\label{tab:approx-m}
\vskip -0.15in
\end{table}
% \znote{This table can be compressed}
Note that the overlap can be low for random-label datasets which do not have a clear eigengap (as in \cref{fig:UTAU_H_spec}). Understanding how the data could change the behavior of the Hessian is an interesting open problem. Other papers including \citet{papyan2019measurements} have given alternative explanations which are not directly comparable to ours, however ours is the only one that gives a closed-form formula for top eigenspace. In \cref{sec:appendix_M_struct} we will discuss the other explanations in more details.

%% file: Appendix_Sections/appendix_exps/low_rank.tex
\subsection{Dominating Eigenvectors of Layer-wise Hessian are Low Rank}
\label{sec:appendix_low_rank_eigenvector}
% The structure of Hessians' eigenvectors is also important for assessing Hessian properties. 

A natural corollary for the Kronecker factorization approximation of layer-wise Hessians is that the eigenvectors of the layer-wise Hessians are low rank. Let $\vh_i$ be the $i$-th eigenvector of a layer-wise Hessian.
The rank of $\Mat(\vh_i)$ can be considered as an indicator of the complexity of the eigenvector. Consider the case that $\vh_i$ is one of the top eigenvectors. From \sectionref{sec:models}, we have $\vh_i \approx \vu_i \otimes \hE[\vx]$. % for some $\vu_i \in \R^m$.
Thus, $\Mat(\vh_i) \approx \vu_i\hE[\vx]^\T$, which is approximately rank 1. Experiments shows that first singular values of $\Mat(\vh_i)$ divided by its Frobenius Norm are usually much larger than 0.5, indicating the top eigenvectors of the layer-wise Hessians are very close to rank 1.
\figureref{fig:eigen_lowrank} shows first singular values of $\Mat(\vh_i)$ divided by its Frobenius Norm for $i$ from 1 to 200. We can see that the top eigenvectors of the layer-wise Hessians are very close to rank 1.
 \begin{figure}[h]
     \includegraphics[width=\columnwidth]{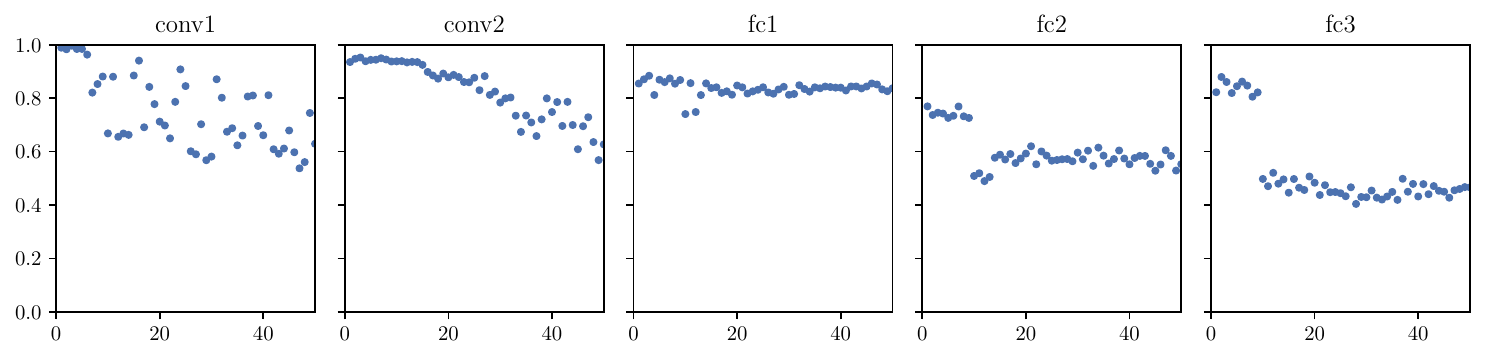}
    %  \captionsetup{justification=centering}
      \vspace{-0.2in}
     \caption{Ratio between top singular value and Frobenius norm of matricized dominating eigenvectors. (LeNet5 on CIFAR10). The horizontal axes denote the index $i$ of eigenvector $\vh_i$, and the vertical axes denote $\|\Mat(\vh_i)\|/\|\Mat(\vh_i)\|_F$.}
    \label{fig:eigen_lowrank}
     \vspace{-2em}
 \end{figure}

%% file: Appendix_Sections/appendix_exps/overlap.tex
\subsection{Eigenspace Overlap of Different Models}
\label{sec:appendix_model_overlap}
% \begin{figure}[ht]
%     \centering
%     \includegraphics[width=\textwidth]{Appendix_Figures/Overlap_different_models/DimOverlap_CIFAR10_LeNet5_normnew_fixlr0.01_appendix_full.pdf}
%     \caption{Eigenspace overlap of different models of LeNet5}
%     \label{fig:app_adexp_lenet5_fixlr0.01}
% \end{figure}
% %In this section we explain
% %\begin{enumerate}
%     %\item linear growth (in general) and why no linear growth for fc3
%     %\item why fc2 bad (M too low rank)
%     %\item why conv1 high
%     %\item why fc2 still exhibits a linear segment at the begining.
% %\end{enumerate}
% \figureref{fig:app_adexp_lenet5_fixlr0.01} shows the average pairwise overlap between the top eigenspaces of the layer-wise Hessian of 5 different LeNet5 models.
From the experiment results in \sectionref{sec:appendix_exp_res} together with \figureref{fig:overlap}, we can see that our approximation and explanation stated in \sectionref{sec:models} of the main text is approximately correct but may not be so accurate for some layers.
We now present a more general explanation which addresses why the overlap before rank-$m$ grows linearly. We will also explain some exceptional cases as shown in \sectionref{sec:appendix_expres_ovlp} and possible discrepancies of our approximation.

Let $\vh_i$ be the $i$-th eigenvector of the layer-wise Hessian $\mH_\Ls(\vw^{(p)})$, under the assumption that the autocorrelation matrix $\E[\vx\vx^\T]$ is approximately rank 1 that $\E[\vx\vx^\T]\approx \E[\vx]\E[\vx]^\T$, for all $i \leq m$, we can approximate the $\vh_i$ as $\vu_i\otimes (\E[\vx]/\|\E[\vx]\|)$ where $\vu_i$ is the $i$-th eigenvector of $\E[\mM]$. Formally, the trend of top eigenspace can be characterized by the following theorem. For simplicity of notations, we abuse the superscript within parentheses to refer the two models instead of layer number in this section.
% If this approximation is reasonably accurate, the eigenspace overlap pattern is explained. However, this is actually not a necessary condition for this pattern and we can loose this requirement while still explaining the phenomenon.

% According to \figureref{fig:Corr_fc} and \sectionref{sec:app_exp_corr}, for the top $m$ eigenvectors of layer-wise Hessian, we can see that the approximation is usually more accurate for the $\E[\vx\vx^\T]$ part than the $\E[\mM]$ part and $\vh_i$ usually have a high correspondence with the top eigenvector of $\E[\vx\vx^\T]$. Indeed, this is the only condition we need. We can then have this theorem, with $\Corr(\vt, \vh_i) = \|\Mat(\vh_i)\vt\|^2$ as in \cref{subsec:correspondence}.
\begin{theorem}
Consider 2 different models with the same network structure trained on the same dataset. Fix the $p$-th hidden layer with input dimension $n$ and output dimension $m$. For the first model, denote its output Hessian as $\E[\mM]^{(1)}$ with eigenvalues $\tau^{(1)}_1 \geq \tau^{(1)}_2 \geq \cdots \geq \tau^{(1)}_m \geq 0$ and eigenvectors $\vr^{(1)}_1, \cdots, \vr^{(1)}_m\in\R^m$; denote its autocorrelation matrix as $\E[\vx\vx^\T]^{(1)}$, with eigenvalues $\gamma^{(1)}_1 \geq \gamma^{(1)}_2 \geq \cdots \geq \gamma^{(1)}_m \geq 0$ and eigenvectors $\vt^{(1)}_1, \cdots, \vt^{(1)}_n\in\R^n$. The variables for the second matrices are defined identically by changing 1 in the superscript parenthesis to 2.

Assume the Kronecker factorization approximation is accurate that $\mH_\Ls(\vw^{(p)})^{(1)}\approx \E[\mM]^{(1)}\otimes \E[\vx\vx^\T]^{(1)}$ and $\mH_\Ls(\vw^{(p)})^{(2)}\approx \E[\mM]^{(2)}\otimes \E[\vx\vx^\T]^{(2)}$.
Also assume the autocorrelation matrices of two models are sufficiently close to rank 1 in the sense that $\tau^{(1)}_m\gamma^{(1)}_1 > \tau^{(1)}_1\gamma^{(1)}_2$ and $\tau^{(2)}_m\gamma^{(2)}_1 > \tau^{(2)}_1\gamma^{(2)}_2$.
Then for all $k\leq m$, the overlap of top $k$ eigenspace between their layerwise Hessians $\mH_\Ls(\vw^{(p)})^{(1)}$ and $\mH_\Ls(\vw^{(p)})^{(2)}$ will be approximately 
$\frac{k}{m}(\vt^{(1)}_1 \cdot \vt^{(2)}_1)^2.$
Consequently, the top eigenspace overlap will show a linear growth before it reaches dimension $m$. The peak at $m$ is approximately $(\vt_1 \cdot \vt_2)^2$.
\label{thm:model_overlap}
\end{theorem}

% Note that if this theorem holds, then

\begin{proofof}{\theoremref{thm:model_overlap}}
Let $\vh^{(2)}_i$ be the $i$-th eigenvector of the layer-wise Hessian for the first model $\mH_\Ls(\vw^{(p)})^{(1)}$, and $\vg_i$ be that of the second model $\mH_\Ls(\vw^{(p)})^{(2)}$. Consider the first model. By the Kronecker factorization approximation, since $\tau^{(1)}_m\gamma^{(1)}_1 > \tau^{(1)}_1\gamma^{(1)}_2$, the top $m$ eigenvalues of the layer-wise Hessian are $\gamma^{(1)}_1\tau^{(1)}_1,\cdots, \gamma^{(1)}_1\tau^{(1)}_m$. Consequently, for all $i \leq m$ we have $\vh_i \approx \vr^{(1)\T}_i\otimes \vt^{(1)}_1$.
Thus, for any $k \leq m$, we have its top $k$ eigenspace as $\mV_k^{(1)} \otimes \vt_1^{(1)}$, where $\mV^{(1)}_k\in\R^{m\times k}$ has column vectors $\vr_1^{(1)}, \ldots, \vr_k^{(1)}$.
Similarly, for the second model we have $\vh^{(2)}_i \approx \vr^{(2)}_i\otimes \vt^{(2)}_1$ and the top $k$ eigenspace as $\mV^{(2)}_k \otimes \vt^{(2)}_1$, where $\mV^{(2)}_k$ has column vectors $\vr^{(2)}_1, \ldots, \vr^{(2)}_k$.
The eigenspace overlap of the 2 models at dimension $k$ is thus

\begin{align}
    \begin{split}
        \Overlap\left(\mV^{(1)}_k \otimes \vt^{(1)}_1, \mV^{(2)}_k \otimes \vt^{(2)}_1\right) &=\frac{1}{k}{\left\|\mV_k^{(1)\T}\mV^{(2)}_k \otimes \vt^{(1)\T}_1\vt^{(2)}_1\right\|}^2_F\\
    &= {\left(\vt^{(1)}_1 \cdot \vt^{(2)}_1\right)}^2\Overlap\left(\mV^{(1)}_k, \mV^{(2)}_k\right).
    \end{split}
\label{eqn:appendix_model_overlap}
\end{align}

Note that for all $i\leq m$, $\vr^{(1)}_i, \vr^{(2)}_i \in \R^n$, which is the space corresponding to the neurons. Since for hidden layers, the output neurons (channels for convolutional layers) can be arbitrarily permuted to give equivalent models while changing eigenvectors. For $\vh_i \approx \vr_i\otimes \vt_1$, permuting neurons will permute entries in $\vr_i$. Thus, we can assume that for two models, $\vr^{(1)}_i$ and $\vr^{(2)}_i$ are not correlated and thus have an expected inner product of $\sqrt{1/m}$.
%Thus, even the 2 models are equivalent, $\vs_i$ can be equal to $\vr_i$ with entries being randomly permuted.

It follows from \definitionref{def:overlap} that 
\begin{equation}
    \E[\Overlap(\mV^{(1)}_k, \mV^{(2)}_k)] = \sum_{i=1}^k\E[{(\vr_i^{(1)}\cdot\vr_i^{(2)})}^2] = k(\frac{1}{m}) = \frac{k}{m}
\end{equation}
and thus the eigenspace overlap of at dimension $k$ would be approximately $\frac{k}{m}(\vt^{(1)}_1 \cdot \vt^{(2)}_1)^2$. This explains the peak at dimension $m$ and the linear growth before it.
\end{proofof}

From our results on autocorrelation matrices in \sectionref{sec:xxT} and \sectionref{sec:appendix_xxT}, we have $\hE[\vx]^{(1)}\approx \vt^{(1)}_1$ and $\hE[\vx]^{(2)}\approx \vt^{(2)}_1$ where $\hE$ is the normalized expectation. Hence when $k=m$, the overlap is approximately $(\hE[\vx]^{(1)} \cdot \hE[\vx]^{(2)})^2$. Since $\hE[\vx]^{(1)}$ and $\hE[\vx]^{(2)}$ are the identical for the input layers, the overlap is expected to be very high at dimension $m$ for input layers. For other hidden layers in a ReLU network, $\vx$ are output of ReLU and thus non-negative. Two non-negative vectors $\hE[\vx]^{(1)}$ and $\hE[\vx]^{(2)}$ still have relatively large dot product, which contributes to the high overlap peak.

\subsubsection{The Decreasing Overlap After Output Dimension}
\label{sec:app_ovlp_dec}
Consider the $(m+1)$-th eigenvector $\vh^{(1)}_{m+1}$ of the first model. Following the Kronecker factorization approximation and assumptions in \theoremref{thm:model_overlap}, we have $\vh^{(1)}_{m+1}\approx \vr^{(1)}_1\otimes \vt^{(1)}_2$. Since top $m$ eigenspace of the first model is approximately $\mI_m \otimes \vt^{(1)}_1$ and $\vt^{(1)}_2$ is orthogonal to $\vt^{(1)}_1$, the $\vh^{(1)}_{m+1}$ eigenvector will be orthogonal to the top $m$ eigenspace of the first model. It will also have low overlap with $\mI_m \otimes \vt^{(2)}_1$ since $(\hE[\vx]^{(1)} \cdot \hE[\vx]^{(2)})^2$ is large.

Moreover, since the remaining eigenvectors of the autocorrelation matrix no longer has the all positive property as the first eigenvector and structure of the convariance $\Sigma_\vx$ is directly associated with the ordering of the input neurons which are randomly permuted across different models, the overlap between other eigenvectors of the autocorrelation matrix across different models will be close to random, hence the overlap after the top $m$ dimension will decrease until the eigenspaces has sufficiently many basis vectors to make the random overlap large.

\subsubsection{The Output Layer}
Note that for the last layer satisfying the assumptions in \theoremref{thm:model_overlap}, the overlap will stay high before dimension $m$ and be approximately $(\vt_1 \cdot \vt_2)^2$ since the output neurons directly correspondence to classes, and hence neurons cannot be permuted.
In this case, the overlap will be approximately $(\vt_1 \cdot \vt_2)^2$ for all dimension $k\leq m$. This is consistent with our observations.

% \begin{figure}[H]
%     \centering
%     % \vspace{-1em}
%     \subfigure[\centering\small{fc3:F-$200^2$\\(MNIST)}]{\includegraphics[width=0.23\linewidth]{Appendix_Figures/Overlap_large_model/overlap_raw/last_layer/DimOverlap_MNIST_FC2_fixlr0.01_fc3.pdf}}
%     \subfigure[\centering\small{fc3:LeNet5\\(CIFAR10)}]{\includegraphics[width=0.23\linewidth]{Appendix_Figures/Overlap_large_model/overlap_raw/last_layer/DimOverlap_CIFAR10_LeNet5_normnew_fixlr0.01_fc3.pdf}}
%     \subfigure[\centering\small{fc1:VGG11-W200\\(CIFAR10)}]{\includegraphics[width=0.23\linewidth]{Appendix_Figures/Overlap_large_model/overlap_raw/last_layer/CIFAR10_VGG11W200_fxlr0.01_fc1.pdf}}
%     \subfigure[\centering\small{fc1:ResNet18-W64\\(CIFAR100)}]{\includegraphics[width=0.23\linewidth]{Appendix_Figures/Overlap_large_model/overlap_raw/last_layer/DimOverlap_CIFAR100_Resnet18W64New_nobn_fixlr0.01_fc1.pdf}}
%     \caption{Top eigenspace overlap for the final fully connected layer.}
%     \label{fig:app_adexp_last}
% \end{figure}

\begin{figure}[H]
    \centering
    \begin{subfigure}[b]{0.24\textwidth}
        \centering
        \captionsetup{justification=centering}
        \includegraphics[width=\textwidth]{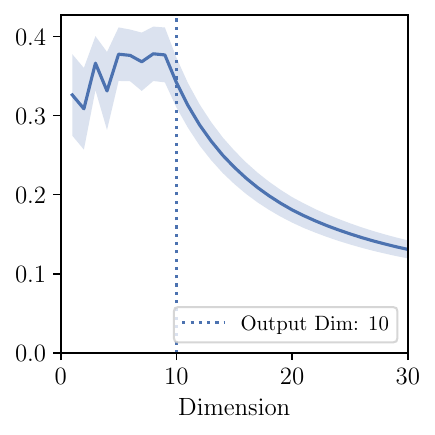}
        \caption{fc3:F-$200^2$\\(MNIST)}
        \label{fig:app_adexp_last_fc2}
    \end{subfigure}
    \begin{subfigure}[b]{0.24\textwidth}
        \centering
        \captionsetup{justification=centering}
        \includegraphics[width=\textwidth]{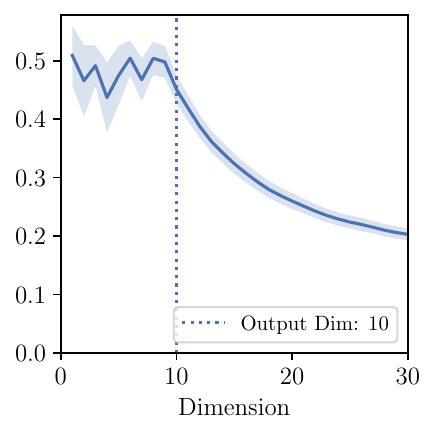}
        \caption{fc3:LeNet5\\(CIFAR10)}
        \label{fig:app_adexp_last_LeNet}
    \end{subfigure}
    \begin{subfigure}[b]{0.24\textwidth}
        \centering
        \captionsetup{justification=centering}
        \includegraphics[width=\textwidth]{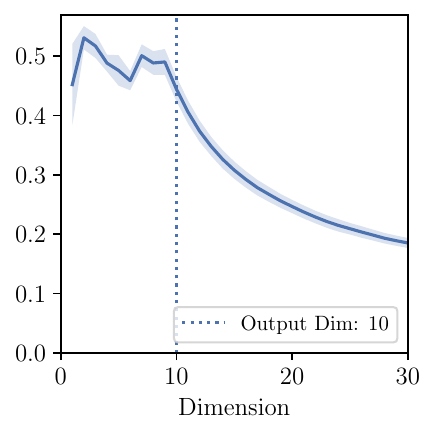}
        \caption{fc1:VGG11-W200\\(CIFAR10)}
        \label{fig:app_adexp_last_vgg}
    \end{subfigure}
    \begin{subfigure}[b]{0.24\textwidth}
        \centering
        \captionsetup{justification=centering}
        \includegraphics[width=\textwidth]{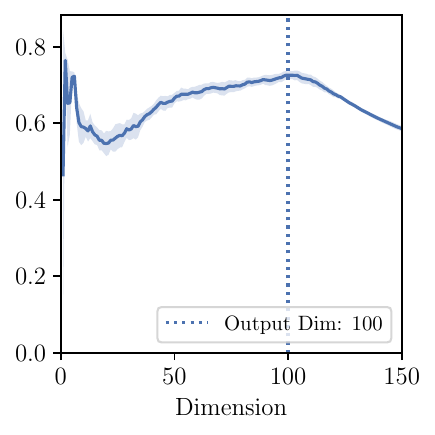}
        \caption{fc1:ResNet18-W64\\(CIFAR100)}
        \label{fig:app_adexp_last_resnet}
    \end{subfigure}
    \captionsetup{justification=centering}
    \caption{Top eigenspace overlap for the final fully connected layer.}
    \label{fig:app_adexp_last}
\end{figure}

% \newpage

\subsubsection{Explaining ``Failed Cases'' of Eigenspace Overlap}
\label{sec:appendix-failure}
As shown in \figureref{fig:app_adexp_failure_early} and \figureref{fig:app_adexp_failure_late}, the nontrivial top eigenspace overlap does not necessarily peak at the output dimension for all layers. Some layers has a low peak at very small dimensions and others has a peak at a larger dimension. With the more complete analysis provided above, we now proceed to explain these two phenomenons.
The major reason for such phenomenons is that the assumption of autocorrelation matrix being sufficiently close to rank 1 is not always satisfied. In particular, following the notations in \theoremref{thm:model_overlap}, for these exceptional layers we have $\tau_m\gamma_1 < \tau_1\gamma_2$.
We first consider the first phenomenon (early peak of low overlap) and take fc2:F-$200^2$ (MNIST) in as an example. Here \figureref{fig:app_adexp_fc2}(a) is identical to \figureref{fig:app_adexp_failure_early}(a), which displays the early peak around $m=10$.

\begin{figure}[H]
    \centering
    \begin{subfigure}[b]{0.24\textwidth}
        \centering
        \captionsetup{justification=centering}
        \includegraphics[width=\textwidth]{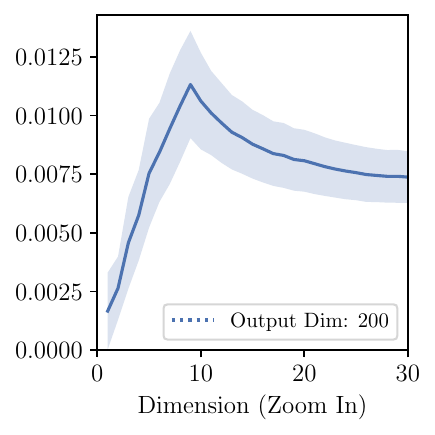}
        \caption{Eigenspace overlap (zoomed in)}
        \label{fig:app_adexp_fc2_ovlp}
    \end{subfigure}%
    \begin{subfigure}[b]{0.24\textwidth}
        \centering
        \captionsetup{justification=centering}
        \includegraphics[width=\textwidth]{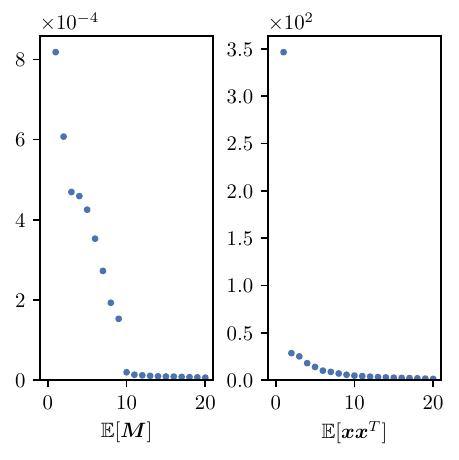}
        \caption{Eigenspectrum of $\E[\mM]$ and $\E[\vx\vx^\T]$}
        \label{fig:app_adexp_fc2_sig}
    \end{subfigure}%
    \begin{subfigure}[b]{0.24\textwidth}
        \centering
        \captionsetup{justification=centering}
        \includegraphics[width=\textwidth]{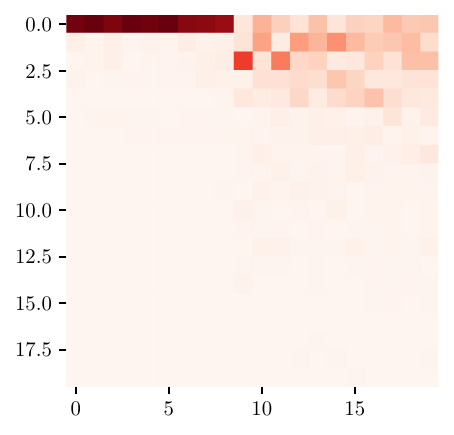}
        \caption{True Hessian with $\E[\vx\vx^\T]$}
        \label{fig:app_adexp_fc2_corr_real}
    \end{subfigure}
    \begin{subfigure}[b]{0.24\textwidth}
        \centering
        \captionsetup{justification=centering}
        \includegraphics[width=\textwidth]{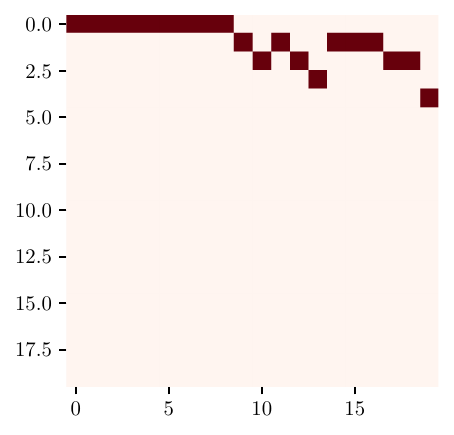}
        \caption{Approximated Hessian with $\E[\vx\vx^\T]$}
        \label{fig:app_adexp_fc2_corr_est}
    \end{subfigure}
    \captionsetup{justification=centering}
    \caption{Eigenspace overlap, eigenspectrum, and cropped (upper $20\times 20$ block)\\eigenvector correspondence matrices for fc2:F-$200^2$ (MNIST)}
    \vspace{-0.1in}
    \label{fig:app_adexp_fc2}
\end{figure}

\begin{figure}[H]
    \centering
    \begin{subfigure}[b]{0.24\textwidth}
        \centering
        \captionsetup{justification=centering}
        \includegraphics[width=\textwidth]{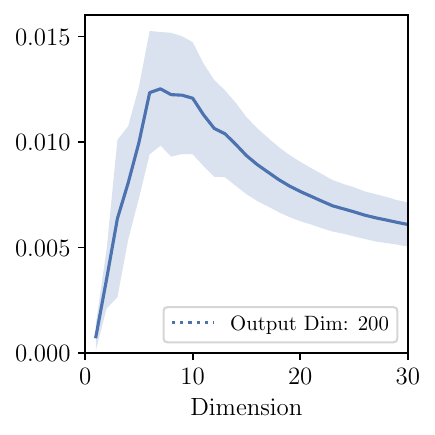}
        \caption{Eigenspace overlap (zoomed in)}
        \label{fig:app_adexp_vgg_ovlp}
    \end{subfigure}%
    \begin{subfigure}[b]{0.24\textwidth}
        \centering
        \captionsetup{justification=centering}
        \includegraphics[width=\textwidth]{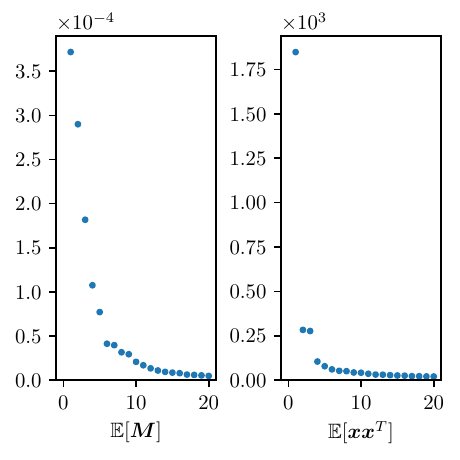}
        \caption{Eigenspectrum of $\E[\mM]$ and $\E[\vx\vx^\T]$}
        \label{fig:app_adexp_vgg_sig}
    \end{subfigure}%
    \begin{subfigure}[b]{0.24\textwidth}
        \centering
        \captionsetup{justification=centering}
        \includegraphics[width=\textwidth]{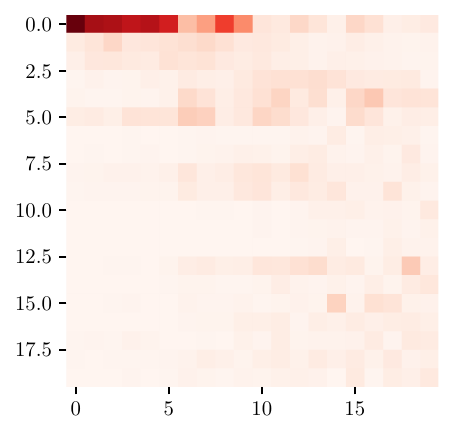}
        \caption{True Hessian with $\E[\vx\vx^\T]$}
        \label{fig:app_adexp_vgg_corr_real}
    \end{subfigure}
    \begin{subfigure}[b]{0.24\textwidth}
        \centering
        \captionsetup{justification=centering}
        \includegraphics[width=\textwidth]{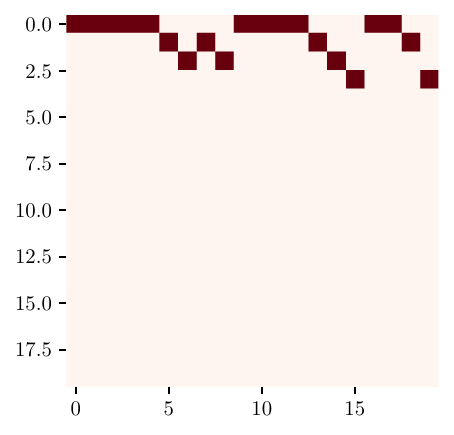}
        \caption{Approximated Hessian with $\E[\vx\vx^\T]$}
        \label{fig:app_adexp_vgg_corr_est}
    \end{subfigure}
    \captionsetup{justification=centering}
    \caption{Eigenspace overlap, eigenspectrum, and cropped (upper $50\times 50$ block)\\eigenvector correspondence matrices for conv2:VGG11-W200 (CIFAR10)}
    \vspace{-0.1in}
    \label{fig:app_adexp_vgg_fail}
\end{figure}

As shown in \figureref{fig:app_adexp_fc2}(b), the second eigenvalue of the auto correlation $\E[\vx\vx^\T]$ is as large as approximately 1/10 of the first eigenvalue.
With the output Hessian have $c-1=9$ significant large eigenvalues as described in \label{sec:emp_outlier}, it has $\tau_{10}\gamma_1 < \tau_1\gamma_2$.
Thus through the Kronecker factorization approximation, the top $m$ dimensional eigenspace is no longer simply $\mI_m\otimes \hE[\vx]$, but a subset of top eigenvectors of the output Hessian Kroneckered with a subset of top eigenvectors of $\E[\vx\vx^\T]$ as reflected in \figureref{fig:app_adexp_fc2}(d). This ``mixture'' of Kronecker product is moreover verified in \figureref{fig:app_adexp_fc2}(c).

As reflected by the first row of \figureref{fig:app_adexp_fc2}(c) and \figureref{fig:app_adexp_fc2}(d), for $i\leq 9$ we have $\vh_i\approx \vr_i\otimes\hE[\vx]$, which falls in the regime of \theoremref{thm:model_overlap}. Hence we are seeing an linearly growing pattern of the overlap for dimension less than 10 and reaches a mean overlap of around 0.012 by dimension 9. If following this linear trend, the overlap would be close to 0.25 by the output dimension of 200. However, since the 10-th eigenvalue of the output Hessian is significantly smaller, little of the 10-19 dimensional eigenspace were contributed by $\hE[\vx]$, hence the overlap of dimension larger than 10 falls into the regime discussed in \sectionref{sec:app_ovlp_dec}, for which we see a sharp decrease of overlap after dimension 9.
Note that this example shows that Kronecker factorization can be used to predict when our conditions in \theoremref{thm:model_overlap} fails and also predict the condition can be satisfied up to which dimension. As shown in \figureref{fig:app_adexp_vgg_fail}, similar explanation also applies to convolutional layers in larger networks.

We then consider the second phenomenon (delayed peak) and take conv2:VGG11-W200 (CIFAR10) in as an example. Here \figureref{fig:app_adexp_vgg2}(a) is identical to \figureref{fig:app_adexp_failure_late}(d), which has the overlap peak later than the output dimension 200. In this case, the second eigenvalue of the auto correlation matrix is still not negligible compared to the top eigenvalue. What differentiate this case from the first phenomenon is that the eigenvalues of the output Hessian no longer has a significant peak \--- instead it has a heavy tail which is necessary for high overlap.

Towards dimension $m$ there gradually exhibits higher correspondence to later eigenvectors of the input autocorrelation matrix and hence less correspondence to $\hE[\vx]$. This eventually results in the delayed and flattened peak.

\begin{figure}[H]
    \centering
    \begin{subfigure}[b]{0.24\textwidth}
        \centering
        \captionsetup{justification=centering}
        \includegraphics[width=\textwidth]{Appendix_Figures/Overlap_large_model/FailCases/late/CIFAR10_VGG11W200_fxlr0.01_conv2.pdf}
        \caption{Eigenspace overlap (zoomed in)}
        \label{fig:app_adexp_vgg2_ovlp}
    \end{subfigure}%
    \begin{subfigure}[b]{0.24\textwidth}
        \centering
        \captionsetup{justification=centering}
        \includegraphics[width=\textwidth]{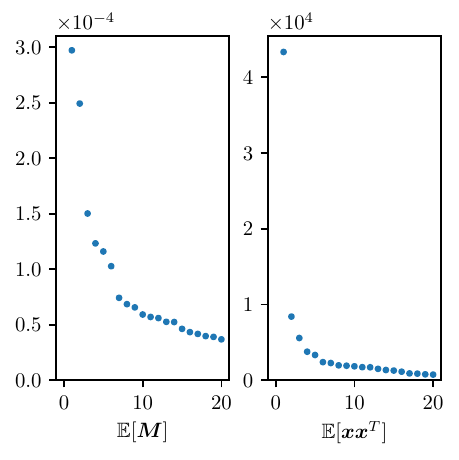}
        \caption{Eigenspectrum of $\E[\mM]$ and $\E[\vx\vx^\T]$}
        \label{fig:app_adexp_vgg2_sig}
    \end{subfigure}%
    \begin{subfigure}[b]{0.24\textwidth}
        \centering
        \captionsetup{justification=centering}
        \includegraphics[width=\textwidth]{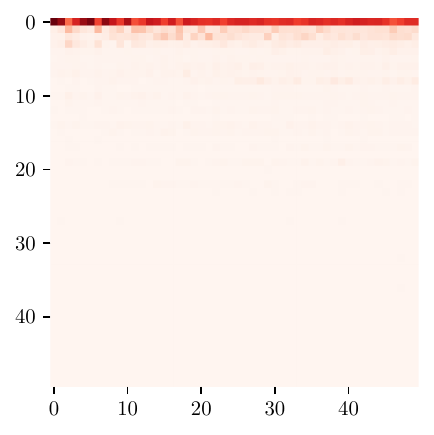}
        \caption{True Hessian with $\E[\vx\vx^\T]$}
        \label{fig:app_adexp_vgg2_corr_real}
    \end{subfigure}
    \begin{subfigure}[b]{0.24\textwidth}
        \centering
        \captionsetup{justification=centering}
        \includegraphics[width=\textwidth]{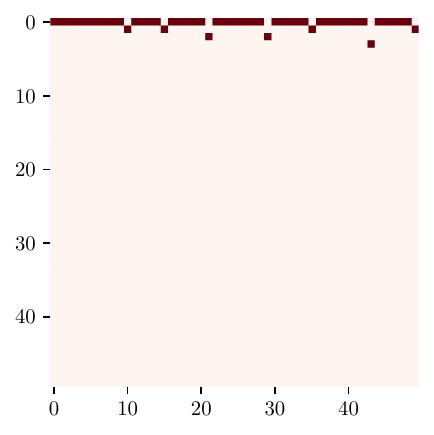}
        \caption{Approximated Hessian with $\E[\vx\vx^\T]$}
        \label{fig:app_adexp_vgg2_corr_est}
    \end{subfigure}\\
    \vspace{0.15in}
    \begin{subfigure}[b]{0.49\textwidth}
        \centering
        \captionsetup{justification=centering}
        \includegraphics[width=\textwidth]{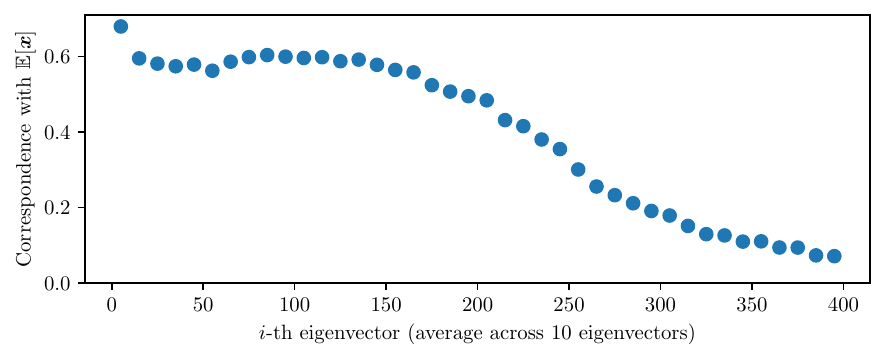}
        \caption{First Row of Correspondence Matrix of True Hessian with $\E[\vx\vx^\T]$}
        \label{fig:app_adexp_vgg2_corr_true_firstline}
    \end{subfigure}
    \begin{subfigure}[b]{0.49\textwidth}
        \centering
        \captionsetup{justification=centering}
        \includegraphics[width=\textwidth]{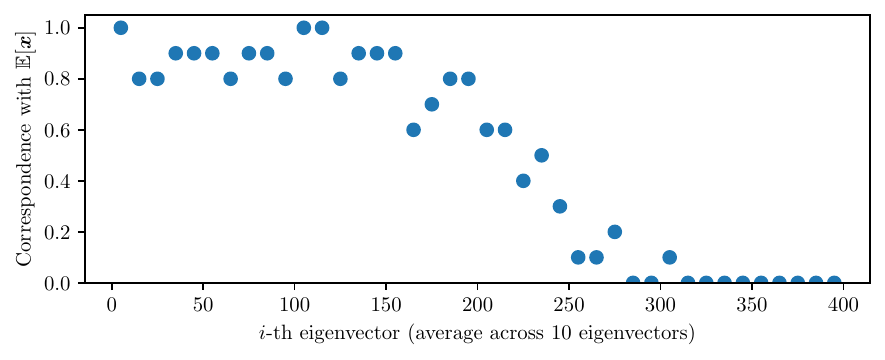}
        \caption{First Row of Correspondence Matrix of Approximated Hessian with $\E[\vx\vx^\T]$}
        \label{fig:app_adexp_vgg2_corr_est_firstline}
    \end{subfigure}
    \captionsetup{justification=centering}
    \caption{Eigenspace overlap, eigenspectrum, and cropped (upper $50\times 50$ block)\\eigenvector correspondence matrices for conv2:VGG11-W200 (CIFAR10)}
    \vspace{-0.1in}
    \label{fig:app_adexp_vgg2}
\end{figure}

Since the full correspondence matrices are too large to be visualized, we plotted their first rows up to 400 dimensions in \figureref{fig:app_adexp_vgg2}(e) and \figureref{fig:app_adexp_vgg2}(f), in which each dot represents the average of correlation with $\hE[\vx]$ for the 10 eigenvector nearby. From these figures it is straightforward to see the gradual decreasing correlation with $\hE[\vx].$

%% file: Appendix_Sections/appendix_exps/batch_norm.tex
\subsection{Batch Normalization and Zero-mean Input}
\label{sec:appendix_batchnorm}
In this section, we show the results on networks with using Batch normalization (BN) \citep{ioffe2015batch}. For layers after BN, we have $\E[\vx]\approx 0$ so that $\E[\vx]\E[\vx]^\T$ no longer dominates $\mSigma_\vx$ and the low rank structure of $\E[\vx\vx^\T]$ should disappear. Thus, we can further expect that the overlap between top eigenspace of layer-wise Hessian among different models will not have a peak.

\tableref{tab:appendix_bn_xxT} shows the same experiments done in \tableref{tab:appendix_xxT_spec_fc}. The values for each network are the average of 3 different models. It is clear that the high inner product and large spectral ratio both do not hold here, except for the first layer where there is no normalization applied. Note that we had channel-wise normalization (zero-mean for each channel but not zero-mean for $\vx$) for conv1 in LeNet5 so that the spectral ratio is also small.
\begin{table}[ht]
\small
    \centering
    \caption{Structure of $\E[\vx\vx^\T]$ for BN networks}
    \vskip 0.1in
    \begin{center}
    \begin{small}
\begin{tabular}{llccccccc}
\toprule
        &              &       & \multicolumn{3}{c}{$(v_1^\T\hE[\vx])^2$} & \multicolumn{3}{c}{$\lambda_1/\lambda_2$} \\
Dataset & Network      & \# fc & mean        & min         & max         & mean         & min          & max         \\
\midrule
MNIST   & F-$200^2$-BN & 2     & 0.062       & 0.001       & 0.260       & 1.16         & 1.04         & 1.30        \\
        & F-$600^2$-BN & 2     & 0.026       & 0.000       & 0.063       & 1.13         & 1.02         & 1.26        \\
        & F-$600^4$-BN & 4     & 0.027       & 0.000       & 0.146       & 1.11         & 1.03         & 1.19        \\
        \midrule
CIFAR10 & LeNet5-BN & 3     & 0.210       & 0.001       & 0.803       & 1.54         & 1.20         & 1.89        \\
        \bottomrule
\end{tabular}
    \end{small}
    \end{center}
  \label{tab:appendix_bn_xxT}%
\end{table}

\figureref{fig:app_exp_bn_overlap}(a) shows that $\E[\vx\vx^\T]$ is no longer close to rank 1 when having BN. This is as expected. However, $\E[\vx\vx^\T]$ still has a few large eigenvalues.

\figureref{fig:app_exp_bn_overlap}(b) shows the eigenvector correspondence matrix of True Hessian with $\E[\vx\vx^\T]$ for fc1:LeNet5. Because $\E[\vx\vx^\T]$ is no longer close to rank 1, only very few eigenvectors of the layer-wise Hessian will have high correspondence with the top eigenvector of $\E[\vx\vx^\T]$, as expected. This directly leads to the disappearance of peak in top eigenspace overlap of different models, as shown in \figureref{fig:app_exp_bn_overlap}. The peak still exists in conv1 because BN is not applied to the input.

% \begin{figure}[th]
%     \centering
%     \vspace{-1em}
%     \subfigure[\small{Eigenspectrum of  $\E[\vx\vx^\T]$}]{\includegraphics[width=0.48\linewidth]{Appendix_Figures/Explanation_LeNet5Case/BN/sigvals_xxt_t20_CIFAR10_Exp1_LeNet5_BN_nl_fixlr0.01R2_E-1.pdf}}
%     \subfigure[\small{True Hessian (fc1)}]{\includegraphics[width=0.24\linewidth]{Appendix_Figures/Explanation_LeNet5Case/BN/xxT_Trueest_real_corr_expand_t20_CIFAR10_Exp1_LeNet5_BN_nl_fixlr0.01R2_E-1_fc1.pdf}}
%     \subfigure[\small{Approx Hessian (fc1)}]{\includegraphics[width=0.24\linewidth]{Appendix_Figures/Explanation_LeNet5Case/BN/xxT_Approxest_real_corr_expand_t20_CIFAR10_Exp1_LeNet5_BN_nl_fixlr0.01R2_E-1_fc1.pdf}}
%     \caption{Eigenspectrum and Eigenvector correspondence matrices with $\E[\vx\vx^T]$ for LeNet5-BN.}
%     \label{fig:app_exp_bn_overlap}
% \end{figure}

\begin{figure}[ht]
    \centering
\begin{subfigure}[b]{0.5\textwidth}
    \includegraphics[width=\textwidth]{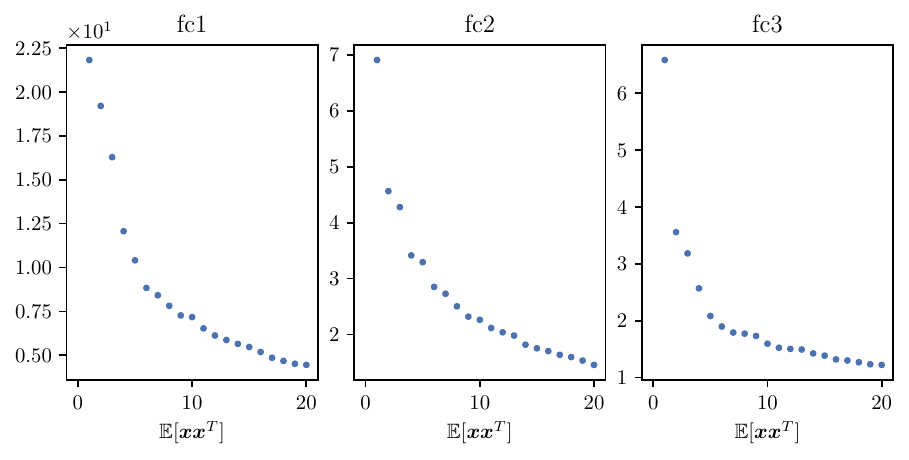}
    \caption{Eigenspectrum for $\E[\vx\vx^\T]$}
    \label{fig:app_exp_bn_xxT_fc_eigenspec}
\end{subfigure}%
\begin{subfigure}[b]{0.25\textwidth}
    \includegraphics[width=\textwidth]{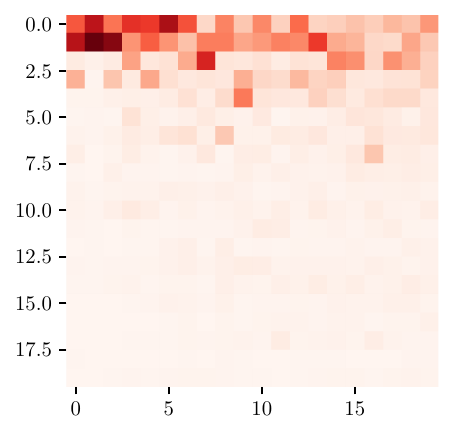}
    \caption{True Hessian with\\ $\E[\vx\vx^\T]$ (fc1:LeNet5-BN)}
    \label{fig:app_exp_bn_corr_true}
\end{subfigure}%
\begin{subfigure}[b]{0.25\textwidth}
    \includegraphics[width=\textwidth]{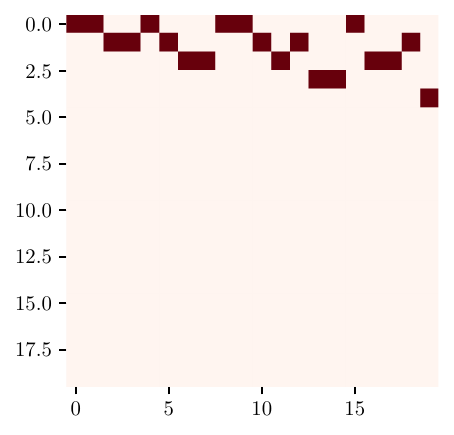}
    \caption{Approx Hessian with\\ $\E[\vx\vx^\T]$ (fc1:LeNet5-BN)}
    \label{fig:app_exp_bn_corr_est}
\end{subfigure}
\label{fig:app_exp_bn_corr}
\caption{Eigenspectrum and Eigenvector correspondence matrices with $\E[\vx\vx^\T]$ for LeNet5-BN.}
\end{figure}
\begin{figure}[ht]
    \centering
    \includegraphics[width=\textwidth]{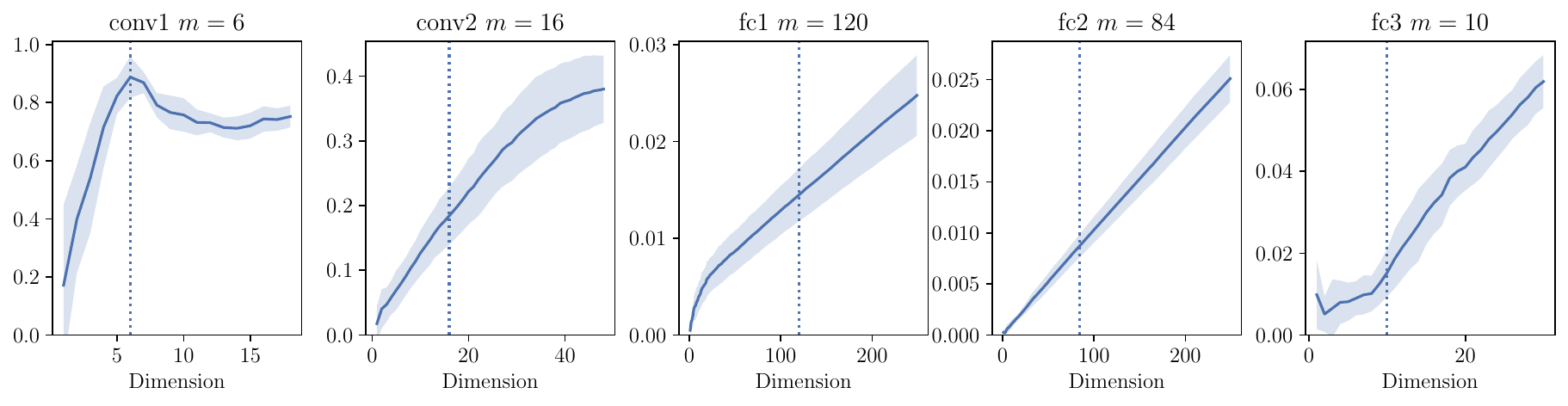}
    \caption{Eigenspace overlap of different models of LeNet5-BN.}
    \label{fig:app_exp_bn_overlap}
\end{figure}

Comparing \figureref{fig:app_exp_bn_overlap}(b) and \figureref{fig:app_exp_bn_overlap}(c), we can see that the Kronecker factorization still gives a reasonable approximation for the eigenvector correspondence matrix with $\E[\vx\vx^\T]$, although worse than the cases without BN (\figureref{fig:Corr_fc}).

\begin{figure}[ht]
\centering
\begin{subfigure}[b]{0.35\textwidth}
        \captionsetup{justification=centering}
    \includegraphics[width=\textwidth]{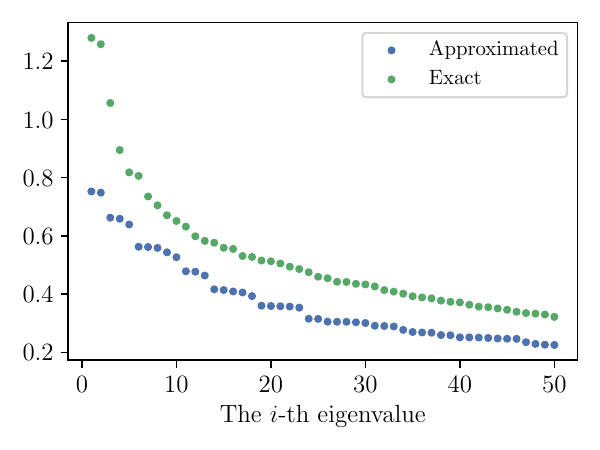}
    \caption{Top eigenvalues of approximated \\and exact layer-wise Hessian for fc2.}
    \label{fig:app_exp_bn_approx_eigenvalues}
\end{subfigure}%
\begin{subfigure}[b]{0.35\textwidth}
        \captionsetup{justification=centering}
    \includegraphics[width=\textwidth]{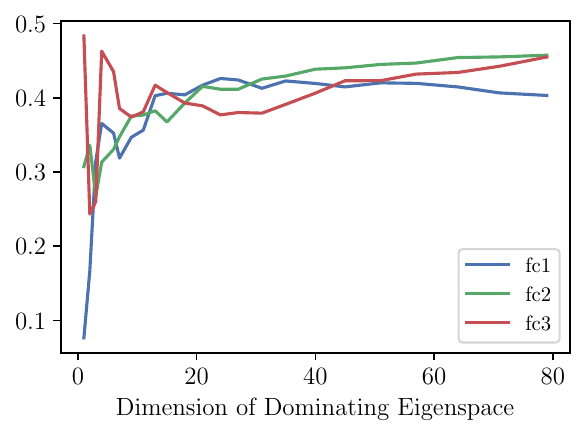}
    \caption{Top eigenspace overlap between \\approximated and true layer-wise Hessian.}
    \label{fig:app_exp_bn_approx_overlap}
\end{subfigure}
\caption{Comparison between the true and approximated layer-wise Hessians for LeNet5-BN.}
\label{fig:app_exp_bn_approx}
\end{figure}
\figureref{fig:app_exp_bn_approx} compare the eigenvalues and top eigenspaces of the approximated Hessian and the true Hessian for LeNet5 with BN. The approximation using Kronecker factorization is also worse than the case without BN (\figureref{fig:eigeninfo_approx}). However, the approximation still gives meaningful information as the overlap of top eigenspace is still highly nontrivial.

% \begin{figure}[th]
%     \centering
%     \vspace{-1em}
%     \subfigure[\centering\small{Top eigenvalues of approximated \\and exact layer-wise Hessian for fc2.}]{\includegraphics[width=0.48\linewidth]{Appendix_Figures/Explanation_LeNet5Case/BN/eigenval_compare_top50_CIFAR10_Exp1_LeNet5_BN_nl_fixlr0.01R2_E-1_fc1.pdf}}
%     \subfigure[\centering\small{Top eigenspace overlap between \\approximated and true layer-wise Hessian.}]{\includegraphics[width=0.48\linewidth]{Appendix_Figures/Explanation_LeNet5Case/BN/sample_kron_decomp_traceoverlap_d80_CIFAR10_Exp1_LeNet5_BN_nl_fixlr0.01R2_E-1.pdf}}
%     \caption{Comparison between the true and approximated layer-wise Hessians for LeNet5-BN.}
%     \label{fig:app_exp_bn_approx}
% \end{figure}

%% file: Appendix_Sections/appendix_exps/M_structure.tex
\newpage
\subsection{Outliers in Hessian Eigenspectrum}
\label{sec:appendix_M_struct}
One characteristic of Hessian that has been mentioned by many is the outliers in the spectrum of eigenvalues. \citet{sagun2017empirical} suggests that there is a gap in Hessian eigenvalue distribution around the number of classes $c$ in most cases, where $c=10$ in our case. A popular theory to explain the gap is the class / logit clustering of the logit gradients \citep{fort2019emergent, papyan2019measurements, papyan2020traces}. 
Note that these explanations can be consistent with our heuristic formula for the top eigenspace of output Hessian at initialization\--- in the two-layer setting we considered the logit gradients are indeed clustered.

In the layer-wise setting, the clustering claim can be formalized as follows: For each class $k\in[c]$ and logit entry $l\in[c]$, with $\mQ$ be defined as in  \equationref{eqn:app_qx}, and $(\rvx, \ry)$ as the input, label pair, let
\begin{equation}
    \mDelta_{i,j} = \E\left[\left.\mQ_{\rvx}\frac{\partial \vz_{\rvx}}{\partial \vw^{(p)}_j}\right\vert\ervy=i\right].
\end{equation}
Then at the initialization, for each logit entry $j$, $\{\mDelta_{i,j}\}_{i\in[c]}$ is clustered around the ``logit center'' $\hat{\mDelta}_j \triangleq \E_{i\in[c]}[\mDelta_{i,j}]$; at the minima, for each class $i$, $\{\mDelta_{i,j}\}_{j\in[c]}$ is clustered around the ``class center''  $\hat{\mDelta}_i \triangleq \E_{j\in[c]}[\mDelta_{i,j}]$. With the decoupling conjectures, we may also consider similar claims for output Hessians, where
\begin{equation}
    \mGamma_{i,j} = \E\left[\left.\mQ_{\rvx}\frac{\partial \vz_{\rvx}}{\partial {\vz^{(p)}_{\rvx}}_j}\right\vert\ervy=i\right].
\end{equation}
A natural extension of the clustering phenomenon on output Hessians is then as follows: At the initialization, for each logit entry $j$, $ \{\mGamma_{i,j}\}_{i\in[c]}$ is clustered around $\hat{\mGamma}_j \triangleq\E_{i\in[c]}[\mGamma_{i,j}]$; at the minima, for each class $i$, $\{\mGamma_{i,j}\}_{j\in[c]}$ is clustered around $\hat{\mGamma}_i \triangleq \E_{j\in[c]}[\mGamma_{i,j}]$.
Note that we have the layer-wise Hessian and layer-wise output Hessian satisfying
\begin{equation}
    \mH_\Ls(\vw^{(p)}) = \mathop{\E}_{i,j\in[c]}[\mDelta_{i,j}^\T\mDelta_{i,j}],\ \  \mM^{(p)} = \mathop{\E}_{i,j\in[c]}[\mGamma_{i,j}^\T\mGamma_{i,j}].
\end{equation}
 %\textcolor{red}{Rong: Please check if this sentence makes sense}

\paragraph{Low-rank Hessian at Random Initialization and Logit Gradient Clustering}\text{} \\
We first briefly recapture our explanation on the low-rankness of Hessian at random initialization. In \sectionref{sec:theoretical} and \sectionref{sec:detailed-proof}, we have shown that for a two layer ReLU network with Gaussian random initialization and Gaussian random input, the output hessian of the first layer $\mM^{(1)}$ is approximately $\frac{1}{4}\mW^{(2)T}\mA\mW^{(2)}$. We then heuristically extend this approximation to a randomly initialized $L$-layer network, that with $\mS^{(p)} = \mW^{(L)}\mW^{(L-1)}\cdots\mW^{(p+1)}$, the output Hessian of the $p$-th layer $\mH^{(p)}$ can be approximated by $\Tilde{\mM}^{(p)}$ where
\begin{equation}
    \label{eqn:M_SAS_Approx}
    \Tilde{\mM}^{(p)} \triangleq \frac{1}{4^{L-p}}\mS^{(p)T}\mA\mS^{(P)}.
\end{equation}
Since $\mA$ is strictly rank $c-1$ with null space of the all-one vector, $\mH^{(p)}$ is strictly rank $c-1$. Thus $\mH^{(p)}$ is approximately rank $c-1$, and so is the corresponding layerwise Hessian according to the decoupling conjecture.

Now we discuss the connection between our analysis with the theory of logit gradient clustering. As previously observed by \citet{papyan2019measurements}, for each logit entry $l$, $\{\mDelta_{i,j}\}_{l\in[c]}$ are clustered around the logit gradients $\E_{l\in[c]}[\mDelta_{i,j}]$. Similar clustering effects for $\{\mGamma_{i,j}\}_{l\in[c]}$ were also empirically observed by our experiments. Moreover, through the approximation above and the decoupling conjecture, for each logit entry $j$, the cluster centers $\hat{\mGamma}_j$ and $\hat{\mDelta}_j$ can be approximated by
\begin{equation}
\begin{split}
    \hat{\mGamma}_j \approx \Breve{\mGamma}_j &\triangleq (\mS^\T\mQ)_j\\ \hat{\mDelta}_j \approx \Breve{\mDelta}_j &\triangleq ((\E[\vx]\otimes\mS^\T)\E[\mQ])_j.
\end{split}
\end{equation}

Following \citet{papyan2019measurements}, we used t-SNE \citep{vanDerMaaten2008} to visualize the logit gradients. As we see in \figureref{fig:tsne1}, the ``logit centers'' of the clustering directly corresponds to the approximated dominating eigenvectors of the Hessian, which is consistent with our analysis.

\paragraph{Gradient Clustering at Minima}

Currently our theory does not provide an explanation to the low rank structure of Hessian at the minima. However we have observed that the class clustering of logit gradients does not universally apply to all models at the minima, even when the models have around $c$ significant large eigenvalues. As shown in \figureref{fig:tsne2}, the class clustering is very weak but there are still around $c$ significant large eigenvalues. We conjecture that the class clustering of logit gradients may be a sufficient but not necessary condition for the Hessian to be low rank at minima.

% \begin{figure}[H]
%     \centering
%     \subfigure[\centering\small{Eigenspectrum of $\E[\mM]$ at initialization.}]{\includegraphics[width=0.25\linewidth]{Appendix_Figures/TSNE_M/new/spec_0.pdf}}
%     \subfigure[\centering\small{Clustering of $\mDelta$ with logits at initialization.}]{\includegraphics[width=0.36\linewidth]{Appendix_Figures/TSNE_M/new/E0_gamma_ccp_cluster_pred_fc1.pdf}}
%     \subfigure[\centering\small{Clustering of $\mGamma$ with logits at initialization.}]{\includegraphics[width=0.36\linewidth]{Appendix_Figures/TSNE_M/new/E0_delta_ccp_cluster_pred_fc1.pdf}}
%     \caption{Logit clustering behavior of $\mDelta$ and $\mGamma$ at initialization (fc1:T-$200^2$)}
%     \label{fig:tsne1}
% \end{figure}
\begin{figure}[h]
    \centering
    \begin{subfigure}[b]{0.2\textwidth}
        \centering
        \captionsetup{justification=centering}
        \includegraphics[width=\textwidth]{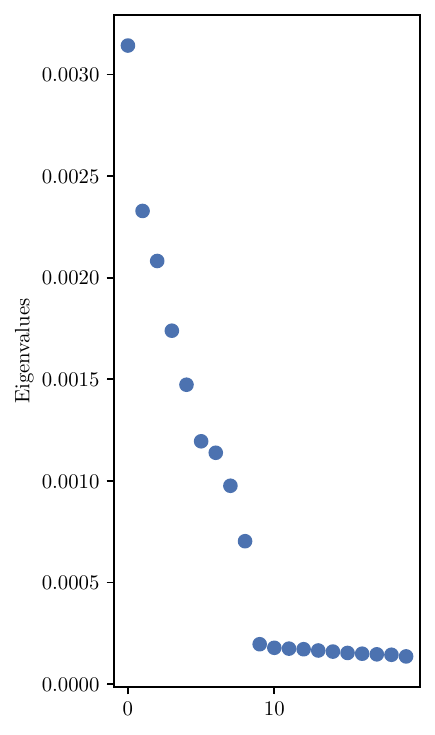}
        \caption{Eigenspectrum of $\E[\mM]$ at initialization.}
        \label{fig:app_tsne_h_init}
    \end{subfigure}%
    \begin{subfigure}[b]{0.4\textwidth}
        \centering
        \captionsetup{justification=centering}
        \includegraphics[width=\textwidth]{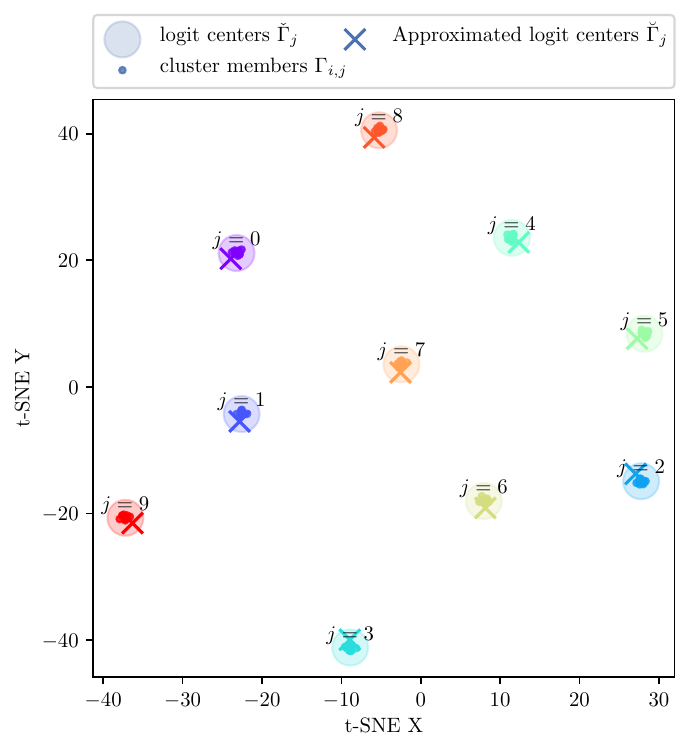}
        \caption{Clustering of $\mGamma$ with logits at initialization.}
        \label{fig:app_tsne_h_min}
    \end{subfigure}%
    \begin{subfigure}[b]{0.4\textwidth}
        \centering
        \captionsetup{justification=centering}
        \includegraphics[width=\textwidth]{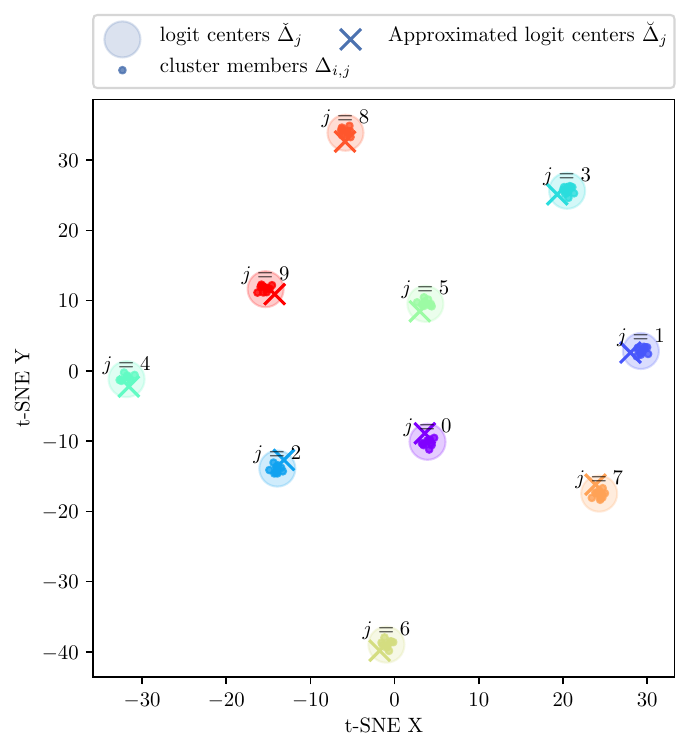}
        \caption{Clustering of $\mDelta$ with logits at initialization.}
        \label{fig:app_tsne_m_init}
    \end{subfigure}%
    \captionsetup{justification=centering}
    \caption{Logit clustering behavior of $\mDelta$ and $\mGamma$ at initialization (fc1:T-$200^2$)}
    \label{fig:tsne1}
\end{figure}

% \begin{figure}[H]
%     \centering
%     \subfigure[\centering\small{Eigenspectrum of $\E[\mM]$ at initialization.}]{\includegraphics[width=0.25\linewidth]{Appendix_Figures/TSNE_M/new/spec_E-1_gamma_ccp_cluster_pred_fc1.pdf}}
%     \subfigure[\centering\small{Clustering of $\mGamma$ with logits at initialization.}]{\includegraphics[width=0.36\linewidth]{Appendix_Figures/TSNE_M/new/E-1_gamma_ccp_cluster_pred_fc1.pdf}}
%     \subfigure[\centering\small{Clustering of $\mDelta$ with logits at initialization.}]{\includegraphics[width=0.36\linewidth]{Appendix_Figures/TSNE_M/new/E-1_delta_ccp_cluster_pred_fc1.pdf}}
%     \caption{Logit clustering behavior of $\mDelta$ and $\mGamma$ at minimum (fc1:T-$200^2$)}
%     \label{fig:tsne2}
% \end{figure}

\begin{figure}[H]
    \centering
    \begin{subfigure}[b]{0.2\textwidth}
        \centering
        \captionsetup{justification=centering}
        \includegraphics[width=\textwidth]{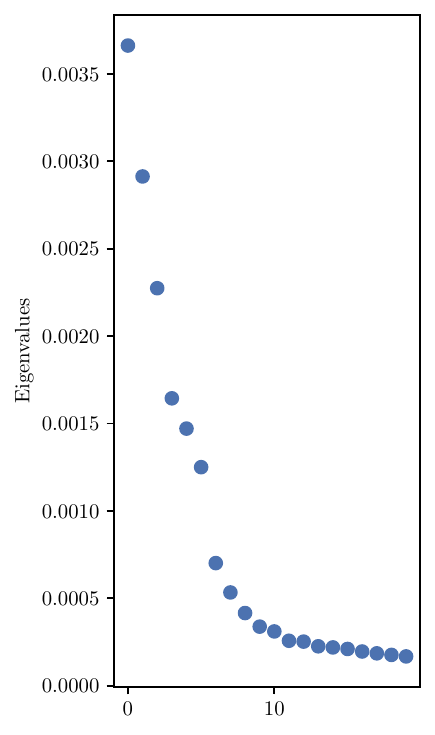}
        \caption{Eigenspectrum of $\E[\mM]$ at minimum.}
        \label{fig:app_tsne_h_init}
    \end{subfigure}%
    \begin{subfigure}[b]{0.4\textwidth}
        \centering
        \captionsetup{justification=centering}
        \includegraphics[width=\textwidth]{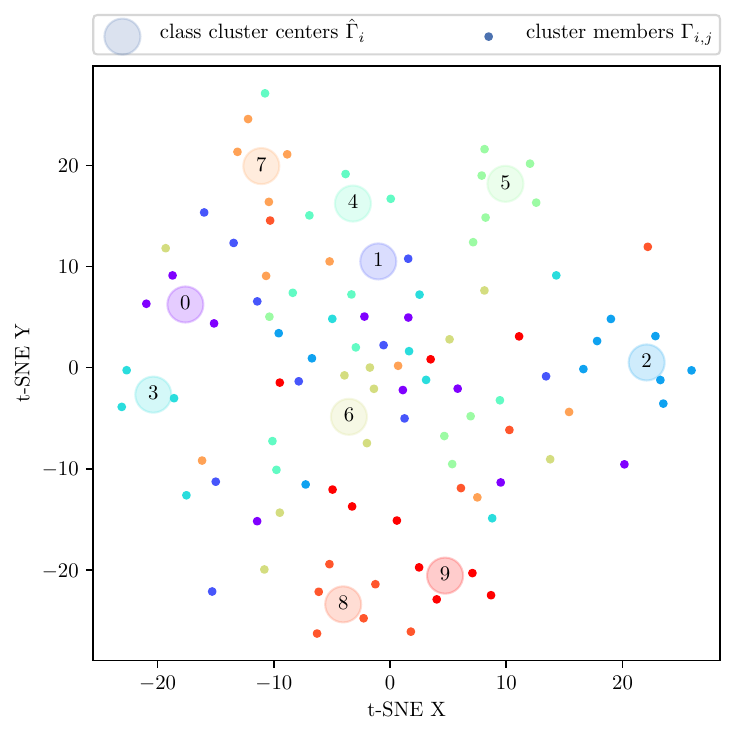}
        \caption{Clustering of $\mGamma$ with class at minimum.}
        \label{fig:app_tsne_h_min}
    \end{subfigure}%
    \begin{subfigure}[b]{0.4\textwidth}
        \centering
        \captionsetup{justification=centering}
        \includegraphics[width=\textwidth]{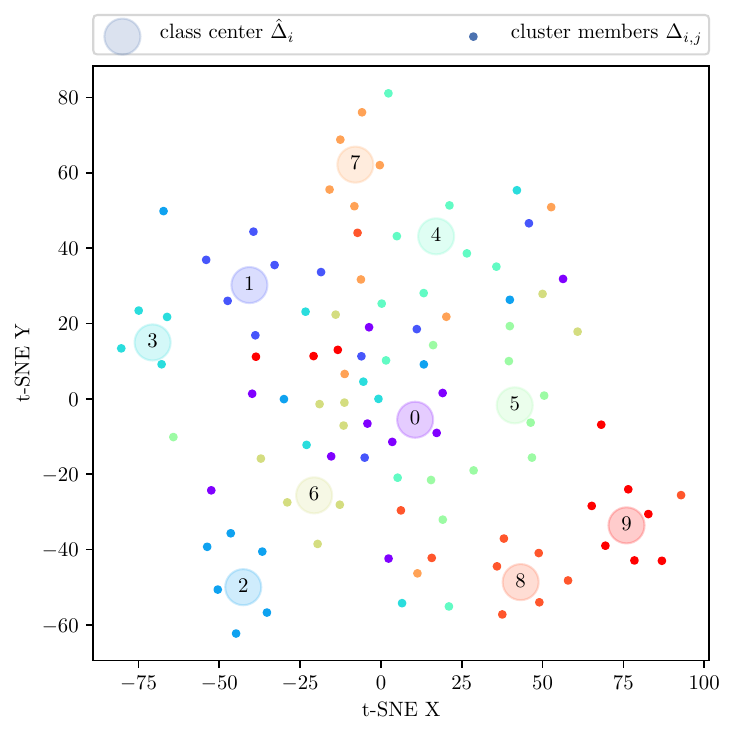}
        \caption{Clustering of $\mDelta$ with class at minimum.}
        \label{fig:app_tsne_m_init}
    \end{subfigure}%
    \captionsetup{justification=centering}
    \caption{Class clustering behavior of $\mDelta$ and $\mGamma$ at minimum. (fc1:T-$200^2$)}
    \label{fig:tsne2}
\end{figure}

%% file: Appendix_Sections/pac_bayes.tex
\label{sec:appendix_pac}
Given a model parameterized with $\theta$ and an input-label pair $(\vx,\vy)\in\R^d\times \R^c$, the classification error of $\theta$ over the input sample $\vx$ is $\breve{l}(\theta, \vx) := \1[\arg\max f_\theta(\vx) = \arg\max\vy].$ With the underlying data distribution $D$ and training set $S$ i.i.d. sampled from $D$, we define
\begin{equation}
e(\theta):=\E_{(\vx, \vy)\sim D}[\breve{l}(\theta,\vx)],\qquad \hat{e}(\theta):=\frac{1}{N}\sum_{i=1}^N[\breve{l}(\theta,\vx_i)]
\end{equation}
as the expected and empirical classification error of $\theta$, respectively.
We define the measurable hypothesis space of parameters $\gH:= \R^P$.
%\znote{measurability is STATEd for KL to be generally well-defined, do we need to define $P$ and $Q$ as probability measure as well?}
For any probabilistic measure $P$ in $\gH$, let $e(P) = \E_{\theta\sim P}e(\theta)$, $\hat{e}(P) = \E_{\theta\sim P}\hat{e}(\theta)$, and $\breve{e}(P) = \E_{\theta\sim P}\mathcal{L}(\theta)$. Here $\breve{e}(P)$ serves as a differentiable convex surrogate of $\hat{e}(P).$

\begin{theorem}[Pac-Bayes Bound]
\citep{mcallester1999some}\citep{langford2001bounds}
For any prior distribution $P$ in $\gH$ that is chosen independently from the training set $S$, and any posterior distribution $Q$ in $\gH$ whose choice may inference $S$, with probability $1-\delta$,
\begin{equation}
    \label{eqn:appendix_pac_bayes_inf}
    \KL\left(\hat{e}(Q)\Vert e(Q)\right)\leq \frac{\KL(Q\Vert P) + \log\frac{|S|}{\delta}}{|S|-1}.
\end{equation}
\end{theorem}
Fix some constant $b, c \geq 0$ and $\theta_0\in\gH$ as a random initialization, \citet{dziugaite2017computing} shows that when setting $Q = \gN(\vw, \diag(\vs))$, $P = \gN(\theta_0, \lambda\mI_P)$, where $\vw, \vs\in\gH$ and $\lambda = c\exp{(-j/ b)}$ for some $j\in\mathbb{N}$, and solve the optimization problem
\begin{equation}
    \min_{\vw,\vs,\lambda}\breve{e}(Q) + \sqrt{\frac{\KL(Q\Vert P) + \log\frac{|S|}{\delta}}{2(|S|-1)},}
\end{equation} with initialization $\vw = \theta$, $\vs = \theta^2$,
one can achieved a nonvacous PAC-Bayes bound by \equationref{eqn:appendix_pac_bayes_inf}.

In order to avoid discrete optimization for $j\in \N$, \citet{dziugaite2017computing} uses the $\BRE$ term to replace the bound in \tableref{eqn:appendix_pac_bayes_inf}. The $\BRE$ term is defined as
\begin{equation}
    \BRE(\vw,\vs,\lambda; \delta) = \frac{\KL(P\Vert Q)+2\log(b\log \frac{c}{\lambda})+\log \frac{\pi^2 |S|}{6\delta}} {|S|-1},
\end{equation}
where $Q = \gN(\vw, \diag(\vs))$, $P = \gN(\theta_0, \lambda\mI_P)$.
The optimization goal actually used in the implementation is thus
\begin{equation}
    \min_{\vw \in \R^P,\vs\in \R^P_+,\lambda\in (0,c)}\breve{e}(Q) + \sqrt{\frac{1}{2}\BRE(\vw,\vs,\lambda; \delta)}.
\end{equation}

\algorithmref{alg:pac} shows the algorithm for \emph{Iterative Hessian} (\textsc{Iter}) PAC-Bayes Optimization. If we set $\eta = T$, the algorithm will be come \emph{Approximate Hessian} (\textsc{Appr}) PAC-Bayes Optimization. It is based on Algorithm 1 in \citet{dziugaite2017computing}. The initialization of $\vw$ is different from \citet{dziugaite2017computing} because we believe what they wrote, $\abso(\vw)$ is a typo and $\log[\abso(\vw)]$ is what they actually means. It is more reasonable to initialize the variance $\vs$ as $\vw^2$ instead of $\exp[2\abso(\vw)]$.

\begin{algorithm}[ht]
\caption{PAC-Bayes bound optimization using layer-wise Hessian eigenbasis}
\textbf{Input:}\\
\algind$\vw_0\in\R^P$\Comment{Network parameters (Initialization)}\\
\algind$\vw\in \R^P$\Comment{Network parameters (SGD solution)}\\
\algind$S$ \Comment{Training examples}\\
\algind$\delta \in (0,1)$ \Comment{Confidence parameter}\\
\algind$b \in \mathbb{N}, c \in (0,1)$ \Comment{Precision and bound for $\lambda$}\\
\algind$\tau\in(0,1), T \in\mathbb{N}$ \Comment{Learning rate; No. of iterations}\\
\algind$\eta \in \mathbb{N}$ \Comment{Epoch interval for Hessian calculation}\\
\textbf{Output}\\
\algind$\vw$\Comment{Optimized network parameters}\\
\algind$\vs$\Comment{Optimized posterior variances in Hessian eigenbasis}\\
\algind$\lambda$\Comment{Optimized prior variancce}
\begin{algorithmic}[1]
\Procedure{Iterative-Hessian-PAC-Bayes}{}
    \State $\vvarsigma \gets \log[\abso(\vw)]$\Comment{where $\vs(\vvarsigma)= \exp(2\vvarsigma)$}
    \State $\varrho \gets -3$\Comment{where $\lambda(\varrho) = \exp(2\varrho)$}
    \State $R(\vw, \vs, \lambda) = \sqrt{\frac{1}{2}\BRE(\vw,\vs,\lambda; \delta)}$
    \Comment{BRE term}
    \State $B(\vw,\vs,\lambda,\vw') = \Ls(\vw')+R(\vw,\vs,\lambda)$\Comment{Optimization goal}
    \For{$t = 0 \to T-1$}\Comment{Run SGD for T iterations}
        \If{$t\mod \eta == 0$}
            \State $\textsc{HessianCalc}(w)$
        \EndIf
        \State \text{Sample} $\vxi \sim \N(0,1)^P$ 
        \State $\vw'(\vw,\vvarsigma)= \vw +\textsc{ToStandard}\left(\vxi \odot \exp(\vvarsigma)\right)$ \Comment{Generate noisy parameter for SNN}
        \State $\vw \gets \vw- \tau\left[\nabla_\vw R(\vw, \vs, \lambda)+\nabla_{\vw'}\Ls(\vw')\right]$
        \State $\vvarsigma \gets \vvarsigma - \tau\left[\nabla_\vvarsigma R(\vw, \vs(\vvarsigma), \lambda)+\textsc{ToHessian}\left(\nabla_{\vw'}\Ls(\vw')\right)\odot \vxi \odot \exp(\vvarsigma)\right]$
        \State $\varrho \gets \varrho - \tau\nabla_{\varrho}R(\vw,\vs,\lambda(\varrho))$ \Comment{Gradient descent}
    \EndFor
    \State \Return $w, s(\vvarsigma), \lambda(\varrho)$
\EndProcedure
\end{algorithmic}
\label{alg:pac}
\end{algorithm}

In the algorithm, \textsc{HessianCalc}$(\vw)$ is the process to calculate Hessian information with respect to the posterior mean $\vw$ in order to produce the Hessian eigenbasis to perform the change of basis. For very small networks, we can calculate Hessian explicitly but it is prohibitive for most common networks. However, efficient approximate change of basis can be performed using our approximated layer-wise Hessians. In this case, we would just need to calculate the full eigenspace of $\E[\mM]$ and that of $\E[\vx\vx^\T]$ for each layer. For $p$th layer, we denote them as $\mU^{(p)}$ and $\mV^{(p)}$ respectively with eigenvectors as columns. We can also store the corresponding eigenvalues by doing pairwise multiplications between eigenvalues of $\E[\mM]$ and $\E[\vx\vx^\T]$.

After getting the eigenspaces, we can perform the change of basis. Note that we perform change of basis on vectors with the same dimensionality as the parameter vector (or the posterior mean). $\textsc{ToHessian}(\vu)$ is the process to put a vector $\vu$ in the standard basis to the Hessian eigenbasis. We first break $\vu$ into different layers and let $\vu^{(p)}$ be the vector for the $p$th layer. We then define $\Mat^{(p)}$ as the reshape of a vector to the shape of the parameter matrix $\mW^{(p)}$ of that layer. We have the new vector $\vv^{(p)}$ in Hessian basis as
\begin{equation}
\label{eqn:appendix_pacbayes_changebasis}
    \vv^{(p)} = \vect\left[\mU^{(p)T}\Mat^{(p)}(\vu^{(p)})\mV^{(p)}\right].
\end{equation}
The new vector $\vv = \textsc{ToHessian}(\vu)$ is thus the concatenation of all the $\vv^{(p)}$.

$\textsc{ToStandard}(\vv)$ is the process to put a vector $\vv$ in the Hessian eigenbasis to the standard basis. It is the reverse process to $\textsc{ToHessian}$. We also break $\vv$ into layers and let the vector for the $p$th layer be $\vv^{(p)}$. Then, the new vector $\vu^{(p)}$ is
\begin{equation}
    \vu^{(p)} = \vect\left[\mU^{(p)}\Mat^{(p)}(\vv^{(p)})\mV^{(p)T}\right],
\end{equation}
The new vector $\vu = \textsc{ToStandard}(\vv)$ is thus the concatenation of all $\vu^{(p)}$.

After getting optimized $\vw, \vs, \lambda$, we compute the final bound using Monte Carlo methods same as in \citet{dziugaite2017computing}.

Note that the prior $P$ is invariant with respect to the change of basis, since its covariance matrix is a multiple of identity $\lambda\mI_P$. Thus, the KL divergence can be calculate in the Hessian eigenbasis without changing the value of $\lambda$. In the \emph{Iterative Hessian with approximated output Hessian} (\textsc{Iter.M}), we use $\Tilde{M}$ to approximate $\E[\mM]$, as in \equationref{eqn:M_SAS_Approx}.

We followed the experiment setting proposed by \citet{dziugaite2017computing} in general. In all the results we present, we first trained the models from Gaussian random initialization $w_0$ to the initial posterior mean estimate $w$ using SGD (lr=0.01) with batch-size 128 and epoch number 1000.

We then optimize the posterior mean and variance with layer-wise Hessian information using \algorithmref{alg:pac},
where $\delta = 0.025$, $b=100$, and $c=0.1$.
We train for 2000 epochs, with learning rate $\tau$ initialized at 0.001 and decays with ratio 0.1 every 400 epochs. For \emph{Approximated Hessian} algorithm, we set $\eta=1$. For \emph{Iterative Hessian} algorithm, we set $\eta=10$. We also tried $\eta$ with the same decay schedule as learning rate (multiply $\eta$ by 10 every time the learning rate is multiplied by 0.1) and the results are similar to those without decay.
We also used the same Monte Carlo method as in \citet{dziugaite2017computing} to calculate the final PAC-Bayes bound. Except that we used 50000 iterations instead of 150000 iterations because extra iterations do not further tighten the bound significantly. We use sample frequency 100 and $\delta'=0.01$ as in that paper.

The complete experiment results are listed in \tableref{tab:app_pac}. We follow the same naming convention as in \citet{dziugaite2017computing} except adding T-$200^2$ we introduced in \sectionref{sec:hessian}. T-$600_{10}$, T-$600^2_{10}$, and T-$200^2_{10}$ are trained on standard MNIST with 10 classes, and others are trained on MNIST-2 (see \sectionref{sec:appendix_exp_dataset}), in which we combined class 0-4 and class 5-9.

In \tableref{tab:app_pac}, Prev means the previous results in \citet{dziugaite2017computing}, \textsc{Appr} means \emph{Approximated Hessian}, \textsc{Iter} means \emph{Iterative Hessian}, \textsc{Iter} (D) means \emph{Iterative Hessian} with decaying $\eta$, \textsc{Iter.M} means \emph{Iterative Hessian with approximated output Hessian}. \textsc{Base} are Base PAC-Bayes optimization as in the previous paper.

We also plotted the final posterior variance, $\vs$. \figureref{fig:app_PAC} shown below is for T-$200^2_{10}$. For posterior variance optimized with our algorithms (\textsc{Appr}, \textsc{Iter}, and \textsc{Iter.M}) we can see that direction associated with larger eigenvalue has a smaller variance. This agrees with our presumption that top eigenvectors are aligned with sharper directions and should have smaller variance after optimization. The effect is more significant and consistent for Iterative Hessian, where the PAC-Bayes bound is also tighter.
\begin{figure}[H]
    \centering
    \includegraphics[width=\textwidth]{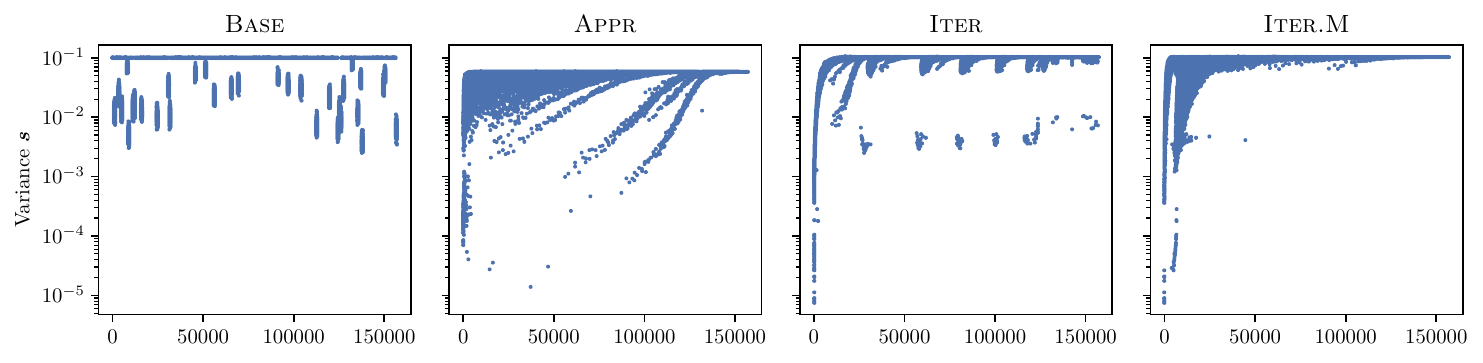}
    % \captionsetup{justification=centering}
    \caption{Optimized posterior variance, $\vs$. (fc1:T-$200^2$, trained on MNIST), the horizontal axis is ordered with decreasing eigenvalues.}
    \label{fig:app_PAC}
\end{figure}

\newpage

\begin{table}[H]
  \centering
  \caption{Full PAC-Bayes bound optimization results}
  \vskip 0.1in
    \begin{center}
    \begin{tabular}{llccccc}
    \toprule
    \textbf{Network} & \textbf{Method} &  \shortstack{PAC-Bayes\\Bound} & \shortstack{KL\\Divergence} & \shortstack{SNN\\loss} & \shortstack{$\lambda$ (prior)} &\shortstack{Test\\Error} 
    \\ \midrule
    T-600          & \textsc{Prev}         & 0.161  & 5144   & 0.028   &    -    & 0.017 \\
                   & \textsc{Base}      & 0.154  & 4612.6 & 0.03373 & -1.3313 & 0.0153 \\
                   & \textsc{Appr}     & 0.1432 & 3980.6 & 0.03417 & -1.6063 & 0.0153 \\
                   & \textsc{Iter}       & \textbf{0.1198} & 3766.1 & 0.02347 & -1.2913 & 0.0153 \\
                   & \textsc{Iter(D)} & 0.1199 & 3751.1 & 0.02366 & -1.2913 & 0.0153 \\
                   & \textsc{Iter.M} & 0.1255 & 3929.9 & 0.02494 & -1.3213 & 0.0153\\
    \hline\rule{0pt}{2.5ex}
    T-$600^2$      & \textsc{Prev}        & 0.186  & 6534   & 0.028   &    -    & 0.016 \\
                   & \textsc{Base}      & 0.1921 & 6966.6 & 0.03262 & -1.4163 & 0.0148 \\
                   & \textsc{Appr}     & 0.1658 & 5176.1 & 0.03468 & -2.0963 & 0.0148 \\
                   & \textsc{Iter}       & 0.1456 & 5086.5 & 0.02473 & -1.7963 & 0.0148 \\
                   & \textsc{Iter(D)} & \textbf{0.1443} & 4956.8 & 0.02523 & -1.7963 & 0.0148 \\
                   & \textsc{Iter.M} & 0.1502 & 5024.5 & 0.02767 & -1.8363 & 0.0148\\
    \hline\rule{0pt}{2.5ex}

    T-1200         & \textsc{Prev}         & 0.179  & 5977   & 0.027   &    -    & 0.016 \\
                   & \textsc{Base}      & 0.1754 & 5917.6 & 0.03295 & -1.5463 & 0.0161 \\
                   & \textsc{Appr}     & 0.1725 & 5318.8 & 0.03701 & -1.8313 & 0.0161 \\
                   & \textsc{Iter}       & 0.1417 & 5071  & 0.02292 & -1.4763 & 0.0161 \\
                   & \textsc{Iter(D)} & \textbf{0.1413} & 5021.1 & 0.02316 & -1.4763 & 0.0161 \\
                   & \textsc{Iter.M} & 0.1493 & 5185.4 & 0.02576	& -1.5363 &	0.0161\\
    \hline\rule{0pt}{2.5ex}
    T-$300^2$      & \textsc{Prev}         & 0.17   & 5791   & 0.027   &    -    & 0.015 \\
                   & \textsc{Base}      & 0.1686 & 5514.9 & 0.03329 & -1.1513 & 0.015 \\
                   & \textsc{Appr}     & 0.1434 & 4105.4 & 0.03296 & -1.8063 & 0.015 \\
                   & \textsc{Iter}       & 0.1249 & 3873.2 & 0.02514 & -1.4763 & 0.015 \\
                   & \textsc{Iter(D)} & \textbf{0.1244} & 3833.7 & 0.02526 & -1.4763 & 0.015 \\
                   & \textsc{Iter.M} & 0.1308	& 3987.2 &	0.02721	 & -1.5713 & 0.015 \\
    \hline\rule{0pt}{2.5ex}
    R-600          & \textsc{Prev}         & 1.352  & 201131 & 0.112   &    -    & 0.501 \\
                   & \textsc{Base}      & 0.6046 & 1144.8 & 0.507   & -1.8263 & 0.4925 \\
                   & \textsc{Appr}     & 0.5653 & 390.25 & 0.5066  & -2.4713 & 0.4925 \\
                   & \textsc{Iter(D)} & 0.5681 & 431.62 & 0.5066  & -2.4513 & 0.4925 \\
                   & \textsc{Iter.M} & \textbf{0.5616} & 340.62 &	0.5065	& -2.5263 &	0.4925\\
    \hline\rule{0pt}{2.5ex}
    T-$200^2_{10}$     & \textsc{Base}      & 0.4165 & 21896  & 0.04706 & -1.1513 & 0.0208 \\
                   & \textsc{Appr}     & 0.2621 & 11068  & 0.0366  & -1.4213 & 0.0208 \\
                   & \textsc{Iter}       & \textbf{0.2145} & 9821  & 0.02229 & -1.1513 & 0.0208 \\
                   & \textsc{Iter(D)} & 0.2311 & 9758.5 & 0.03071 & -1.1513 & 0.0208 \\
                   & \textsc{Iter.M} & 0.2728 & 13406 & 0.02605	& -1.1513 &	0.0208\\
    \hline\rule{0pt}{2.5ex}
    T-$600_{10}$   & \textsc{Base}      & 0.2879 & 12674  & 0.03854 & -1.1513 & 0.018 \\
                   & \textsc{Appr}     & 0.2424 & 9095.8 & 0.04159 & -1.6013 & 0.018 \\
                   & \textsc{Iter}       & \textbf{0.2132} & 8697.9 & 0.02947 & -1.3063 & 0.018 \\
                   & \textsc{Iter.M} & 0.2227	& 8870.9	& 0.03294	& -1.4613	& 0.018 \\
    \hline\rule{0pt}{2.5ex}
    T-$600^2_{10}$ & \textsc{Base}      & 0.3472 & 17212  & 0.03884 & -1.1513 & 0.0186 \\
                   & \textsc{Appr}     & 0.2896 & 11618  & 0.04723 & -2.0563 & 0.0186 \\
                   & \textsc{Iter}       & \textbf{0.2431} & 10568 & 0.03057 & -1.5713 & 0.0186 \\
    \bottomrule
    \end{tabular}
\end{center}
  \label{tab:app_pac}
\end{table}